%% file: main.tex
\documentclass[
	twoside,openright,titlepage,numbers=noenddot,headinclude,
        footinclude=true,cleardoublepage=empty,
	dottedtoc,
	BCOR=5mm,paper=a4,fontsize=11pt,
	american,
]{scrreprt}

\input{sty/classicthesis-config}
\input{notation/init}
\begin{document}

    \frenchspacing
    \raggedbottom
    % \selectlanguage{american}
    \pagenumbering{roman}
    \pagestyle{plain}

    \include{frontbackmatter/titlepage}
    \include{frontbackmatter/titleback}
    \include{frontbackmatter/abstract}
    \include{frontbackmatter/acknowledgments}

    \include{frontbackmatter/contents}

    \pagenumbering{arabic}

    \include{content/intro}
    \include{content/background}

    \include{content/nar}
    \include{content/mindp}
    \include{content/conar}
    \include{content/conclusions}
    %% \ctparttext{}
    % \part{Reasoning Algorithmically}

    %% \include{Chapters/P1C0.Background}

    %% \ctparttext{}
    %% \part{Algorithmic Reasoning and Combinatorial Optimisation}

    % \include{Chapters/...}

    \include{frontbackmatter/bibliography}
    % \appendix
    \part{Appendix}
    \begin{appendix}
    \include{frontbackmatter/publications}
    \include{frontbackmatter/talks}
    \end{appendix}
\end{document}

%% file: sty/classicthesis-config.tex
\PassOptionsToPackage{utf8}{inputenc}
  \usepackage{inputenc}

\PassOptionsToPackage{T1}{fontenc}
  \usepackage{fontenc}

\PassOptionsToPackage{
  drafting=false,    % print version information on the bottom of the pages
  tocaligned=false, % the left column of the toc will be aligned (no indentation)
  dottedtoc=true,  % page numbers in ToC flushed right
  eulerchapternumbers=true, % use AMS Euler for chapter font (otherwise Palatino)
  linedheaders=false,       % chaper headers will have line above and beneath
  floatperchapter=true,     % numbering per chapter for all floats (i.e., Figure 1.1)
  eulermath=false,  % use awesome Euler fonts for mathematical formulae (only with pdfLaTeX)
  beramono=true,    % toggle a nice monospaced font (w/ bold)
  palatino=false,    % deactivate standard font for loading another one, see the last section at the end of this file for suggestions
  style=arsclassica % classicthesis, arsclassica
}{classicthesis}

\newcommand{\myTitle}{Reasoning Algorithmically in Graph Neural Networks\xspace}
\newcommand{\mySubtitle}{Doctoral Dissertation\xspace}

\newcommand{\myName}{Danilo Numeroso\xspace}
\newcommand{\mySupervisor}{Davide Bacciu\xspace}
\newcommand{\myDepartment}{Department of Computer Science\xspace}
\newcommand{\myUni}{Università di Pisa\xspace}

\newcommand{\myTime}{February 16, 2024\xspace}
\newcommand{\myVersion}{\classicthesis}

\providecommand{\mLyX}{L\kern-.1667em\lower.25em\hbox{Y}\kern-.125emX\@}

\PassOptionsToPackage{american}{babel}
\usepackage{babel}

\usepackage{csquotes}
\usepackage{natbib}

\PassOptionsToPackage{fleqn}{amsmath}
\usepackage{amsmath,amsfonts,amsthm,amsbsy,amssymb}
\usepackage{mathtools}
\usepackage{esvect}
\usepackage{dsfont}
\usepackage{xfrac}

\usepackage{graphicx}
\usepackage{scrhack} % fix warnings when using KOMA with listings package
\usepackage{xspace} % to get the spacing after macros right
\usepackage{hyperref}

\PassOptionsToPackage{printonlyused,smaller}{acronym}
  \usepackage{acronym} % nice macros for handling all acronyms in the thesis
\usepackage[inline]{enumitem}
\usepackage{subcaption}
\usepackage{hyphenat}

\usepackage[dvipsnames]{xcolor}
\usepackage{tikz}

\usetikzlibrary{matrix, backgrounds, arrows, calc, cd, decorations.pathmorphing, fit, shapes.multipart, positioning, shapes, patterns, snakes}

\usepackage{tabularx} % better tables
\usepackage{listings}
\usepackage{algorithm}
\usepackage[noend]{algpseudocode}

\algblock{ParFor}{EndParFor}
\algnewcommand\algorithmicparfor{\textbf{for}}
\algnewcommand\algorithmicpardo{\textbf{do in parallel}}
\algnewcommand\algorithmicendparfor{\extbf{end for}}
\algrenewtext{ParFor}[1]{\algorithmicparfor\ #1\ \algorithmicpardo}
\algrenewtext{EndParFor}{\algorithmicendparfor}

\makeatletter
\ifthenelse{\equal{\ALG@noend}{t}}%
  {\algtext*{EndParFor}}
  {}%
\makeatother

\usepackage{classicthesis}

\usepackage[osf,ScaleRM=1.05,ScaleSF=1.05]{libertinus}
\usepackage[euler-digits,OT1]{eulervm} % small
\usepackage{beramono}

\usepackage{multirow}

\usepackage{pgfplots}
\usepgfplotslibrary{fillbetween}

%% file: notation/init.tex
\PassOptionsToPackage{fleqn}{amsmath}
\usepackage{amsmath,amsfonts,amsthm,amsbsy,amssymb}
\usepackage{mathtools}
\usepackage{bm}
\usepackage{environ}
\usepackage{tabularx}
\usepackage{lipsum}

\input{notation/delimiters}
\input{notation/environments}
\input{notation/symbols}

\newcommand{\mathsc}[1]{{\normalfont\textsc{#1}}}

%% file: notation/delimiters.tex
\DeclarePairedDelimiter{\tuple}{\langle}{\rangle}

\DeclarePairedDelimiter{\norm}{\lVert}{\rVert}
\DeclarePairedDelimiter{\set}{\{}{\}}

%% file: notation/environments.tex
\newtheorem{definition}{Definition}
\newtheorem{problem}{Problem}
\newtheorem{theorem}{Theorem}
\newtheorem{proposition}{Proposition}

\newcolumntype{L}{>{$}l<{$}}
\newcolumntype{C}{>{$}c<{$}}
\newcolumntype{R}{>{$}r<{$}}

%% file: notation/symbols.tex
\def\1{\bm{1}}
\newcommand{\train}{\mathcal{D}}

\newcommand{\inp}{\texttt{input}}
\newcommand{\hints}{\texttt{hints}}
\newcommand{\out}{\texttt{output}}

\def\va{{\bm{a}}}
\def\vb{{\bm{b}}}
\def\vc{{\bm{c}}}
\def\vd{{\bm{d}}}

\def\vf{{\bm{f}}}
\def\vg{{\bm{g}}}
\def\vh{{\bm{h}}}

\def\vo{{\bm{o}}}
\def\vp{{\bm{p}}}

\def\vr{{\bm{r}}}

\def\vx{{\bm{x}}}
\def\vy{{\bm{y}}}
\def\vz{{\bm{z}}}

\def\evs{{s}}

\def\mA{{\bm{A}}}
\def\mB{{\bm{B}}}
\def\mC{{\bm{C}}}

\def\mF{{\bm{F}}}

\def\mH{{\bm{H}}}
\def\mI{{\bm{I}}}

\def\mP{{\bm{P}}}

\def\mW{{\bm{W}}}

\DeclareMathAlphabet{\mathsfit}{\encodingdefault}{\sfdefault}{m}{sl}
\SetMathAlphabet{\mathsfit}{bold}{\encodingdefault}{\sfdefault}{bx}{n}

\def\gA{{\mathcal{A}}}
\def\gB{{\mathcal{B}}}
\def\gC{{\mathcal{C}}}

\def\gE{{\mathcal{E}}}
\def\gF{{\mathcal{F}}}

\def\gI{{\mathcal{I}}}

\def\gL{{\mathcal{L}}}

\def\gN{{\mathcal{N}}}

\def\gP{{\mathcal{P}}}

\def\gS{{\mathcal{S}}}
\def\gT{{\mathcal{T}}}
\def\gU{{\mathcal{U}}}
\def\gV{{\mathcal{V}}}
\def\gW{{\mathcal{W}}}
\def\gX{{\mathcal{X}}}
\def\gY{{\mathcal{Y}}}

\def\sA{{\mathbb{A}}}
\def\sB{{\mathbb{B}}}

\def\sL{{\mathbb{L}}}

\def\sN{{\mathbb{N}}}

\def\sR{{\mathbb{R}}}

\def\sT{{\mathbb{T}}}

\def\sZ{{\mathbb{Z}}}

\newcommand{\normltwo}{L^2}

\newcommand*{\eqqst}{\stackrel{\text{?}}{=}}

\newcommand{\relu}{\ensuremath\operatorname{ReLU}}
\newcommand{\gin}{\textsc{gin}}
\newcommand{\mlp}{\textsc{mlp}}

\DeclareMathOperator{\net}{net}
\DeclareMathOperator*{\argmax}{arg\,max}
\DeclareMathOperator*{\argmin}{arg\,min}
\DeclareMathOperator*{\mean}{mean}
\DeclareMathOperator*{\sumt}{sum}

\DeclareMathOperator{\goal}{\texttt{goal}}

%% file: frontbackmatter/titlepage.tex
\begin{titlepage}
    \begin{addmargin}[-1cm]{-3cm}
    \begin{center}
        \large

        \begingroup
            \includegraphics[width=3cm]{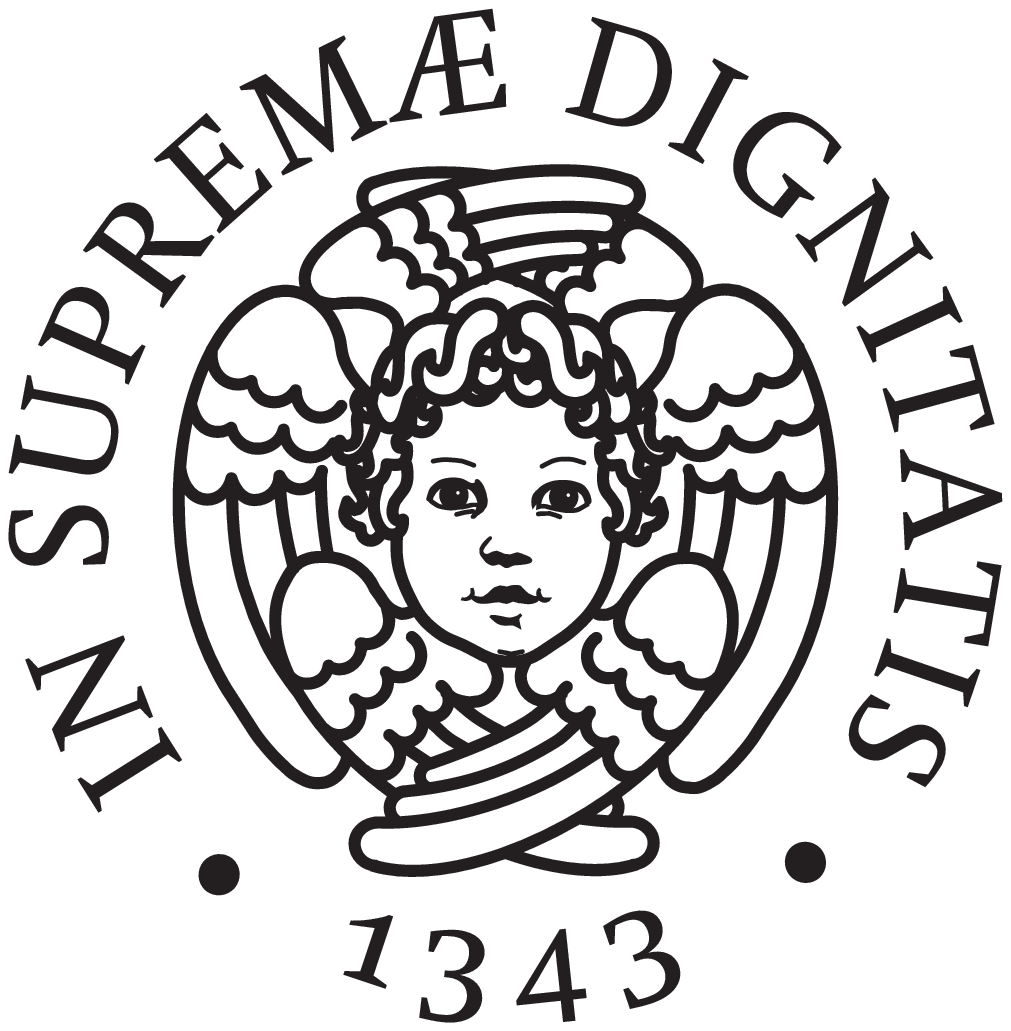}\\\medskip
            \Large\textsc{\myUni}\\\smallskip
            \large{\itshape\myDepartment}
        \endgroup

        \hfill
        \vfill

        \begingroup
        \Large
            \color{CTtitle}\spacedallcaps{\myTitle} \bigskip
        \endgroup

        {\Large\spacedlowsmallcaps{\mySubtitle}}

        \vfill

        % \hspace{0.15\linewidth}%
        % \begin{minipage}[t]{0.3\linewidth}
        % \begin{flushleft}
        \spacedlowsmallcaps{Candidate}\\\smallskip
        \myName\\
        ~\\
        \spacedlowsmallcaps{Supervisor}\\\smallskip
        \mySupervisor
        % \end{flushleft}
        % \end{minipage}\hfill%
        % \begin{minipage}[t]{0.3\linewidth}
        % \begin{flushright}
        % \spacedlowsmallcaps{Committee}\\\smallskip
        % X X\\
        % X Y
        % \end{flushright}
        % \end{minipage}\hspace{0.15\linewidth}

        \vfill

        \myTime

        \vfill

    \end{center}
  \end{addmargin}
\end{titlepage}

%% file: frontbackmatter/titleback.tex
\thispagestyle{empty}

\hfill

\vfill

\noindent\myName: \textit{\myTitle,} \mySubtitle
\newline\textcopyright\ \myTime

%% file: frontbackmatter/abstract.tex
\pdfbookmark[1]{Abstract}{Abstract}

\begingroup
\let\clearpage\relax
\let\cleardoublepage\relax
\let\cleardoublepage\relax

\chapter*{Abstract}
The development of artificial intelligence systems with advanced
reasoning capabilities represents a persistent and long-standing
research question. Traditionally, the primary strategy to address this
challenge involved the adoption of \textit{symbolic} approaches, where
knowledge was explicitly represented by means of symbols and
explicitly programmed rules. However, with the advent of machine
learning, there has been a paradigm shift towards systems that can
autonomously learn from data, requiring minimal human guidance. In
light of this shift, in latest years, there has been increasing
interest and efforts at endowing neural networks with the ability to
reason, bridging the gap between data-driven learning and logical
reasoning.  Within this context, Neural Algorithmic Reasoning (NAR)
stands out as a promising research field, aiming to integrate the
structured and rule-based reasoning of algorithms with the adaptive
learning capabilities of neural networks, typically by tasking neural
models to mimic classical algorithms. In this dissertation, we provide
theoretical and practical contributions to this area of research. We
begin with a review of foundational principles necessary for the
understanding of the notions presented in this thesis, followed by a
comprehensive overview of the most important NAR
principles. Proceeding forward, we explore the connections between
neural networks and tropical algebra, deriving powerful architectures
that are \emph{aligned} with algorithm execution. These architectures
are thus proven to be able to approximate some \emph{min}-aggregated
dynamic programming algorithms up to arbitrary precision. Furthermore,
we discuss and show the ability of such neural \emph{reasoners} to
learn and manipulate complex algorithmic and combinatorial
optimization concepts, such as the principle of strong duality. Here,
we rigorously evaluate this capacity through extensive quantitative
and qualitative studies.  Finally, in our empirical efforts, we
validate the real-world utility of NAR networks across different
practical scenarios.  This includes tasks as diverse as planning
problems, large-scale edge classification tasks and the learning of
polynomial-time approximate algorithms for NP-hard combinatorial
optimisation problems.  Through this comprehensive exploration, we aim
to showcase the potential and versatility of integrating algorithmic
reasoning in machine learning models.

\endgroup
\vfill

%% file: frontbackmatter/acknowledgments.tex
\pdfbookmark[1]{Acknowledgements}{acknowledgements}

\begingroup
\let\clearpage\relax
\let\cleardoublepage\relax
\let\cleardoublepage\relax
\chapter*{Acknowledgements}
% Finding myself once again in front of this blank page, makes me realise that I have come to the end of another journey in my life. This journey, however, would not have been half as enjoyable without all the people who have accompanied me during these three long years.

% \bigskip
I \textit{owe} the inception of these acknowledgements to a truly exceptional human being, \textbf{Davide Bacciu}. Since the time of my master's thesis, I have consistently felt an unwavering trust from him -- perhaps more than I ever had in myself. Throughout these three years, you have been a supervisor, a psychologist, and, most importantly, a friend. Your continuous support has been invaluable, and I can never express my gratitude enough.
\bigskip
\\
I wish to thank the international reviewers of this thesis, \textbf{Christopher Morris} and 
\textbf{Nils Morten Kriege}, for giving me the opportunity to revise some of the concepts
I (thought I) grasped, through their comments and constructive criticisms. I also wish to
thank my PhD coordinator: \textbf{Antonio Brogi}, and my internal committee: \textbf{Alessio Conte}; \textbf{Alina Sirbu}; and \textbf{Antonio Frangioni}, for guiding and providing me with valuable feedback during these years.
\bigskip
\\
Many thanks to \textbf{Alessio Gravina} and \textbf{Valerio De Caro}, with whom I
had the pleasure to share the office throughout the entire duration of this PhD. Thanks for all 
the discussions, the occasional battles for GPU resources, and the laughs we had. 
This journey would have been only half as enjoyable without your presence.
\bigskip
\\
I would also like to thank each and every one of the UniPi people I had the opportunity to 
know during these years, especially \textbf{Francesco Landolfi} who always had the time to listen
to my ramblings (and vice versa!).
\bigskip
\\
Sincere thanks are due to \textbf{Petar Veli\v{c}kovi\'{c}}. Since our first meeting, he always
took my ideas seriously, sharing with me his honest and valuable thoughts and giving me the
opportunity to grow professionally. 
\bigskip
\\
Immense gratitude belongs to \textbf{Yingqian Zhang}, for warmly welcoming me at the \textit{Technische Universiteit Eindhoven} and for giving me the opportunity to collaborate
with amazing people. In this regard, I sincerely thank \textbf{Wouter Kool}, with whom I had
the pleasure (and the duty!) to organise parts of the \textit{EURO Meets NeurIPS 2022}
competition.

Special thanks are due to \textbf{Luca Begnardi}, \textbf{Riccardo Lo Bianco}, \textbf{Fabio Mercurio} and \textbf{Ya Song}. The memories of the time we spent together, during coffee breaks/tea breaks is something I will never forget about.

Heartfelt thanks go to two of the most genuinely good-hearted people I have ever known, 
``\textbf{Pit}'' and \textbf{Marcin}. You'll probably never read these words but... \textit{dzi\k{e}ki}. 

I couldn't conclude this paragraph without first thanking \textbf{Edwin van Weert}. 
You truly made me feel home and really added that extra to the overall dutch experience.
Say \textit{hi} to \textit{Coco}. I nearly forgot... thank you \textbf{Ziva}! As 
strange as it may sound, thanks to you, I don't think I will ever forget how to count 
in dutch: \textit{één, twee, drie, vier, vijf, zes, zeven, acht, negen, tien}! This is
some memory I will jealously guard with me. \textit{Dank je wel}!
\bigskip
\\
I had the luck to visit the 
\textit{University of Cambridge}, where I had \textit{the} most intense and fun period
of my entire PhD. A huge thank you goes to \textbf{Pietro Li\`{o}}, to whom I will 
always be grateful for giving me such opportunity and for his kindness. Throughout my stay
in Cambridge, I always had the impression that he genuinely cared I was, first of all,
enjoying myself. I would also like to thank \textbf{Dobrik Georgiev}, for all our inspiring
discussions and exchange of ideas, and of course for his relentless dedication and commitment. 
This paragraph could never be complete without a word of acknowledgements
to the amazing folks I met there: \textbf{Donato Crisostomi}; \textbf{Federico Siciliano}; 
\textbf{Francesco Ceccarelli}; \textbf{Francesco Prinzi}; \textbf{Gianluca Carlini}; \textbf{Lorenzo Giusti}; \textbf{Maria Sofia Bucarelli}; and \textbf{Pietro Barbiero} (thanks for the bike!). I will always remember the \textit{intense} night before NeurIPS submission(s), as well as our barbecues at the Jesus Green.
\bigskip
\\
I suspect the cardinality of the set of people I met during these years and influenced me positively
is uncountable, so a general thanks goes to all of the folks I didn't have the opportunity to mention here.
\bigskip
\\
\paragraph{Personal acknowledgements} -- \textit{in italian}\\
Un grazie speciale va a \textbf{Anna Orfan\`{o}}, per avermi aiutato a superare alcuni dei momenti più difficili che ho affrontato nella mia vita. Porterò sempre con me le nostre chiacchierate.
\bigskip
\\
Il pi\`{u} sentito dei grazie va alla mia famiglia:

- \textbf{Mamma e Babbo}, a nulla sarebbe valso tutto il mio impegno se voi non foste stati sempre al mio fianco, a sorreggermi. Grazie.

- \textbf{Matti}, avere la possibilit\`{a} di essere tuo fratello e di poterti aiutare nelle tue scelte future \`{e} uno dei miei pi\`{u} grandi orgogli. Averti nella mia vita \`{e} una fortuna.

- \textbf{Mara}, perch\`{e} ormai sei famiglia. Affrontare questo cammino con te al mio fianco \`{e} stato, ed \`{e}, un privilegio. Qualsiasi cosa affronter\`{o} in futuro, bella o brutta, so che la affronter\`{o} insieme a te. Grazie per esserci sempre.
\bigskip
\\
Grazie a tutta la \textit{Varano crew}: \textbf{Ciro Leonardi}; \textbf{Francesco Bondi}; \textbf{Gaia Volpi}; \textbf{Lorenzo Cassi}; \textbf{Marcello Ceresini} e
\textbf{Riccardo Verbeni}. Grazie per aver aggiunto spensieratezza e leggerezza in questi tre anni.
\bigskip
\\
Un ultimo saluto, e un doveroso grazie, va ad una persona che queste parole non potr\`{a} mai leggerle.
Credo tu non le possa neanche ascoltare, ma spero di sbagliarmi. \textit{Cia', 'a no'!}
\endgroup

%% file: frontbackmatter/contents.tex
\pagestyle{scrheadings}

\pdfbookmark[1]{\contentsname}{tableofcontents}
\setcounter{tocdepth}{3}
\setcounter{secnumdepth}{3}
\manualmark
\markboth{\spacedlowsmallcaps{\contentsname}}{\spacedlowsmallcaps{\contentsname}}
\tableofcontents
\automark[section]{chapter}
\renewcommand{\chaptermark}[1]{\markboth{\spacedlowsmallcaps{#1}}{\spacedlowsmallcaps{#1}}}
\renewcommand{\sectionmark}[1]{\markright{\textsc{\thesection}\enspace\spacedlowsmallcaps{#1}}}

\clearpage

\begingroup
    \let\clearpage\relax
    \let\cleardoublepage\relax

\endgroup

%% file: content/intro.tex
\chapter{Introduction}\label{ch:introduction}
From the earliest forms of calculation conducted with the Sumerian
abacus to the intricate mechanisms of the Antikythera mechanism, the
essence of computation has always lingered in the shadows of human
advancement. Throughout history, our innate desire to understand,
quantify, and predict the world around us has driven the development
of instruments and methods that aid our cognitive capabilities. With
the dawn of the 20th century, computer science emerged, not merely as
an academic discipline, but as the backbone of the modern society.
Today, computer science stands at the intersection of mathematics,
logic and engineering, driving innovation, reshaping industries, and
revolutionising how we devise solutions to everyday challenges.

In this regard, in the vast landscape of computer science, two
distinct paradigms of problem-solving have arisen over the years --
algorithms and machine learning.

Algorithms, in particular, have been at the heart of computer science
since its inception and represent the keystone upon which every
computer scientist builds their knowledge and expertise. Algorithms
are sequences of instructions that are meticulously juxtaposed to
devise solutions to computational problems, often bolstered by
rigorous theoretical guarantees -- typically in the form of pre- and
post-conditions. When executed with inputs that adhere to their
specified preconditions, algorithms can solve complex problems with
precisions, delivering results that consistently meet the expected
postconditions. However, their deterministic nature is both their
strength and limitation. Most algorithmic inputs originate from
problems observed in the vast complexity of the real world. These
problems, rich in detail and context, are then translated and
abstracted to accommodate the strict preconditions that algorithms
demand. This translation, unfortunately, comes at the price of
possible loss of information, making it hard to find exact
solutions with classical algorithms.

In such a noisy and ever-evolving world, machine learning has emerged
as a flexible paradigm, adept at navigating the uncertainties and
intricacies inherent in \emph{natural}\footnote{As opposed to
  ``abstract'' data typical of algorithms, we use the term ``natural''
  to refer to the unprocessed, raw data that can be observed in the
  real world.} data. Machine learning, especially its most notable
subset, Neural Networks, represents the modern computational approach
to solve tasks and problems. At its core, machine learning is about
enabling computers to learn and discover patterns from data, using
this knowledge to make decisions or predictions. Unlike algorithms,
then, neural networks are not explicitly programmed to perform a task
in a certain way, but they rather learn to adapt themselves as they
are exposed to more and more examples. Their strength, therefore, lies
in their adaptability. Neural networks can process and generalise
across diverse inputs, making them particularly well-suited for
natural data. Often, neural networks that are trained on a particular
data distribution can be used as-is or as a starting point to tackle
computational tasks that are different from the original one
\citep{reyes2015fine}. Algorithms, in comparison, do not show this
degree of adaptability. They are designed to tackle a specific problem
and typically need manual intervetions to adapt to a new one. Machine
learning flexibility, however, also comes at a cost. Indeed, neural
networks, \emph{deep} neural models, often lack the robust theoretical
guarantees that accompany classical algorithms. As a general rule of
thumb, the more complex the neural architecture is, the less we can
say about the learnt function, even though great efforts are being
made in order to make deep learning \textit{interpretable}
\citep{burkart2021survey, numeroso2021meg,
  bacciu2022explaining}. Furthermore, neural networks usually struggle
to generalise their knowledge in
out-of-distribution\footnote{Out-of-distribution scenarios refer to
  situations where neural networks face data that differs noticeably
  from their training distribution.} scenarios
\citep{neyshabur2017exploring}. Algorithms, on the other hand, have
the inherent ability to generalise to any input meeting its
pre-conditions.

At first glance, thus, the algorithmic and neural paradigms may be
perceived as diametrically opposed. However, by careful assessment of
their properties, they can actually be envisioned as complementary
approaches, each addressing the shortcomings of the other. Where
algorithms fall short in the face of task flexiblity and data
adaptability, neural networks excel; and where neural networks operate
in absence of theoretical guarantees, algorithms excel. This synergy
suggests a potential convergence of the two paradigms -- a unified
approach that harnesses the deterministic strength of algorithms and
the adaptive flexibility of neural networks.

Stimulated by this fascinating perspective, researchers began
exploring towards incorporating the principled and structured
knowledge of algorithms into learning models, laying the foundation
for what is now recognised as Neural Algorithmic Reasoning (NAR)
\citep{velickovic2021neural}. Broadly speaking, neural algorithmic
reasoning strives to enable neural networks to demonstrate reasoning
abilities, specifically mirroring the reasoning inherent in classical
algorithms (i.e., the ability of slowly processing input data in a
principled manner in order to return a desired outcome). In
particular, the integration of \emph{algorithmic knowledge} in machine
learning models is typically performed by instructing neural networks
to \emph{execute} classical algorithms.

The key motivations for such integration are multiple. First, enabling
reasoning in neural networks could mitigate the limitations of
traditional neural networks, particularly in their susceptibility to
overfitting, not to mention that it is a long-standing research goal
\citep{khardon1997learning}. Moreover, \emph{neural algorithmic
  reasoners}\footnote{Throughout this dissertation, we will be using
  this term to refer to neural networks that can execute classical
  algorithms.} are, by their very nature, knowledge of classical
algorithms encoded in their parameters. This unlocks various
possibilities to leverage such algorithmic expertise as \emph{prior
  knowledge} and utilise it for problems that might benefit from the
application of said algorithmic knowledge (i.e., effectively using
algorithms as \emph{inductive bias}).

\section{Objectives}
Inspired by the above premises, this dissertation aims to explore the
potential of neural algorithmic reasoners, particularly regarding
their abilities to \emph{learn to execute classical algorithms} and
and \emph{validate} the effectiveness of employing trained algorithmic
reasoners as inductive priors for related downstream tasks.

The main contributions of this thesis are aimed at addressing these
two research questions, particularly in the context of graphs, given
that many classical algorithms of interest are developed and designed
for structured data \citep{cormen2009introduction}. Additionally, we
will be seeking to provide evidence for the previously mentioned
questions from both theoretical and empirical viewpoints.

To address the concern regarding \emph{learnability} of classical
algorithms, we propose a theoretical framework drawing connections
between graphs, neural networks and tropical algebra
\citep{landolfi2023tropical}. In this setting, equivalence between
algorithms, particularly dynamic programming algorithms, and neural
networks will be established. We will also demonstrate how to derive a
powerful neural network architecture suitable for learning algorithms
based on such connections.

Moving outside the context of dynamic programming algorithms, we
propose to learn algorithms through duality \citep{numeroso2023dual},
effectively showing how we can borrow concepts from various areas
related to algorithms, such as combinatorial optimisation, to enhance
the extent to which algorithmic reasoning can be encoded into
neural networks. This contribution also serves as a first practical
example of how algorithms, used as inductive priors, can aid to solve
standard machine learning tasks more accurately.

On this line, we present two more contributions: an algorithmic
reasoner that learns \emph{consistent} heuristic functions for
planning problems \citep{numeroso2022learning}; and an extensive study
on the effectiveness of \emph{transferring} algorithmic knowledge to
NP-hard combinatorial optimisation problems \citep{georgiev2023neural}.

Moreover, as a side objective, this thesis also endeavors to serve as
an introductory guide to the world of Neural Algorithmic Reasoning,
particularly through its \autoref{ch:nar}, tailored to be
comprehensible for those unfamiliar with NAR.

\section{Thesis outline}
In the following, we provide a brief outline regarding the content of
this thesis:

In \autoref{ch:background}, we review foundational and theoretical
concepts pivotal to this research. This chapter introduces the
fundamental principles of graph theory, combinatorial optimisation,
algorithms, and finally, an exploration into the machine learning
paradigm that utilizes graph structures, known as Graph Networks.

In \autoref{ch:nar}, we present an introduction to Neural Algorithmic
Reasoning (NAR). This chapter will explore and discuss practical
scenarios in which the straightforward application of algorithms fail,
and actionable interventions to mitigate said issues.

In \autoref{ch:contribution-dp}, we focus on the primary contributions
of this thesis related to the learning of \emph{min}-aggregated
dynamic programming algorithms. The chapter delves into both the
theoretical and practical facets of these algorithms, drawing
connections with tropical algebra and introducing strategies to
enhance path planning in dynamic programming contexts using learnt
heuristics.

In \autoref{ch:contribution-conar}, our exploration revolves
around the learning of algorithms specifically aimed at combinatorial
optimisation problems. This chapter presents empirical validations on
the utility of various algorithmic biases, including leveraging strong
duality the application of algorithmic priors for solving
combinatorial problems.

In \autoref{ch:conclusions}, we summarise the content of this
dissertation, as well as adding some concluding remarks and discuss
future research avenues.

Finally, the appendix contains the full list of publications and talks
given during my PhD.

%% file: content/background.tex
\part{Background}
\chapter{Preliminaries}\label{ch:background}
In this chapter, we delve into the foundational and theoretical
concepts essential to the understanding and development of the
research presented in this PhD thesis. This chapter is structured as
follows.

The first section, \textit{Graphs}, serves as a basis for
comprehending the subsequent topics. Here, we explore the fundamental
principles and definitions of graph theory, a branch of mathematics
and computer science that deals with the study of graphs and their
properties.

The second section, \textit{Combinatorial Optimisation}, delves into
the area of combinatorial problems. We define what combinatorial
optimisation problems are and discuss their connections to decision
problems, as well as introduce a categorisation into different classes
of complexity.

Moving forward, the third section, \textit{Algorithms}, delves into
the area of computational problem-solving techniques. It discusses the
deep interconnection between algorithms and combinatorial optimisation
problems, and how the former have historically been applied for
solving the latter. We also examine algorithms as discrete functions,
discussing the consequential challenges of learning them through
gradient-based methods.

The final section, \textit{Graph Networks}, introduces a machine
learning paradigm that leverages and processes graph structures.  This
section provides the necessary background for understanding how graph
networks work and their similarities and applications for learning and
behaving like algorithms.

\section{Graphs}
Graphs are mathematical structures that offer a flexible and powerful
framework to model a vast array of real-world systems and
phenomena. Their inherent strength lies in the capability to represent
pair-wise\footnote{For the sake of completeness, in the context of
  hypergraphs, relationships can extend beyond pair-wise connections
  to involve multiple entities simultaneously. However, within the
  context of this thesis we will always consider standard graphs and
  pair-wise relationships among objects.} relationships and
dependencies between entities in a simple and visually comprehensible
manner. This characteristic positions graphs as the primary choice for
modelling complex notions across numerous scientific fields.

In the realm of social sciences, for instance, graphs are instrumental
in modeling social networks where nodes represent individuals and
edges depict social connections \citep{netletton2013data}. This
modeling enables deep insights into social dynamics and community
structures \citep{liben2003link, kumar2006structure,
kumar2022influence}.  Similarly, in biological sciences, graphs have
proven useful for enhancing our understanding of genetic networks and
protein-to-protein interaction networks, thanks to their inherent
ability to represent molecular structures where nodes represent atoms
of elements and edges symbolise chemical bonds
\citep{rietman2011review, winterbach2013topology}.

The domain of combinatorial optimisation, a branch of mathematics and
computer science that seeks optimal solutions in discrete and
combinatorial structures, is another key area where graphs are of
invaluable importance.  A more detailed exploration of combinatorial
optimisation problems will be carried out in \autoref{sec:co}. For the
purpose of this section, it is sufficient to mention that numerous
combinatorial problems, such as the \textit{Maximum Flow Problem}
\citep{fulkerson1955computation} and the \textit{Travelling Salesman
  Problem} \citep{robinson1949hamiltonian} can be efficiently encoded
within the domain of graphs. Not only does this simplify the
visualisation of these problems, but also aids in finding solutions,
through the use of (graph) algorithms.  For instance, transportation
is a prime example of the practical usage of graphs in the field of
combinatorial optimisation.  Here, nodes can represent waypoints like
airports or train stations, and edges embody the connections between
these points, enabling optimization of route planning
\citep{dantzig1959truck}, congestion management
\citep{kumar2005congestion}, and infrastructure development
\citep{hakimi1964optimum}.

In essence, graphs are of fundamental importance in the field of
computer science, and the continuing evolution of graph-based
methodologies (including in the sphere of deep learning) motivates the
necessity for a comprehensive discussion of graphs in the subsequent
sections.

\subsection{Graph Theory}\label{sec:graph-theory}
In the field of graph theory, a graph $g$ is defined as a tuple $g$ =
$\tuple{\gV, \gE, \gX_\gV, \gX_\gE}$. In particular, $\gV$ represents
a collection of entities known as \textit{nodes}, and
$\gE$ denotes a set of \textit{edges} (or \textit{links}), expressing
relations between elements of $\gV$.  $\gX_{\gV}$ and
$\gX_{\gE}$ indicate the potential set of node and edge features,
respectively, providing the possibility to attach useful information
to them. Such pieces of information are highly valuable when
conducting operations on graphs, such as computing statistics,
executing algorithms and extracting potentially useful patterns from
them.

Typically, graphs can be classified into different types, mainly
depending on the structure of the edge set $\gE$. In the following, we
furnish formal definitions of general graph structures and their
properties.  Unless specified otherwise, most of the definitions
hereafter refer to graphs devoid of node and edge features, denoted as
$g = \tuple{\gV, \gE}$.
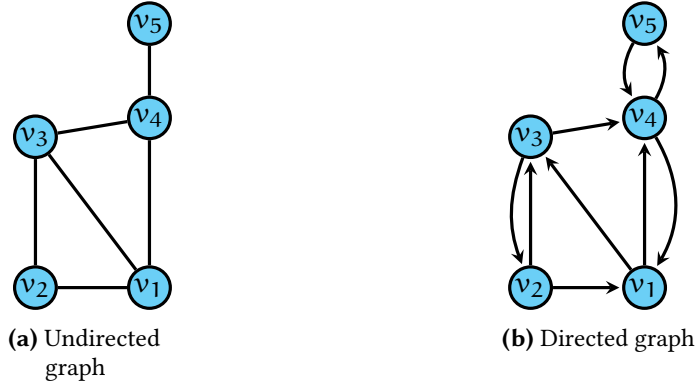
\begin{figure}
    \centering
    \subcaptionbox{Undirected graph\label{subfig:undirected-graph}}{%
        \input{gfx/background/undirected_graph.tex}
      }
    \hspace{4cm}
    \subcaptionbox{Directed graph\label{subfig:directed-graph}}{%
        \input{gfx/background/directed_graph.tex}
      }
      \caption{Visual representation of undirected and directed
        graphs. In \textbf{(b)}, node adjacencies
        $\set{\set{v_1, v_4}, \set{v_2, v_3}, \set{v_4, v_5}}$ are
        bi-directional connections and may be expressed as undirected
        edges. Equivalently, all edges in \textbf{(a)} may be replaced
        by two directed edges.}
    \label{fig:graphs}
\end{figure}

\begin{definition}[Undirected Graph]
    Let $g$ be a graph. Then, $g$ is an undirected
    graph if and only if the elements in $\gE$ are unordered pair
    of vertices, hence $\gE \subseteq \set{\set{u, v} \mid u, v \in \gV}$.
\end{definition}
\begin{definition}[Directed Graph]
    Let $g$ be a graph. Then, $g$ is a directed
    graph if and only if the elements in $\gE$ are ordered pair
    of vertices, hence $\gE \subseteq \set{(u, v) \mid u, v \in \gV}$.
    In such case, $u$ is called head (or source) and
    $v$ is called tail (or target).
\end{definition}
The two above definitions outline the difference between undirected
and directed graphs (see illustration in \autoref{fig:graphs}). In
other words, a directed edge $(u, v)$ expresses a unidirectional
relation from $u$ to $v$. It is important to point out that an
undirected edge can be thought of as the union of two contrarily
oriented edges, i.e. $\{u, v\} = \{(u, v), (v, u)\}$. Therefore,
w.l.o.g, we consider undirected graphs as their directed versions, by
substituting each non-oriented edge $\{u, v\}$ for two oriented edges
$(u, v)$ and $(v, u)$. For the rest of this thesis, whenever needed,
we always refer to undirected edges with curly brackets notation,
i.e. $\{u, v\}$ and to directed edges with round brackets, i.e.
$(u, v)$.
\begin{definition}[Outgoing/Ingoing Edges]
    Let $g$ be a directed graph. The set of
    outgoing edges of an arbitrary node $v \in \gV$ is defined as
    $\gE^{out}_v = \{(v,x) \mid \forall x \in \gV : (v, x) \in \gE\}$.
    The set of ingoing edges of $v$ is then defined as $\gE^{in}_v =
    \{(x,v) \mid \forall x \in \gV : (x, v) \in \gE\}$.
\end{definition}
\begin{definition}[Adjacency]
    Let $g$ be an undirected graph and $u, v
    \in \gV$.  Then, $u$ and $v$ are adjacent if $\{u, v\} \in \gE$.
\end{definition}
\begin{definition}[Adjacency Matrix]
    Let $g$ be a graph. The adjacency matrix $\mA
    = [a_{uv}] \in \{0, 1\}^{|\gV| \times |\gV|}$ is a square matrix
    where $a_{uv} = (u, v) \in \gE$.
\end{definition}
The adjacency property of two nodes can be defined for directed graphs
as well. The only difference is that the adjacency between two nodes
$u$ and $v$ depends on the direction of the edge joining the
nodes. More precisely, in a directed graph if $(u,v) \in \gE^{out}_u$,
$u$ can be said adjacent to $v$, but this does not necessarily hold
from $v$ to $u$, unless $(v,u) \in \gE^{in}_u$ too (see nodes $u$ and
$w$ in \autoref{subfig:directed-graph}).

This entails a subtle difference for an adjacency matrix of an
undirected graph and that of a directed one. In particular, in case of
undirected graphs, if there is an edge from $u$ to $v$, there is also
an edge for $v$ to $u$. That means that the adjacency matrix $\mA$ is
symmetric, i.e. $\mA = \mA^T$, whereas for a directed graph it may occur
that $\mA \neq \mA^T$.
\begin{definition}[Neighbourhood]
  Let $g$ be an undirected graph. Let $v \in \gV$ be a node of
$g$. Thus, the neighbourhood of $v$ is defined as:
\begin{equation*}
  \gN(v) = \{x \in \gV \mid \{x, v\} \in \gE\}.
\end{equation*}
\end{definition}
\begin{definition}[Node degree]
    Let $g$ be an undirected graph.  Let $v \in
\gV$ be a node of $g$. The node degree $d(v): \gV \rightarrow \sN$ is
defined as the cardinality of its neighbourhood, i.e. $d(v) =
|\gN(v)|$.
\end{definition}
Also for neighbourhood and node degree, there is distinction for nodes
in undirected graphs and directed graphs. As relations in undirected
graphs are symmetric, if $u \in \gN(v)$, then $v \in \gN_u$ as
well. However, in directed graphs there is clear distinction between
the set of in-neighbours $\gN^{in}_v = \{x \in \gV \mid (x, v) \in
\gE^{in}_v\}$ and out-neighbours $\gN^{out}_v = \{x \in \gV \mid (v,
x) \in \gE^{out}_v\}$, similarly to the concepts of in-going and
out-going edges. As a consequence, we derive in-degree and out-degree
from in and out neighbours.
\begin{definition}[Finite Path] \label{def:path}
  Let $g$ be a graph. A finite path in $g$ is a sequence of nodes
  $\pi = v_1 v_2 \dots v_n$ for which it holds that
  $\forall v_i,v_{i+1} \in \pi : (v_i, v_{i+1}) \in \gE$.
\end{definition}
\begin{definition}[Finite Cycle] \label{def:cycle}
  Let $g$ be a graph. $\pi$ is a finite cycle if
  and only if $\pi$ is a finite path and $\pi=v_1 v_2 \dots v_n v_1$.
\end{definition}
\begin{definition}[Reachability]\label{def:reachability}
  Let $g$ be a graph and $u, v \in \gV$ two
  arbitrary nodes.  We say that $u$ is reachable from $v$ if there exist
  a path $\pi$ from $u$ to $v$, i.e. $\pi = u x_1 x_2 \dots x_n v$,
  where $x_i \in \gV$.
\end{definition}
\begin{definition}[Connectivity]
  Let $g$ be a graph. $g$ is said to be connected
if and only if $\forall u,v \in \gV$, v is reachable from u. A graph
which is not connected is called disconnected.
\end{definition}
These seemingly simple concepts are of crucial importance for many
problems in the real-world. For example, via connectivity and
reachability one can verify that a computer network is connected,
hence all end-points can communicate with one another. Furthermore,
concepts of cycles and paths are at the basis of the class of vehicle
routing problems, which are widely used to model real problems such as
route-delivery optimisation.
\begin{definition}[Hamiltonian Path]\label{def:hamiltonian-path}
  Let $g$ be a graph. $\pi$ is a hamiltonian path if and only if $\pi$
  is a path such that $|\pi| = |\gV|$ and $\forall 1 \leq i < j \leq
  |\pi| \;:\; \pi_i \neq \pi_j$.
\end{definition}
A hamiltonian path can be thought of as a path that passes through
every single node exactly once. This concept is of high importance as
it is closely related to the more familiar Travelling Salesman Problem
(TSP). Consequently, hamiltonian paths are relevant for all families
of route opimisation problems as well as task scheduling
\citep{simonin2011isomorphic}.
\begin{definition}[Isomorphism]
  Let $g$ and $g'$ be graphs.
If there is a bijective function $f: \gV_g \rightarrow \gV_{g'}$
such that $(u, v) \in \gE_g \iff (f(u), f(v)) \in \gE_{g'}$, then $g$
and $g'$ are called isomorphic.
\end{definition}
In simpler terms, two graphs are isomorphic if there is a bijection
between their nodes that preserves adjacency. This means that the two
graphs are essentially the same graph with a different labelling of
nodes. From a machine learning perspective, graph isomorphism are
deeply interwined with the expressive capacity of learning models for
graphs. As we will see in \autoref{sec:gn}, graph networks, the
central deep learning architectures for learning on graphs, are
essentially \textit{structural transductions} and deeply connected to
the notion of graph isomorphism. Transductions can be intended as
mathematical functions $\gT: \sA \rightarrow \sB$. In the context of
deep learning for graphs, any element in $\sA$ and $\sB$ is assumed to
be a graph $g = \tuple{\gV, \gE, \gX_\gV, \gX_\gE}$. Now, we conclude
this section by discussing IO-isomorphic transductions.
\begin{definition}[Skeleton]
  Let $g=\tuple{\gV, \gE, \gX_\gV, \gX_\gE}$ be a graph. The skeleton
  of $g$, denoted by $\text{skel}(g)$, is the graph obtained by
  ignoring the set of node and edge features (i.e.,
  $\text{skel}(g) = \tuple{\gV, \gE}$.
\end{definition}
\begin{definition}[IO-isomorphism]\label{def:io-iso}
  Let $g=\tuple{\gV, \gE, \gX_\gV, \gX_\gE}$ be a graph and let $\gT:
\gX \rightarrow \gY$ be a structural transduction. Then if
$\forall g \in \gX \;:\; \text{skel}(g) = \text{skel}(\gT(g))$, then
$\gT$ is said to be an IO-isomorphism.
\end{definition}
Fundamentally, an IO-isomorphism is a structural transduction that
preserves the topology of $g$, while possibly altering its set of node
and edge features. Hence, such transductions can be framed as
\textbf{graph relabelling} methods, which refer to the process of
assigning new labels to the set of nodes (and possibly edges),
preserving the underlying structure. Some popular applications of
graph relabelling methods involve \textit{graph isomorphism} tests
as well as graph processing with deep learning architectures.
Both applications will be discussed more thoroughly in the next
sections.

\subsection{Instances of graphs}
\label{sec:graph-types}
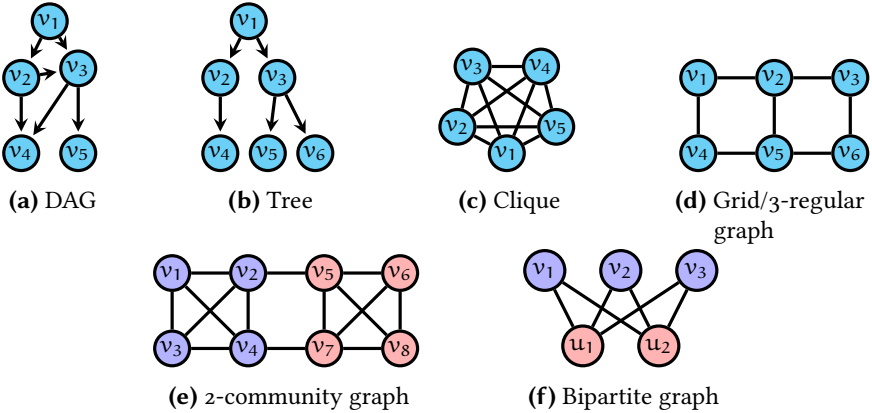
\begin{figure}
    \centering
    \subcaptionbox{DAG\label{subfig:dag}}{%
      \footnotesize
      \input{gfx/background/dag.tex}
    }
    \hspace{1cm}
    \subcaptionbox{Tree\label{subfig:tree}}{%
      \footnotesize
        \input{gfx/background/tree.tex}
      }
    \hspace{1cm}
    \subcaptionbox{Clique\label{subfig:clique}}{%
      \footnotesize
      \input{gfx/background/clique.tex}
      }
    \hspace{1cm}
    \subcaptionbox{Grid/3-regular graph\label{subfig:grid}}{%
      \footnotesize
      \input{gfx/background/grid.tex}
      }
    \hspace{1cm}
    \subcaptionbox{2-community graph\label{subfig:community}}{%
       \footnotesize
       \input{gfx/background/kcommunity.tex}
      }
    \hspace{1cm}
    \subcaptionbox{Bipartite graph\label{subfig:bipartite}}{%
       \footnotesize
       \input{gfx/background/bipartite.tex}
      }
      \caption{Illustration of various graph types.  For better
        clarity, self-loops are not shown. In the grid illustration,
        $v_1, v_3, v_4$ and $v_6$ possess self-loops, ensuring the graph's
        3-regularity.}
      \label{fig:graph-types}
\end{figure}

Above, we introduced the general concept of graph. We can, however,
categorise graphs into different families that exhibit a very specific
(and similar within the corresponding family) topology. In the
following, we discuss some of the most well-known families of graphs
that will ease understanding of the research contributions presented
in this thesis. \autoref{fig:graph-types} illustrates a visual example
for each type of graphs defined herein.
\begin{definition}[Directed Acyclic Graph (DAG)]
  A directed graph $g$ is a Directed Acyclic Graph (DAG) if there are
  no directed cycles. Formally, if $g = (\gV, \gE)$ where
  $\gE \subset \gV \times \gV$, then $g$ is a DAG if and only if
  $\nexists v_1, v_2, \ldots, v_k \in \gV$ such that
  $(v_i, v_{i+1}) \in \gE$ for $1 \leq i < k$ and
  $(v_k, v_1) \in \gE$.
\end{definition}
DAGs are used in many fields including computer science and
physics. They are notably used to represent dependencies among tasks,
as in a project management tool or in a build system, where each
vertex is a task and there is an edge from $v$ to $w$ if task $v$ must
be completed before task $w$ can start.
\begin{definition}[Tree]
  Let $g$ be a DAG. Thus, $g$ is a tree if and only if $\exists! r 
  \in \gV \;.\; |\gE^{in}_v| = 0$ and $\forall v \in \gV \setminus 
  \set{r} \;:\; (v, v) \notin \gE \;\land\; |\gE^{in}_v| = 1$.
\end{definition}
Intuitively, each node in a tree can be connected to multiple 
``successors'' (but not itself), and is required to only have one 
parent (i.e.,  in-degree must be 1). The only exception is 
represented by the root node \textit{r}, which has in-degree equal to 0.
A tree can then be categorised as a special case of a DAG. 
Trees are frequently used as a data structure in computer
science to represent hierarchical relations. In fact, as opposed to
DAGs, we can always partition the set of nodes in different ``sets'',
called levels. To easily understand this concept, consider the
function $\text{succ}(v) = \set{u \mid u \in \gN^{out}(v)}$ that,
given a node $v$, returns the set of ``out neighbours'', called
\textit{successors}. Intuitively, we can define the root node $r$ as
being at level 0 and its direct successors (i.e., $\text{succ}(r)$) as
being at level 1. Consequently, $\text{succ}(\text{succ}(r))$ are at
level 2, and so on. This tree-structure underpins, for instance,
filesystems in modern operative system \citep{mathur2007new} as well
as programming languages \citep{slonneger1995formal}.
\begin{definition}[Clique]
  Let $g$ be a graph. $g$ is a clique if and only if
  $\forall u \neq v \in \gV \;:\; (u,v) \in \gE$.
\end{definition}
In other words, each node is connected to every other node. Cliques
represent completely interconnected groups of entity. This topology
can often be used to represent problems of computational geometry in
the form of graphs, such as the convex hull or the aforementioned TSP.
\begin{definition}[k-regular graph]
  Let $g$ be a graph. $g$ is $k$-regular if and only if
  $\forall v \in \gV \;:\; |\gN(v)| = k$.
\end{definition}
\begin{definition}[2D grid]
  Let $g$ be a graph. $g$ is a 2D grid its nodes and edges can be
  placed on a regular lattice in two dimensions, such that adjacent
  nodes are only connected (possibly) to themselves and horizontally
  and vertically with each other.
\end{definition}
For a visual representation of 2D grids, refer to
\autoref{fig:graph-types}. Grid graphs are often used to represent
spatial relationships. For instance, in computer vision or image
processing, pixels in an image may be treated as nodes of a grid
graph, with edges connecting each pixel to its neighbours.
\begin{definition}[k-community graph]
  Let $g$ be a graph. $g$ is a $k$-community graph if
  $\exists C_1, C_2, \dots, C_k \subset \gV \;:\; \gV = C_1 \cup C_2
\cup \dots \cup C_k$ and $\forall i \neq j C_i \cap C_j = \emptyset$,
and such that $\forall u, v \in C_i, w \in C_j \;:\; i \neq j$, the
probability that there is an edge $(u, v)$ is higher than the
probability that there is an edge $(v, w)$.
\end{definition}
Hence, a $k$-community graph is partitioned into $k$ communities $C$,
where a community $C \subset \gV$ is a subset of nodes such that the
nodes within the community are more densely connected to each other
than to the nodes outside the community.
\begin{definition}[Bipartite graph]
    Let $g = \tuple{\gV, \gE, \dots}$ be a graph.  $g$ is a bipartite
graph if $\exists \gV_1, \gV_2 \subset \gV : \gV_1 \cap \gV_2 =
\emptyset \land \gV = \gV_1 \cup \gV_2$ such that $\forall (u, v) \in
\gE : u \in \gV_1 \land v \in \gV_2$.
\end{definition}
Bipartite graphs can be thought of as two sets of nodes representing
different entities which depend on each other. For example, bipartite
graphs are exploited to represent the state of mixed integer programs
(MIPs), where we aim to find the minimising assignment to a set of
variables ($\gV_1$) subjected to a set of constraints ($\gV_2$).

\subsection{Computational problems on graphs}
\label{sec:problems-on-graphs}
In this section, we explore some popular computational
problems associated with graphs. As we will analyse in
\autoref{sec:algorithms}, these problems are often targeted by
algorithms for their resolution and their inherent difficulty
can vary widely.
\begin{problem}[Reachability]\label{prob:reachability}
  Let $g$ be a graph and $u, v \in \gV$. Then, the reachability
  problem asks to determine whether there is a path $\pi$ from node
  $u$ to node $v$ in the graph.
\end{problem}
Similarly to \autoref{def:reachability}, reachability is the problem
of determining whether a certain destination can be reached from a
given starting point. This problem forms the basis of many other graph
problems, typically in the form of a ``subproblem'' (i.e., it must be
solved prior to resolution of the main problem).
\begin{problem}[Shortest path]\label{prob:shortest-path}
  Let $g = \tuple{\gV, \gE, \gW}$ be a graph and $u, v \in \gV$, where
  $w_{uv} \in \gW$ is a ``weight'' assigned to each edge $(u, v)$ in
  the graph. The shortest path problem asks to find a path
  $\pi=u \dots v$ such that $c(\pi) = \sum_{i} w_{\pi_i, \pi_{i+1}}$ is
  mininum.
\end{problem}
The shortest path problem is common in applications like navigation,
where the nodes represent locations and the edge weights represent the
cost of traversing that edge (typically in terms of time or distance).
\begin{problem}[Minimum Spanning Tree (MST)]
  \label{prob:mst}
  Let $g = \tuple{\gV, \gE, \gW}$ be an undirected, connected graph.
  A minimum spanning tree $t \subseteq g$ is a subgraph of $g$
  such that $t$ is a tree spanning all the nodes in $g$ and the total
  weight of the selected edges in $t$ is minimised.
\end{problem}
The MST problem is related to shortest path, and constitutes an
important subproblem for a rich collections of more complex graph
problems. Indeed, it is used as a subgoal in the popular
\textit{Christofides} heuristic \citep{christofides1976worst}, an
algorithm that is used to retrieve suboptimal solutions for the TSP.
\begin{problem}[Maximum Flow (Max-Flow)]
  \label{prob:maxflow}
  Let $g = \tuple{\gV, \gE, \gC}$ be a graph, with $c_{uv} \in \gC$
  being a ``capacity'' assigned to each edge $(u, v)$ in the
  graph. Let $s \in \gV$ be the ``source'' node and $t \in \gV$ be the
  ``target'' node.  The maximum flow problem is concerned with finding
  a flow $f: \gE \rightarrow \sR$ such that the total amount of flow
  from $s$ to $t$ is maximised, and flow assignments $f_{uv}$ do not
  exceed capacities $c_{uv}$.
\end{problem}
The Maximum Flow problem models scenarios such as traffic flow or data
transfer in a network, where the goal is to maximise the throughput
from source to sink. Differently from the presented problems so far,
MaxFlow poses additional challenges. In fact, an \textit{optimal}
solution for this problem has to satisfy two major constraints. The
first one is presented in the above problem definition (i.e., edge
capacity constraint), while the second is called \textit{flow
  conservations}. Respectively, the edge capacity constraint can
be described as follows:
\begin{align}\label{eq:capacity-constraint}
  \forall (u, v) \in \gV \;.\; f_{uv} \leq c_{uv},
\end{align}
whereas flow conservation is formally stated as:
\begin{align} \label{eq:flow-conservations}
  \begin{split}
    \forall u \in \gV \setminus \{s, t\}
    : &\left(\sum_{(u,v) \in E} f_{uv} + \sum_{(v,u) \in E} f_{vu}\right) = 0\\
    \land &\left(\sum_{(s, v) \in \gE} f_{sv} = -\sum_{(v, t) \in \gE} f_{vt}\right).
  \end{split}
\end{align}
Equation \eqref{eq:flow-conservations} states precisely that all the
flow coming out from the source must reach the target node (i.e.,
their algebraic sum is 0), or, equivalently, that no intermediate node
is allowed to retain flow. Typically, MF seeks the maximum solution to
\eqref{eq:capacity-constraint} and \eqref{eq:flow-conservations}.
The entire MF problem can then be mathematically described as:
\begin{equation}\label{eq:mf}
  \max \sum_{(s, v) \in \gE} f_{sj} \quad
  \text{subject to \eqref{eq:capacity-constraint}} \, \land \,
  \text{\eqref{eq:flow-conservations}}.
\end{equation}

\begin{problem}[Minimum $(s, t)$ Cut (Min-Cut)]
  \label{prob:mincut}
  Let $g = \tuple{\gV, \gE, \gW}$ be a graph, with $w_{ij} \in \gW$
  being a ``weight'' assigned to each edge $(i, j)$ in the graph. Let
  $s, t \in \gV$. The minimum $(s, t)$ cut problem seeks to find a
  partition of nodes $(\gV_s, \gV_t)$ such that $s \in \gV_s$ and $t \in
  \gV_t$ and the sum of the weights $w_{ij}$ associated with edges $(i,
  j)$ crossing the partition (i.e., $i \in \gV_s$ and $j \in \gV_t$) is
  minimised.
\end{problem}
Intuitively, the minimum $(s, t)$ cut problem seeks to identify the
set of \textit{bottleneck} edges, such that the sum of their weights
is minimal, whose removal \textit{disconnects} $s$ from $t$.  This
problem is closely related to maximum flow. Actually, the two problems
are linked by the Max-Flow Min-Cut theorem \citep{dantzig2003max},
which states that the value of the maximum flow is equal to the value
of the minimum cut in a flow network. Furthermore, this theorem
establishes a \textit{dual} relation between min-cut and
max-flow. Duality is a property of, especially, linear programs, which
will be discussed more thoroughly in \autoref{sec:co}. Both min-cut
and max-flow fit in the linear programming framework.
\begin{problem}[Travelling Salesman Problem]
  \label{prob:tsp}
  Let $g = \tuple{\gV, \gE, \gW}$ be a clique. Given a starting node
  $v$, The Travelling Salesman Problem (TSP) asks for the shortest
  possible hamiltomian path $\pi$ that visits each node exactly once
  and returns to $v$.
\end{problem}
\begin{problem}[Vehicle Routing Problem]
  \label{prob:vrp}
  Let $g = \tuple{\gV, \gE, \gW}$ be a clique. The Vehicle Routing
  Problem (VRP) is a generalisation of the TSP which asks for $n$
  paths $\pi_1, \dots, \pi_n$ such that
  $\pi_1 \cap \dots \cap \pi_n = \emptyset$ and
  $\pi_1 \cup \dots \cup \pi_n = \gV$. Furthermore, each $\pi_i$
  is a hamiltonian path on the subgraph formed by the nodes traversed
  by $\pi_i$.
\end{problem}
TSP and VRP are classic problems in computer science and operations
research. They have several practical applications, such as route
optimisation for delivery services.
\begin{problem}[Vertex K Center Problem]
  \label{prob:vkc}
  Let $g$ be a graph. Given $k \in \sN$, the Vertex K
  Center Problem asks for a set of $k$ nodes $C \subseteq \gV$
  such that the maximum distance from any node to the nearest center
  in $C$ is minimised.
\end{problem}
The Vertex K Center Problem \citep{hakimi1964optimum} is extremely
important to optimise emergency facility locations. For example, in a
network representing a city, the vertices could represent potential
locations for hospitals, and the goal is to place $k$ hospitals in a
way that minimises the maximum distance anyone has to travel to reach
the nearest hospital.

\section{Combinatorial Optimisation}\label{sec:co}
Combinatorial Optimisation (CO) is a branch of computer science which
is concerned with finding the best configuration of elements within a
discrete and finite set of objects, subject to a (possibly empty) set
of constraints. Issues related to CO problems hold substantial
practical relevance, especially concerning modern industrial
applications. Actually, the Travelling Salesman Problem, described in
\autoref{prob:tsp}, can be categorised as a combinatorial optimisation
problem, with wide-ranging applications spanning from genome
arrangement \citep{sankoff1998multiple} and system biology
\citep{johnson2006protein} to route optimisation \citep{xu2019brief}
and logistics \citep{baniasadi2020transformation}. Furthermore, the
Vertex K Center Problem (\autoref{prob:vkc}) may also be considered
under the framework of combinatorial optimisation

Formally, a combinatorial optimisation problem is a tuple
$\gP=\tuple{\gI, \gF, m}$, where $\gI$ is a set of possible instances
for $\gP$, given a instance $p \in \gI$, $\gF$ is the set of
\textit{feasible} solutions (i.e., the ones satisfying all the
constraints) and $m$ is the \textit{objective function} which given a
feasibile solution $y \in \gF$, assigns a measure of \textit{quality}
to it. Solving $p$ requires finding the optimal solution $y^*$, which
is typically defined as:
\begin{equation}
  y^* = \argmax \set{m(y) \mid y \in \gF}.
\end{equation}
W.l.o.g, any combinatorial optimisation problem can be expressed
either as a minimisation or a maximisation problem.  As in
combinatorial problems the optimisation is conducted on (usually
large) discrete spaces, graphs serve as a natural framework for
encoding such problems, given their inherent ability to represent
discrete objects and relations that exist between them. Indeed,
all the problems presented in \autoref{sec:problems-on-graphs} are
combinatorial optimisation problems.

Furthermore, CO problems can vary in complexity, in the sense that
given two combinatorial optimisation problems (e.g., TSP and Shortest
Path), one may be more difficult than the other. In particular, for
each combinatorial problem $\gP$, we can derive its corresponding
\textit{decision} problem. The decision version seeks to determine
whether there exists a feasible solution $y$, satisfying a particular
measure $m_0$. Hence, the main difference between combinatorial and
decision problems lies in the output space of solutions. In
particular, the latter asks a question that requires a binary solution
(i.e., \textit{yes} or \textit{no}), whereas a solution to the former
is an assignment of the problem variables. Notwithstanding this
distinction, the close relation between combinatorial and decision
problems enables to (informally) ``categorise'' combinatorial problems
in the well-known classes of complexity which are originally defined
for their decision versions: P, NP, NP-Complete and NP-Hard. A more
comprehensive exploration of the differences between these classes
will be provided in \autoref{sec:complexity-classes}.

\subsection{The P vs NP dilemma}
\label{sec:complexity-classes}
Leveraging their connection to decision problems, combinatorial
optimisation problems can be classified into a well-defined
stratification of complexity classes. Although there exist a variety
of different complexity classes and different categorisations, here we
focus on a small fraction (but arguably the most crucial): P and NP.

P is defined as the class of problems that can be solved efficiently
in \textit{polynomial time}. In other words, the time required to find
an optimal solution to these problems is proportional to a polynomial
function of the problem's size. Precisely, given a problem $\gP$,
there exists an algorithm $A$ that can solve optimally all of its
instances $p \sim \gP$ in at most $O(n^k)$ operations, where $n$ is
$p$'s size and $k$ is some constant.

NP, instead, is the class of problems that can be \textit{verified} in
polynomial time. This means that if we had a \textit{certificate} of a
solution, we could use it to verify that the given solution is correct
in polynomial time.

A straightforward consequence of these definitions is that
$\text{P} \subseteq \text{NP}$, namely any problem \textit{solvable}
in polynomial time is also trivially \textit{verifiable} in polynomial
time. The converse implication (i.e., $\text{NP} \subseteq \text{P}$),
which asks whether a problem that is verifiable in polynomial time can
also be solved efficiently is, however, still unverified. This
constitutes the popular P \textit{vs} NP question \citep{cook2000the}
and is arguably the most fundamental question within the field of
computer science.

Furthermore, broader categories such as NP-\textit{completeness}
(NP-C) and NP-\textit{hardness} (NP-H) are usually defined to account
for the most challenging computational problems within NP and
beyond. Informally, NP-C consists of problems that are both in NP and
are ``at least as hard'' as any other problem in NP. NP-H, instead,
includes problems for which verifiability is not even possible in
polynomial time (thus, they may not even be in NP itself), making
them the hardest problems to deal with.

The consequence of the P vs NP dilemma is that, at the time of
writing, for some challenging NP-C and NP-H problems there exists 
no polynomial-time exact algorithm for solving them. This
implies that for many practical problems of interests, we have to rely
on polynomial-time approximation algorithms, also referred to as
\textit{heuristics}. When employing such strategies, the focus becomes
to find ``good'' solutions in a reasonable amount of time, even though
they are not optimal. Certainly, \textit{exact} algorithms still exist
but they require exponential time (i.e., $O(k^n)$), often making their
use infeasible for practical large-scale problems.

Historically, development of specialised heuristics for NP-C/NP-H
combinatorial problems has been a central topic in combinatorial
optimisation and computer science for a very long time
\citep{robinson1949hamiltonian, fulkerson1955computation,
  dantzig1959truck, hakimi1964optimum}. Such heuristics typically
require significant expertise and effort in their development and
engineering.

\subsection{Linear Programming}
\label{sec:lp}
Linear Programming (or \textit{linear programs}) (LP), is a
mathematical framework used to model a variety of combinatorial
optimisation problems where dependencies between the problem variables
and between the problem constraints can be expressed using linear
relationships.  Mathematically, a LP can be phrased as:
\begin{equation}\label{eq:lp}
  \min \left\{\vc^{\top} \vx \mid \vx \in
    \sR^n \wedge \mA \vx \leq \vb \wedge \vx \geq 0\right\}.
\end{equation}
Here, $\vc$ is the vector of coefficients for the objective function,
$\vx$ is the vector of variables, $\mA$ is the matrix of coefficients
for the constraints, and $\vb$ is the vector of constants in the
constraints.

In many practical situations, however, the variables of a problem must
take on integer values. This leads to the concept of Integer Linear
Programming (ILP), where the constraints are linear, as in LP, but
some constraints may enforce integrality of some problem
variables:
\begin{equation}
  \min \left\{\vc^{\top} \begin{bmatrix}
                   \vx \\
                   \vy \\
  \end{bmatrix} \mid \vx \in
    \sR^{n-k} \wedge \vy \in \sZ^k \wedge
    \mA \vx \leq \vb \wedge \vx \geq 0\right\}.
\end{equation}
Despite their apparent simplicity, LP and ILP accounts for a multitude
of problems of practical relevance. By extension, ILP broadens the
applicability of LP to problems requiring discrete solutions. For
instance, equation \eqref{eq:mf} is actually a LP. More in general,
all problems presented in \autoref{sec:problems-on-graphs} find their
representations within either the LP (e.g., max-flow/min-cut) or the
ILP framework (e.g., TSP and Shortest Path).

\subsection{Duality}
\label{sec:duality}
The concept of duality is of fundamental importance in the theory of
linear programming. Specifically, duality establishes a relationship
between a given LP problem, known as the \textit{primal problem}, and
another LP problem called the \textit{dual problem}. In the duality
theory, these two problems are tightly connected and interesting
properties arise from their relation. Notably, the dual problem is
derived from the primal problem and possesses the characteristic that
the optimal value of the dual problem provides a bound (either upper
or lower) on the optimal value of the primal problem. In order to see
such upper/lower bound relationship, first consider the primal LP
problem in explicit form as follows:
\begin{equation}\label{eq:lp-primal}
(P) \quad \begin{aligned}
& \underset{\vx}{\text{maximise}}
& & \vc^{\top} \vx \\
& \text{subject to}
& & \mA \vx \leq \vb \\
& & & \vx \geq 0,
\end{aligned}
\end{equation}
which is equivalent to \eqref{eq:lp}. From (P), the dual problem (D)
is constructed by introducing a vector of dual variables $\vy$
corresponding to the constraints of the primal problem. The dual
problem can thus be formulated as follows:
\begin{equation}\label{eq:lp-dual}
(D) \quad \begin{aligned}
& \underset{\vy}{\text{minimise}}
& & \vb^{\top} \vy \\
& \text{subject to}
& & \mA^{\top} \vy \geq \vc \\
& & & \vy \geq 0.
\end{aligned}
\end{equation}
By analising \eqref{eq:lp-primal} and \eqref{eq:lp-dual}, we note that
the dual problem is a maximisation problem if the primal problem is a
minimisation problem (and viceversa). The \textit{Weak Duality
  Theorem} formally establishes the bound relationship between (P) and
(D). W.l.o.g, assuming a maximisation (P) and a minimisation (D), then
the Weak Duality Theorem is stated as follows:
\begin{theorem}[Weak Duality Theorem]
  Let $\vx_0$ be a feasible solution of the primal problem (P). Let
  $\vy_0$ be a feasible solution of the dual problem (D). Then:
  \[
    \vb^{\top} \vy_0 \geq \vc^{\top} \vx_0.
  \]
\end{theorem}
In essence, this theorem guarantees that the objective value of the
dual solution will always be greater than or equal to the objective
value of the primal solution. Conversely, if (P) is a minimization
problem and (D) is a maximization problem, the theorem's conclusion
would be $\vc^{\top} \vx_0 \geq \vb^{\top} \vy_0$.

Furthermore, relationship between a (P) and a (D) can be even
stronger.  Specifically, solving one of the two problems might yield
an objective function value that exactly matches that of its
counterpart. This concept is formally captured in the \textit{Strong
  Duality Theorem}:
\begin{theorem}[Strong Duality Theorem]
  \label{th:strong-duality-theorem}
  If the primal problem (P) has an optimal solution $\vx^*$, then the
  dual problem (D) also has an optimal solution $\vy^*$, and the
  optimal values of the primal and dual problems are equal, that is,
  \[
    \vb^{\top} \vy^* = \vc^{\top} \vx^*.
  \]
\end{theorem}
Both strong and weak duality have practical implications of
significant value. In particular, the aforementioned bound relations
underpin various algorithms and optimization techniques in
combinatorial optimisation. Notable examples in this regards are the
\textit{simplex method} \citep{dantzig1990origins} and the family of
\textit{branch-and-bound} methods, effectively showing practical
relevance of duality even in the context of ILP. Another notable
implication is that duality provides a framework for establishing
relationships between seemingly different optimisation
problems. Indeed, the dual of one problem may correspond to another
combinatorial problem of practical relevance. Max-flow and min-cut, as
discussed earlier in \autoref{sec:problems-on-graphs}, are a
well-known example of such a connection.

To conclude this section, duality is then a significant concept in
combinatorial optimisation. As we will see in the next chapters of
this thesis, duality is at the basis of one of the developed
approaches for learning algorithms during this thesis work.

% \subsection{Column Generation}

\section{Algorithms}\label{sec:algorithms}
Algorithms lie at the heart of computer science and discrete
mathematics, encompassing principles, methodologies, and techniques
that form the building blocks of computational problem-solving. The
study and development of algorithms influence multiple areas, ranging
from mathematics to biology and (combinatorial)
optimisation. Recently, algorithms have made their way into artificial
intelligence models, aiming at empowering learning models by
leveraging the inherent \textit{reasoning} capabilities of these
paradigms. The task of teaching algorithms, however, presents
substantial challenges, particularly for gradient-based methods. In
this section, we establish the notation used to represent algorithms,
essential for comprehending the precise ways in which algorithms can
be processed and learnt (which will be analysed in the following
chapters). Additionally, we discuss the main difficulties algorithms
present to machine learning models and discuss several graph
algorithms within the context of the previously introduced notation.

\subsection{A simple notation for representing algorithms}
\label{sec:algo-notation}
An algorithm is typically defined as a computational procedure that,
taken an input, performs a \textit{sequence} of well-defined
instructions on that input to eventually produce an output
\citep{cormen2009introduction}. At a high level, algorithms operate
similarly to mathematical functions $f: \sA \rightarrow \sB$, in the
sense that they accept an input and produce an output (both possibly
multidimensional). The subtle difference lies in the fact that, while
functions simply establish a relation between an input and a output
(i.e., $f(\va) = \vb$), algorithms lay out the exact procedure needed
to convert input $\va$ to output $\vb$. This procedure can be
considered as a \textit{reasoning} process, where we apply the
sequence of algorithm steps, denoted as $A_t$, to slowly transform
$\va$ to $\vb$:
\begin{equation*}
  \va \overset{A_0}{\longrightarrow} \va^{(1)} \overset{A_1}{\longrightarrow} \dots
  \overset{A_{T-2}}{\longrightarrow} \va^{(T-1)} \overset{A_{T-1}}{\longrightarrow} \vb.
\end{equation*}
As outlined above, then, an algorithm $A$ is composed of a set of
\texttt{inputs}, a set of \textit{intermediate} transformations of the
inputs (i.e., \texttt{hints}), and the set of final \texttt{outputs}.
For the rest of this thesis we will always refer to and consider an
algorithm $A$ as the set of its \inp{}, \hints{} and \out{}. The whole
set of \hints{} plus the final \out{} is referred to as
\textit{algorithm trajectory}.  Furthermore, as we will explore in the
next sections, algorithms usually store and use (some of) the
intermediate \hints{} $\set{\va^{(1)},\dots,\va^{(T-1)}}$ for
subsequent steps. This is why these types of algorithms are often
referred to as \textit{iterative algorithms}, emphasising that the
solution is constructed through a series of iterations where the
output of an iteration may become (part of) the input for the
subsequent step. It is not unusual that at time $t$, $\va^{(t)}$
represents the complete \textit{state} of algorithm $A$, meaning that
no additional information beyond $\va^{(t)}$ is required to compute
the value at step $t+1$. In this scenario, the hint $\va^{(t)}$
becomes the new input for the computation of
$\va^{(t+1)} = A_t(\va^{(t)})$. As we will explore in
\autoref{sec:graph-algos}, this concept holds true for a considerable
number of important algorithms.

\subsection{Algorithmic guarantees}\label{sec:algo-guarantees}
In the preceding section, we explored the notion that algorithms are
fundamentally a series of \textit{extremely specific}
instructions. When executed in the correct sequence, these
instructions invariably produce a predictable outcome. Indeed, the
main criterion for labeling a sequence of instructions as an algorithm
is that they must be unambiguous, leaving no possibility for
misinterpretation. Specifically, any given instruction $I$, when
applied to a well-defined state $S_t$, will result in a well-defined
subsequent state $S_{t+1}$. In simpler terms, given our knowledge of
the characteristic of a state $S_t$ and the \textit{effect} that a
specific instruction will have when applied to it, we can invariably
\textit{predict} the subsequent state $S_{t+1}$, for any possible
instantiation of the state $S_t$.

This intuition is at the core of algorithms, and it is referred to as
\textit{correctness} of algorithms. In essence, an algorithm is
provably correct when, for any input that meets a set of specified
\textit{preconditions}, it invariably produces an output that meets a
set of specified \textit{postconditions} (i.e., ``correct'' output).
This implies that, when analysing the correctness of a specific
algorithm $A$, it is crucial to define precisely which states $S_t$
can be processed by the instructions of $A$. This reasoning applies
all the way back to $S_0$, effectively imposing constraints on the
\texttt{inputs} and thus establishing the acceptable inputs values
of $A$ (i.e, preconditions).

Given these two ``ingredients'' (i.e., well-definess of the states and
specificity of the instructions), an algorithm is ensured to operate
within a clearly defined space, and its dynamics can be analysed and
proven correct. This is, perhaps, the most important feature of
algorithmics in general, as it enables us to confirm with certainty
that an algorithm will consistently produce the correct output for a
given set of inputs.

\subsection{Algorithms as discrete functions}
\label{sec:algos-discrete-functions}
In the domain of deep learning and neural networks, the predominant
use case is approximating continuous functions. Indeed, the
\textit{Universal Approximation Theorem} \citep{hornik1989multilayer}
for neural networks essentially states that feedforward networks with
(at least) a single hidden layer containing a finite number of neurons
can approximate \textit{any} continuous functions on compact subsets
of $\sR^n$, under mild assumptions on the activation functions. This
theorem however, does not extend to the learnability of discrete
functions. Previously, we mentioned that algorithms can be thought of
as being similar to mathematical functions. If analysed under this
lens, however, the \textit{modus operandi} of algorithms makes them
resemble discrete functions rather than continuous functions, since
algorithms involve discrete operations and decisions.

Take, for instance, the algorithm of searching an element within a
list of integers. This algorithm can be represented as a function
$A: \sL \times \sN \rightarrow \{0, 1\}$ (with $\sL$ being the space
of all possible lists) such that:
\begin{equation}
A(L, x) =
\begin{cases} 1 & \text{if } x \in L \\ 0 & \text{otherwise}
\end{cases}.
\end{equation}
Similarly, consider the operation of pushing an element $e$ onto a
stack or queue $Q$:
\begin{equation}
  \text{push}(Q, e) = Q' = Q \cup \set{e}.
\end{equation}
It is evident from the aforementioned examples that these algorithms
can indeed be framed as discrete functions and, as such, do not admit
differentiability across their entire domain. This poses a challenge
when we wish to integrate such functions into neural network
architectures, which inherently rely on the notion of gradients and
differentiability to learn and adapt. This is a problem not only when
considering end-to-end learnability of algorithms, but also when we
plan to integrate neural networks with external memory modules, such
as queues. Push and pop operations to and from these types of memories
do not provide the necessary gradients to enable gradient-based
optimisation of neural networks.

\subsection{Graph algorithms}\label{sec:graph-algos}
Here, we present and discuss a selection of relevant \textit{graph
  algorithms} (i.e., algorithms designed to operate on graphs). We
explore in detail how these algorithms play an importat role in
solving common combinatorial problems, such as shortest path, max-flow
and the TSP. We anticipate that many of the algorithms discussed
herein form the backbone of many of the research investigations
undertaken and presented in this dissertation.

\paragraph{Bellman-Ford algorithm.}
The Bellman-Ford algorithm \citep{bellman1958routing} is a computational
procedure designed to solve the shortest path problem (see
\autoref{prob:shortest-path}) in weighted graphs
$g=\tuple{\gV, \gE, \mW}$. A pseudocode for Bellman-Ford is provided
in \autoref{alg:bf}.

\input{content/algs/bf.tex}

If analysed under the notation presented in
\autoref{sec:algo-notation}, Bellman-Ford can be decomposed in \inp{},
\hints{} and \out{}.  Precisely, $\inp{} = \set{\mA, \mW, s,
\vd^{(0)}}$, where $\mA$ and $\mW$ are, respectively, the adjacency
matrix and weight matrix of the graph, $s$ is the source node and
$\vd^{(0)}$ is the initial estimates of the shortest paths from $s$ to
all other nodes in the graph. The algorithm begins by initialising
distances to the source node as 0 and to all other nodes as $+\infty$.
It then performs a sequence of intermediate transformations of the
inputs, called \textit{relaxation} steps, in order to transform the
initial distance vector $\vd^{(0)}$ to the final distance vector
$\vd$:
\begin{equation*}
   \vd^{(0)} \overset{A_0}{\longrightarrow} \vd^{(1)}
   \overset{A_1}{\longrightarrow} \dots
   \overset{A_{|\gV|-2}}{\longrightarrow} \vd^{(|\gV|-1)}
   \overset{A_{|\gV|-1}}{\longrightarrow} \vd.
\end{equation*}
Here, the set of intermediate distance vectors $\vd^{(t)}$ represents
the set of \hints{} of the algorithm (i.e., $\hints{} =
\set{\vd^{(1)},\dots,\vd^{(|\gV|-1)}}$). Furthermore, at each step
$t$, the algorithm iteratively relaxes all the edges, comparing the
current distance estimate $d[u]^{(t)}$ with the sum of the distance
estimate of a neighbouring node $d[v]^{(t)}$ and the weight of the
edge connecting them, $w_{uv}$, and updating it if a shorter path is
found. Mathematically, the relaxation step for node $u$ can be
formulated as:
\begin{equation}\label{eq:bf-relaxation}
  d^{(t+1)}[u] = \min \set{d^{(t)}[u],
    \min \set{d^{(t)}[v] + w_{uv} \mid v \in \gN(u)}}.
\end{equation}
The final \out{} is thus the final vector of distances $\vd^{|\gV|} =
\vd$, alongside the vector \texttt{pred}, where $\text{pred[u]}$
accounts for the \textit{predecessor} of node $u$ in the final
shortest path $s \leadsto u$.

\paragraph{Dijkstra's algorithm.}
The Dijkstra's algorithm \citep{dijkstra1959note} is perhaps the most
popular shortest path algorithm. Essentially, it solves the exact same
problem as Bellman-Ford. As such, the sets of \inp{}, \hints{} and
\out{} are unchanged compared to those of Bellman-Ford. Pseudo-code
is provided in \autoref{alg:dijkstra}.

\input{content/algs/dijkstra.tex}

Notably, Dijkstra improves on Bellman-Ford by utilising a priority
queue (line~\ref{dij:line:queue}) for selecting the next node to
evaluate during shortest path computation.

\paragraph{Breadth First Search.}
The Breadth First Search (BFS) \citep{cormen2009introduction} is an
algorithm for solving the reachability problem
(\autoref{prob:reachability}) in graphs $g=\tuple{\gV, \gE, \dots}$.
We can break down BFS in \inp{}, \hints{} and \out{}, similarly to
Bellman-Ford. In this case, $\inp{} = \set{\mA, s}$, where $\mA$ and
$s$ are analogous to Bellman-Ford (i.e., adjacency matrix and source
node). The $\hints{}=\set{\vr^{(1)},\dots,\vr^{(|\gV|-1)}}$ represent
all intemerdiate steps, where $\vr_v^{(t)} \in \set{0, 1}$ indicates
whether node $v$ is reachable from $s$ using at most $t$ edges.
Final output can be considered to be the vector $\vr^{|(\gV|)}$.
A pseudocode for the BFS procedure is provided in \autoref{alg:bfs}.

\input{content/algs/bfs.tex}

\paragraph{Prim's algorithm.}
The Prim's algorithm \citep{cormen2009introduction} computes the
MST (see \autoref{prob:mst}) of weighted graphs. Its pseudocode
is reported in \autoref{alg:prim}.

\input{content/algs/prim.tex}

As can be seen by comparison with \autoref{alg:dijkstra}, the Prim's
algorithm and Dijkstra bear a noticeable resemblance, as they both
employ the use a greedy strategy for updating the paths and a priority
queue. Decomposition to \inp{}, \hints{}, \out{} is then derived
similarly to Dijkstra. Specifically, $\inp{}=\set{\mA}$,
$\hints{}=\set{\text{key}^{(1)}, \text{pred}^{(1)}, \dots,
  \text{key}^{(n)}, \text{pred}^{(n)}}$ and
$\out{}=\set{\text{key}, \text{pred}}$.

\paragraph{Ford-Fulkerson.}
The Ford-Fulkerson algorithm \citep{ford1956maximal} is used to
compute the maximum flow in a capacitated graph (see
\autoref{prob:maxflow}). We provide the pseudocode for Ford-Fulkerson
in \autoref{alg:ff}.

\input{content/algs/ff.tex}

Differently from the previous algorithms, Ford-Fulkerson is composed
of two distinct sub-routines: (i) finding an \textit{augmenting path}
from $s$ to $t$ (i.e., a path within a residual graph\footnote{A
  residual graph is a representation of the remaining capacity of
  edges in a flow network after some flow has been pushed through.}
that allows for increased flow through the network); (ii) update the
flow assignment $\vf^{(t)}$. In \autoref{alg:ff}, sub-routine (i) is
encoded at line 4, whereas (ii) spans line 5--8.  Furthermore,
sub-routine (i) can be implemented as the Bellman-Ford algorithm.

At high level, the algorithm looks for an augmenting path from the
source $s$ to the target $t$ in the residual graph $g_f$, through,
say, the Bellman-Ford algorithm. The flow is then augmented along this
path by the bottleneck capacity, the minimum remaining capacity of the
edges in the path.  The iterative process continues until no
augmenting paths can be found in the residual graph. The final
$\vf$ represents the maximum flow from source to target.

By following our notation, thus,
$\inp{} = \set{\mA, \mC, s, t, \vf^{(0)}}$ plus all the inputs that
are inherited from the Bellman-Ford procedure. Similarly, \hints{}
also contains all the previously defined Bellman-Ford hints, with the
addition of the sequence of flow transformations
$\set{\vf^{(1)},\dots,\vf^{(T-1)}}$.
Finally, $\out{} = \set{\vf}$.

\paragraph{Christofides' algorithm.}
The Christofides algorithm, whose pseudo-code is presented in
\autoref{alg:christofides}, is the first example of a polynomial-time
algorithm for finding approximate solution to a NP-H problem, namely
the TSP (\autoref{prob:tsp}).

\input{content/algs/christofides.tex}

It operates by employing a series of algorithmic primitives, combining
them in a unique way to construct a feasible tour:

\begin{enumerate}

\item \textbf{Minimum Spanning Tree}: first, the algorithm finds a
  minimum spanning tree of the given graph.

\item \textbf{Minimum Weight Perfect Matching}:
  ensures that all nodes in the MST have an even degree by executing
  a minimum weight perfect matching on all nodes with an odd degree.
  This is a crucial step for forming Eulerian cycles.

\item \textbf{Eulerian Circuit}: forms an Eulerian cycle by
  combining the previous two steps.

\item \textbf{Hamiltonian Circuit}: an Hamiltonian cycle is derived
  from the Eulerian path by skipping visited nodes.
\end{enumerate}
An interesting property of the Christofides algorithm lies in its
guarantee -- it assures a solution within a $\frac{3}{2}$ factor from
the optimal TSP tour. Notably, the Christofides algorithm is not the
only example of approximation heuristics that leverage strong
algorithmic primitives. In fact, the Clarke-Wright savings
algorithm \citep{doyuran2011a} computes solution for VRP instances
by leveraging lists and sorting algorithms.

\paragraph{Gon algorithm.}
Similarly to the Christofides' algorithm, the Gon algorithm
\citep{gonzalez1985clustering, dyer1985simple} is a 2-approximation
strategy for the VKC problem (\autoref{prob:vkc}).  The algorithm
proceeds to find 2-approximate solutions by arbitrarily choosing a
``center'' node $v_1$ to add to the center set $C$, and then
iteratively selecting the additional centers $v_i$. The new $v_i$
are chosen such that $v_i$ is the farthest away from the already
selected centers in $C$. For consistency with other algorithms,
we report Gon's pseudocode in \autoref{alg:gon-algorithm}.

\input{content/algs/gon.tex}

\paragraph{Graph isomorphism.}
In \autoref{sec:graph-theory}, we introduced the notion of
\textit{graph isomorphism}. When observed in the context of decision
problems, this concept holds immense theoretical significance. Indeed,
determining whether two graphs are isomorphic is recognized as an NP
problem, but its classification as either a P problem or an
NP-complete problem remains an open question
\citep{garey1979computers}. Interestingly, it is also a potential
candidate for the class of NP-intermediate problems (NPI), a class
that would only exist if $\text{P} \neq \text{NP}$. Consequently,
finding a solution to the graph isomorphism problem would be a crucial
step in answering the long-debated $\text{P} \eqqst \text{NP}$
question.

\input{content/algs/wl.tex}

Since, as discussed in \autoref{sec:complexity-classes}, the debate is
still unresolved, we have to rely to approximation strategies for the
graph isomorphism problem.  In this regard, the most popular algorithm
is the \textit{Weisfeiler-Lehman test} \citep{lehman1968reduction}. A
pseudo-code for the WL test algorithm is given in
Algorithm~\ref{alg:wl-test}. In essence, the WL test is a graph
coloring algorithm, which assigns an initial color $c$ to each node of
the graph and iteratively updates this coloring based on the
distribution of colors in the neighbouring nodes, for each node. The
iterative update can be formulated as:
\begin{equation}\label{eq:wl}
  c^{t+1}(v)=h\left(c^t(v),
    \oplus\left(\left\{c^t(u) \mid u \in \gN(v)\right\}\right)\right),
\end{equation}
where, $\oplus$ is a function that transforms a multiset of colors
into a sequence of colors and $h$ is a hash function that generates a
new color based on the current color of a vertex and the colors of its
neighbouring nodes. The choice of $h$ and $\oplus$ depends on
practical implementation. Usually, $\oplus$ is a function that sorts
the colors in a multiset in a consistent way (typically by imposing a
total order on the colors), and $h$ is chosen to meet the typical
requirements of hash functions (i.e., avoiding hash collisions). Such
node-coloring scheme, however, has some limitations. In fact,
the Weisfeiler-Lehman test may erroneously classify two non-isomorphic 
graphs as isomorphic.
However, the converse is accurate. In other words, if the WL test
determines that two graphs are not isomorphic, then it is certain that
the two graphs are not isomorphic.

In order to distinguish non-isomorphic graphs more effectively, we can
create a hierarchy of higher-order WL tests, which are commonly
referred to as $k$-WL, with $k$ being the dimensionality of the
coloring node vector. As equation \eqref{eq:wl} only uses a single
color for each node, it is also known as $1$-WL. Finally, we
anticipate that the WL-test is tightly connected to machine learning
models and neural networks that operate on graphs.  We will explore
this connection more deeply in \autoref{sec:gn}.

\subsection{Dynamic programming}
\label{sec:dp}
\textit{Dynamic Programming} (DP) is a powerful computational paradigm
\citep{bellman1966dynamic} that targets complex problems by breaking
them down into simpler, overlapping subproblems, solving each
subproblem just once, and building up the solution to the original
problem from these solutions. DP finds applications in many areas of
computer science, such as operations research
\citep{bellman2015applied}, bioinformatics
\citep{giegerich2000systematic}, and artificial intelligence
\citep{bertsekas1995neuro, lewis2009reinforcement}. Formally, a DP
algorithm can be framed as:
\begin{equation}\label{eq:dp}
  S_i^{(k+1)} = U(\{ S_j^{(k)} \mid j = 1, \dots, n \}),
\end{equation}
where $S_i^{(k)}$ represents the solution of the $i$-th subproblem at
iteration $k$, and $U$ updates the solution at step $k+1$ based on all
solutions at the previous step $k$. Usually, practical applications
assume the existence of some metric $m$ that is set to be either
minimised or maximised, in which case $U$ is chosen to be either
$\min$ or $\max$.

In general, DP assumes that the optimal solution to a problem can be
constructed from optimal solutions of its subproblems. This assumption
is usually referred to as \textit{principle of optimality}
\citep{henig1985principle}. A popular example of a DP algorithm is the
Bellman-Ford algorithm (presented in the previous section) in the
sense that it solves the shortest path problem by following the
principle of optimality. Precisely, the shortest path problem
$u \leadsto v$ through exactly $k$ edges depends on the shortest path
problem $u \leadsto v$ through exactly $k-1$ edges. In simpler terms,
an optimal solution to $d^{(k)}[v]$ (i.e., vector of distances in
Bellman-Ford) requires an optimal solution to
$d^{(k-1)}[v]$. Additionally, this property can be derived from
equation \eqref{eq:bf-relaxation}. Indeed, one can note that
\eqref{eq:bf-relaxation} is a special case of \eqref{eq:dp}, where we
replace $U$ with $\min$ and the solutions $S_j^{(k)}$ with
$d^{(k)}[j] + w_{ij}$ (i.e., $w_{ii}$ is assumed zero).

\paragraph{Tropical Algebra.}
In the context of dynamic programming, it is often noteworthy to
discuss a closely related field, which is that of Tropical Algebra
\citep{brugalle2000brief}. Tropical algebra provides a mathematical
framework that performs arithmetic operations on alternative semirings
(usually referred to as \textit{tropical semiring}) rather than the
``conventional'' one defined on the real numbers (i.e.,
$\tuple{\sR, +, \cdot, 0, 1}$). Formally, a semiring is an algebraic
structure of the form
$\tuple{\sA, \oplus, \odot, id_\oplus, id_\odot}$ where:
\begin{enumerate}
\item $\sA$ is a set, such as the set of real numbers $\sR$ or natural
  numbers $\sN$.
\item $\oplus$ is a commutative \textit{binary} operation, called
  \textit{addition}.
\item $\odot$ is \textit{binary} operation, called
  \textit{multiplication}.
\item $id_\oplus$ is the \textit{identity} element of addition, such
  that
  $\forall a \in \sA \,:\, (a \oplus id_\oplus) = (id_\oplus \oplus a)
  = a$.  Additionally, $id_\oplus$ acts as an \textit{absorbing}
  element with respect to multiplication (i.e.,
  $\forall a \in \sA \,:\, (a \odot id_\oplus) = id_\oplus$.
\item $id_\odot$ is the \textit{identity} element of multiplication, such that
  $\forall a \in \sA \,:\, (a \odot id_\odot) = a$.
\end{enumerate}
Thus, tropical semirings are actually a \textit{collections} of
semirings (see \citep{gunawardena1998correspondence} for a
comprehensive overview), where we ``swap'' the usual arithmetic
operations of addition $\oplus=+$ and multiplication $\odot=\cdot$
with different operators, typically chosen to reflect operations
common to DP algorithms. Specifically, one instance of a tropical
semiring is the \textit{min-plus} semiring, defined as
$\tuple{\sT = \sR \cup \set{+\infty}, \oplus, \odot, +\infty, 0}$,
where $\oplus$ and $\odot$ are defined as follows:
\begin{align*}
  a \oplus b &= \min(a, b)\\
  a \odot b &= a + b,
\end{align*}
for each $a$ and $b$ in $\sT$, where $+\infty$ and $0$ are the
identity elements of $\oplus=\min$ and $\odot=+$, respectively. Such
formalism provides a compact way to express many DP algorithms.  To
conclude the section, let us provide practical examples of the above
claim. Consider two popular algorithms: (i) the Breadth First Search
(BFS) algorithm; (ii) and the Bellman-Ford algorithm. These algorithms
can be easily formulated within the min-plus semiring. First, note that
the generalised matrix-matrix multiplication in the min-plus semiring
can be expressed as:
\begin{equation}\label{eq:tropical-mm-mul}
  \left[{\mA} \odot \mB\right]_{ij} = \bigoplus_{k}
  \mA_{ik} \odot \mB_{kj} = \min_k \mA_{ik} + \mB_{kj}.
\end{equation}
Similarly, we can also define the power of matrix as:
\begin{equation}\label{eq:power-mm-mul}
  \mA^{\odot k} = \overbrace{\mW \odot \mW \odot \cdots \odot \mW}^{k
  \text{ times}}.
\end{equation}
Consequently, to approximate BFS, assume we are given a graph
$g=\tuple{\gV, \gE}$. Now, consider the matrix $\mA_{\sT}$ as its
adjacency matrix, where all zeros are replaced with $+\infty$ and all
ones are replaced with $0$. Consider also the vector $\chi_s$
having $0$ on the $s$-th coordinate and $+\infty$ elsewhere. Thus, BFS
can be seamlessly represented in the tropical semiring as:
\begin{equation}\label{eq:tropical-bfs}
  \vr^{(|\gV|)} = \chi_s \odot \mA_\sT^{\odot |\gV|}.
\end{equation}
Here, $\vr^{(|\gV|)}$ is a vector such that $\vr^{(|\gV|)} = 0$ if $v$
is reachable from $s$ and $\vr^{(|\gV|)} = +\infty$
otherwise. Furthermore, earlier steps $\vr^{(t)}$ with $t < |\gV|$
will correctly represent intermediate steps of the BFS algorithm.

Similarly, we can also compute shortest paths from a given source $s$
to any other nodes in the graph. Suppose we are given a weighted graph
$g = \tuple{\gV, \gE, \gW}$, we can define the matrix $\mW_\sT$ to be
the weighted ``tropical'' adjacency matrix of $g$, where all zeros in
$\gW$ are replaced with $+\infty$, similarly to what we have done for
BFS. Precisely, $\mW$ is described as follows:
\begin{equation}\label{eq:tropical-distance-matrix}
  \mW_{uv} = \begin{cases}
               {d}_{uv} \in \sT & (u, v) \in \gE, \\
               0 & u = v, \\
               \infty & \text{otherwise}.
             \end{cases}
\end{equation}
Equivalently to \eqref{eq:tropical-bfs}, then, the following
equation will compute the shortest paths starting from a node $s$:
\begin{equation}\label{eq:tropical-bf}
  \vd^{(|\gV|)} = \chi_s \odot \mW_\sT^{\odot |\gV|}.
\end{equation}
It is easy to verify that \eqref{eq:tropical-mm-mul} applied to
the tropical weight matrix is equivalent to the relaxation step
as in \eqref{eq:bf-relaxation}.

\section{Graph Networks}\label{sec:gn}
Graph Networks (GNs) \citep{bacciu2020gentle} are a particular class
of deep learning models that are specialised in handling, processing
and learning from structured data in the form of graphs. Unlike
traditional neural networks that operate on $n$-dimensional euclidean
spaces, i.e. $\sR^n$, or on structured data like images and sequences,
GNs can be intended to work in geometrical spaces where each data
point is a graph $g=\tuple{\gV,\gE,\dots}$, defined as in
\autoref{sec:graph-theory}. The central idea behind GNs is to
effectively build vectorial representations of each node by
aggregating information from neighbours, allowing the model to capture
global context and dependencies in the graph through multiple
iterations of this mechanism, i.e. \textit{context propagation}. These
node representations need to contain sufficient information to infer
desired properties on the input graphs and generalise to unseen
instances. Analogous to neural networks operating on euclidean data,
then, the objective is to learn \textit{representations} that are
meaningful for regression and classification tasks. Moreover, as we
will explore in the following sections, node representations can be
combined to obtain a single vector that serves as a representation of
the entire graph. The way we combine such pieces of information is
extremely important in order to achieve specific desirata of GNs, such
as the ability to process \textit{any} graph, regardless of its size
(i.e., number of nodes) and topology.  In this regard, these
attributes have played a pivotal role in the dissemination of GNs
across numerous graph-based learning tasks, supported by a significant
body of literature attesting to the efficacy of these models. Their
applications span diverse domains, including computational chemistry
\citep{bacciu2023deep}, social networks \citep{min2021stgsn}, and more
recently, algorithms and combinatorial optimisation
\citep{bello2017neural, velickovic2019neural, kool2019attention}.

Below, we discuss the foundational concepts underpinning the
development and popularisation of graph networks, highlighting key
factors that helped surpass early stage limitations, and providing
initial hints into how their computation is effectively suited for
learning algorithms.

\subsection{A brief historical tour of GNs} \label{sec:gn-history}
Despite becoming popular only in the last decade, the origins of graph
networks, and more generally of learning for graphs, can be traced
back to the 1990s. The first efforts attempted to learn
representations mainly for acyclic structures, such as trees and DAGs.
Notably, \cite{sperduti1997supervised}, briefly followed by
\cite{frasconi1998a} introduced the very first learnable IO-isomorph
transductions $\tilde{\gT}$ (see \autoref{def:io-iso}). These learnt
transductions were heavily inspired by recursive neurons, essentially
generalising on them to incorporate recursion through structure and
trained via back-propagation through time
\citep{werbos1990backpropagation}.  Although these primal versions of
``graph networks'' look different to their contemporary counterparts,
they demonstrated a groundbreaking capability of neural networks to
process and learn to infer properties in structured data for the first
time. Since then, these classes of models have been continuously
studied and are still relevant nowadays
\citep{bianchini2001processing,bianchini2004recursive,bacciu2013an}.

However, the aforementioned works typically presented a few, yet
strong, limitations. Firstly, the architectural decisions made in the
design of such models directly influence on the graph topology
accepted by these types of networks. In essence, the functions
executed by these models rely on \textit{a priori} knowledge, such as
the maximum number of ``children'' that any node in the graph may have
or the existence of a \textit{supersource} node, from which all the
other nodes can be reached. Secondly, and more importantly, these
models had difficulties in handling possibly cyclic graphs, due to
their \textit{recursive} nature. To elaborate, when dealing with an
arbitrary \textit{cyclic} graph $g$ and a specific node $v$, computing
the representation of $v$ recursively becomes unfeasible due to its
dependence on itself within the cycle.

However, a promising alternative lies in tracking the evolution of
$v$'s representation, enabling the computation of its updated
representation based on its past representation at a specific point in
time. This is the breakthrough that was presented, concurrently, by
\cite{scarselli2009the} and \cite{micheli2009neural} in the late
2000s. This idea essentially ``breaks'' cycles in the graph structure
via a simple iterative \textbf{message-passing} scheme, enabling the
computation of node representations within cycles and consequently the
processing of cyclic graphs.  Although the two algorithms share the
same view, the authors implemented it in two different
``flavours''. \cite{scarselli2009the} relied on a recurrent paradigm,
constructing the graph network as a contraction mapping and
iteratively computing node representations until a fixed point was
attained. \cite{micheli2009neural}, instead, proposed a feedforward
architecture which implements the message-passing scheme via multiple
layers, akin to the standard modern GNs.

Building upon these pioneering works, numerous subsequent studies
emerged, progressively integrating a wealth of concepts from both
graph theory and neural networks into the realm of graph networks.
A few noteworthy citations will be discussed more in depth in
\autoref{sec:it-vs-ff-mp}. For a comprehensive overview of graph
networks, refer to \citep{velickovic2023everything}.

\subsection{Desiderata of Graph Networks}\label{sec:gn-desiderata}
\begin{figure}
  \centering
  \input{gfx/background/graph_convolution}
  \caption{This graphic showcases how messages are computed and passed
    in graph networks. At step $\ell = 0$, we focus on node $v$. To
    update its representation $\vh^{(0)}_v$, we compute messages
    $m_x = \psi(\vh^{(0)}_{v},\vh^{(0)}_{x}, \vh^{(0)}_{xv})$ and
    $m_z = \psi(\vh^{(0)}_{v},\vh^{(0)}_{z}, \vh^{(0)}_{zv})$ from its
    neighbours $x$ and $z$. $\vh^{1}_v$ is then obtained through
    application of \eqref{eq:gn-mp}. At step $\ell=1$, $u$ is updated
    based on its only neighbour $v$, whose has already gathered
    information from $x$ and $z$ at $\ell=0$.  As a result, at a
    subsequent step $\ell=2$, $\vh^{(2)}$ will be conditioned not only
    on $\vh^{(1)}_v$ but also on $\vh^{(0)}_x$ and $\vh^{(0)}_z$,
    essentially expanding the range of information $u$ has access to.}
  \label{fig:gn-mp}
\end{figure}
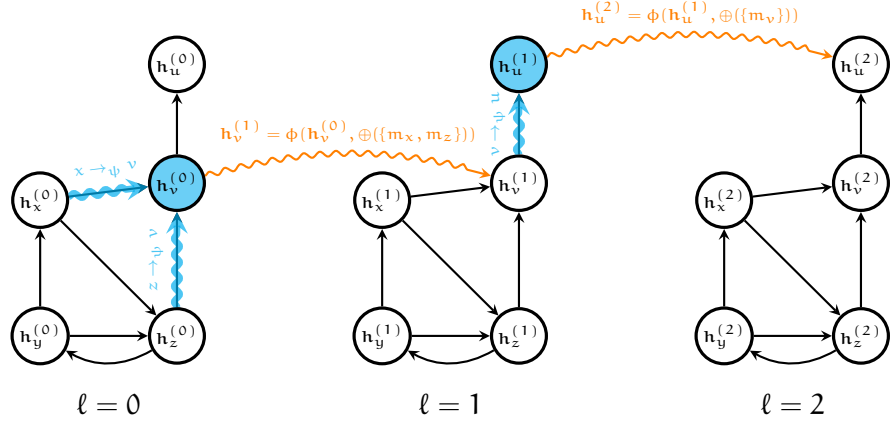
Above, we hinted at the fact that graph networks should provide a
versatile framework for processing graph-structured information. In
this section, we articulate the key desiderata of graph networks,
outlining the precise equations and algorithms they adopt to fulfill
these requirements:
\begin{itemize}

\item \textbf{Topological invariance}: graph networks need to process
  graphs $g = \tuple{\gV, \gE, \gX_\gV, \gX_\gE}$ without imposing
  constraints on: (i) the cardinality of $\gV$ and $\gE$; (ii) the
  graph cyclic or acyclic nature.

\item \textbf{Stationarity}: the representation of a node $v$ is
  independent on $v$'s specific identity or its position within the
  graph. In simpler terms, the learnt representation of a node should
  depend solely on its input features $\vx_v \in \gX_\gV$ and its
  local neighbourhood $\gN(v)$.

\item \textbf{Effective discrimination of graphs}: graph networks
  should exhibit enough expressiveness to differentiate between
  distinct graphs and to identify isomorphic graphs as effectively as
  possible.

\end{itemize}
In order to embody these properties, graph networks apply a ``learnt''
transformation to the input features of the nodes within the graph,
augmented with a message-passing algorithm to leverage the topological
structure of the given graph effectively. Formally, in their most
general configuration, the equation encapsulating these concepts can
be distilled as (see also \autoref{fig:gn-mp}):
\begin{equation}\label{eq:gn-mp}
  \vh^{(\ell+1)}_v = \phi \left(
    \vh^{(\ell)}_v,
    \oplus \left(\left\{
      \psi(\vh^{(\ell)}_v,
      \vh^{(\ell)}_u,
      \vh^{(\ell)}_{uv}) \mid u \in \gN(v) \right\}\right)
\right).
\end{equation}
Specifically, the above equation describes the iterative evolution of
a node representation, denoted by $\vh_v$, over ``timesteps'' denoted
by $\ell$ (with $\vh^0_v$ representing the initial vector of features
$\vx_v \in \gX_\gV$). Note that by $\vh_{uv}$ we denote the
representation of the edge $(u, v)$. This representation may be either
a specific value (i.e., edge attribute), in which case it 
remains fixed for all ``steps'' (i.e., 
$\vh^{(\ell)}_{uv} = \vx_{uv} \in \gX_\gE$, for all $\ell$), or updated
in a similar fashion to \eqref{eq:gn-mp}. To simplify the 
upcoming explanation, we assume the former for the reminder of this section.
Typically, the above equation is
characterised as a \textbf{message-passing} algorithm, composed of two
distinct operations:
\begin{enumerate}
  \item \textbf{Message computation}: a message is
computed for each node in the graph and is disseminated to all
neighbouring nodes. In \eqref{eq:gn-mp}, $\psi(\cdot)$ represents a
learnable transformation and its output is the message $u
\rightarrow_\psi v$ for every pair of neighbouring nodes $(u, v)$.

  \item \textbf{State update}: incoming messages are gathered via
an \textit{aggregation function}, symbolised by $\oplus$.
Consequently, the aggregated output is leveraged to calculate the
\textit{new} representation for node $v$ at the subsequent ``time''
$\ell+1$ using a second learnable transformation, denoted by $\phi$.
\end{enumerate}
The whole procedure is run in parallel for all nodes in the graph.
We now discuss how the above equation satisfies all the desiderata
for graph networks.

\paragraph{Local processing of node information.} Looking at
\eqref{eq:gn-mp} carefully, we note that the learning of a node
representation is defined from the perspective of a \textbf{single}
node $v$. This is commonly referred to as \textbf{local processing} of
a graph, meaning that the computation is centered on information that
are directly connected to each node, rather than considering all nodes
and edges simultaneously (essentially attempting to extract features of
the graph as-a-whole, e.g. spectral processing).  Local processing
guarantees flexibility, as equations for graph networks do not assume
knowledge of the number of nodes and edges of the graph, thereby enabling
processing of graphs with any number of nodes and edges (i.e.,
\textbf{size invariance}). It is important to note that, through
multiple iterations of the message-passing scheme presented in
\eqref{eq:gn-mp}, this ``local scope'' can gradually extend and gather
information from more distant parts of the graph, effectively capturing
more \textit{global} information while mantaining the primary focus on
each node's local context (refer to \autoref{fig:gn-mp}).

\paragraph{Processing cyclic graphs.}
Above, we successfully showed how \eqref{eq:gn-mp} makes a graph
network size-invariant. However, to fully adhere to the topological
invariance prerequisite, it is necessary to discuss how the said
equation adeptly processes cyclic graphs. \autoref{fig:gn-mp} shows
that \eqref{eq:gn-mp} is in fact an iterative procedure. By unrolling
the computation of \eqref{eq:gn-mp}, we observe that each update of
the node representations at step $\ell$ is determined in terms of the
node representations at step $\ell-1$. As a result, mutual
dependencies among node representations involved in a cycle can not
exist at the same step $\ell$. In other words, we ``break'' the cycles
with an iterative approximation scheme. We anticipate that there are
different ways of practically implement the ``step'' $\ell$, which
will be examined more thoroughly in \autoref{sec:it-vs-ff-mp}.
However, the iterative message passing described herein effectively
enables the processing of cyclic graphs, consequently accomplishing
the \textbf{topological invariance} property.

\paragraph{Stationarity of node encodings.}
The second desiderata of graph networks precisely states that node
representations must not be dependent on the nodes' specific
identities. This is critical for capturing the underlying structural
relation of nodes within the graph, regardless of their specific
identities. The \textbf{stationarity} property of GNs is derived
primarily from two aspects. First, the aggregation function
$\oplus$ used in \eqref{eq:gn-mp} need to be independent on the
considered ordering of neighbouring nodes. If this is not the case
(i.e., the output of $\oplus$ changes upon reordering of its
inputs), stationarity would be violated. In fact, nodes in a graph do
not usually come with a known order as in \textit{sequences} or
\textit{trees}.  Hence, if different permutations of neighbours yield
different evolutions of a node representation $\vh_v$, the computation
of $\vh_v$ is non-stationary (i.e., the representation depends on the
specific labeling and ordering of the nodes). To mitigate this issue,
$\oplus$ is usually selected to be a \textbf{permutation-invariant}
functions, where the output remains the same for any possible
permutations of the inputs. Common permutation-invariant functions
employed in GNs include \textbf{sum}, \textbf{mean} and element-wise
\textbf{maximum} of neighbouring nodes' representations.
Second, the use of shared parameters across all nodes in GNs
contributes to stationarity. In \eqref{eq:gn-mp}, the same
transformation (i.e., the same weights in a neural network layer) is
applied equally to all nodes. This process ensures that the same
learnt encoding mechanism on every node within the graph.
Permutation-invariant aggregators and weight sharing together ensure
that the representation of a node is dependent solely on its features
and its local neighbourhood, regardless of its location within the graph
and its ordering and labeling, thus achieving the stationarity property.

\paragraph{Relation to the 1-WL test.} Here, the attentive reader may
have noticed the similarities between equation \eqref{eq:wl} and
\eqref{eq:gn-mp}. In fact, the node updating funtion in GNs is
analogous to the color updating function in the 1-WL test. Both
\eqref{eq:wl} and \eqref{eq:gn-mp} perform node-wise updates based on
local neighbourhood information, although there are important
divergences on the \textit{nature} of functions being used. While the
$1$-WL test relies on hash functions and a \textit{discrete} coloring
algorithm (i.e., the set of possible ``node features'' is countable),
GNs extend the procedure to utilise \textit{continuous} functions.
Despite these differences, the analogy between the $1$-WL test and GNs
has proven extremely useful in the design of graph networks. It is
formally proven that, indeed, graph networks can only distinguish
non-isomorphic graphs that can also be distinguished by the $1$-WL
test \citep{xu2019powerful}. This theorical result is, however,
conditioned on the type of permutation-invariant aggregator used in
\eqref{eq:gn-mp}. \cite{xu2019powerful} show that it is possible to
achieve the same expressivity as the $1$-WL test if we choose $\oplus$
to be a sum operation, whereas we achieve \textit{strictly} less
expressivity with mean and maximum. Hence, in general, the expressive
capacity of graph networks is bounded by the expressivity of the
$1$-WL test. From a machine learning perspectives, this means that two
non-isomorphic, but indistiguishable for the $1$-WL test, graphs can
not have two different ``learnt'' representations. Hence, the neural
model can not generate different predictions.  That is why classical
GNs are often referred to as $1$-WL graph networks. These limitations
have inspired the research of increasingly complex $k$-WL graph
networks, inspired by the hierarchy of $k$-WL tests
\citep{morris2019weisfeiler}. These models, however, are rarely used
in practice, as the increased capacity does not always reflect an
increase in model's performance.

\subsection{Iterative vs Feedforward Message Passing}
\label{sec:it-vs-ff-mp}
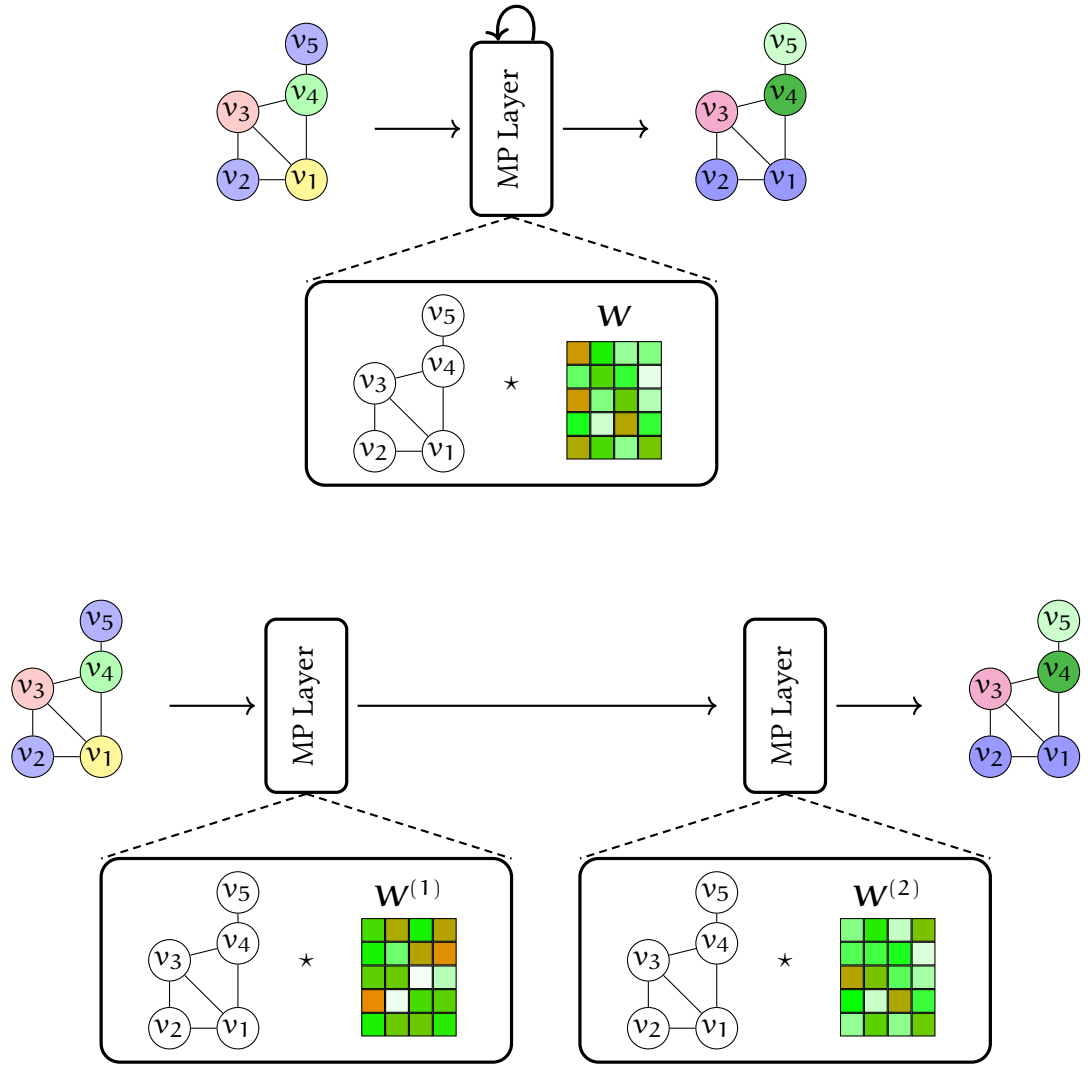
\begin{figure}
  \centering
  \input{gfx/background/iterative_gn}
  \caption{An \textbf{iterative} graph network (above) is compared to
    a two-layer \textbf{feedforward} graph network (below). Above,
    node representations are refined via iterative application of the
    same $\mW$, as symbolised by the recurrent connection in the
    message-passing layer (MP layer). Below, instead, node
    representations are updated first by application of $\mW^{(1)}$ in
    the first MP block and then updated again by the second MP block
    symbolised by $\mW^{(2)}$. In the picture, $\star$ represents
    application of \eqref{eq:gn-mp}.}
  \label{fig:it-vs-ff-mp}
\end{figure}
In the context of graph networks and deep learning for graphs, there
is no one single way of how information is spread across the
graph. Earlier, we identified at least two different ways of how GNs
exchange messages and perform aggregations among neighbours:
\textit{iterative} and \textit{feedforward}. Although different in
style, both paradigms can be studied under the unifying notation of
equation \eqref{eq:gn-mp}. For completeness, \cite{bacciu2020gentle}
also identify a third different category, which is that of
\textit{constructive} architectures. Although much work fall into this
category \citep{micheli2009neural, bacciu2018contextual}, these models
are essentially feedforward architectures where training is performed
layer-wise. Notwithstanding their importance in, for instance,
effectively countering common vanishing/exploding gradient issues, we
will focus on the difference between iterative and feedforward
architectures for the reminder of this section.

The major differences between \textit{iterative} and
\textit{feedforward} architectures lies in how equation
\eqref{eq:gn-mp} is istantiated and evaluated to obtain the evolutions
of the node representations. As showcased in
\autoref{fig:it-vs-ff-mp}, within the iterative paradigm, there is,
usually, one single parametrised layer of message-passing which is
called iteratively ``as many times as needed'' to spread and acquire
information across an increasingly larger portion of the graph.
Contrastingly, \textit{feedforward architectures} address graph
processing by means of several feedforward layers. Precisely,
\eqref{eq:gn-mp} formalises the learnt transformation employed by one
single layer of a feedforward architecture, which may contain
multiple. To understand precisely how message-passing is implemented
differently in iterative and feedforward architectures, consider a
simpler formulation of \eqref{eq:gn-mp}, described in terms of
matrix multiplications:
\begin{equation}
  \label{eq:gn-mp-matrix}
  \mH^{(\ell + 1)} = \sigma((\mA + \mI) \mH^{(\ell)}\mW^{(\ell)}).
\end{equation}
Here, we isolate the parameters $\mW$ from the previous equation
\eqref{eq:gn-mp}, which were implicitly included in $\psi$.  Hence,
$\mW^{(\ell)}$ is the matrix of parameters at GN layer/step $\ell$.
Similarly, $\mH^{(\ell)}$ contains all node representations for each
node $v$ within the graph at $\ell$, organised in matrix form as
follows:
\begin{equation*}
  \mH^{(\ell)} = \begin{bmatrix}
                   \vdots \\
                   \vh^{(\ell)}_v \\
                   \vdots\\
  \end{bmatrix}.
\end{equation*}
Additionally, $\mA$ is the adjacency matrix to which we add
self-loops through the identity matrix $\mI$, and $\sigma$ is
an activation function. From \eqref{eq:gn-mp}, we set
$\oplus = \sum$, which enables us to rewrite the equation
exploiting matrix products.

Now, it is easy to show the main difference between iterative and
feedforward architectures. Precisely, within the iterative paradigm,
the parameters $\mW^{(\ell)}$ are the same for any step $\ell$ of
\eqref{eq:gn-mp-matrix} (i.e.,
$\mW^{(0)} = \mW^{(1)} = \dots = \mW^{(\ell)} = \dots = \mW$).  This
is why message-passing herein is called ``iterative'', since we do not
have access to ``layers'' in the proper term, but rather we perform
iterations of the same equation.

Contrastingly, feedforward architectures typically have different
parametrisation for different $\ell$ (i.e.,
$\mW^{(\ell)} \neq \mW^{(\ell+1)}$), and we only perform a single
``iteration'' of \eqref{eq:gn-mp-matrix} for a given
$\mW^{(\ell)}$. In practice, within the feedforward paradigm we choose
the number of different parametrisations, or \textit{layers}, in
advance. \autoref{fig:it-vs-ff-mp}, indeed, shows a 2-layer GN.

Such differences have two consequences. First, iterative architectures
typically require \textit{less} trainable parameters compared to their
feedforward counterpart. The second fundamental difference is the
\textbf{context radius}. The context radius accounts for the largest
portion of the graph that influences a node representation $\vh_v$.
In fact, $k$ iterations of message-passing entails that each
representation $\vh_v$ will be influenced by all and only the nodes at
most distance $k$ from $v$, termed $k$-hop neighbours (see
\autoref{fig:gn-mp}).

For iterative GNs, this radius is linked to the number of iterative
message-passing ``calls'', which is typically unbounded. Indeed, we
can, in principle, perform an unlimited number of iterations. From a
practical standpoint, though, the context radius is limited by
vanishing and exploding gradient issues, particularly for long
sequence of iterations, as well as by well-known oversmoothing issues
\citep{chen2020measuring}. To counteract these issues, one can
typically perform a \textit{fixed} number of iterations or iterate
until a certain stopping condition is met \citep{tang2020towards}.  In
this respect, \textit{stability} of the node representations has
historically played an important role to determine the ``right time''
to stop iterations. Under this lens, GNs are often seen as dynamical
systems on which contractivity constraints are imposed
\citep{scarselli2009the, gallicchio2010graph}, in order to ensure
``convergence'' of the set of node representations.  Such convergence
is usually ensured under the hypothesis that $\forall u, v \in \gV$:
\begin{equation*}
  \norm{\vh^{(\ell+1)}_v - \vh^{(\ell+1)}_u} \leq
  c \cdot \norm{\vh^{(\ell)}_v - \vh^{(\ell)}_u}
\end{equation*}
with $ 0 \leq c < 1$, that is \eqref{eq:gn-mp-matrix} is a contractive
mapping.

For feedforward architecture, instead, the context radius is usually
equal to the depth of the architecture (i.e., number of layers), since
we only perform one step for each layer. For instance, in
\autoref{fig:gn-mp} every node in the feed-forward architecture is
conditioned at most on they 2-hop neighbours.

Such architetural differences make the two approaches suitable for
different kinds of tasks. Intuitevively, if we aim to learn functions
that can be expressed as a composition of different functions, we
should tend towards using feedforward architectures, as a different
parametrisation at different layers allows for simpler modelling of
the composition (e.g., different layers can indeed learn different
functions). However, iterative approaches seem to fit more effectivily
scenarios where the target graph function consist of repeteadly
performing the same operation on all nodes within the graph. Indeed,
iterative GNs are used in physical simulations
\citep{hamrick2018relational, pfaff2020learning} and neural execution
of algorithms \citep{velickovic2019neural}. While we postpone a more
thorough discussion of the latter to \autoref{ch:nar}, we conclude
this chapter by highlighting and discussing similarities shared by
traditional graph algorithms and graph networks, especially
considering their iterative versions. In this respect, many
traditional algorithms are designed to execute the same exact function
across all sections of their input. Consider, for instance, sorting
algorithms that continuously swap elements in a list until it is
sorted, or shortest path algorithms like Bellman-Ford that
consistently apply the same relaxation step (i.e., equation
\eqref{eq:bf-relaxation}) to every node in the graph. As we
anticipated, iterative GNs better fits modelling of such functions.
The reason is that repeteadly applying the same operation intuitevely
correspond to having a single parametrised weight map $\mW$ that is
called iteratively.  Indeed, consider a scenario in which Bellman-Ford
is the target learning function, even if we had different
parametrisations $\mW^{(\ell)}$, they would still have to learn the
same \eqref{eq:bf-relaxation} function for each $\ell$.  Hence, in
such cases, a single $\mW$ gives us more guarantees that the applied
learnt function is consistent across all steps.  Another issue to
evaluate is that the number of the iterations in algorithms is usually
dependent on the size of the input. Precisely, Bellman-Ford performs
equation \eqref{eq:bf-relaxation} $O(|\gV|)$ times. This makes the
feedforward paradigm, given their limited \textit{context radius},
inapplicable in such scenarios. In fact, given a feedforward
architecture with $k$ layers, such network is inherently unable to
generalise the Bellman-Ford dynamics for graphs with diameter
larger than $k$. Contrastingly, with the iterative paradigms we can
increase the number of iterations to match that of Bellman-Ford.

We conclude this chapter by stating that whenever we refer to GNs
in the reminder of this thesis, we always refer to GNs based on the
iterative paradigm, as they represent the most suitable architecture
for learning to execute classical algorithms.

%% file: gfx/background/undirected_graph.tex
\begin{tikzpicture}[-,>=stealth,auto, very thick]

  \tikzstyle{vertex}=[circle,fill=cyan!50!,draw,inner sep=1pt, minimum size=15pt]

  \node[vertex] (v1) at (0,-1.5) {$v_1$};
  \node[vertex] (v2) at (-1.5,-1.5) {$v_2$};
  \node[vertex] (v3) at (-1.5,0.5) {$v_3$};
  \node[vertex] (v4) at (0,0.75) {$v_4$};
  \node[vertex] (v5) at (0,2) {$v_5$};

  % Straight Edges
  \draw (v1) edge (v3);
  \draw (v1) edge (v4);
  \draw (v2) edge (v1);
  \draw (v2) edge (v3);
  \draw (v3) edge (v4);
  \draw (v4) edge (v5);

\end{tikzpicture}

%% file: gfx/background/directed_graph.tex
\begin{tikzpicture}[->,>=stealth,shorten >=1pt,auto, very thick]

  \tikzstyle{vertex}=[circle,fill=cyan!50!,draw,inner sep=1pt, minimum size=15pt]

  \node[vertex] (v1) at (0,-1.5) {$v_1$};
  \node[vertex] (v2) at (-1.5,-1.5) {$v_2$};
  \node[vertex] (v3) at (-1.5,0.5) {$v_3$};
  \node[vertex] (v4) at (0,0.75) {$v_4$};
  \node[vertex] (v5) at (0,2) {$v_5$};

  % Straight Edges
  \draw (v1) edge (v3);
  \draw (v1) edge (v4);
  \draw (v2) edge (v1);
  \draw (v2) edge (v3);
  \draw (v3) edge (v4);

  % Bent Edges
  \draw (v3) edge[bend right=20] (v2);
  \draw (v4) edge[bend left=30] (v1);
  \draw (v5) edge[bend right=30] (v4);
  \draw (v4) edge[bend right=30] (v5);

\end{tikzpicture}

%% file: gfx/background/dag.tex
\begin{tikzpicture}[->,>=stealth, shorten >=1pt, auto, very thick, scale=.5]

  \tikzstyle{vertex}=[circle,fill=cyan!50!,draw,inner sep=1pt, minimum size=7pt]

  \node[vertex] (v1) at (0,-1.5) {$v_5$};
  \node[vertex] (v2) at (-1.5,-1.5) {$v_4$};
  \node[vertex] (v3) at (-1.5,0.5) {$v_2$};
  \node[vertex] (v4) at (0,0.75) {$v_3$};
  \node[vertex] (v5) at (-0.75,2) {$v_1$};

  % Straight Edges
  \draw (v5) edge (v4);
  \draw (v5) edge (v3);
  \draw (v4) edge (v1);
  \draw (v4) edge (v2);
  \draw (v3) edge (v4);
  \draw (v3) edge (v2);
\end{tikzpicture}

%% file: gfx/background/tree.tex
\begin{tikzpicture}[->,>=stealth, shorten >=1pt, auto, very thick, scale=.5]

  \tikzstyle{vertex}=[circle,fill=cyan!50!,draw,inner sep=1pt, minimum size=7pt]

  \node[vertex] (v1) at (1,-1.5) {$v_6$};
  \node[vertex] (v2) at (-1.5,-1.5) {$v_4$};
  \node[vertex] (v3) at (-1.5,0.5) {$v_2$};
  \node[vertex] (v4) at (0,0.5) {$v_3$};
  \node[vertex] (v5) at (-0.75,2) {$v_1$};
  \node[vertex] (v6) at (-0.25,-1.5) {$v_5$};

  % Straight Edges
  \draw (v5) edge (v4);
  \draw (v5) edge (v3);
  \draw (v4) edge (v1);
  \draw (v4) edge (v6);
  \draw (v3) edge (v2);
\end{tikzpicture}

%% file: gfx/background/clique.tex
\begin{tikzpicture}[-,>=stealth, auto, very thick, scale=.5]

  \tikzstyle{vertex}=[circle,fill=cyan!50!,draw,inner sep=1pt, minimum size=7pt]

  \node[vertex] (v1) at (0,-1.2) {$v_1$};
  \node[vertex] (v2) at (-1.3,-0.5) {$v_2$};
  \node[vertex] (v3) at (-0.9,1.1) {$v_3$};
  \node[vertex] (v4) at (0.9,1.1) {$v_4$};
  \node[vertex] (v5) at (1.3,-0.5) {$v_5$};

  \draw (v1) -- (v2);
  \draw (v1) -- (v3);
  \draw (v1) -- (v4);
  \draw (v1) -- (v5);
  \draw (v2) -- (v3);
  \draw (v2) -- (v4);
  \draw (v2) -- (v5);
  \draw (v3) -- (v4);
  \draw (v3) -- (v5);
  \draw (v4) -- (v5);
\end{tikzpicture}

%% file: gfx/background/grid.tex
\begin{tikzpicture}[-,>=stealth, auto, very thick, scale=.5]

  \tikzstyle{vertex}=[circle,fill=cyan!50!,draw,inner sep=1pt, minimum size=7pt]

  % Define the nodes
  \node[vertex] (v1) at (0,0) {$v_1$};
  \node[vertex] (v2) at (2,0) {$v_2$};
  \node[vertex] (v3) at (4,0) {$v_3$};
  \node[vertex] (v4) at (0,-2) {$v_4$};
  \node[vertex] (v5) at (2,-2) {$v_5$};
  \node[vertex] (v6) at (4,-2) {$v_6$};

  \draw (v1) edge (v2);
  \draw (v1) edge (v4);
  \draw (v2) edge (v3);
  \draw (v2) edge (v5);
  \draw (v3) edge (v6);

  \draw (v4) edge (v5);
  \draw (v5) edge (v6);

\end{tikzpicture}

%% file: gfx/background/kcommunity.tex
\begin{tikzpicture}[-,>=stealth, auto, very thick, scale=.5]

  \tikzstyle{vertex}=[circle,fill=cyan!50!,draw,inner sep=1pt, minimum size=7pt]
  \tikzstyle{community1}=[fill=blue!30!white]
  \tikzstyle{community2}=[fill=red!30!white]

  % Define nodes in the first community
  \node[vertex,community1] (v1) at (0,2) {$v_1$};
  \node[vertex,community1] (v2) at (2,2) {$v_2$};
  \node[vertex,community1] (v3) at (0,0) {$v_3$};
  \node[vertex,community1] (v4) at (2,0) {$v_4$};

  % Define nodes in the second community
  \node[vertex,community2] (v5) at (4,2) {$v_5$};
  \node[vertex,community2] (v6) at (6,2) {$v_6$};
  \node[vertex,community2] (v7) at (4,0) {$v_7$};
  \node[vertex,community2] (v8) at (6,0) {$v_8$};

  \draw (v1) -- (v2);
  \draw (v1) -- (v3);
  \draw (v1) -- (v4);
  \draw (v2) -- (v3);
  \draw (v2) -- (v4);
  \draw (v3) -- (v4);

  \draw (v5) -- (v6);
  \draw (v5) -- (v7);
  \draw (v5) -- (v8);
  \draw (v6) -- (v7);
  \draw (v6) -- (v8);
  \draw (v7) -- (v8);

  \draw (v2) -- (v5);
  \draw (v4) -- (v7);

\end{tikzpicture}

%% file: gfx/background/bipartite.tex
\begin{tikzpicture}[-,>=stealth, auto, very thick]

  \tikzstyle{vertex}=[circle,fill=cyan!50!,draw,inner sep=1pt, minimum size=15pt]
  \tikzstyle{layer1}=[fill=blue!30!white]
  \tikzstyle{layer2}=[fill=red!30!white]

  % Define nodes in the first layer
  \node[vertex,layer1] (v1) at (0,1) {$v_1$};
  \node[vertex,layer1] (v2) at (1,1) {$v_2$};
  \node[vertex,layer1] (v3) at (2,1) {$v_3$};

  % Define nodes in the second layer
  \node[vertex,layer2] (v4) at (0.5,0) {$u_1$};
  \node[vertex,layer2] (v5) at (1.5,0) {$u_2$};

  % Connect nodes from the first layer to the second layer
  \foreach \i in {1,2,3}
    \foreach \j in {4,5}
      \draw (v\i) -- (v\j);

\end{tikzpicture}

%% file: content/algs/bf.tex
\begin{algorithm}
\caption{Bellman-Ford algorithm}
\label{alg:bf}
\begin{algorithmic}[1]
\Procedure{BellmanFord}{$g$, $s$}
    \ForAll{node $v$ in $\gV_g$}
        \State $d[v] \gets \infty$
        \State $pred[v] \gets \text{NULL}$
    \EndFor
    \State $d[s] \gets 0$
    \For{$i = 1$ to $|\gV_g|-1$}
        \ForAll{edge $(u, v)$ in $\gE_g$}
            \If{$d[v] > d[u] + w_{uv}$}
                \State $d[v] \gets d[u] + w_{uv}$
                \State $pred[v] \gets u$
            \EndIf
        \EndFor
    \EndFor
    \State \textbf{return} $\tuple{d, pred}$
\EndProcedure
\end{algorithmic}
\end{algorithm}

%% file: content/algs/dijkstra.tex
\begin{algorithm}
\caption{Dijkstra's algorithm}
\label{alg:dijkstra}
\begin{algorithmic}[1]
\Procedure{Dijkstra}{$g$, $s$}
    \ForAll{vertex $v$ in $\gV_g$}
        \State $d[v] \gets \infty$
        \State $pred[v] \gets \text{NULL}$
    \EndFor
    \State $d[s] \gets 0$
    \State $Q \gets \gV_g$
    \While{$Q \neq \emptyset$}
        \State $u \gets \text{extract-min}(Q, d[u])$ \label{dij:line:queue}
        \State remove $u$ from $Q$
        \ForAll{neighbor $v$ of $u$ in $Q$}
            \If{$d[v] > d[u] + w_{uv}$}
                \State $d[v] \gets d[u] + w_{uv}$
                \State $pred[v] \gets u$
            \EndIf
        \EndFor
    \EndWhile
    \State \textbf{return} $\tuple{d, pred}$
\EndProcedure
\end{algorithmic}
\end{algorithm}

%% file: content/algs/bfs.tex
\begin{algorithm}
\caption{Breadth-First Search algorithm}
\label{alg:bfs}
\begin{algorithmic}[1]
\Procedure{BFS}{$g$, $s$}
    \ForAll{node $v$ in $\gV_g$}
        \State $r[v] \gets 0$
    \EndFor
    \State $r[s] \gets 1$
    \State initialise an empty queue $Q$
    \State push($Q$, $s$)
    \While{$Q$ is not empty}
        \State $v \gets$ pop($Q$)
        \ForAll{$u$ in $\gN(v)$}
            \If{$r[u] = 0$}
                \State $r[u] \gets 1$
                \State push($Q$, $u$)
            \EndIf
        \EndFor
    \EndWhile
    \State \textbf{return} $r$
\EndProcedure
\end{algorithmic}
\end{algorithm}

%% file: content/algs/prim.tex
\begin{algorithm}
\caption{Prim's algorithm}
\label{alg:prim}
\begin{algorithmic}[1]
\Procedure{Prim}{$g$}
    \ForAll{node $v$ in $\gV_g$}
        \State key[v] $\gets \infty$
        \State pred[v] $\gets$ NULL
    \EndFor
    \State Pick any node $s$
    \State key[s] $\gets 0$
    \State $Q \gets \gV_g$
    \While{$Q \neq \emptyset$}
        \State $u \gets \text{extract-min}(Q, key[s])$
        \ForAll{node $v$ adjacent to $u$}
            \If{$v \in Q$ and $w_{uv} < key[v]$}
                \State pred[v] $\gets u$
                \State key[v] $\gets w_{uv}$
            \EndIf
        \EndFor
    \EndWhile
    \State \textbf{return} $\tuple{key, pred}$
\EndProcedure
\end{algorithmic}
\end{algorithm}

%% file: content/algs/ff.tex
\begin{algorithm}
\caption{Ford-Fulkerson algorithm}
\label{alg:ff}
\begin{algorithmic}[1]
\Procedure{FordFulkerson}{g, s, t}
    \ForAll{edge $(u, v)$ in $\gE_g$}
        \State $f_{uv} \gets 0$
    \EndFor
    \While{$\exists$ augmenting path $\pi=s,\dots,t$ in $g_f$}
        \State $df = \min\{c_{uv} : (u,v) \in \pi\}$
        \ForAll{edge $(u,v)$ in $\pi$}
            \State $f_{uv} \gets f_{uv} + df$
            \State $f_{vu} \gets f_{vu} - df$
        \EndFor
    \EndWhile
    \State \textbf{return} $f$
\EndProcedure
\end{algorithmic}
\end{algorithm}

%% file: content/algs/christofides.tex
\begin{algorithm}
\caption{Christofides' algorithm}
\label{alg:christofides}
\begin{algorithmic}[1]
\Procedure{Christofides}{$g$}
    \State $T \gets$ \Call{ComputeMST}{$g$}
    \State $O \gets$ \Call{GetOddDegreeNodes}{$T$}
    \State $M \gets$ \Call{ComputeMinWeightPerfectMatching}{$O$}
    \State $H \gets T \cup M$
    \State $E' \gets$ \Call{ComputeEulerianPath}{$H$}
    \State $C \gets \text{Hamiltonian cycle from } E'$
    \State \textbf{return} $C$
\EndProcedure
\end{algorithmic}
\end{algorithm}

%% file: content/algs/gon.tex
\begin{algorithm}
  \caption{Gon algorithm}
  \label{alg:gon-algorithm}
  \begin{algorithmic}[1]
    \Procedure{Gon}{$G, k$}
    \State Choose an arbitrary node $v_1$ from $G$ and add it to center set $C$
    \For{$i = 2$ to $k$}
    \State Choose a vertex $v_i$ that is farthest from the centers already in $C$
    \State Add $v_i$ to $C$
    \EndFor
    \State \textbf{return} $C$
    \EndProcedure
  \end{algorithmic}
\end{algorithm}

%% file: content/algs/wl.tex
\begin{algorithm}
\caption{Weisfeiler-Lehman graph isomorphism test}
\label{alg:wl-test}
\begin{algorithmic}[1]
\Procedure{WLTest}{$g_1, g_2$}
\State $c^0_1 \gets$ \Call{InitialColoring}{$g_1$}
\State $c^0_2 \gets$ \Call{InitialColoring}{$g_2$}
\While{colors have not converged}
    \State $c^{t+1}_1 \gets$ \Call{UpdateColoring}{$g_1, c^t_1$}
    \State $c^{t+1}_2 \gets$ \Call{UpdateColoring}{$g_2, c^t_2$}
\EndWhile
\State \Return \Call{HaveSameDistribution}{$C_1, C_2$}
\EndProcedure
\Procedure{UpdateColoring}{$g, c$}
\State Update coloring $c$ based on equation \eqref{eq:wl}
\EndProcedure
\Procedure{HaveSameDistribution}{$c_1, c_2$}
\State Check if the color distributions in $c_1$ and $c_2$ are the same
\EndProcedure
\end{algorithmic}
\end{algorithm}

%% file: gfx/background/graph_convolution.tex
\begin{tikzpicture}[->,>=stealth,auto, thick,scale=0.90]
  \definecolor{dgreen}{HTML}{00871B}

  \tikzstyle{message}=[->, decorate,
  decoration={snake,
    amplitude=1,
    segment length=6,
    post length=7
  },
  line width=2.5pt]
  \tikzstyle{vertex}=[circle,fill=white,draw,inner sep=1pt, minimum size=15pt, very thick]

  \node[vertex] (1) at (-5,-1.5) {\tiny$\vh_z^{(0)}$};
  \node[vertex] (2) at (-7,-1.5) {\tiny$\vh_y^{(0)}$};
  \node[vertex] (3) at (-7,0.5) {\tiny$\vh_x^{(0)}$};
  \node[vertex, fill=cyan!50] (4) at (-5,0.75) {\tiny$\vh_v^{(0)}$};
  \node[vertex] (5) at (-5,2.5) {\tiny$\vh_u^{(0)}$};

  \node (label) at (-6, -2.5) {$\ell=0$};

  \draw (1) edge (4);
  \draw (1) edge[bend left=30] (2);
  \draw (2) edge (1);
  \draw (2) edge (3);
  \draw (3) edge (1);
  \draw (3) edge (4);
  \draw (4) edge (5);

  \draw (3) edge[message, cyan, opacity=0.7] node[midway, sloped] {\tiny$x \rightarrow_\psi v$} (4);
  \draw (1) edge[message, cyan, opacity=0.7, ] node[midway, sloped] {\tiny$z \rightarrow_\psi v$} (4);

  \node[vertex] (s1) at (0,-1.5) {\tiny$\vh_z^{(1)}$};
  \node[vertex] (s2) at (-2,-1.5) {\tiny$\vh_y^{(1)}$};
  \node[vertex] (s3) at (-2,0.5) {\tiny$\vh_x^{(1)}$};
  \node[vertex] (s4) at (0,0.75) {\tiny$\vh_v^{(1)}$};
  \node[vertex, fill=cyan!50] (s5) at (0,2.5) {\tiny$\vh_u^{(1)}$};

  \node (label) at (-1, -2.5) {$\ell=1$};

  \draw (s1) edge (s4);
  \draw (s1) edge[bend left=30] (s2);
  \draw (s2) edge (s1);
  \draw (s2) edge (s3);
  \draw (s3) edge (s1);
  \draw (s3) edge (s4);
  \draw (s4) edge (s5);

  \draw (s4) edge[message, cyan, opacity=0.7] node[midway, sloped] {\tiny$v \rightarrow_\psi u$} (s5);

  \draw (4) edge[message, orange, bend left=15, thick] node[midway, sloped] {\tiny$\vh^{(1)}_v = \phi(\vh^{(0)}_v, \oplus(\set{m_x, m_z}))$} (s4);

  \node[vertex] (t1) at (5,-1.5) {\tiny$\vh_z^{(2)}$};
  \node[vertex] (t2) at (3,-1.5) {\tiny$\vh_y^{(2)}$};
  \node[vertex] (t3) at (3,0.5) {\tiny$\vh_x^{(2)}$};
  \node[vertex] (t4) at (5,0.75) {\tiny$\vh_v^{(2)}$};
  \node[vertex] (t5) at (5,2.5) {\tiny$\vh_u^{(2)}$};

  \node (label) at (4, -2.5) {$\ell=2$};

  \draw (t1) edge (t4);
  \draw (t1) edge[bend left=30] (t2);
  \draw (t2) edge (t1);
  \draw (t2) edge (t3);
  \draw (t3) edge (t1);
  \draw (t3) edge (t4);
  \draw (t4) edge (t5);

  \draw (s5) edge[message, orange, bend left=15, thick] node[midway, sloped] {\tiny$\vh^{(2)}_u = \phi(\vh^{(1)}_u, \oplus(\set{m_v}))$} (t5);

\end{tikzpicture}

%% file: gfx/background/iterative_gn.tex
\begin{tikzpicture}[scale=0.9]

  \tikzstyle{vertex}=[circle,draw,inner sep=1pt, minimum size=15pt]

  % In Graph
  \node[vertex, fill=yellow!50] (v1) at (-3,-0.5) {$v_1$};
  \node[vertex, fill=blue!30!] (v2) at (-4,-0.5) {$v_2$};
  \node[vertex, fill=red!20!] (v3) at (-4,0.5) {$v_3$};
  \node[vertex, fill=green!30!] (v4) at (-3,0.75) {$v_4$};
  \node[vertex, fill=blue!30!] (v5) at (-3,1.5) {$v_5$};

  \draw (v1) edge (v3);
  \draw (v1) edge (v4);
  \draw (v2) edge (v1);
  \draw (v2) edge (v3);
  \draw (v3) edge (v4);
  \draw (v4) edge (v5);

  \draw (-2, 0.25) edge[->, thick] (-0.75, 0.25);

  % MP Layer
  \node[draw, rounded corners, inner sep=10pt, rotate=90, very thick] (r) at (0, 0.25) {MP Layer};

  % \draw[blue] (r) edge[loop left = 1, line width=3.5pt, looseness=5, -stealth, opacity=0.7] node [above, opacity=1.0] {\LARGE $P$} (r);

  \draw[->, very thick] (r) to [out=55, in=80, looseness=15] (r);
  \draw (0.75, 0.25) edge[->, thick] (2, 0.25);

  % MP Layer box
  \draw[rounded corners=7pt, very thick] (-3, -5) rectangle (3,-2);

  \draw[-, densely dashed, thick] (r.west) to (-3, -2);
  \draw[-, densely dashed, thick] (r.west) to (3, -2);

  \node[vertex] (1) at (-1,-4.5) {$v_1$};
  \node[vertex] (2) at (-2,-4.5) {$v_2$};
  \node[vertex] (3) at (-2,-3.5) {$v_3$};
  \node[vertex] (4) at (-1,-3.25) {$v_4$};
  \node[vertex] (5) at (-1,-2.5) {$v_5$};

  \draw (1) edge (3);
  \draw (1) edge (4);
  \draw (2) edge (1);
  \draw (2) edge (3);
  \draw (3) edge (4);
  \draw (4) edge (5);

  \node at (0, -3.5) {$\star$};

  \matrix (m) at (1.5, -3.75) [matrix of nodes,
    nodes={minimum size=3mm, draw, inner sep=0pt},
  ]
  {
    |[fill=green!20!orange]| & |[fill=green!90!orange]| & |[fill=green!40]| & |[fill=green!50]| \\
    |[fill=green!60]| & |[fill=green!70!orange]| & |[fill=green!80]| & |[fill=green!10]| \\
    |[fill=green!20!orange]| & |[fill=green!50]| & |[fill=green!60!orange]| & |[fill=green!30]| \\
    |[fill=green!90]| & |[fill=green!20]| & |[fill=green!30!orange]| & |[fill=green!80]| \\
    |[fill=green!33!orange]| & |[fill=green!71!orange]| & |[fill=green!44]| & |[fill=green!55!orange]| \\
  };

  \node at (1.5, -2.5) {$\mW$};

  % End MP Layer box

  % Output Graph
  \node[vertex, fill=blue!40] (x1) at (4,-0.5) {$v_1$};
  \node[vertex, fill=blue!40] (x2) at (3,-0.5) {$v_2$};
  \node[vertex, fill=magenta!40] (x3) at (3,0.5) {$v_3$};
  \node[vertex, fill=yellow!30!green] (x4) at (4,0.75) {$v_4$};
  \node[vertex, fill=green!20] (x5) at (4,1.5) {$v_5$};

  \draw (x1) edge (x3);
  \draw (x1) edge (x4);
  \draw (x2) edge (x1);
  \draw (x2) edge (x3);
  \draw (x3) edge (x4);
  \draw (x4) edge (x5);

  %% FEEDFORWARD

  \def\xoffset{0}

  % In Graph
  \node[vertex, fill=yellow!50] (ff1v1) at (-6-\xoffset,-9) {$v_1$};
  \node[vertex, fill=blue!30!] (ff1v2) at (-7-\xoffset,-9) {$v_2$};
  \node[vertex, fill=red!20!] (ff1v3) at (-7-\xoffset,-8) {$v_3$};
  \node[vertex, fill=green!30!] (ff1v4) at (-6-\xoffset,-7.75) {$v_4$};
  \node[vertex, fill=blue!30!] (ff1v5) at (-6-\xoffset,-7) {$v_5$};

  \draw (ff1v1) edge (ff1v3);
  \draw (ff1v1) edge (ff1v4);
  \draw (ff1v2) edge (ff1v1);
  \draw (ff1v2) edge (ff1v3);
  \draw (ff1v3) edge (ff1v4);
  \draw (ff1v4) edge (ff1v5);

  \draw (-5-\xoffset, -8.25) edge[->, thick] (-3.75-\xoffset, -8.25);

  % MP Layer
  \node[draw, rounded corners, inner sep=10pt, rotate=90, very thick] (ff1r) at (-3-\xoffset, -8.25) {MP Layer};
  \draw (-2.25-\xoffset, -8.25) edge[->, thick] (3-\xoffset, -8.25);

  % MP Layer box
  \draw[rounded corners=7pt, very thick] (-6-\xoffset, -13.5) rectangle (0-\xoffset,-10.5);

  \draw[-, densely dashed, thick] (ff1r.west) to (-6-\xoffset, -10.5);
  \draw[-, densely dashed, thick] (ff1r.west) to (0-\xoffset, -10.5);

  \node[vertex] (ff2v1) at (-4-\xoffset,-13) {$v_1$};
  \node[vertex] (ff2v2) at (-5-\xoffset,-13) {$v_2$};
  \node[vertex] (ff2v3) at (-5-\xoffset,-12) {$v_3$};
  \node[vertex] (ff2v4) at (-4-\xoffset,-11.75) {$v_4$};
  \node[vertex] (ff2v5) at (-4-\xoffset,-11) {$v_5$};

  \draw (ff2v1) edge (ff2v3);
  \draw (ff2v1) edge (ff2v4);
  \draw (ff2v2) edge (ff2v1);
  \draw (ff2v2) edge (ff2v3);
  \draw (ff2v3) edge (ff2v4);
  \draw (ff2v4) edge (ff2v5);

  \node at (-3-\xoffset, -12) {$\star$};

  \matrix (m) at (-1.5-\xoffset, -12.25) [matrix of nodes,
    nodes={minimum size=3mm, draw, inner sep=0pt},
  ]
  {
    |[fill=green!73!orange]| & |[fill=green!33!orange]| & |[fill=green!91!orange]| & |[fill=green!29!orange]| \\
    |[fill=green!92!orange]| & |[fill=green!57]| & |[fill=green!29!orange]| & |[fill=green!13!orange]| \\
    |[fill=green!63!orange]| & |[fill=green!58!orange]| & |[fill=green!4]| & |[fill=green!28]| \\
    |[fill=green!9!orange]| & |[fill=green!7]| & |[fill=green!72!orange]| & |[fill=green!64!orange]| \\
    |[fill=green!90!orange]| & |[fill=green!58!orange]| & |[fill=green!58!orange]| & |[fill=green!84!orange]| \\
  };

  \node at (-1.5-\xoffset, -11) {$\mW^{(1)}$};

  % End MP Layer Box

  % MP Layer
  \def\xoffset{-7}

  \node[draw, rounded corners, inner sep=10pt, rotate=90, very thick] (ff1r) at (-3-\xoffset, -8.25) {MP Layer};
  \draw (-2.25-\xoffset, -8.25) edge[->, thick] (-1-\xoffset, -8.25);

  % MP Layer box
  \draw[rounded corners=7pt, very thick] (-6-\xoffset, -13.5) rectangle (0-\xoffset,-10.5);

  \draw[-, densely dashed, thick] (ff1r.west) to (-6-\xoffset, -10.5);
  \draw[-, densely dashed, thick] (ff1r.west) to (0-\xoffset, -10.5);

  \node[vertex] (ff2v1) at (-4-\xoffset,-13) {$v_1$};
  \node[vertex] (ff2v2) at (-5-\xoffset,-13) {$v_2$};
  \node[vertex] (ff2v3) at (-5-\xoffset,-12) {$v_3$};
  \node[vertex] (ff2v4) at (-4-\xoffset,-11.75) {$v_4$};
  \node[vertex] (ff2v5) at (-4-\xoffset,-11) {$v_5$};

  \draw (ff2v1) edge (ff2v3);
  \draw (ff2v1) edge (ff2v4);
  \draw (ff2v2) edge (ff2v1);
  \draw (ff2v2) edge (ff2v3);
  \draw (ff2v3) edge (ff2v4);
  \draw (ff2v4) edge (ff2v5);

  \node at (-3-\xoffset, -12) {$\star$};

  \matrix (m) at (-1.5-\xoffset, -12.25) [matrix of nodes,
    nodes={minimum size=3mm, draw, inner sep=0pt},
  ]
  {
    |[fill=green!47]| & |[fill=green!86!orange]| & |[fill=green!23]| & |[fill=green!52!orange]| \\
    |[fill=green!67]| & |[fill=green!74]| & |[fill=green!89]| & |[fill=green!14]| \\
    |[fill=green!29!orange]| & |[fill=green!53!orange]| & |[fill=green!65]| & |[fill=green!35]| \\
    |[fill=green!98]| & |[fill=green!21]| & |[fill=green!32!orange]| & |[fill=green!77]| \\
    |[fill=green!42]| & |[fill=green!64!orange]| & |[fill=green!45]| & |[fill=green!59!orange]| \\
  };

  \node at (-1.5-\xoffset, -11) {$\mW^{(2)}$};

  % End MP Layer Box

  % Output Graph
  \node[vertex, fill=blue!40] (x1) at (8,-9) {$v_1$};
  \node[vertex, fill=blue!40] (x2) at (7,-9) {$v_2$};
  \node[vertex, fill=magenta!40] (x3) at (7,-8) {$v_3$};
  \node[vertex, fill=yellow!30!green] (x4) at (8,-7.75) {$v_4$};
  \node[vertex, fill=green!20] (x5) at (8,-7) {$v_5$};

  \draw (x1) edge (x3);
  \draw (x1) edge (x4);
  \draw (x2) edge (x1);
  \draw (x2) edge (x3);
  \draw (x3) edge (x4);
  \draw (x4) edge (x5);
\end{tikzpicture}

%% file: content/nar.tex
\chapter{Neural Algorithmic Reasoning} \label{ch:nar}
Neural Algorithmic Reasoning (NAR) is a burgeoning field at the
intersection of computer science, machine learning, and algorithms. It
is dedicated to the studying and understanding of the degree to which
neural architectures can learn to execute algorithms, thereby
connecting traditional algorithmic concepts with modern neural network
models. This chapter will serve as an introduction to the foundational
concepts of this field of research.

The first section, \textit{Why learn algorithms?}, we discuss
potential challenges that arise from the application of classical
algorithms to real-world data, underlining the rationale behind
transitioning classical algorithms to a neural framework.

Thus, \textit{Principles of Algorithmic Reasoning}, elucidates the
crucial aspects of learning to execute algorithms. We define neural
architectures arising from \textit{algorithmic alignment} requirements
and discuss different training pipelines in comparison to conventional
machine learning approaches.

Concluding the chapter, \textit{Learnt algorithms as inductive priors}
showcases different modalities through which we may leverage
algorithmically-informed neural networks to tackle learning of new
tasks.

\section{Why learn algorithms?}
\label{sec:nar-motivation}
As explained in \autoref{sec:algo-guarantees}, the field of computer
science is rooted in algorithms. Algorithms are essentially
problem-solving methods that provide us with many guarantees, such
as their theoretical correctness, contrarily to deep learning, where
we usually can not say much about the
generalisation capabilities of the learnt function. This brings up
a very obvious debated question: ``\textit{Why should we learn
  algorithms if algorithms are provably correct?}''. The answer lies
in the gap between the theoretical guarantees provided by algorithms
and the practical realities of applying these algorithms to real-world
data.

\begin{figure}
  \includegraphics[width=1\linewidth]{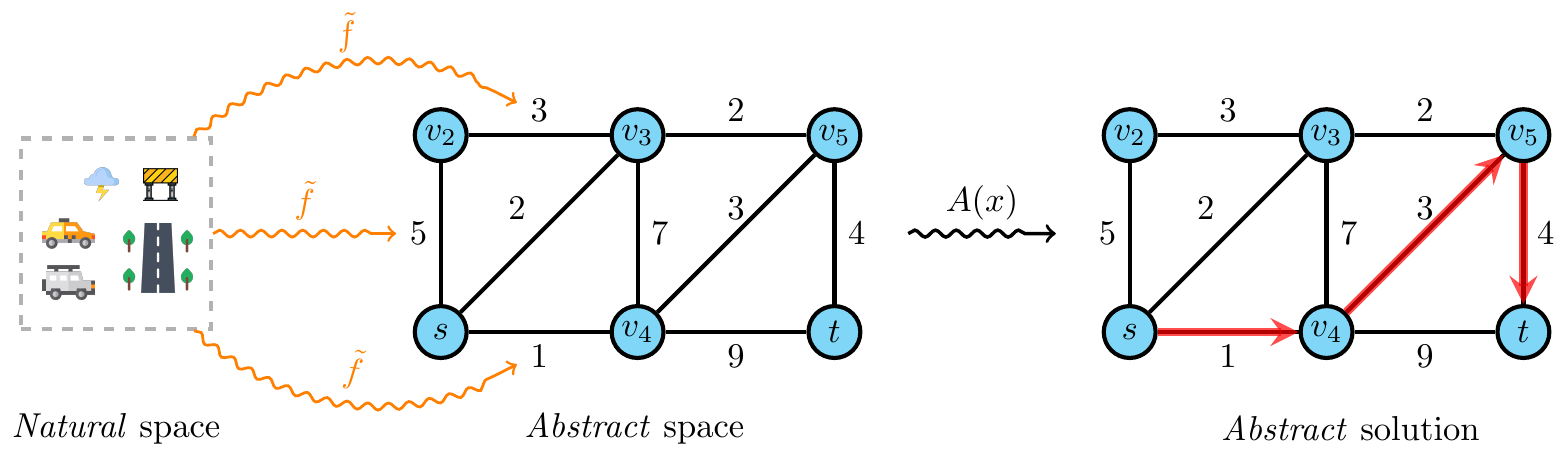}
  \caption{Illustration of the algorithmic bottleneck phenomena. Here,
    we model a shortest path problem in a real-world environment
    (\emph{Natural} space) as a graph in the \emph{abstract} space of
    algorithms. However, the encoding of multi-dimensional, noisy
    natural data (i.e., $\tilde{f}$) to a single dimension (i.e.,
    scalars of algorithms) is often performed manually and leads to
    loss of information \citep{harris1955fundamentals}. Then, we get a
    provably correct, but likely \textbf{suboptimal}, solution in the
    abstract space through applications of a shortest path algorithm
    (e.g., Dijkstra).}
  \label{fig:algo-flaws-example}
\end{figure}

\paragraph{Algorithms flaws.}
Algorithms are often perceived as flawless executors that invariably
produce correct solutions to specific problems, mainly because of
their theoretical guarantees regarding correctness and invariances
throughout their execution. However, the truth is that they can only
meet these expectations when the abstract spaces in which they operate
align precisely with the real-world contexts where their results are
intended to be applied. To clarify this point, consider the following
example, also depicted in \autoref{fig:algo-flaws-example}.

Consider a scenario where we aim to solve a routing problem within a
large road and traffic network. Now, consider an agent navigating this
environment whose goal is to identify the \textit{best} possible path
from point A to point B. Theoretically, we can model the large network
with a weighted graph $g=\tuple{\gV, \gE, \gW}$, where edges represent
roads and nodes symbolise locations (e.g., intersection points, point
of interests, etc.). Each edge $(u, v)$ is assigned a weight $w_{uv}$
denoting the \textit{cost} of traversing that road. By employing this
formalism, we can use our favourite shortest path algorithm to
obtain a solution which is provably the optimal one.

However, there are significant challenges in translating this
theoretical model into the real-world scenario. In particular,
\textit{how can we be certain that the edge weights are estimated
  correctly?} Most of classical graph algorithms often commit to use a
scalar value to represent the cost of traversing an edge. However,
quantifying accurately these scalars is extremely difficult in
general. Factors such as fluctuating traffic conditions, unpredictable
weather, and unforeseen incidents all need to be quantified jointly
and compressed into a single dimension to enable the use of graph
algorithms. If this estimation is imprecise, our theoretically optimal
solution might not be truly optimal in the context of our real-world
scenario.

This issue has been known for a long time. Indeed, in 1955,
\cite{harris1955fundamentals} analysed and discussed methods for
evaluating railways networks capacities, effectively modelling the
problem as a max-flow problem. In the very first paragraphs of their
investigation, the authors acknowledged that there were no known
tested mathematical models for robustly estimate the scalar capacities
of the edges in the rail network. This lack of formal methods for
estimating such scalars persists even to these days.

In light of these observations, algorithms' outputs, although provably
``optimal'', can still be considered approximations in case their
abstract space can not model our problem of interest exactly. In
simpler terms, theoretical correctness does not matter if we execute
the algorithm on the \textit{wrong} inputs.

As a last note, we highlight that classical algorithms may also come
across other problematics, such as dealing with partially observable
data and missing data. In such cases, algorithms quickly become
inapplicable as their preconditions are often violated by partially
missing specifications.

\paragraph{How does neural algorithmic reasoning help?}
One possible approach to address the problems mentioned earlier is to
employ neural networks to estimate the edge scalars, effectively
replacing the manual feature engineering, and then run an algorithm on
the predicted configuration. However, this strategy presents several
limitations. First, we would need to backpropagate through the
operations of the algorithm. However, algorithms are inherently
discrete functions (\autoref{sec:algos-discrete-functions}), and as
such, do not provide useful gradients for learning, although research
has started exploring ways to overcome such issue
\citep{poganvcic2019differentiation}. Second, the accuracy of the
estimated scalars might be compromised when the available training
data is limited. Additionally, attempting to compress the complexities
of the real world (e.g., traffic and weather conditions in the
previous example) into a single scalar value might lead to a drastic
loss of information. Such problem is often referred to as
\textit{algorithmic bottleneck} \citep{velickovic2021neural}.
Consequently, once the algorithm commits to these scalars, the result
could be suboptimal and, potentially, incorrect.

One alternative solution is to consider performing the algorithm
directly in a neural network's latent space. This is the objective of
Neural Algorithmic Reasoning, where we aim at building neural networks
that behave as close as possible to a chosen algorithm. Such an
approach, assuming that a neural network can indeed behave
algorithmically, actually solves both problems. First, the process
becomes fully differentiable, making learning far more
feasible. Second, neural networks operate in a high-dimensional space,
allowing for a more nuanced representation of the real-world
complexities that are difficult to capture in a single scalar. By
performing the algorithmic computations directly within the neural
network's high-dimensional space, we bypass the limitations associated
with the algorithmic bottleneck. Consequently, this allows for a more
faithful and adaptable representations of real-world scenarios. On a
broad scale, we fundamentally trade off algorithmic guarantees with
more accurate representations.

\paragraph{Is that all?}
One might argue if the above observations encompass the full objective
of neural algorithmic reasoning. The answer, however, is that there
remains much more to explore and understand within this domain.

Indeed, from a machine learning perspective, the learning of
algorithms is an intruiguing and fascinating challenge in its own
right. One field which might greatly benefit from NAR is the field of
Neural Combinatorial Optimisation (NCO). Precisely, NCO tackles
solving combinatorial optimisation problems through the use of neural
networks, either in conjuction with established CO methods
\citep{gasse2019exact, gupta2020hybrid} or as end-to-end neural
solvers \citep{bello2017neural, kool2019attention}. In this respect,
the objective is to solve CO problems to a higher precision by
crafting algorithmically-inspired neural networks that incorporate
algorithms as prior knowledge, and apply them to a related CO problem.
As highlighted in \autoref{sec:graph-algos}, there exist various
approximation algorithms that are reliant on \textit{algorithmic
  primitives} (e.g., MST, shortest path and sorting algorithms, etc.).
This supports the idea that incorporating algorithmic bias into neural
networks is useful. In other words, algorithms constitute an important
\textit{prior} for solving combinatorial problems. Injecting such
priors into neural networks, thus, could lead to a better exploration
of the solution space and a better understanding of the underlying
structure of these problems. One practical example is, for instance, a
neural network targeting the resolution of a TSP. In this scenario,
prior knowledge of how to compute a minimum spanning tree is of
fundamental importance. Such knowledge can be integrated by learning
an algorithm for computing the MST (e.g., Prim or Kruskal
\citep{cormen2009introduction}).

One of the profound aspirations in applying NAR to combinatorial
optimisation, thus, lies in the possibility of neural networks
discerning, via algorithmic priors, heuristics for NP-H problems that
surpass the advancements made over years in standard approximation
algorithms \citep{michalewicz2013solve, hromkovivc2013algorithmics,
  burke2014search, mart2018handbook}. Notably, the current NAR
literature already comprehends examples of how algorithmic biases have
proven a useful prior for neural networks \citep{davies2021advancing,
  blundell2022towards, beurer2022learning}.

As a last note, NAR's objectives also aim at building neural networks
that extrapolate beyond the support of training data -- commonly
termed out-of-distribution (OOD) generalisation. Indeed, overfitting
is a well-known problem in the machine learning community. Neural
models often excel at fitting the training data but struggle to
generalise when the test data distribution \textit{shifts} from the
training distribution \citep{neyshabur2017exploring}. Algorithms, on
the other hand, have the inherent ability to generalise on any
(acceptable) input, thanks to their strong theoretical guarantees.
Hence, in order to build neural networks that effectively behave like
algorithms, much effort is put into preserving this property inside
neural approaches. In the following sections, we discuss how this
objective is pursued in practice.

\section{Principles of Algorithmic Reasoning}
In this section, we present and clarify how the task of performing
algorithmic reasoning through neural networks is structured. We
explore how algorithms are learnt in an end-to-end fashion by
analysing each one of the three fundamental principles of neural
algorithmic reasoning: \textit{algorithmic alignment},
the \textit{encode-process-decode} architecture and \textit{step-wise}
supervision.

\subsection{Algorithmic alignment}
\label{sec:algorithmic-alignment}
In the context of NAR, the term algorithmic alignment
\citep{xu2020can} refers to the capacity of specific neural
architectures to learn and extrapolate on reasoning tasks, such as
executing classical algorithm. Here, \textit{alignment} implies that
some neural architectures inherently lean towards learning certain
families of functions more easily than others. To better understand this
concept, consider two popular neural network architectures: MLPs and
GNs. Notably, the latter includes the former in its equations (see \eqref{eq:gn-mp}), 
building upon MLPs and sharing similar approximation capabilities
\citep{scarselli2008computational}. However, even though
both MLPs and GNs have theoretically similar expressive power, the
different architectural bias may make the finding of a good
approximation to $f$ more or less ``difficult''. Now, consider a
graph algorithm $A$ and a neural network $\hat{f}$, a more formal
definition of algorithmic alignment is given as follows:
\begin{definition}
  Let $\hat{f}_1, \ldots, \hat{f}_{n}$ be the modules of $\hat{f}$ and
  let $A_1, \ldots, A_n$ be the modules of the algorithm. Thus, $\hat{f}$
  and $A$ are said to be $M$-algorithmically aligned if: (i)
  if we replace $\hat{f}_i$ with $A_i$ for each $i$, then the network
  $\hat{f}$ emulates $A$; (ii) the learning problem $\hat{f}_i \approx A_i$
  exhibits sample complexity $\gC(\hat{f}_i, A_i)$ such that
  $\sum_{i=1}^n \gC(\hat{f}_i, A_i) \leq M$.
\end{definition}
Informally, the sample complexity $\gC(\hat{f}_i, A_i)$
\citep{hanneke2016optimal} can be phrased as the minimum number of
training samples required to get a ``good'' approximation of $A_i$ via
$\hat{f}_i$. Hence, the concept of algorithmic alignment fundamentally
states that a good alignment implies that each of the algorithm's
submodules are ``easy'' to learn for the neural architecture (i.e.,
``small'' $M$).  That is the reason why considering the architectural
bias is crucial when dealing with learning of classical
algorithms. Actually, a first hint at the suitability of graph
networks to learn graph algorithms was given in
\autoref{sec:gn-desiderata}, where we discussed how GNs process
information equivalently to the $1$-WL test algorithm. Since, the
$1$-WL test algorithm can be phrased under the message-passing
framework, one may argue that GNs are a good fit for any
message-passing algorithms. Indeed, in this direction,
\cite{dudzik2022graph} show that GNs actually aligns with dynamic
programming algorithms. These intuitions suggest that graph networks
are the preferable choice over other kinds of architectures (e.g.,
MLPs) when the target function is an algorithm. To show this
precisely, consider the learning task of executing the Bellman-Ford
algorithm. Interestingly, Bellman-Ford relaxation steps, outlined in
\eqref{eq:bf-relaxation}, can be framed in terms of message-passing
through node and edge features. In fact, \eqref{eq:bf-relaxation} is
essentially a special case of \eqref{eq:gn-mp} where $\oplus$ and
$\phi$ are replaced with minimum (or maximum\footnote{Note that
  $\min(a,b) = |\max(-a, -b)|$ thus $\min$ and $\max$ can be used
  interchangeably}) operations. This seamlessly aligns the GN with
Bellman-Ford, as detailed by the following equations:
\begin{align}
  % \begin{split}
    &\text{(MP)}& &\min \left(
                    \vh_v,
                    \min \left(\left\{
                    \psi(\vh_v,
                    \vh_u,
                    \vh_{uv}) \mid u \in \gN(v) \right\}\right)
                    \right) \label{eq:gn-mp-alignment} \\
    &\text{(BF)}& &\min \left(
                    d_v,
                    \min \left(\left\{
                    d_u + e_{uv}
                    \mid u \in \gN(v) \right\}\right)
                    \right). \label{eq:bf-mp-alignment}
  % \end{split}
\end{align}
From the presented equations, it is clear that the parametrised
network $\psi$ only needs to learn a linear function (i.e.,
$\psi(\vh_v, \vh_u, \vh_{uv}) \approx d_u + e_{uv}$), whereas other
aggregations and operations are inherently addressed through
architectural bias. Contrastingly, learning such relaxation step
through a different neural network architecture, such as an MLP, would
require to learn aggregating the representations and recognising
neighbourhoods, which is far beyond the task of ``simply'' learning a
linear function. Another interesting point of discussion is the
consequences that this ``linear'' algorithmic alignment have on OOD
generalisation. \cite{xu2021neural} provide empirical evidence
showing that neural networks with ReLU activations typically
generalise in a linear fashion along any dimension outside their
training distribution. While this might seem a limitation in standard
machine learning tasks, it is actually a desirable property when
dealing with most of classical algorithms. By analysing
\eqref{eq:gn-mp-alignment} and \eqref{eq:bf-mp-alignment}, we deduce
that $\psi$ must extrapolate linearly to align with the behaviour of
the relaxation step. Consequently, given its inclination towards
linear generalisation, ReLU stands out as the ideal activation
function for use within $\psi$.

In the latest years, more alignment properties between GNs and DP
algorithms have been demonstrated in the literature, whether
empirically or more formally, testifying that this is an important
property to achieve. \cite{fereydounian2022exact} formally state that
when the target function is \textit{permutation-compatible} (i.e., the
function is not influenced by re-labeling of the graph), then GNs can
\textit{generate} it. This has direct implications on the alignment
between GNs and algorithms, since permutation compatibility can be
proved to hold for some algorithmic functions (e.g., minimum cut).
\cite{bevilacqua2023neural} argue that most of classical algorithms
exhibit a \textit{causal model} \citep{sloman2005causal} underneath,
and aligning to it is indeed beneficial for OOD generalisation. Much
work have been done also to show that GNs align to \textit{parallel}
algorithms \citep{engelmayer2023parallel} and to augment GNs with
external memory to align more closely to recursion
\citep{jayalath2023recursive, jain2023neural}.

As it is evident from the above discussion, algorithmic alignment does
not restrict to the architectural bias alone. In fact,
\cite{cappart2023combinatorial} actually present a list of
``prescriptions'' that spans architectures, learning setups and loss
functions. In the following sections, we discuss the most important
principles, namely the encode-process-decode architecture and
step-wise supervision, and we refer the interested reader to
\citep{cappart2023combinatorial} for a more comprehensive list.

\subsection{The Encode-Process-Decode architecture}
\label{sec:epd}
\begin{figure}
  \input{gfx/nar/epd}
  \caption{The Encode-Process-Decode architecture. The processor is
    shared across multiple tasks through the application of several
    encoders/decoders that map from inputs to $p$'s latent space and
    to outputs. In this pipeline, the processor $p$ is trained to
    align to a shortest path algorithm (\textcolor{blue}{blue}
    path). Then, we use any of the transfer learning methodologies
    discussed in \autoref{sec:transfer} to reason on \emph{natural}
    inputs (\textcolor{red}{red} path) while maintaining alignment to
    shortest path.  }
  \label{fig:epd}
\end{figure}
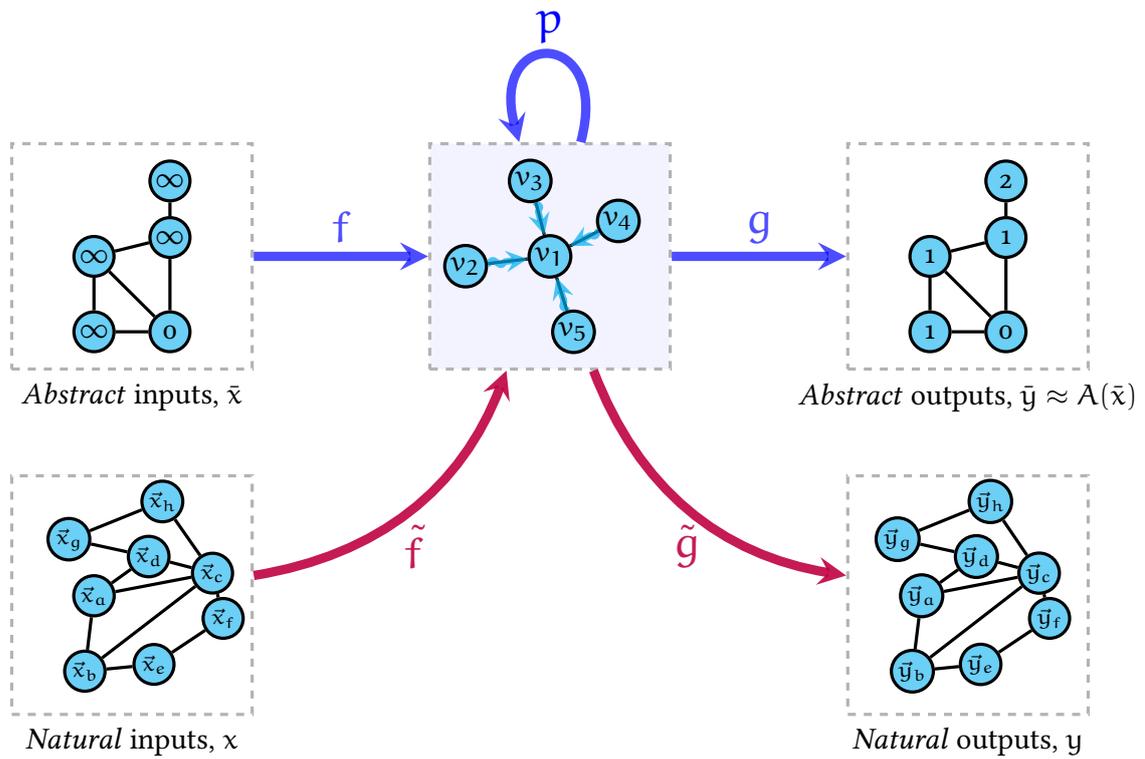
The Encode-Process-Decode (EPD) architecture
\citep{hamrick2018relational}, depicted in \autoref{fig:epd}, is a
specific neural network architecture used in a variety of machine
learning fields, particularly in the field of neural algorithmic
reasoning \citep{velickovic2021neural, cappart2023combinatorial}. The
EPD architecture consists of three learnable components: the encoder
$f$, the processor $p$, and the decoder $g$. The whole architecture
can be described as a composition of these three elements
$\net=g \circ p \circ f$.  Notably, the whole architecture is called
iteratetively $k$ times (i.e., $net^k$), symbolysing the $k$ steps of an
algorithm. This essentially formalises the need to using an
\textit{iterative} neural network architecture to obey the constraints
of OOD generalisation discussed in \autoref{sec:it-vs-ff-mp}. We now
describe the general forward pass (at each step $k$) by examining
independently each component of the EPD architecture. In the
following, we consider the notation introduced in
\autoref{sec:algo-notation} with $\vx$ and $\vy$ denoting
\texttt{inputs} and \texttt{outputs}.

\paragraph{Encoder.}
The encoder $f: \gX \rightarrow \sR^d$ primary objective is to
transform raw input data into a suitable format for the subsequent
processing stage. This is the first step to ``break'' the scalar
bottleneck. Typically, in an EPD architecture we would like our processor
to be the only component to be ``algorithmically-aligned'' to a set
of chosen algorithms. Thus, encoders (and decoders) are generally chosen
among the simplest classes of learning architectures, namely linear models.
Hence, suppose $\vx_v \in \gX_\gV$ and $\vx_{uv} \in \gX_\gE$ are
respectively arbitrary node and edge \texttt{inputs}, the encoder outputs
their encoded representations:
\begin{align*}
  &\vz^{(t)}_v = f_\theta(\vx^{(t)}_v)\\
  &\vz^{(t)}_{uv} = f_\theta(\vx^{(t)}_{uv}),
\end{align*}
where $\theta$ symbolises the learnable parameters.
\paragraph{Processor.}
The processor $p: \sR^d \rightarrow \sR^d$ is the core component of
the EPD pipeline that operates on the encoded representations $\vz$
generated by the encoder. This component is usually given an
architectural bias towards the computation dynamics of the target
algorithms. Since we mostly deal with graph algorithms, the processor
is assumed to be a graph network. Thus, the processor network performs
iterative step-wise updates of the node representations:
\begin{align*}
  \vh^{(t)}_v = p_\theta\Big(&\vz^{(t)}_v, \big\{\vz_{uv}^{(t)}\big\}_{(u, v) \in \gE},\vh_{v}^{(t-1)},\\
  &\big\{\vh_{u}^{(t-1)}\big\}_{u \in \gN(v)}, \big\{\vh_{uv}^{(t)}\big\}_{(u, v) \in \gE}\Big),
\end{align*}
as well as edge representations:
\begin{equation*}
  \vh^{(t)}_{uv} = p_\theta\left(\vz^{(t)}_v, \vz_{u}^{(t)}, \vz_{uv}^{(t)},
    \vh_{uv}^{(t-1)}, \vh_{u}^{(t-1)}, \vh_{v}^{(t-1)}\right).
\end{equation*}
Here, $\vh^{(0)}$ is initialised as the null vector both for node and
edge representations.

In practice, the selection of the message-passing function in the
processor network is guided by both algorithmic alignment properties
and empirical performance. Here, we generally seek simple and general
architectures, avoiding overly complicated graph convolution
operations that stray from algorithmic behaviour. For instance,
notorious message-passing methods based on Chebyshev polynomials
\citep{defferrard2016convolutional}, Laplacian matrices
\citep{kipf2017semi} and graph isomorphism \citep{xu2019powerful} are
not typically used as NAR processors. The rationale behind this is
that spectral processing, normalisation techniques and aggregators
commonly used in these operations do not conform to the alignment
principles discussed in \autoref{sec:algorithmic-alignment}.

In the following, we list some of the most used processor networks:
\begin{itemize}
\item \textbf{Message-Passing Neural Networks}
  \citep{gilmer2017neural}: MPNNs corresponds exactly to equation
  \eqref{eq:gn-mp}, where typically $\oplus = \max$ to better align
  with DP algorithms. Furthermore, graphs are always processed as
  cliques, regardless of their true connectivity. Herein, there exist
  also variants such as MPNNs based on \textit{triplet reasoning}
  (Tri-MPNNs) \citep{ibarz2022generalist} that consider
  \textit{edge-based} message passing, in order to align more closely
  to a broader set of algorithms (e.g., Floyd-Warshall).

\item \textbf{Pointer Graph Networks} \citep{velickovic2020pointer}:
  PGNs are special cases of MPNNs, where messages are computed and
  exchanged within the true neighbourhoods $\gN(v)$ in the graph,
  hence not considering graphs as fully-connected.

\item \textbf{Graph Attention Networks} \citep{velickovic2018graph}:
  GATs use attention \citep{vaswani2017attention} to
  weigh contributions of neighbouring nodes. These models have
  empirically shown to be more effective to tackle computational
  geometry algorithms \citep{velickovic2022clrs}.

\end{itemize}
Indeed, the aforementioned networks are the highest-performing
baselines in the CLRS algorithmic benchmark
\citep{velickovic2022clrs}.

Finally, one significant aspect to consider is the termination
condition of the processor. Many algorithms, referred to as
\textit{iterative algorithms}, terminate after a given loop condition
is met. To align with these algorithms, it is essential to introduce
termination strategies in our processors. Usually, termination
requires the use of a a \textit{termination network} $T$ that uses a
sigmoid activation to determine when to halt iterations. This network
often takes the average of the most recent $n$ node or edge
representations as its input, as follows:
\begin{equation*}
  \evs^{(t)} = \sigma\left(T(\{\bar{\vh}_v \mid v \in \gV\})\right).
\end{equation*}
Here, $\sigma$ is the logistic function and $\bar{\vh}$ denotes the
averaged representation of the few iterations of node $v$. As a
result, $\evs^{(t)} \in [0, 1]$ indicating whether to terminate the
processor at step $t$. Thus, the termination network essentially
learns a stopping criterion. This strategy is reminiscent of the
IterGNN \citep{tang2020towards} approach.

Additionally, it is also feasible to apply the inherent termination
conditions of the target algorithm.  For example, Bellman-Ford
(similarly to other shortest path algorithm) terminate once the
solution ``converges'' (i.e., no ``new'' shortest paths are found
within two successive iterations). In such cases, we can
straightforwardly apply the same stopping criterion to our processor.
We conclude by noting that, whenever we have access to an
\textit{upper bound} on the maximum number of iterations (e.g.,
diameter of the graph), we can include such stopping condition in
our processor.

\paragraph{Decoder.}
Finally, the decoder $g: \sR^d \rightarrow \gY$ maps the embeddings
produced by the processor to the actual space of solutions of the
algorithm:
\begin{align*}
  &\vy^{(t)}_v = g_\theta(\vz^{(t)}_v, \vh^{(t)}_v, \vh^{(t-1)}_v)\\
  \\
  &\vy^{(t)}_{uv} = g_\theta(\vz^{(t)}_{uv}, \vh^{(t)}_{uv}, \vh^{(t-1)}_{uv}).
\end{align*}
Notably, the decoder is applied to the embeddings produced at every
step $t$, in order to obtain intermediate outputs (i.e.,
\texttt{hints}) that might be fed back to the encoder in the
subsequent step. This means that parts of the inputs $\vx^{(t)}$ may
include step-wise outputs $\vy^{(t-1)}$.

\subsection{Step-wise supervision}
In standard machine learning setups, the one objective is learning an
approximation to a hidden function $f: \sR^n \rightarrow \sR^m$. The
most typical way in which we pursue this objective is by learning a
stack of non-linear transformations (e.g., MLPs), imposing little or
no constraints on the intermediate transformations of the hidden
layers. This results in an optimisation problem described as:
\begin{equation*}
  \theta^* = \argmin_\theta \frac{1}{|\train|}\sum_{x \in |\train|}\gL(f(x), \net(x \mid \theta)),
\end{equation*}
where $\net$ represents the neural network and $\train$ represents
the training data.

Furthermore, the above learning problem has another implication: we
assume that we only want to supervise on
\texttt{input}-\texttt{output} pairs. Indeed, when we think of machine
learning and deep learning we tacitly imply that optimisation is only
performed in this manner, and understandably so. In fact, in the most
general case we have no indications of \textit{how} the target
function generated its \texttt{output} (i.e., we do not have access to
\texttt{hints} as for algorithms), that is, in fact, the learning
objective. This is the primary reason why we often imply that neural
networks learn from pairs of input-output data.

However, within the context of \textit{neural execution} of
algorithms, supervising a model under such conditions has historically
exhibited poor performance. Consider, for instance, the Neural Turing
Machines (NTM) \citep{graves2014neural} -- an early effort to teach
classical algorithms to neural networks. This model is equipped with
an auxiliary differentiable memory and are supervised solely on
input-output pairs of the target algorithmic functions. The neural
network is then tasked to learn the dynamics of \textit{reading} and
\textit{writing} to this memory, without explicit guidance on its
management as per the target algorithms. This usually leads to high
variance in the results, as we have few guarantees that the neural
network will learn to use the memory effectively.

A straightforward approach to tackle this problem is to show the
network the step-by-step process of how target algorithms interact
with data structures and input data. This is the idea behind
\textit{step-wise supervision}, illustrated for the first time in
\citep{velickovic2019neural}. In essence, having access to these
manipulations enables us to supervise on each of them, effectively
including an auxiliary supervision signal for every step of the
algorithm's trajectory. Following the notation of
\autoref{sec:algo-notation}, we refer to these intermediate
manipulations as \texttt{hints}. A mathematical formulation is
provided as follows:
\begin{align} \label{eq:stepwise-supervision-loss}
  \begin{split}
    \theta^* = \argmin_\theta
    &\frac{1}{|\train|} \sum_{x \in |\train|}
      \Big(\gL(A(x), \net(x \mid \theta)) +\\
    &\sum_{h \in \texttt{hints}} \sum^{T-2}_{t=0}
      \gL(A_t(h^{(t)}), \net_t(h^{(t)} \mid \theta))\Big),
  \end{split}
\end{align}
where $A$ is the target algorithm and $h^{(t)}$ represents a possible
\texttt{hint} at step $t$ of the algorithm, with $A_t$ and $\net_t$
indicating the $t$-th step of the algorithm and the algorithmic
reasoner. Including these additional terms to the loss function, thus,
effectively constrains the network to behave more closely to the
target algorithm at each intermediate step $t$. This also lets the
network disambiguate across similar algorithms. In fact, consider
the set of sorting algorithms that, given an input list, returns
a sorted version of the list according to some metric. If we ignore
the intermediate \texttt{hints}, all sorting algorithms become
indistiguishable to one another. Hence, including supervision on
\texttt{hints} is sometimes crucial to align correctly to the dynamics
of a target algorithm.

\section{Learnt algorithms as inductive priors}
\label{sec:transfer}
As introduced in the previous section, an intriguing reason to employ
neural algorithmic reasoning is to adapt algorithms for handling
real-world, noisy inputs (i.e., \textit{natural} inputs). Broadly
speaking, our goal is to learn a mapping from one natural space
$\tilde{\sA}$ to another $\tilde{\sB}$, where this mapping is of
algorithmic nature. However, algorithms operate in an abstracted
version of these natural spaces (i.e., \textit{abstract} space). This
abstract space is what we typically use when teaching algorithms to
machines. Now, denote the learning task on natural spaces as the
\textit{target} task $\gT$, and the task of learning on abstract
spaces as the \textit{base} task $\gB$. In practice, $\gB$ may be any
algorithmic tasks, such as finding the shortest path from point A to
point B or find the maximum flow within a flow network. The target
task $\gT$, in contrast, might be the exact same shortest path problem
but with natural data, or any problem correlated to algorithms (e.g.,
NCO), as per our discussion in \autoref{sec:nar-motivation}. Thus, the
assumption is that $\gT$ and $\gB$ share mutual information and hence
knowledge of $\gB$ may be a valuable inductive prior when solving
$\gT$.

Now, in order to infuse $\gT$ with knowledge of $\gB$ we generally
employ a two-step learning scheme:
\begin{enumerate}
\item Given synthetically generated inputs and ground truths in the
  abstract space, learn the parameters of an EPD architecture to
  execute an algorithm (i.e., solve $\gB$) by optimising
  \eqref{eq:stepwise-supervision-loss}, where
  $net = g \circ p \circ f$. Here, we put emphasis on the fact that
  the EPD model effectively \textit{decouples} the processor network
  $p$ from raw inputs and outputs, with the encoder and decoder
  networks handling the necessary \textit{pre}- and
  \textit{post}-processing (i.e., mapping to and from $p$'s latent
  space). Furthermore, as detailed in \autoref{sec:epd}, $f$ and $g$
  contribute little to the learning of the algorithmic function, being
  implemented as simple linear projections. This ensures that, at the
  end of the optimisation, we end up with all knowledge of $\gB$
  encoded inside the processor's parameters. This allows for easily
  extraction of the processor and effective transfer learning during
  step 2.

\item Discard $f$ and $g$ and apply any suitable transfer learning
  \citep{torrey2010transfer} method to $p$ to learn the target task
  $\gT$.
\end{enumerate}
Through this simple learning scheme, we are able to re-use the
processor $p$ for solving the new task $\gT$. We now define how
we can effectively implement step 2. Consider
$\tilde{\vx}, \tilde{\vy} \sim \gT$ to be input and ground-truth
samples from $\gT$, \cite{xhonneux2021transfer} discuss different
transfer learning settings to practically implement step 2, which we
report in the following along others:
\begin{itemize}
\item \textbf{Pre-train and freeze (PF)}: freeze the parameters of
  $p$. Thus, instantiate new encoders and decoders $\tilde{f}$,
  $\tilde{g}$ to process $\tilde{\vx}$ and $\tilde{\vy}$, and learn their
  parameters by minimising an appropriate loss function $\gL_\gT$:
  \begin{equation*}
    \argmin_\phi \gL_\gT\left(\tilde{\vy},
      (\tilde{g} \circ p \circ \tilde{f})(\vx \mid \phi)\right),
  \end{equation*}
  where $\phi$ are the learnable parameters of the new encoders and
  decoders. This settings is used, for instance, in
  \citep{deac2021neural, numeroso2023dual}.

\item \textbf{Pre-train and fine-tune (PFT)}: unlike the above
  approach, here we let $p$'s parameters change during step 2 (i.e.,
  $\phi$ include parameters of the processor in the above equation).

\item \textbf{2-processor transfer (2PROC)}: here, the two
  aforementioned methods are combined. Precisely, we instantiate a new
  processor $\tilde{p}$ whose parameters are free to change, whereas
  parameters of the pre-trained processor $p$ are kept frozen.
  Activations of the two processors are then combined together through
  any aggregator function prior application of the decoder $\tilde{g}$.
  Suppose $\vh_v$, $\tilde{\vh}_v$ to be the activations of the frozen
  processor $p$ and the new processor $\tilde{p}$, respectively.
  Inferred values of $\tilde{\vy}_v$ are then obtained as:
  \begin{equation*}
    \tilde{\vy}_v = \tilde{g}(\tilde{\oplus}(\vh_v, \tilde{\vh}_v)),
  \end{equation*}
  where $\tilde{\oplus}$ can be a learnable MLP or any simple functions
  such as summation.

\item \textbf{Multi-task learning (MTL)}: here, the two tasks $\gB$
  and $\gT$ are solved together by means of one pair of encoder and
  decoder networks for the task $\gB$ (i.e., $f_\gB$,$g_\gB$) and one
  for $\gT$ (i.e., $f_\gT$, $g_\gT$). These encoders and decoders
  share the same processor network $p$, which is trained to execute
  the two tasks simultaneuosly. As a result, the transfer is implicit
  during step 1 of the above learning scheme.
\end{itemize}
Through the above transfer learning methods, we have everything we
need to successfully apply a learnt $\gB$ reasoner on a new task
$\gT$.  In other words, we illustrated how we can reason
algorithmically on natural inputs, by carefully pre-train a processor
$p$ to imitate an algorithm. Intuitively, freezing (or not) of the
processor regulates how strong of an algorithmic bias we want to
infuse $\gT$ with. Furthermore, $\tilde{f}$ and $\tilde{g}$
effectively replace the manual feature engineering surpassing the
algorithmic bottleneck issue discussed in the previous sections, as
both $\tilde{f}$ and $\tilde{g}$ learn to map to and from a
multi-dimensional space (i.e., the processor $p$'s latent space) and
are forced to re-use the algorithmic knowledge lying inside $p$.  Note
that, while we usually require encoders and decoders to be simple
linear transformations when learning $\gB$, such constraint is not
only not needed when transferring to $\gT$, but can also limit the
performance in case a \textit{non}-linear reasoning is required on the
algorithmic steps. Indeed, especially when applying \textit{pre-train
  and freeze} we can usually implement $\tilde{f}$ and $\tilde{g}$ as
more powerful MLPs or any suitable neural networks.
\cite{velickovic2022reasoning} exemplify this by pre-training an
algorithmic reasoner to perform physical simulations on tabular data,
and then deploy encoders and decoders as deep convolutional neural
networks \citep{lecun1995convolutional} when applying these learnt
simulators to images.

As a last note, let us put our focus on the \textit{multi-task
  learning} setting. There, we discussed how we can use the same
processor for learning two tasks at the same time. This concept can be
generalised to tackle an arbitrary number of tasks. This means that,
for instance, we can actually learn multiple algorithms at the same
time by having several independent pairs $\tuple{f_{A_i}, g_{A_i}}$
implementing the necessary mapping from the abstract spaces of the
different algorithms $\gA = \set{A_1, \ldots, A_n}$. Introduced in
\citep{xhonneux2021transfer}, this learning modality shows to be
beneficial for OOD generalisation as diverse algorithms may share
common sub-routines that the neural algorithmic reasoner can
successfully identify and exploit during
learning. \cite{ibarz2022generalist} scale this idea to map all 30
algorithms of the CLRS benchmark \citep{velickovic2022clrs} to the
same processor's parameter space, with performance comparable to a
network trained exclusively on one single algorithm. Within this
thesis, we utilise this technique when learning dual algorithms in
\autoref{sec:dar}.

%% file: gfx/nar/epd.tex
\begin{tikzpicture}[very thick]

  \tikzset{line/.style={draw,line width=1.5pt}}
  \tikzset{arrow/.style={->,>=stealth}}
  \tikzset{snake/.style={arrow,line width=1.3pt,decorate,decoration={snake,amplitude=1,segment length=4,post length=5}}}

  \tikzstyle{box}=[dash pattern=on 5pt off 2pt,inner sep=5pt,rounded corners=3pt]
  \tikzstyle{vertex}=[circle,fill=cyan!50!,draw,minimum size=15pt, inner sep=1pt]

  \tikzstyle{message}=[arrow, decorate,
  decoration={snake,
    amplitude=1,
    segment length=6,
    post length=7
  },
  line width=2.5pt]

  \node[rectangle, draw, dashed, very thick, black!30, minimum height=85pt, minimum width=90pt] (Xbar) at (-5.5, 0) {};

  \node[rectangle, draw, dashed, very thick, black!30, minimum height=85pt, minimum width=90pt] (Ybar) at (5.5, 0) {};

  \node[rectangle, draw, dashed, very thick, black!30, minimum height=85pt, minimum width=90pt, fill=blue, fill opacity=0.05] (P) at (0, 0) {};

  \node[rectangle, draw, dashed, very thick, black!30, minimum height=90pt, minimum width=90pt, below=3.4em of Xbar] (X) {};

  \node[rectangle, draw, dashed, very thick, black!30, minimum height=90pt, minimum width=90pt, below=3.4em of Ybar] (Y) {};

  \node[below=0em of Xbar] (xbl) {\emph{Abstract} inputs, $\bar{x}$};

  \node[below=0em of X] (xl) {\emph{Natural} inputs, $x$};

  \node[below=0em of Ybar] (ybl) {\emph{Abstract} outputs, $\bar{y}\approx A(\bar{x})$};

  \node[below=0em of Y] (yl) {\emph{Natural} outputs, $y$};

  \draw[line width=3.5pt, -stealth, blue, opacity=0.7] (Xbar) -- node[above] {\LARGE $f$} (P);
  \draw[line width=3.5pt, -stealth, blue, opacity=0.7] (P) -- node[above] {\LARGE $g$} (Ybar);

  \draw[blue] (P) edge[loop above=1, line width=3.5pt, looseness=5, stealth-, opacity=0.7] node [above, opacity=1.0] {\LARGE $p$} (P);

  \draw[purple] (X) edge[bend right, line width=3.5pt, -stealth, opacity=0.9] node[below,xshift=0.1em] {\LARGE $\tilde{f}$} (P);

  \draw[purple] (P) edge[bend right, line width=3.5pt, -stealth, opacity=0.9] node[below,xshift=-0.1em] {\LARGE $\tilde{g}$} (Y);

  \foreach \pos/\name/\lab in {{(0,0)/m1/$v_1$}, {(-1.11,-0.13)/m2/$v_2$}, {(-0.26,1)/m3/$v_3$}, {(0.9,0.47)/m4/$v_4$}, {(0.31,-1)/m5/$v_5$}}
  \node[vertex] (\name) at \pos {\lab};

  \foreach \pos/\name/\nb in {{(-5, -1)/a/0},{(-6, -1)/b/$\infty$},{(-6, 0)/c/$\infty$},{(-5, 0.25)/d/$\infty$},{(-5, 1)/e/$\infty$}}
  \node[vertex] (\name1) at \pos {\nb};

  \foreach \pos/\name/\nb in {{(6 , -1)/a/0},{(5, -1)/b/1},{(5, 0)/c/1},{(6, 0.25)/d/1},{(6, 1)/e/2}}
    \node[vertex] (\name2) at \pos {\nb};

  % \foreach \name/\nb in {{a/$\infty$},{b/$0$},{c/$\infty$},{d/$\infty$},{e/$\infty$}}
  % \node[vertex, below = 3.75cm of \name1] (\name3) {$x_\name$};

  \foreach \pos/\name in {{(-6, -4.5)/a},{(-6.12, -5.5)/b},{(-4.44, -4.2)/c},{(-5.28, -3.98)/d},{(-5.21, -5.41)/e}, {(-4.3, -4.8)/f}, {(-6.33, -3.75)/g}, {(-5.1, -3.25)/h}}
    \node[vertex] (\name3) at \pos {\scriptsize$\vec{x}_\name$};

  \foreach \name in {a,b,c,d,e,f,g,h}
    \node[vertex, right = 25.5em of \name3] (\name4) {\scriptsize$\vec{y}_\name$};

  \draw (m2) edge (m1);
  \draw (m3) edge (m1);
  \draw (m4) edge (m1);
  \draw (m5) edge (m1);
  \draw (m2) edge[message, cyan, opacity=0.7] (m1);
  \draw (m3) edge[message, cyan, opacity=0.7] (m1);
  \draw (m4) edge[message, cyan, opacity=0.7] (m1);
  \draw (m5) edge[message, cyan, opacity=0.7] (m1);

  \draw (a1) edge (b1);
  \draw (a1) edge (c1);
  \draw (a1) edge (d1);
  \draw (b1) edge (c1);
  \draw (c1) edge (d1);
  \draw (d1) edge (e1);

  \draw (a2) edge (b2);
  \draw (a2) edge (c2);
  \draw (a2) edge (d2);
  \draw (b2) edge (c2);
  \draw (c2) edge (d2);
  \draw (d2) edge (e2);

  % \draw[red, line width=3.5pt, opacity=0.7, stealth-] (a2) edge (e2);
  % \draw[red, line width=3.5pt, opacity=0.7, -stealth] (a2) edge (b2);
  % \draw[red, line width=3.5pt, opacity=0.7, stealth-] (c2) edge (d2);
  % \draw[red, line width=3.5pt, opacity=0.7, stealth-] (b2) edge (f2);
  % \draw[red, line width=3.5pt, opacity=0.7, -stealth] (c2) edge (f2);
  % \draw[red, line width=3.5pt, opacity=0.7, -stealth] (g2) edge (f2);
  % \draw[red, line width=3.5pt, opacity=0.7, stealth-] (g2) edge (h2);

  \draw (a3) edge (b3);
  \draw (a3) edge (c3);
  \draw (a3) edge (d3);
  \draw (b3) edge (c3);
  \draw (b3) edge (e3);
  \draw (c3) edge (d3);
  \draw (c3) edge (f3);
  \draw (d3) edge (g3);
  \draw (f3) edge (e3);
  \draw (g3) edge (h3);
  \draw (h3) edge (c3);

  \draw (a4) edge (b4);
  \draw (a4) edge (c4);
  \draw (a4) edge (d4);
  \draw (b4) edge (c4);
  \draw (b4) edge (e4);
  \draw (c4) edge (d4);
  \draw (c4) edge (f4);
  \draw (d4) edge (g4);
  \draw (f4) edge (e4);
  \draw (g4) edge (h4);
  \draw (h4) edge (c4);

\end{tikzpicture}

%% file: content/mindp.tex
\part{Contribution}
\chapter{Learning Dynamic Programming}\label{ch:contribution-dp}
This chapter contains the first core contributions of this thesis.
Herein, we target learning of \emph{min}-aggregated dynamic
programming (Min-DP) algorithms, from both theoretical and practical
standpoints.

In the first section, \textit{Approximating min-DP algorithms}, we
draw connections between the field of \textit{tropical algebra}, where
DP algorithms are commonly thought to ``live'' in
\citep{maclagan2009introduction}, and classical euclidean spaces
(i.e., where neural networks operate).  Such connection will provide
evidence that GNs can approximate Min-DP algorithms up to a arbitrary
precision.

In the second section, \textit{Planning via learnt heuristics},
we show how we can learn a heuristic function to be used in a path
planning problem, which can be framed under dynamic programming. Such
heuristic will allow for more effective path discovery, especially
when deployed in partially observable environments.

\section{Approximate min-DP algorithms}
\label{sec:tropical-gnn}
By now, it should be clear that GNs and algorithms computation
dynamics share similarities. On this line, \cite{dudzik2022graph}
attempt at formalising such similarities through category theory
\citep{awodey2010category}, particularly by employing \textit{integral
  transforms} -- a popular construct in category theory. This integral
transform provides a unified description of both dynamic programming
algorithms and computation of graph networks, effectively drawing an
equivalence between the two at a macro-level.

In this contribution, we take a different route and attempt at drawing
connections between DP algorithms and GNs by leveraging
\textit{tropical algebra} (see \autoref{sec:dp}). In essence, tropical
algebra is a ``degeneration'' of Euclidean algebra which is better
``aligned'' to the dynamic programming paradigm. Precisely, we draw
connections between two \textbf{semirings} (see \autoref{sec:dp}),
that of \textit{real numbers} (i.e., ``Euclidean'' semiring) and the
tropical \textbf{min-plus} semiring. Both the real and min-plus
semirings are particular semirings, which can be represented
mathematically as $\tuple{\sR, +, \cdot, 0, 1}$ and
$\tuple{\sT = \sR \cup \set{+\infty}, \min, +, +\infty, 0}$. The
former semiring encompasses classical addition and multiplication
operations on the set of real numbers $\sR$. Neural networks, in
general, operate within this structure.  In the following, we will
always refer to the tropical min-plus semiring as simply the
``tropical semiring''. Within this algebraic structure, many dynamic
programming algorithms find a precise description
\citep{maclagan2009introduction}, as we also showed in
\autoref{sec:dp}.

In order to show that GNs can effectively approximate DP algorithms up
to arbitrary precision, we proceed by ``reconciling'' these two
semirings. We provide practical proof examples for reachability and
shortest path algorithms, in particular. Precisely, we consider inputs
and ouputs coming from tropical sets and formulate a suitable EPD
architecture with a GN processor that can seamlessly process such
data, effectively casting GNs under the framework of tropical
algebra. In particular, we implement encoders and decoders as
\textit{Maslov} quantisation maps \citep{litvinov2007maslov} and use a
\emph{sum}-aggregated graph network (i.e., Graph Isomorphism Network
\citep{xu2019powerful}) as the processor of such architecture. We
will show that, by employing this architecture, GNs can effectively
approximate \emph{min}-aggregated DP algorithms up to an arbitrary
precision \citep{landolfi2023tropical}.

\subsection{Aligment of encoders and decoders}
\label{sec:tropical-enc-dec-alignment}
Suppose we receive inputs in the form of ``tropical'' objects $\sT$.
Standard EPD pipelines (with linear encoding and decoding functions)
are unable to directly process such inputs, due to the presence of
possible infinities in the input (recall that we consider
$\sT = \sR \cup \set{+\infty}$).

To overcome this limitation, the EPD architecture considered in
this section comprises specialised encoders and decoders, known
as Maslov quantisation and dequantisation maps
\citep{gunawardena1998correspondence, brugalle2000brief}
$q_h: \sT \to \sR_+$ and $d_h: \sR_+ \to \sT$, with $\sR_+$ being
the set of non-negative real numbers. A precise mathematical
formulation is as follows:
\begin{align}\label{eq:quantization}
  q_h(x) &=
  \begin{cases}
    0 & \text{if } x = \infty \\
    e^{-x/h} & \text{otherwise};
  \end{cases}\\
  d_h(x) &=
  \begin{cases}
    \infty & \text{if } x=0 \\
    -h \ln x & \text{otherwise}.
  \end{cases},
\end{align}
where $h > 0$ is a pre-defined constant that regulates the smoothness
of the quantisation and dequantisation operators. Now, we have a
natural map from tropical objects to real numbers and
viceversa. However, such mappings entail a stronger connection between
the min-plus semiring and the semiring of real numbers. In fact, $q_h$
and $d_h$ can be actually shown to be semiring \textit{homomorphisms}
between $\sT$ and $\sR_+$.
\begin{proposition}
  $q_h: \sT \rightarrow \sR_+$ is a semiring homomorphism from
  $\tuple{\sT, \min, +, +\infty, 0}$ to $\tuple{\sR_+, +, \cdot, 0, 1}$
  for $h \rightarrow 0^+$.  As such, $q_h$ satisfies:
  \begin{enumerate}
    \item $q_h(+\infty) = id_{\tuple{\sR_+, +}}$.
    \item $q_h(0) = id_{\tuple{\sR_+, \cdot}}$.
    \item
      $\forall a, b \in \sT \,:\, q_h(a \odot_\sT b) = q_h(a + b) =
      q_h(a) \cdot q_h(b)$.
    \item
      $\forall a, b \in \sT \,:\, q_h(a \oplus_\sT b) = q_h(\min(a,b))
      = q_h(a) + q_h(b)$.
  \end{enumerate}
\end{proposition}
\begin{proof}
  We will prove that $q_h$ satisfies each of the above properties:
  \begin{enumerate}
  \item By definition of $q_h$, we have that
    $q_h(+\infty) = 0 = id_{\tuple{\sR_+, +}}$.
    \item By definition of $q_h$, we have that
      $q_h(0) = e^{0/-h} = 1 = id_{\tuple{\sR_+, \cdot}}$.
    \item Given $a, b \in \sT$, we have that
      $q_h(a \odot_\sT b) = e^{\frac{-(a+b)}{h}} = e^{\frac{-a}{h} +
        \frac{-b}{h}} = e^{\frac{-a}{h}} \cdot e^{\frac{-b}{h}} =
      q_h(a) \cdot q_h(b)$
    \item Given $a, b \in \sT$, consider
      $\lim_{h \to 0^+} q_h(a \oplus_\sT b) = \lim_{h \to 0^+}
      q_h(\min(a,b)) = 0 = \lim_{h \to 0^+} \left[ q_h(a) + q_h(b)
      \right]$.
  \end{enumerate}
\end{proof}
\begin{proposition}
  $d_h: \sR_+ \rightarrow \sT$ is a semiring homomorphism from
  $\tuple{\sR_+, +, \cdot, 0, 1}$ to $\tuple{\sT, \min, +, +\infty, 0}$
  for $h \rightarrow 0^+$. As such, $d_h$ satisfies:
  \begin{enumerate}
    \item $d_h(0) = +\infty = id_{\tuple{\sT, \min}}$.
    \item $d_h(1) = 0 = id_{\tuple{\sT, +}}$.
    \item $\forall a, b \in \sR_+ \,:\, d_h(a \odot_{\sR_+} b) =  d_h(a \cdot b) = d_h(a) + d_h(b)$.
    \item $\forall a, b \in \sR_+ \,:\, d_h(a \oplus_{\sR_+} b) = d_h(a + b) = \min(d_h(a), d_h(b))$.
 \end{enumerate}
\end{proposition}
\begin{proof}
  The proof for $d_h$ is analogous to the previous one:
  \begin{enumerate}
  \item By definition of $d_h$:
    $d_h(0) = +\infty = id_{\tuple{\sT, \min}}$.
    \item $d_h(1) = -h \ln 1 = 0 = id_{\tuple{\sT, +}}$.
    \item Given $a, b \in \sR_+$, we have that
      $d_h(a \cdot b) = -h \ln(ab) = -h \ln a + (-h \ln b) = d_h(a) +
      d_h(b)$
    \item Given $a, b \in \sR_+$, consider
      $\lim_{h \to 0^+} d_h(a + b) = \lim_{h \to 0^+} \left[-h
        \ln(a+b)\right] = 0 = \lim_{h \to 0^+} \min(d_h(a), d_h(b))$.
  \end{enumerate}
\end{proof}
We have established that both $q_h$ and $d_h$ function as
homomorphisms between the two semirings for arbitrarily ``small''
values of $h$. Furthermore, it can be shown that $d_h = q^{-1}_h$ by
observing that $\ln = \exp^{-1}$ and viceversa. Since $q_h$ admits an
inverse, $q_h$ and $d_h$ are, in reality, \textbf{isomorphisms}.
Consequently, any algebraic structure in $\sR_+$ and any algebraic
structure in $\sT$ are isomorphic.  Of course, homomorphisms and
isomorphisms are only valid as $h$ approaches $0$. For ``sufficiently
high'' values of $h$, in fact, addition (condition 4 of homomorphism)
may not be preserved.  Consequently, $q_h$ and $d_h$ effectively map
to and from a family of semirings, where the specific semiring to
which they map is contingent upon the chosen value of $h$. The closer
$h$ is to $0$, the closer $q_h$ maps to the $\sR_+$ semiring, and
viceversa.  Despite this observation, there is still tight connection
between the operations performed in the $h$-semiring and its
``tropical'' counterpart. In fact, consider the following generalised
addition and multiplication:
\begin{align*}
  a \oplus_h b &=  d_h(q_h(a) + q_h(b)) \quad\text{and}\\
  a \odot_h b &= d_h(q_h(a) \cdot q_h(b)).
\end{align*}
From the above formulations, it follows that $\forall a, b \in \sT$:
\begin{align}
  \begin{split}
    a \odot_h b &= d_h(q_h(a) \cdot q_h(b)) \\
                &= -h \ln (e^{-a/h}e^{-b/h}) \\
                &= -h \ln e^{-(a+b)/h} \\
                &= a + b = a \odot_\sT b,
  \end{split}
\end{align}
and
\begin{align}
  \label{eq:h-plus}
  \begin{split}
    a \oplus_h b &= \lim_{h \to 0^+}  d_h(q_h(a) + q_h(b))\\
                 &= \lim_{h \to 0^+}-h \ln (e^{-a/h} + e^{-b/h})\\
                 &=\min(a, b) = a \oplus b.
  \end{split}
\end{align}
Analogously, matrix operations also approximate the generalized dot
product in the tropical semi-ring:
\begin{equation*}
  \lim_{h \to 0^+} d_h[q_h[\mW]^k] = \mW^{\odot_\sT k}
\end{equation*}
where $\mW$ is a matrix and $\mW^k$ symbolyses $\mW$ to the power of
$k$ in the $h$-semiring, $\mW^{\odot_\sT k}$ is its tropical
counterpart and $f[\mW]$ denotes element-wise application of $f$ on
the matrix $\mW$.  ``Tropical'' matrix multiplications are the key
points of our discussion. In fact, through matrix multiplications in
the tropical space we can actually represent entire
\textit{algorithms}. Examples of this property are provided in
\autoref{sec:dp} (i.e., BFS and Bellman-Ford).  Furthermore, the above
formulations explicitate a very important observation: operations
performed in the $h$-semiring ``approximate'' those executed in the
tropical semiring. Thus, we have forged a substantive connection
between the $h$-semiring and the tropical one. Such connection
degenerates into a bijection between the $\sR_+$ semiring and $\sT$
whenever $h$ is sufficiently small. This implies that operations
executed within the $\sR_+$ (or $\sR$) semiring must align with their
counterparts in the ``tropical'' semiring. Furthermore, considering
that tropical operations can represent ``algorithmic steps'',
operations within $\sR_+$ equivalently correspond to these algorithms.

\autoref{fig:tropical-transform} positions all of this within the
context of the EPD architecture and elucidates the various
transformations between the two semirings. In particular, the figure
shows that we have successfully implemented an encoder
$f: \sT \rightarrow \sR_+$ (i.e., $q_h$) which maps from the $\sT$
semiring -- where we supposedly receive DP inputs -- to the $\sR$
semiring, where our processor network $p$ naturally operates, and a
decoder $g = d_h$ that maps back to $\sT$ for computing DP outputs.
Our processor, then, is taught to approximate tropical steps in the
$h$-semiring. Potentially, this enables our processor to approximate
steps of algorithms as well. As a consequence, by showing that a
processor network $p$ can indeed approximate $\sT$ operations up to
arbitrary precision, we will successfully show that $p$ can also
approximate related algorithmic steps arbitrarily well. Note that in
the diagram $p$ should leverage conventional addition and
multiplication. Consequently, supposing $p$ is a graph network, we
should choose a proper architecture reflecting this property. We
explore a suitable choice for our processor in the next section.

\begin{figure}
  \centering
  \input{gfx/contrib/tropical/transform}
  \caption{Category-theory inspired diagrams showcasing the
    equivalences $\min_\sT \equiv +_\sR$ (left-hand diagram) and
    $+_\sT \equiv \cdot_\sR$ (right-hand diagram). Specifically, in
    the left-hand diagram, starting from inputs $[a, b]$ in the
    $\sT$-semiring, we can derive $c=\min(a,b)$ by following two
    paths. First, apply the $\min$-morphism to obtain $c$ directly.
    Second, map $[a,b]$ to $\sR$ through Maslov quantisation $q_h$ and
    perform an update step through the +-morphism to obtain
    $z$. Ultimately, $c$ can be decoded by application of $d_h$ to
    $z$. The same also holds for the right-hand diagram (i.e.,
    equivalence between tropical addition and real multiplication).
    The two diagrams together forms the bijection
    $\tuple{\sT, \min, +, +\infty, 0} \xleftrightarrow{} \tuple{\sR,
      +, \cdot, 0, 1}$. Finally, $f, p$ and $g$ refer to the EPD
    architecture and elucidate how the different EPD component are
    instantiated to enable mapping between the two
    semirings.}\label{fig:tropical-transform}
\end{figure}
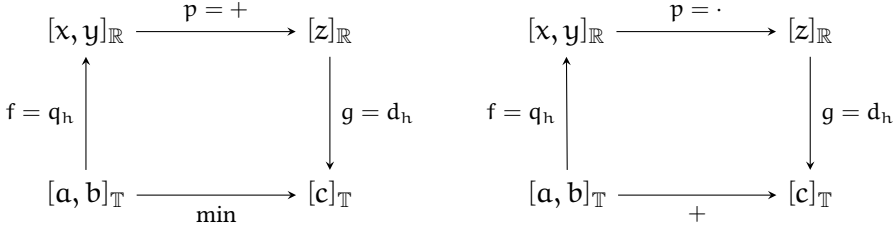

\subsection{Processor alignment and DP approximation}
As thoroughly discussed in \autoref{sec:algorithmic-alignment}, in
the context of NAR, we typically choose graph network processors
with $\max$ aggregators. Here, however, we should carefully consider
that we are executing message-passing steps in an $h$-semiring, where
inputs have appropriately been parsed through Maslov quantisation
and dequantisation maps. As evident from equation \eqref{eq:h-plus},
a $\min$ (or $\max$) operation corresponds to a sum operation in the
$h$-semiring. Hence, we should consequently choose a GN architecture
with a sum aggregator, rather than the typically recommended choice
of $\max$.
For such reasons, the Graph Isomorphism Network (GIN) is the natural
choice:
\begin{equation}\label{eq:gin-mp}
  \vh^{(\ell+1)}_{v} = \mlp\left(\left(1+\epsilon\right)
    \vh_u^{\ell}+\sum_{u \in \gN(v)} \vh_u^{\ell}\right),
\end{equation}
or, equivalently in matrix form:
\begin{equation}\label{eq:gin-mp-matrix}
  \gin(\mA, \mH^{(\ell)}) = \mH^{(\ell + 1)} = \mlp\left((1+\epsilon)\mH^{(\ell)}
    + \mA\mH^{(\ell)}\right).
\end{equation}
Here, $\epsilon$ is a parameter that regulates the importance of self
information in the message passing scheme. Note that we subtly assume
that the original $\mA$ contains no self loops (self loops will be
added directly by the GIN layer). Now, consider a simple EPD
architecture with a GIN processor:
\begin{equation*}
  \vy^{(k)} = (g \circ \gin_\epsilon^k \circ f)(\cdot).
\end{equation*}
The following propositions will show how this simple architecture
can model and learn exemplar DP algorithms up to arbitrary precision.

\paragraph{Reachability approximation (\autoref{prob:reachability}).}
Consider an unweighted graph $\tuple{\gV, \gE}$ and consider its
adjacency matrix $\mA$ and a vector $\chi_v \in \{0, 1\}^{|\gV|}$ such
that $\chi_v = 1$ for a given $v \in \gV$ and $\chi_{u \neq v} = 0$.
Additionally, let the encoder $f$ be the identity function $f(x) = x$
and let $g = \mlp \circ d_1$, with $d_1$ defined as in
\eqref{eq:quantization}.  Consider $\gin_0$ an instation of GIN with
$\epsilon = 0$. Thus, $\gin_0$ can effectively learn to perform BFS
starting from node $v$ (that is, identifying the nodes that are
reachable in the $k$-hop neighbourhood of $v$, $\gN^k(v)$).
\begin{proposition}\label{th:reachability}
  Consider $\mA$ and $\chi_v$, a positive constant $c > 0$, and
  $\vd = \mlp(d_1[\gin_0^k(\mA, \chi_v)])$, with $\gin_0$ and $\mlp$
  having any number of layers with $\relu$ activations. Thus, there
  exists a configuration of their parameters such that
  $|\vd_u - \delta_u(\gN^k(v))| \leq c$ for any $u \in \gV$, where
  $\delta_u$ is the Dirac delta.
\end{proposition}
\begin{proof}
  Fix $w_\theta = 1$ and $b_\theta = 0$ for every layer of
  $\gin$/$\mlp$, if not stated otherwise. Since $\relu$ acts as an
  identity function for non-negative elements, we have that
  $\gin_0^k(\mA, \chi_v) = (\mA + \mI)^k \chi_v$. Let
  $\mW \in \sT^{|\gV|\times |\gV|}$ a matrix defined as in
  \eqref{eq:tropical-distance-matrix}, with $\vd_{uv} = 0$ for all
  $(u, v) \in \gE$.  For any $h > 0$ we have that
  \begin{align*}
    Q_h[\mW] &= \mA + \mI\\
    Q_h[\chi_v] &= \chi_v.
  \end{align*}
  If we set $w_{\theta_1} = h$ in the first layer of $\mlp$, we have
  that
  \begin{align*}
    h\cdot d_1[q_h[\mW]^k \cdot q_h[\chi_v]] &= d_h[q_h[\mW]^k \cdot q_h[\chi_v]]\\ &\xrightarrow[h \to 0^+]{} \mW^{\odot k} \odot \chi_v = \vd_\circ,
  \end{align*}
  where $[\vd_\circ]_u = 0$ if $u \in \gN^k(v)$ and
  $[\vd_\circ]_u = \infty$ otherwise.  By setting, instead,
  $w_{\theta_1} = -h$ and $b_{\theta_1} = 1$ in the first layer of
  $\mlp$, we obtain
  \begin{align*}
    \lim_{h \to 0} \vd &= \lim_{h\to 0} \sigma[1 - h\cdot D_1[(\mA+\mI)^k \chi_v]]\\
    &= \max[0, 1 - \vd_\circ]\\
    &= \delta(\gN^k_+(v)).
  \end{align*}
  For a $h > 0$ small enough we have that
  \begin{itemize}
  \item for $u \in \gN^k_+(v)$, $\lvert \vd_u - \delta_u(\gN^k_+(v))\rvert < c$ by definition of limit, and
  \item for $u \not\in \gN^k_+(v)$, $[\vd_\circ]_u \gg 1$ and hence $\vd_u = \delta_u(\gN^k_+(v)) =0$,
  \end{itemize}
  thus reaching the conclusion.
\end{proof}

\paragraph{Unweighted shortest path (\autoref{prob:shortest-path}).}
Consider the same set of inputs as for the reachability problem in the
previous paragraph. The following proposition shows that
$\gin_\epsilon$ can compute the shortest paths in an unweighted graph
(or, equivalently, a equally weighted graph). Note that here we are
actually approximating the Bellman-Ford algorithm.
\begin{proposition}
  Consider $\mA$ and $\chi_v$, a positive constant $c > 0$, and
  $\vd = \mlp(d_1[\gin_\epsilon^k(\mA, \chi_v)])$, with
  $\gin_\epsilon$ and $\mlp$ having any number of layers with $\relu$
  activation function, there exist a configuration of their parameters
  such that $|\vd_u - d(u, v)| \leq c$ for any
  $u \in \gN^k(v)$.
\end{proposition}
\begin{proof}
  Fix $w_\theta = 1$ and $b_\theta = 0$ for every layer of
  $\gin$/$\mlp$, if not stated otherwise.  Let
  $\mW \in \sT^{n\times n}$ a matrix defined as in
  \eqref{eq:tropical-distance-matrix}, with $\vd_{uv} = 1$ for all
  $(u, v) \in \gE$.
For any $h > 0$ we have that
\begin{align*}
  Q_h[\mW] &= e^{-1/h} \cdot \mA + \mI,\\
  Q_h[\chi_v] &= \chi_v.
\end{align*}
By fixing $\epsilon = e^{1/h} - 1$, and $w_{\theta_1} = e^{-1/h}$ in
the first layer of $\gin$, we have that
\begin{align*}
  \gin_\epsilon^k(\mA, \chi_v) = e^{-1/h} \cdot ((1 + \epsilon) \cdot \mI + \mA)^k \chi_v = (\mI + e^{-1/h}\cdot \mA)^k \chi_v,
\end{align*}
and, by setting also $w_{\theta_1} = h$ in the first layer of $\mlp$,
we have that
\begin{align*}
  \vd &= h\cdot D_1[Q_h[\mW]^k \cdot Q_h[\chi_v]] \\
      &= D_h[Q_h[\mW]^k \cdot Q_h[\chi_v]] \\
      &\xrightarrow[h \to 0^+]{} \mW^{\odot k} \odot \chi_v = \vd_\circ,
\end{align*}
where
\begin{align}\label{eq:dist-vector}
    [\vd_\circ]_u = \begin{cases}
    d(u,v) & \text{if } u \in \gN^k(v) \\
    \infty & \text{otherwise}.
    \end{cases}
\end{align}
By definition of limit (for a $h > 0$ small enough) we reach the
conclusion.
\end{proof}

\paragraph{Weighted shortest path (\autoref{prob:shortest-path}).}
Contrarily to the previous examples, consider a weighted graph
$\tuple{\gV, \gE, \gW}$ with its tropical distance matrix
$\mW \in \sT^{n\times n}$, defined as in
\eqref{eq:tropical-distance-matrix}.  Additionally, w.l.o.g., consider
its diagonal entries set to $+\infty$.  Then
$g \circ \gin_\epsilon \circ f$ can approximate a weighted shortest
path assuming that $f = q_1 \circ \mlp$, with $q_1$ defined as in
\eqref{eq:quantization}, and $g$ is the same as the previous examples.
\begin{proposition}
  Consider $\mW$ and $\chi_v$, a positive constant $c > 0$,
  and
  $\vd =
  \mlp_{\operatorname{d}}[d_1[\gin_\epsilon^k(q_1[\mlp_{\operatorname{E}}[\mW]],
  \chi_v)]]$, with $\gin_\epsilon$,
  $\mlp_{\operatorname{E}}$, and $\mlp_{\operatorname{d}}$ having any
  number of layers with $\relu$ activation function, there exist a
  configuration of their parameters such that
  $|\vd_u - d(u, v)| \leq c$ for any $u \in \gN^k(v)$.
\end{proposition}
\begin{proof}
  Fix $w_\theta = 1$ and $b_\theta = 0$ for every layer of
  $\gin$/$\mlp$, apart in the first layer of
  $\mlp_{\operatorname{E}}$, where we set $w_{\theta_1} = 1/h$, and in
  the first layer of $\mlp_{\operatorname{D}}$, where we set instead
  $w_{\theta_1} = h$.  As in the proof of \eqref{th:reachability}, we
  have that $\gin_0^k(\mA, \chi_v) = (\mA + \mI)^k \chi_v$.  Noticing
  that $q_1(x/h) = q_h(x)$ and $h\cdot d_1(x) = d_h(x)$, we obtain
  that
\begin{align*}
  \vd &= h\cdot d_1[q_1[\tfrac{1}{h}\mW]^k \cdot \chi_v]\\
  &=
  d_h[q_h[\mW]^k \cdot Q_h[\chi_v]]\\
  &\xrightarrow[h \to 0^+]{}
  \mW^{\odot k} \odot \chi_v = \vd_\circ,
\end{align*}
where $\vd_\circ$ is defined as in \eqref{eq:dist-vector}.
By definition of limit (for a $h > 0$ small enough) we reach the
conclusion.
\end{proof}

To conclude, we have now established a direct link between tropical
algebra and graph networks. Particularly, we showed that an EPD
architecture with $g=(\mlp \circ d_1)$, $p=\gin_\epsilon$ and
$f=q_1 \circ \mlp$ can approximate reachability and shortest path
algorithms on graphs, by learning the $h$ parameter. Furthermore,
oftentimes DP can be understood as generalised path-finding algorithms
\citep{dudzik2022graph} similar to shortest path. Hence, propositions
herein show potential to encompass a wider range of DP problems that
can be framed as path-finding (e.g., task scheduling
\citep{shirazi1990analysis}).

\subsection{Effects on network design}
\label{sec:tropical-effects}
\begin{figure}
  \centering
  \subcaptionbox{1-unit ReLU}{%
    \input{gfx/contrib/tropical/relu1.tex}
  }
  \subcaptionbox{4-unit ReLU}{%
    \input{gfx/contrib/tropical/relu4.tex}
  }
  \subcaptionbox{8-unit ReLU}{%
    \input{gfx/contrib/tropical/relu8.tex}
  }
  \subcaptionbox{32-unit ReLU}{%
    \input{gfx/contrib/tropical/relu32.tex}
  }
  \caption{Illustration of approximating an exponential function with
    piece-wise linear functions (ReLU). Plots have been obtained by
    approximating $e^{x}$ with ReLU-MLPs with 1, 4, 8 and 32 neurons.
  }
  \label{fig:relu-vs-exp}
\end{figure}
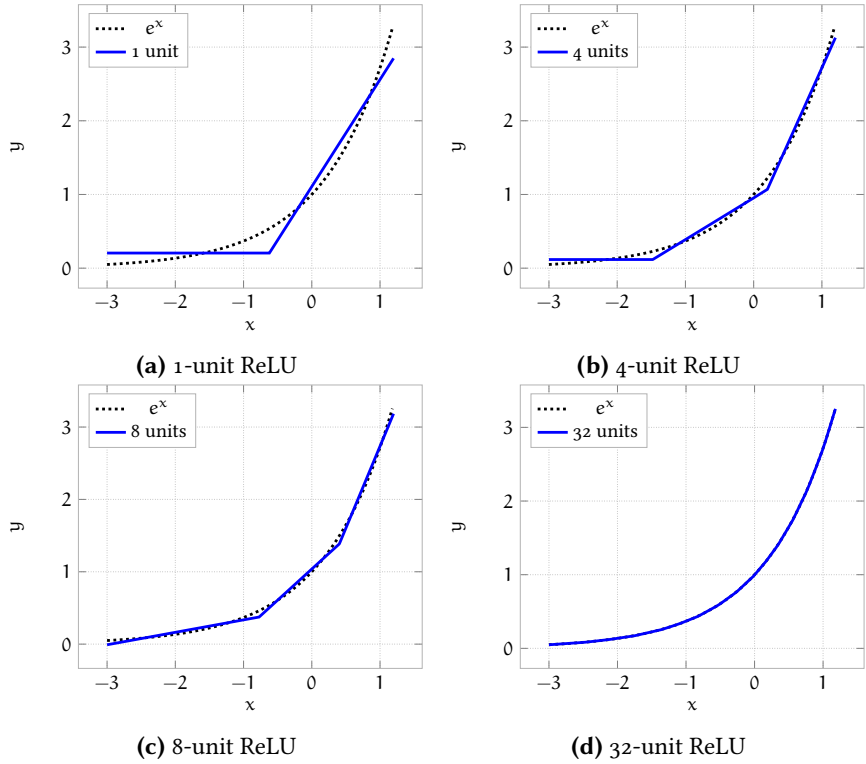
The propositions presented in the previous section also serve as
theoretical ground to support prior NAR findings. An example of this
lies in the commonly suggested choice of using $\max$ as aggregator in
the GN processor. In the literature, this suggestion is often based on
empirical evidence \citep{velickovic2019neural, velickovic2022clrs,
  ibarz2022generalist} and the answer to \textit{``why choose a max
  aggregator?''} is addressed in terms of algorithmic alignment
\citep{xu2020can, xu2021neural, dudzik2022graph}.  However, here we
seek to answer to this question in a more mathematical sense.  First,
we note that our architectural choices do reflect the use of a
component-wise maximum aggregator, even though we plug $\gin$ which
uses addition to aggregate node representations. The reason is that,
in our regime, $\gin$ operates on an $h$-semiring where its inputs and
outputs are adequatly processed by Maslov maps. As shown in
\autoref{eq:h-plus}, thus, in this setting the sum operation in the
$h$-semiring is equivalent to a $\min$ (or $\max$) operation in the DP
``tropical'' space. As a result, our sum-aggregated GN \textit{aligns}
with our target \emph{min}-aggregated DP algorithm.

However, in common NAR settings we usually deploy simple linear
encoders and decoders (\autoref{sec:epd}) and we do not directly
process data belonging to a tropical semiring. Therefore, assuming we
target the same \emph{min}-aggregated class of DP algorithms,
\emph{how would a different choice of the aggregator impact the
  learning problem?}. Consider the target function to be Bellman-Ford
as in \autoref{sec:algorithmic-alignment} and let us put our focus on
the aggregation step only. Furthermore, assume we choose a
\emph{sum}-aggregated processor similar to $\gin$. We report both
equations in the following:
\begin{align}
    &\text{($\Sigma$-MP)}& &\sumt \left(\left\{
                    \psi(\vh_v,
                    \vh_u,
                    \vh_{uv}) \mid u \in \gN(v) \right\}\right)
                    \label{eq:gin-alignment} \\
    &\text{(BF)}& &\min \left(\left\{
                    d_u + e_{uv}
                    \mid u \in \gN(v) \right\}\right).
                    \label{eq:bf-truncated-alignment}
\end{align}
Hence, in order for \eqref{eq:gin-alignment} to learn a $\min$
operation, there are two possibilities:
\begin{enumerate}
\item Encoders and decoders $f$ and $g$ learn exponential and
  logarithmic maps similar to \eqref{eq:quantization}.
\item The processor $p$ learns a \emph{softplus} function on
  neighbourhoods:
  \begin{equation*}
    \max(\{u \in \gN(v)\}) \approx h\log(\sum_{u \in \gN(v)} e^{\frac{\psi(\vh_v, \vh_u, \vh_{uv})}{h}}).
  \end{equation*}
  Here, $h$ represents the \emph{temperature} of the approximation and
  is the equivalent of the $h$ term in equation
  \eqref{eq:quantization}.  Note that this is in line with prior
  findings in linear algorithmic alignment (see,
  \autoref{sec:algorithmic-alignment}) as the processor here is
  required to learn a non-linear transformation.
\end{enumerate}
Consequently, we can effectively measure and quantify the drawbacks
that a misaligned aggregator would impose on the learning model. In
fact, this either forces usage of non-linear encoders and decoders, as
learning exponential and logarithmic functions through linear layers
is impractical, or it places an additional learning load on the
processor. Indeed, in the latter case the processor would be required
to both learn the target DP algorithm \textit{and} learn the
aforementioned softplus aggregation with ``low'' values of temperature
$h$ (recall from \autoref{sec:tropical-enc-dec-alignment} that
$h \to 0$).  Furthermore, NAR modules are usually deployed with ReLU
activations to comply with (linear) algorithmic alignment
requirements. Hence, the processor (or non-linear encoders and
decoders) are tasked to approximate exponential (and logarithmic)
mappings through piece-wise linear
functions. \autoref{fig:relu-vs-exp} clearly shows that the
``goodness'' of such approximation strongly depends on the number of
points of discontinuity (i.e., number of neurons in MLPs) we utilise.
This means that a sub-optimal aggregator choice directly impacts the
latent space complexity of NAR modules, effectively forcing to use a
higher number of neurons to get an equally good approximation
(compared to a GN with aligned aggregators). Finally, in light of the
previously discussed linear extrapolation tendency of ReLU networks
(\autoref{sec:algorithmic-alignment} and \citep{xu2021neural}), we
will likely \textit{never} align to the exponential mappings outside
the support of the training distribution, unless using exponential
activations in the neural network at the cost, however, of possible
numerical instabilities.

\section{Planning via learnt heuristics}
\label{sec:neural-planning}
In the previous section, we showed the GNs are a good learning model
when learning \emph{max}- and/or \emph{min}-aggregated DP functions.
This section capitalizes on this property to develop heuristic
functions designed to enhance the efficiency of \textit{planning}
algorithms \citep{numeroso2022learning}, with particular emphasis on
the A* search algorithm \citep{hart1968formal}. A* is a path-planning
algorithm that directly improves on the Dijkstra's algorithm
\citep{dijkstra1959note} and is particularly suited for solving
planning tasks.  Planning problems, on the other hand, are standard
instances of DP problems that involve executing \textit{actions}
within an \textit{environment} to achieve a specific
\textit{goal}. They can be formally defined as tuples
$\tuple{\gS, s_0, \gA, \goal}$, where:
\begin{itemize}
\item $\gS$ is a set of possible \textit{states} representing the
  possible configurations of the environment.
\item $s_0$ is the \textit{initial state}.
\item $\gA$ is the set of possible actions. An action $a \in \gA$ is a
  function $a: \gS \rightarrow \gS$ that describes a transition from
  one state to another.
\item $\goal: \gS \rightarrow \set{0, 1}$ is a function that
  determines whether a given state is a goal state, denoted as $s^*$.
\end{itemize}
Consequently, a solution to a planning problem is an action sequence
$a_0, a_1, \dots, a_n$ that transitions from the initial state to a
goal state. This is akin to the notion of graph paths (see
\autoref{def:path}). Indeed, we typically apply algorithms that find a
path in the \textit{state-action} graph\footnote{a graph where states
  are nodes $\gV$ and actions are connections (or edges $\gE$) among
  states} -- which represents the planning problem -- from $s_0$ to
$s_*$. A visual representation of a state-action graph is depicted in
\autoref{fig:sa-graph}.

\begin{figure}
  \centering
  \input{gfx/contrib/a_star/sa_graph}
  \caption{An example of a popular state-action graph representing a
    game of chess. Here, \textit{states} $\gS$ represent possible
    board positions, whereas actions $\gA$ are moves that change the
    configuration of the chessboard (represented as edges). Note that
    actions taken in state $s_0$ directly affect the range of
    potential actions available in the subsequente state $s_i$.}
  \label{fig:sa-graph}
\end{figure}
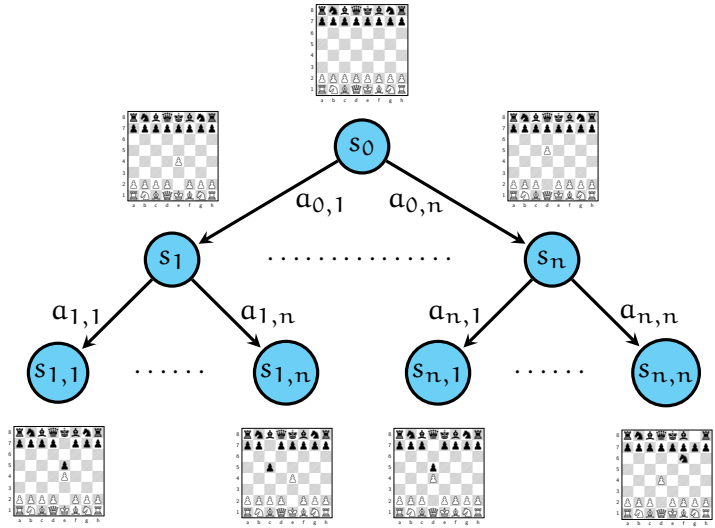

Problems of this nature, and their corresponding solution algorithms
find applications across numerous fields, including Robotics
\citep{mac2016heuristic}, Computer Science \citep{gebser2013domain,
  chen2020retro} and Chemistry \citep{yeh2012pathway}. When seeking
for the goal state $s^*$, the most efficient efficient search
algorithms usually exploit heuristic functions to navigate through the
state-action graph. Informally, these heuristics are functions
$h: \gS \rightarrow \sR$ that, given a state $s$, return a real number
indicating the quality of that state. In such context, $h(s)$ refers
to how ``fast'' we can reach the goal state through visiting $s$.
Such heuristics, thus, enable faster convergence towards the goal
state. In the context of A*, for instance, the heuristic is used for
selecting the next node to evaluate during path planning computation.
Precisely, A* can be seen as a generalisation of the Dijkstra's
algorithm where, the priority key in line~\ref{dij:line:queue} of
\autoref{alg:dijkstra} is changed from $d[s]$ to $d[s] + h[s]$.
Hence, A* fundamentally prioritises the nodes which are believed to be
closer to the goal state. As a result, the effectiveness of the whole
algorithm relies on a careful design of such heuristic, which is often
specialised for the problem at hand \citep{khalidi2020heuristic}. Such
hand-crafted heuristics necessitate domain knowledge and are often
subject to extensive trial-and-error phases during their development.
To circumvent this manual design process, an appealing alternative is
to automate the discovery of such heuristics, by learning them from
data. In the literature, there exist solutions based on supervised
learning \citep{us2013learning, wilt2015building, kim2020learning} and
imitation learning/reinforcement learning \citep{pandy2022learning}.
In contrast, this work formalises an unsupervised learning problem to
learn a proficient heuristic function for graph-based path-planning
problems. Such heuristic is then plugged into the A* algorithm. We
delve into technicalities of the approach in the following sections.

\subsection{Learning consistent heuristics}
In this work, we aim to learn an \textit{admissible} and
\textit{consistent} heuristic function for A* on weighted graphs
$G = \tuple{\gV, \gE, \gW}$.
\begin{definition}[Admissible Heuristic]
  A heuristic $h: \gV \rightarrow \sR$ is said to be
  \textit{admissible} for a graph $g = \tuple{\gV, \gE, \gW}$ and
  a target node $s^*$ if, for any node $v \in \gV$, it holds that
  \[ h(v) \leq h^*(v) \] where $h^*(v)$ is the true cost of the
  ``shortest path'' from node $v$ to the target node $s^*$.
\end{definition}
In other words, $h(v)$ never overestimates the real cost of reaching the
target node $s^*$ from $v$.
\begin{definition}[Consistent Heuristic]
  A heuristic $h: \gV \rightarrow \sR$ is said to be
  \textit{consistent} (or \textit{monotonic}) for a graph
  $g = \tuple{\gV, \gE, \gW}$ if, for any $(u, v) \in \gE$ with cost
  $w_{uv}$, it satisfies the following condition:
  \[ h(u) \leq w_{uv} + h(v) \]
\end{definition}
That is, for every node $u$ and ``successor'' $v$, the triangle
inequality is satisifed. Note that admissibility directly follows from
the definition of consistency, hence any consistent heuristic is also
admissibile.  In the context of A*, an admissible $h$ will ensure that
the algorithm will return the optimal path from $s_0$ to $s^*$ (or, in
planning terms, the best action sequence), where as consistency
guarantees that each intermediate node (or state) is evaluated no more
than once during optimal path computation (i.e., better efficiency).

\paragraph{Learning objective.}
To learn a consistent heuristic, we frame the learning problem for
our neural network as a constrained minimisation problem:
\begin{align} \label{eq:heur-objective}
  \min \, & h(s^*) - h(s_0) \\
  \text{subject to} \,\, & h(v) - h(u) \leq w_{uv} \,\, \forall (u, v)
                           \in E.
\end{align}
Note that this formulation guides the learnt heuristic values to low
values in the proximity of $s^*$ and progressively higher values to
nodes far from the goal. Additionally, the constraints serves to
ensure the heuristic consistency property.

Within this study, we also aim to make our learnt heuristic function
invariant to graph sizes (i.e., we want the heuristic to be applicable
regardless of the number of states in the planning problem). Given
that this is one of the key prospects of neural algorithmic reasoning,
we exploit it to attain this particular objective. In doing so, we
apply the MTL technique presented in \autoref{sec:transfer} to learn
jointly the Dijkstra's algorithm \citep{dijkstra1959note} and an
optimal assignment to \eqref{eq:heur-objective}. This way, we
constrain our learnt heuristic function to be computed in an
``algorithmic'' fashion, effectively biasing it towards size
invariance.  Parameters of the model are then optimised through a
combination of a cross-entropy error signal computed and supervised on
the ground truth Dijkstra's outputs at each step $t$ (i.e., step-wise
supervision) and the unsupervised objective described in
\eqref{eq:heur-objective}. Specifically, we optimise
\eqref{eq:heur-objective} by relaxing its constraints and introducing
penalties corresponding to the magnitude of each constraint
violation. Furthermore, as the main objective $\{h(t) - h(s)\}$ is
unbounded, we add an additional penalty term proportional to the
squared $\normltwo$-norm of the predicted heuristics (i.e.,
$||\vy^{(t)}||^2$).

The resulting objective is formulated as follows:
\begin{align*}
  \gL_i(\theta ; \lambda) =
  &\frac{1}{T} \sum_t^T \gL_{CE}(\vo^{(t)}, \vp^{(t)}) + \\
  &\frac{1}{T} \sum_t^T \Big[
    (y^{(t)}_{s_0} - y^{(t)}_{s^*}) +\\ &\qquad\sum_{(u,v) \in E} (y_v - y_u) \cdot \mathbf{1}_{[y_v - y_u > w_uv]} + \lambda ||\vy^{(t)}||_2^2 \Big],
\end{align*}
where $\gL_{CE}$ is the cross-entropy function,
$\mathbf{1}_{[y_v - y_u > w_{uv}]}$ is an indicator function checking
for satisfaction of the specific constraint, and $\lambda$ is a
weighing hyper-parameter.

\paragraph{Model configuration.}
The neural algorithmic reasoner architecture is a specification of the
encode-process-decode and is depicted in
\autoref{fig:planning-net}. The neural forward pass is the same as
described in \autoref{sec:epd}. In the following we provide a
specification of the different components.

The encoder layer comprises encoders for Dijkstra's inputs $f_D$ and
for positional features for the computation of heuristic values (i.e.,
identity of the starting node $s_0$ and the target node $s^*$), which
are processed through a second encoder $f_h$. Both $f_D$ and $f_h$ are
linear transformations. The processor $p$ is a MPNN
\citep{gilmer2017neural} with max aggregator. Decoder specifications
include a linear projector $g_h$ that maps to heuristic values
$\tilde{h}^{(t)}(v) = y_v = g_h(\cdot)$ and $g_D$ that predicts Dijkstra's
outputs $\vo^{(t)} \approx \vp^{(t)}$.
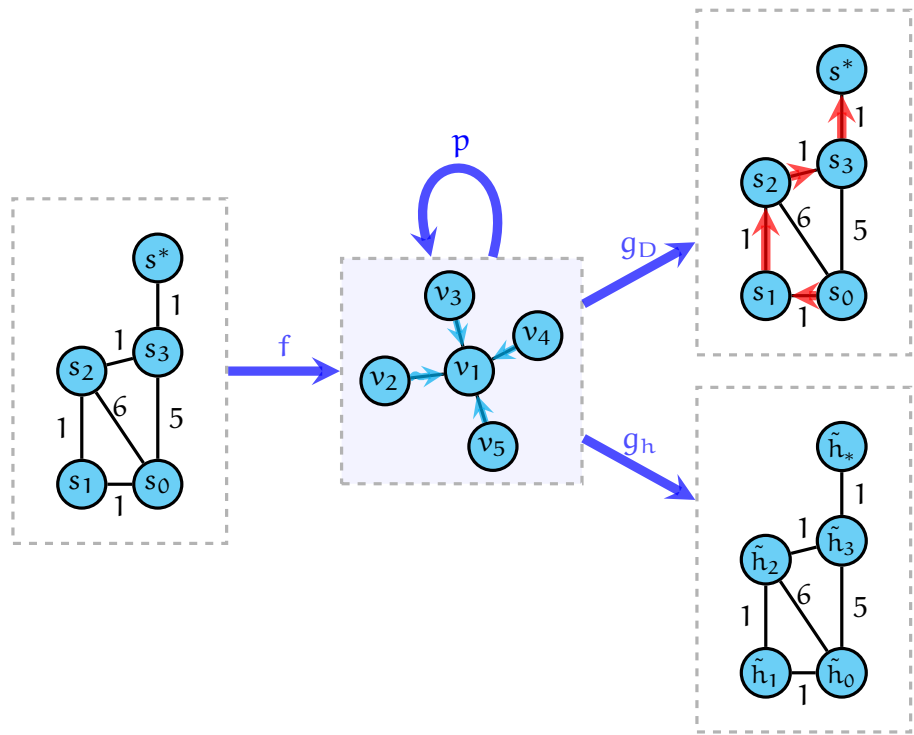
\begin{figure}
  \centering
  \input{gfx/contrib/a_star/net}
  \caption{Illustration of model architecture used to learn A*
    heuristics. We employ MTL on the Dijkstra's algorithm and the
    learning of the heuristic function $\tilde{h}$ (i.e.,
    $\tilde{h}(s_i) = \tilde{h}_i$ in the figure). }
  \label{fig:planning-net}
\end{figure}

\paragraph{Test inference.}
At test time, we run A* where the priority key for the queue
at line~\ref{dij:line:queue} of \autoref{alg:dijkstra} is substituted
for:
\begin{equation*}
  \texttt{key}(v) = d(v) + \tilde{h}(v),
\end{equation*}
where $d(v)$ represents the distance from node $v$ to the source node
$s_0$ and $\tilde{h}$ is our learnt heuristic function. Compared to
vanilla Dijkstra, we introduce an additional overhead due to the
computation of heuristic values $\tilde{h}(v)$ for each node $v$
within the graph. To keep this overhead limited, we compute
$\tilde{h}(v)$ for any $v$ in the graph by only executing one
\textbf{single pass} of the processor $p$. This has two implications.
First, the introduced overhead can be exactly quantified as
$O(|\gV| \bar{n})$ where $\bar{n}$ is the average node degree.
Second, $\tilde{h}(v)$ is only computed based on its immediate
neighbours, and thus a node $v$ might not directly ``see'' the goal
state $s^*$. This means that we are essentially navigating in a
\textbf{partially observable environment}, where our heuristic
function has to give good indications without actually knowing where
the goal state is at early stages of the algorithm. This poses an
additional challenge to our learner, as in common A* settings the
heuristic function is assumed to be able to access $s^*$'s location.

\subsection{Experimental evaluation}
Here, we detail the experimental setting, including model selection
and results.

\paragraph{Model selection.}
\begin{table}
  \centering
  \begin{tabular}{l | c C c}
    \toprule
    Level & Hyperparameter & [a,b] & Distribution \\
    \midrule
    \multirow{4}{4em}{Level 1} & Learning rate & [1 \cdot 10^{-4}, 1 \cdot 10^{-2}] & log-uniform\\
          & Weight decay & [1 \cdot 10^{-5}, 1 \cdot 10^{-1}] & log-uniform\\
          & Hidden dimension & [16, 512] & uniform\\
          & $\lambda$ & [1 \cdot 10^{-3}, 1 \cdot 10^{-1}] & uniform\\
    \midrule
    \multirow{4}{4em}{Level 2} & Learning rate & [9 \cdot 10^{-4}, 4 \cdot 10^{-3}] & log-uniform\\
          & Weight decay & [9 \cdot 10^{-4},  4 \cdot 10^{-3}] & log-uniform\\
          & Hidden dimension & [80, 100] & uniform\\
          & $\lambda$ & [5 \cdot 10^{-2}, 1 \cdot 10^{-2}] & uniform\\
    \bottomrule
  \end{tabular}
  \caption{First-level and second-level random search configurations
    used to optimise learning of $\tilde{h}$.}
  \label{tab:rs-bilevel-heur}
\end{table}
To select the best model hyper-parameters, we utilise a bi-level
random search strategy, where the best sampled hyper-parameters in the
initial stage are furher refined through a second, more granular,
random search.  Details of the two-stage random search are presented
in \autoref{tab:rs-bilevel-heur}. We sampled $n=200$ configurations
for each level of the random search.

\paragraph{Data generation.}
We generate three different graph sets, labelled as \textsc{very dense},
\textsc{dense}, and \textsc{sparse}. Each of the three sets includes
Erdős--Rényi \citep{erdos1960evolution} graphs with varying edge
sampling probabilities $p=0.5$, $p=0.35$, and
$p=\frac{\log |\gV|}{|\gV|}$, respectively. To rigorously evaluate our
initial claim of size invariancy of the learnt heuristic function, we
produce training and validation sets involving small graphs -- 16
nodes -- and nine separate test sets, having increasing number of
nodes: 16, 32, 64, 96, 128, 160, 192, 224, and 256.
We utilise 1000 instances for training and 128 for validation and
testing. Notably, our model is solely optimised on the \textsc{dense}
distribution, whereas the other two distributions are only used for
testing purposes.

\paragraph{Metrics.}
We keep track and report four different metrics:
\begin{enumerate}
\item $\gC$: which refers to the percentage of nodes for which the triangle
  inequality in \eqref{eq:heur-objective} is satisfied. Fundamentally,
  $\gC$ evaluates the consistency of the heuristic function.

\item \textsc{gap}: which is the optimality gap. This is computed by
  considering the cost of the path $\pi_{\tilde{h}} = s_0 \dots s^*$
  found by executing A* with the learnt heuristic function, and
  comparing it to the cost of the optimal path $\pi_*$ found by
  executing Dijkstra. The optimality gap is calculated as:
  \[
    \mathsc{gap} = \frac{c(\pi_{\tilde{h}})}{\pi_*} - 1
    \qquad \text{where} \quad c(\pi) = \sum_{i} w_{\pi_i, \pi_{i+1}}.
  \]

\item $\gI$: which refers to the number of iterations needed to reach
  the goal state (or target node) $s^*$. This metric is included to
  measure the efficiency of the learnt heuristic function.

\item \textsc{speedup}: which measures the actual speed-up runtime of
  running A* equipped with the learnt heuristic compared to running
  vanilla Dijkstra. The formula to compute the speedup is:
  \[\mathsc{speedup} = \frac{\texttt{Dijkstra runtime}}{\texttt{A* runtime}}.\]
\end{enumerate}

\paragraph{Results analysis.}
We evaluate our model by executing the A* algorithm, employing learnt
heuristics, across all nine test sets of the three datasets, and
average the outcomes over five trials. The results from our
experimental analysis are presented in \autoref{tab:heur-res}.
Additionally, we also include a graphic illustration of results for
\textit{dense} graphs in \autoref{fig:heur-res}.

Our learnt heuristic function is compared against a ``zero'' heuristic
function (i.e., $h(v) = 0$ for all $v$), which is equivalent to the
Dijkstra's algorithm, and a random baseline that samples values
uniformly in $[0, 1]$ (i.e., $h(v) \sim \gU[0,1]$). The first
conclusion that can be drawn from empirical results of
\autoref{tab:heur-res} is that the heuristic function predicted by the
model is very close to be fully consistent, as the number of violated
triangular inequalities is very low (i.e., $\leq 0.9 \%$ on average)
on each of the three different test distributions. Furthermore,
heuristics seem to be mostly \textit{size invariant}, as performance
appear not to degrade as we test it on 2x, 6x, 12x and 16x times
larger graphs compared to training data. The optimality gap also
reflects a good tendency, particularly for \textsc{sparse} data, of
the model with respect to finding the optimal paths while evaluating
fewer nodes compared to Dijkstra (i.e., $\gI_{A*}$ vs. $\gI_D$). For
\textsc{dense} and \textsc{very dense} graphs, however, the optimality
gap does degrade for larger graphs, as the number of paths gets
combinatorially larger with high connectivity. We recall, however,
that heuristics are computed only based on 1-hop neighbours, which
might limit the performance in case of high node degree and high
number of paths, as it is the case for the two \textsc{dense} and
\textsc{very dense} distributions.

\begin{table}
    \footnotesize
    \centering
    \begin{tabular}{lllll}
      \toprule
      \textbf{Metrics} & \emph{32 nodes} & \emph{96 nodes} & \emph{192 nodes} & \emph{256 nodes}\\
      \midrule
                       & \multicolumn{3}{c}{\textsc{sparse}} \\
      \midrule
      $\gC$ $(\uparrow)$& $99.4_{\pm 0.06}\%$ & $99.7_{\pm 0.04}\%$ & $99.8_{\pm 0.03}\%$ & $99.8_{\pm 0.03}\%$ \\
      \textsc{gap} $(\downarrow)$ & $0.03\%_{0.01}\%$ & $0.18_{\pm 0.05}\%$ & $0.19_{\pm 0.09}\%$ & $0.09_{\pm 0.01}\%$ \\
      $\gI_{A*}$ $(\downarrow)$& $17_{\pm 1}$ & $43_{\pm 1}\%$ & $74_{\pm 2}$ & $99_{\pm 3}$ \\
      $\gI_D$ $(\downarrow)$& $19$ & $53$ & $97$ & $132$ \\
      \textsc{speedup} $(\uparrow)$& $0.94$ & $1.08$ & $1.19$ & $1.23$\\
      \midrule
      \midrule
                       & \multicolumn{3}{c}{\textsc{dense}}\\
      \midrule
      $\gC$ $(\uparrow)$& $99.2_{\pm 0.08}\%$ & $99.4_{\pm 0.09}\%$ & $99.5_{\pm 0.09}\%$ & $99.5_{\pm 0.09}\%$ \\
      \textsc{gap} $(\downarrow)$& $0.98_{\pm 0.16}\%$ & $10.75_{\pm 0.77}\%$ & $33.5_{\pm 1.79}\%$ & $52.8_{\pm 4.55}\%$ \\
      $\gI_{A*}$ $(\downarrow)$& $12_{\pm 1}$ & $22_{\pm 2}$ & $22_{\pm 3}$ & $21_{\pm 5}$ \\
      $\gI_{D}$ $(\downarrow)$& $17$ & $51$ & $97$ & $132$ \\
      \textsc{speedup} $(\uparrow)$& $0.93$ & $2.29$ & $4.30$ & $4.83$\\
      \midrule
      \midrule
                       & \multicolumn{3}{c}{\textsc{very dense}} \\
      \midrule
      $\gC$ $(\uparrow)$& $99.1_{\pm 0.08}\%$ & $99.4_{\pm 0.11}\%$ & $99.5_{\pm 0.11}\%$ & $99.5_{\pm 0.11}\%$ \\
      \textsc{gap} $(\downarrow)$& $0.7_{\pm 0.15}\%$ & $19.8_{\pm 1.53}\%$ & $59.4_{\pm 5.00}\%$ & $73.5_{\pm 6.75}\%$ \\
      $\gI_{A*}$ $(\downarrow)$& $11_{\pm 1}$ & $16_{\pm 2}\%$ & $16_{\pm 3}$ & $16_{\pm 4}$ \\
      $\gI_D$ $(\downarrow)$& $17$ & $50$ & $90$ & $130$ \\
      \textsc{speedup} $(\uparrow)$& $1.22$ & $2.60$ & $5.06$ & $5.24$\\

    \bottomrule
    \end{tabular}
    \caption{Performance of A* with learnt heuristic function on 4 out
      of the 9 test sets, i.e. 32, 96, 192, 256 nodes.
      $\gC$ represents percentage of nodes for which the triangle inequality
      is satisfied, \textsc{gap} and \textsc{speedup} refer to the optimality
      gap and the runtime speedup, respectively, and $\gI_{A*}$ and $\gI_{D}$
      indicate the iterations performed by A* with learnt heuristic and Dijkstra,
      respectively. $(\uparrow)$ means ``the higher the better'', $(\downarrow)$
      means ``the lower the better''.}
    \label{tab:heur-res}
\end{table}
\begin{figure}
  \centering
  \subcaptionbox{Optimality gap.}{
    \input{gfx/contrib/a_star/gap.tex}
  }
  \subcaptionbox{Number of iterations.}{%
    \input{gfx/contrib/a_star/iter.tex}
  }
  \caption{\textbf{(a).} Comparison of the optimality gap achieved by
    the learnt heuristic and a random baseline. \textbf{(b).}  Number
    of iterations performed by A* with different heuristic
    function. Our learnt heuristic performs the best among all,
    keeping the number of iterations constant across different graph
    sizes. Note that the zero heuristic corresponds to running the
    Dijkstra's algorithm.}
  \label{fig:heur-res}
\end{figure}

Performance-wise, A* equipped with the learnt heuristic function is
more efficient, by far, than Dijkstra (see results on \textsc{very
  dense} 256-node graphs in \autoref{tab:heur-res}, where
$\gI_{A*} \approx 16$ and $\gI_{D} = 130$). This results is confirmed
also by the \textsc{speedup} metric, which measures the exact
computational runtime of computing the heuristics plus running the A*
search, indicating that the overall time complexity of the learnt
search does not exceed that of running classic Dijkstra.  This result
can also be visualised in \autoref{fig:heur-res}, for dense graphs.

Finally, the consistent results achieved on each of the three
different graph distributions suggest that the model can generalise
across different Erdős--Rényi distributions (recall that the model is
only trained on \textsc{dense} graphs).

%% file: gfx/contrib/tropical/transform.tex
\begin{tikzpicture}[->, >=stealth]
    \node (A) {$[x,y]_{\sR}$};
    \node[right=2.1cm of A] (B) {$[z]_{\sR}$};
    \node[below=1.5cm of A] (C) {$[a,b]_{\sT}$};
    \node[below=1.5cm of B] (D) {$[c]_{\sT}$};

    \draw (C) to node[midway,left] {\footnotesize$f = q_h$} (A);
    \draw (A) to node[midway,above] {\footnotesize$p=+$} (B);
    \draw (B) to node[midway,right] {\footnotesize$g=d_h$} (D);
    \draw (C) to node[midway,below] {\footnotesize$\min$} (D);

    \node[right=5cm of A] (A1) {$[x,y]_{\sR}$};;
    \node[right=2.1cm of A1] (B1) {$[z]_{\sR}$};
    \node[right=5cm of C] (C1)  {$[a,b]_{\sT}$};
    \node[below=1.5cm of B1] (D1) {$[c]_{\sT}$};

    \draw (C1) to node[midway,left] {\footnotesize$f = q_h$} (A1);
    \draw (A1) to node[midway,above] {\footnotesize$p=\cdot$} (B1);
    \draw (B1) to node[midway,right] {\footnotesize$g=d_h$} (D1);
    \draw (C1) to node[midway,below] {\footnotesize$+$} (D1);

\end{tikzpicture}

%% file: gfx/contrib/tropical/relu1.tex
\begin{tikzpicture}[scale=.66]
\definecolor{dark green}{HTML}{00611B}
\begin{axis}[
    xlabel={$x$},
    ylabel={$y$},
    grid=major,
    legend entries={$e^x$, 1 unit},
    legend pos=north west,
    grid style = {black!30, densely dotted},
    % tick label style={font=\scriptsize},
    % label style={font=\small},
    legend style={draw=black!30, fill=white},
    axis line style={black!30},
]
\addplot[black, dotted, ultra thick] table[col sep=comma, x index=0, y index=1]{gfx/contrib/tropical/data/gt.csv};
\addplot[blue, ultra thick] table[col sep=comma, x index=0, y index=1]{gfx/contrib/tropical/data/pred1.csv};
\end{axis}
\end{tikzpicture}

%% file: gfx/contrib/tropical/relu4.tex
\begin{tikzpicture}[scale=0.66]
\definecolor{dark green}{HTML}{00611B}
\begin{axis}[
    xlabel={$x$},
    ylabel={$y$},
    grid=major,
    legend entries={$e^x$, 4 units},
    legend pos=north west,
    grid style = {black!30, densely dotted},
    % tick label style={font=\scriptsize},
    % label style={font=\small},
    legend style={draw=black!30, fill=white},
    axis line style={black!30},
]
\addplot[black, dotted, ultra thick] table[col sep=comma, x index=0, y index=1]{gfx/contrib/tropical/data/gt.csv};
\addplot[blue, ultra thick] table[col sep=comma, x index=0, y index=1]{gfx/contrib/tropical/data/pred4.csv};
\end{axis}
\end{tikzpicture}

%% file: gfx/contrib/tropical/relu8.tex
\begin{tikzpicture}[scale=0.66]
\definecolor{dark green}{HTML}{00611B}
\begin{axis}[
    xlabel={$x$},
    ylabel={$y$},
    grid=major,
    legend entries={$e^x$, 8 units},
    legend pos=north west,
    grid style = {black!30, densely dotted},
    % tick label style={font=\scriptsize},
    % label style={font=\small},
    legend style={draw=black!30, fill=white},
    axis line style={black!30},
]
\addplot[black, dotted, ultra thick] table[col sep=comma, x index=0, y index=1]{gfx/contrib/tropical/data/gt.csv};
\addplot[blue, ultra thick] table[col sep=comma, x index=0, y index=1]{gfx/contrib/tropical/data/pred8.csv};
\end{axis}
\end{tikzpicture}

%% file: gfx/contrib/tropical/relu32.tex
\begin{tikzpicture}[scale=0.66]
\definecolor{dark green}{HTML}{00611B}
\begin{axis}[
    xlabel={$x$},
    ylabel={$y$},
    grid=major,
    legend entries={$e^x$, 32 units},
    legend pos=north west,
    grid style = {black!30, densely dotted},
    legend style={draw=black!30, fill=white},
    axis line style={black!30},
]
\addplot[black, dotted, ultra thick] table[col sep=comma, x index=0, y index=1]{gfx/contrib/tropical/data/gt.csv};
\addplot[blue, solid, ultra thick] table[col sep=comma, x index=0, y index=1]{gfx/contrib/tropical/data/pred32.csv};
\end{axis}
\end{tikzpicture}

%% file: gfx/contrib/a_star/sa_graph.tex
\begin{tikzpicture}[
    level distance=1.5cm,
    level 1/.style={sibling distance=5cm},
    level 2/.style={sibling distance=3cm},
    ->,>=stealth,shorten >=1pt,auto, very thick]

    \tikzstyle{vertex}=[circle,fill=cyan!50!,draw,inner sep=1pt, minimum size=20pt]
    % Root
    \node[vertex] (Root) {$s_0$}
      child {node[vertex] (C1) {$s_1$}
          child {node[vertex] (C1-1) {$s_{1,1}$}}
          child {node[vertex] (S2) {$s_{1,n}$}}
      }
      child {node[vertex] (C3) {$s_n$}
          child {node[vertex] (C3-1) {$s_{n,1}$}}
          child {node (S6)[vertex] {$s_{n,n}$}}
      };

    \path (Root) -- (C1) node[midway,right=2mm] {$a_{0,1}$};

    \path (Root) -- (C3) node[midway,left] {$a_{0,n}$};

    \path (C1) -- (C1-1) node[midway,left] {$a_{1,1}$};
    \path (C1) -- (S2) node[midway,right] {$a_{1,n}$};
    \path (C3) -- (C3-1) node[midway,left] {$a_{n,1}$};
    \path (C3) -- (S6) node[midway,right] {$a_{n,n}$};

    % Ellipsis
    \node[draw=none,fill=none,
    left of=S2, xshift=-0.25cm] {$\cdots$};
    \node[draw=none,fill=none,
    left of=S2, xshift=-0.75cm] {$\cdots$};
    \node[draw=none,fill=none,
    left of=C3, xshift=-0.5cm] {$\cdots$};
    \node[draw=none,fill=none,
    left of=C3, xshift=-1cm] {$\cdots$};
    \node[draw=none,fill=none,
    left of=C3, xshift=-1.5cm] {$\cdots$};
    \node[draw=none,fill=none,
    left of=C3, xshift=-2cm] {$\cdots$};
    \node[draw=none,fill=none,
    left of=C3, xshift=-2.5cm] {$\cdots$};
    \node[draw=none,fill=none,
    left of=S6, xshift=-0.25cm] {$\cdots$};
    \node[draw=none,fill=none,
    left of=S6, xshift=-0.75cm] {$\cdots$};

    \node[above=0pt of Root] {\includegraphics[width=1.5cm]{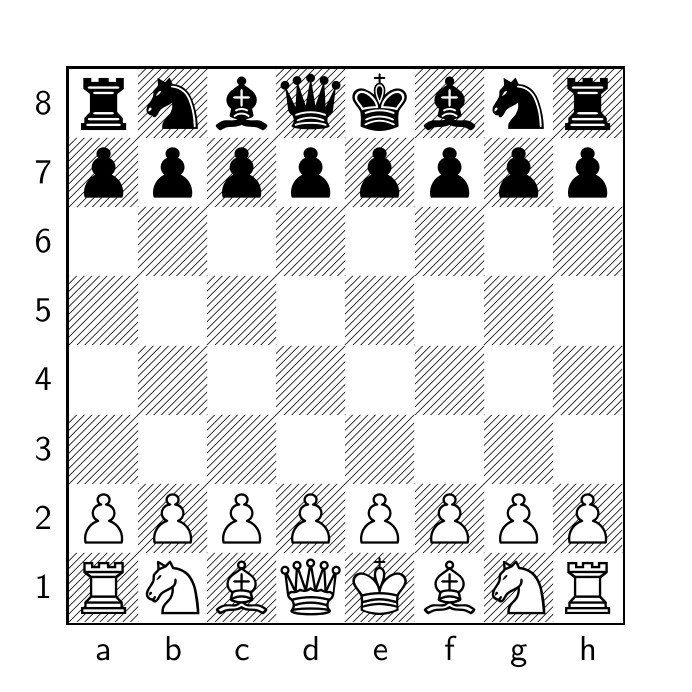}};
    \node[above=3pt of C1] {\includegraphics[width=1.5cm]{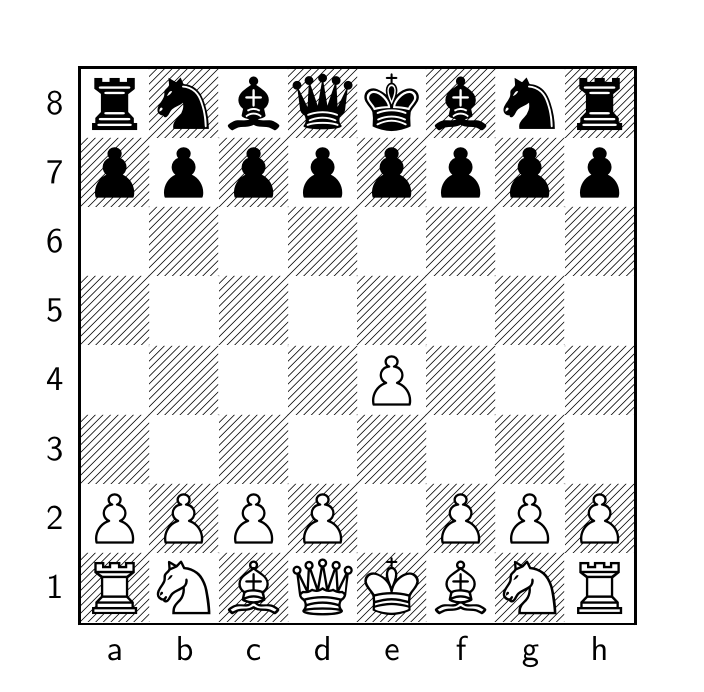}};
    \node[above=3pt of C3] {\includegraphics[width=1.5cm]{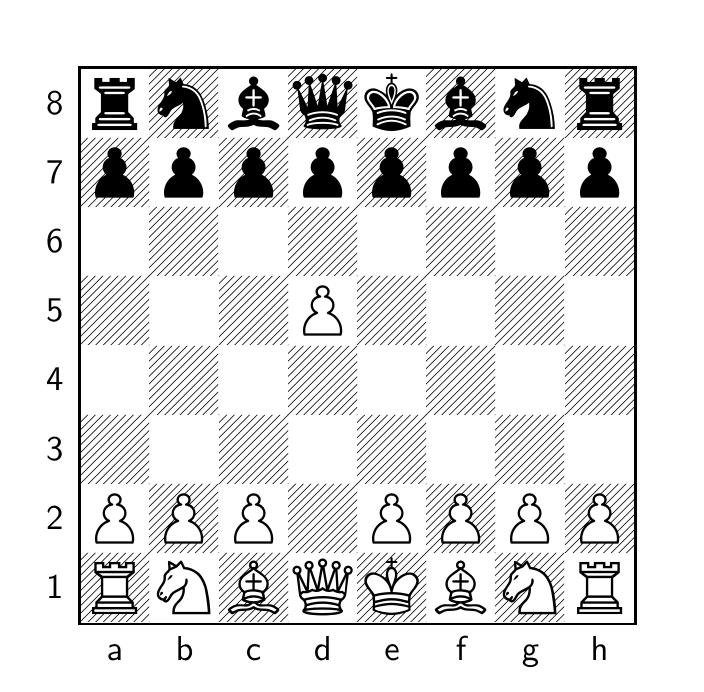}};
    \node[below=0pt of C1-1] {\includegraphics[width=1.5cm]{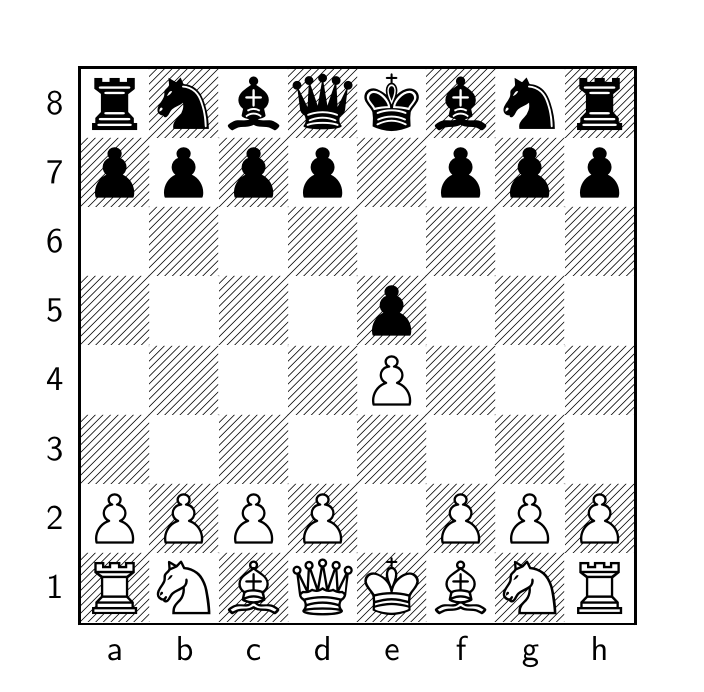}};
    \node[below=0pt of S2] {\includegraphics[width=1.5cm]{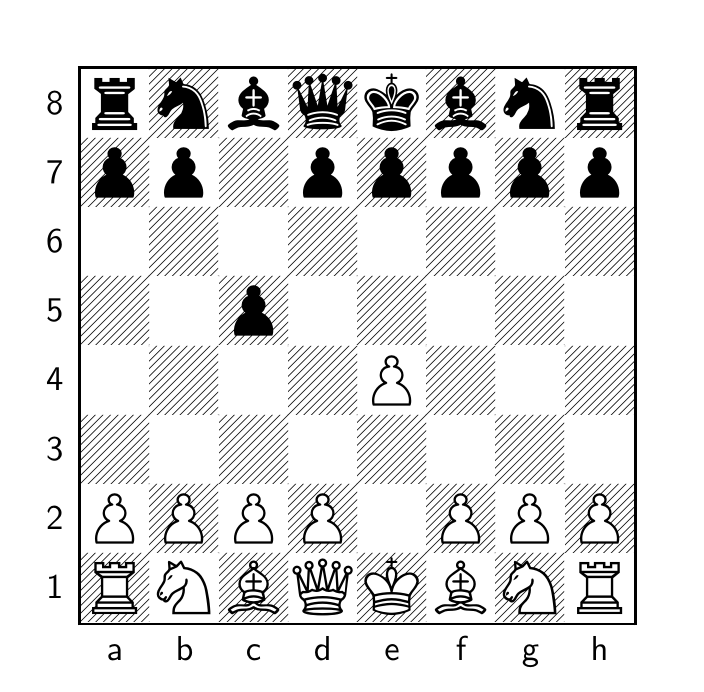}};
    \node[below=0pt of C3-1] {\includegraphics[width=1.5cm]{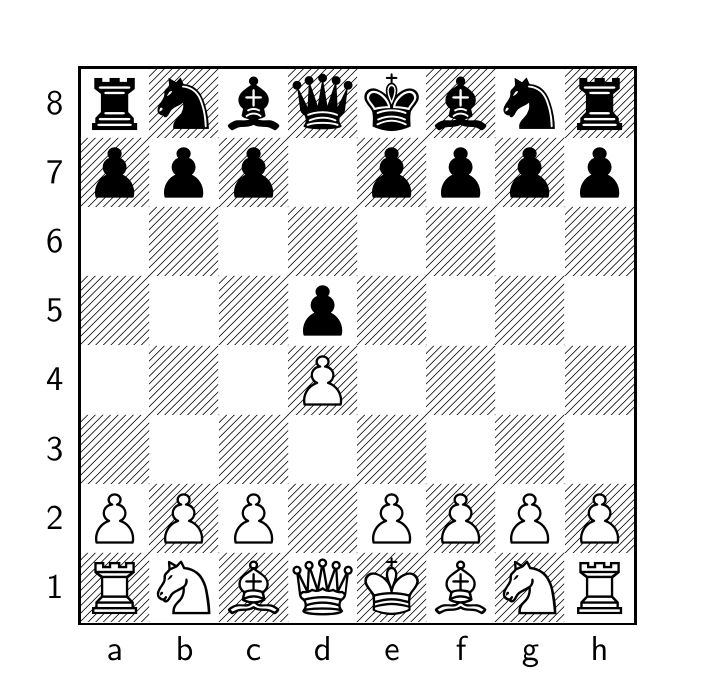}};
    \node[below=0pt of S6] {\includegraphics[width=1.5cm]{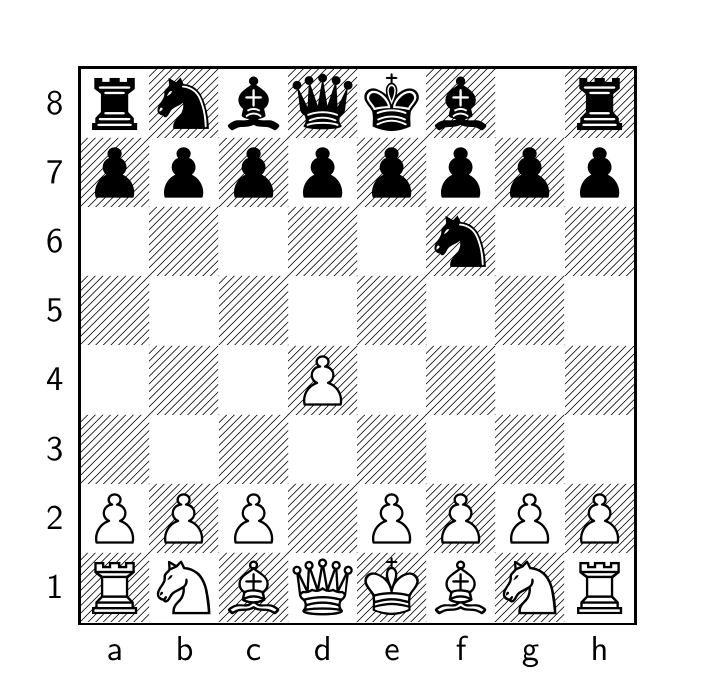}};

\end{tikzpicture}

%% file: gfx/contrib/a_star/net.tex
\begin{tikzpicture}[very thick,scale=1]

  \tikzset{line/.style={draw,line width=1.5pt}}
  \tikzset{arrow/.style={->,>=stealth}}
  \tikzset{snake/.style={arrow,line width=1.3pt,decorate,decoration={snake,amplitude=1,segment length=4,post length=5}}}

  \tikzstyle{box}=[dash pattern=on 5pt off 2pt,inner sep=5pt,rounded corners=3pt]
  \tikzstyle{vertex}=[circle,fill=cyan!50!,draw,minimum size=18pt, inner sep=1pt]

  \tikzstyle{message}=[arrow, decorate,
  decoration={snake,
    amplitude=1,
    segment length=6,
    post length=7
  },
  line width=2.5pt]

  \node[rectangle, draw, dashed, very thick, black!30, minimum height=130pt, minimum width=80pt] (Xbar) at (-4.5, 0) {};

  \node[rectangle, draw, dashed, very thick, black!30, minimum height=130pt, minimum width=80pt] (Ybar) at (4.5, 2.5) {};

  \node[rectangle, draw, dashed, very thick, black!30, minimum height=130pt, minimum width=80pt] (Hbar) at (4.5, -2.5) {};

  \node[rectangle, draw, dashed, very thick, black!30, minimum height=85pt, minimum width=90pt, fill=blue, fill opacity=0.05] (P) at (0, 0) {};

  \draw[line width=3.5pt, -stealth, blue, opacity=0.7] (Xbar) -- node[above] {\large $f$} (P);
  \draw[line width=3.5pt, -stealth, blue, opacity=0.7] (P) -- node[above] {\large $g_D$} (Ybar);
  \draw[line width=3.5pt, -stealth, blue, opacity=0.7] (P) -- node[above] {\large $g_h$} (Hbar);

  \draw[blue] (P) edge[loop above=1, line width=3.5pt, looseness=5, stealth-, opacity=0.7] node [above, opacity=1.0] {\large $p$} (P);

  \foreach \pos/\name/\lab in {{(0.1,0)/m1/$v_1$}, {(-1.01,-0.13)/m2/$v_2$}, {(-0.16,1)/m3/$v_3$}, {(1.0,0.47)/m4/$v_4$}, {(0.41,-1)/m5/$v_5$}}
  \node[vertex] (\name) at \pos {\lab};

  \foreach \pos/\name/\nb in {{(-4, -1.5)/a/$s_0$},{(-5, -1.5)/b/$s_1$},{(-5, 0)/c/$s_2$},{(-4, 0.25)/d/$s_3$},{(-4, 1.5)/e/$s^*$}}
  \node[vertex] (\name1) at \pos {\nb};

  \foreach \pos/\name/\nb in {{(5 , 1)/a/$s_0$},{(4, 1)/b/$s_1$},{(4, 2.5)/c/$s_2$},{(5, 2.75)/d/$s_3$},{(5, 4)/e/$s^*$}}
  \node[vertex] (\name2) at \pos {\nb};

  \foreach \pos/\name/\nb in {{(5 , -4)/a/$\tilde{h}_{0}$},{(4, -4)/b/$\tilde{h}_{1}$},{(4, -2.5)/c/$\tilde{h}_{2}$},{(5, -2.25)/d/$\tilde{h}_{3}$},{(5, -1)/e/$\tilde{h}_{*}$}}
  \node[vertex] (\name3) at \pos {\small\nb};

  % \foreach \name/\nb in {{a/$\infty$},{b/$0$},{c/$\infty$},{d/$\infty$},{e/$\infty$}}
  % \node[vertex, below = 3.75cm of \name1] (\name3) {$x_\name$};

  \draw (m2) edge (m1);
  \draw (m3) edge (m1);
  \draw (m4) edge (m1);
  \draw (m5) edge (m1);
  \draw (m2) edge[message, cyan, opacity=0.7] (m1);
  \draw (m3) edge[message, cyan, opacity=0.7] (m1);
  \draw (m4) edge[message, cyan, opacity=0.7] (m1);
  \draw (m5) edge[message, cyan, opacity=0.7] (m1);

  \draw (a1) edge node[midway, below] {$1$} (b1);
  \draw (a1) edge node[midway, above] {$6$} (c1);
  \draw (a1) edge node[midway, right] {$5$} (d1);
  \draw (b1) edge node[midway, left] {$1$} (c1);
  \draw (c1) edge node[midway, above] {$1$} (d1);
  \draw (d1) edge node[midway, right] {$1$} (e1);

  \draw (a2) edge node[midway, below] {$1$} (b2);
  \draw (a2) edge node[midway, above] {$6$} (c2);
  \draw (a2) edge node[midway, right] {$5$} (d2);
  \draw (b2) edge node[midway, left] {$1$} (c2);
  \draw (c2) edge node[midway, above] {$1$} (d2);
  \draw (d2) edge node[midway, right] {$1$} (e2);

  \draw (a3) edge node[midway, below] {$1$} (b3);
  \draw (a3) edge node[midway, above] {$6$} (c3);
  \draw (a3) edge node[midway, right] {$5$} (d3);
  \draw (b3) edge node[midway, left] {$1$} (c3);
  \draw (c3) edge node[midway, above] {$1$} (d3);
  \draw (d3) edge node[midway, right] {$1$} (e3);

  \draw[red, line width=3.5pt, opacity=0.7, -stealth] (a2) edge (b2);
  \draw[red, line width=3.5pt, opacity=0.7, -stealth] (b2) edge (c2);
  \draw[red, line width=3.5pt, opacity=0.7, -stealth] (c2) edge (d2);
  \draw[red, line width=3.5pt, opacity=0.7, -stealth] (d2) edge (e2);

\end{tikzpicture}

%% file: gfx/contrib/a_star/gap.tex
\begin{tikzpicture}
\begin{axis}[
    xlabel={Graph size},
    ylabel={\textsc{gap}},
    grid=major,
    legend entries={Learnt $\tilde{h}$, Random $\tilde{h}$},
    legend pos=north west,
    grid style={black!30, densely dotted},
    legend style={draw=black!20, fill=white, font=\footnotesize},
    axis line style={black!20},
    xtick={16,32,64,96,128,160,192,224,256},
    xticklabel style={font=\footnotesize},
    yticklabel style={font=\footnotesize}
]
\addplot[blue, ultra thick] table[col sep=comma, x index=0, y index=1]{gfx/contrib/a_star/data/gap_nar.csv};
\addplot[orange, ultra thick] table[col sep=comma, x index=0, y index=1]{gfx/contrib/a_star/data/gap_random.csv};

\addplot[blue!30, name path=plus_sd_nar, forget plot] table[col sep=comma, x index=0, y expr=\thisrowno{1}+\thisrowno{2}]{gfx/contrib/a_star/data/gap_nar.csv};
\addplot[blue!30, name path=minus_sd_nar, forget plot] table[col sep=comma, x index=0, y expr=\thisrowno{1}-\thisrowno{2}]{gfx/contrib/a_star/data/gap_nar.csv};
\addplot[blue!30] fill between[of=plus_sd_nar and minus_sd_nar];

\addplot[orange!30, name path=plus_sd_rand, forget plot] table[col sep=comma, x index=0, y expr=\thisrowno{1}+\thisrowno{2}]{gfx/contrib/a_star/data/gap_random.csv};
\addplot[orange!30, name path=minus_sd_rand, forget plot] table[col sep=comma, x index=0, y expr=\thisrowno{1}-\thisrowno{2}]{gfx/contrib/a_star/data/gap_random.csv};
\addplot[orange!30] fill between[of=plus_sd_rand and minus_sd_rand];
\end{axis}
\end{tikzpicture}

%% file: gfx/contrib/a_star/iter.tex
\begin{tikzpicture}
\begin{axis}[
    xlabel={Graph size},
    ylabel={$\gI$},
    grid=major,
    legend entries={Learnt $\tilde{h}$, Random $\tilde{h}$, Zero $\tilde{h}$},
    legend pos=north west,
    grid style={black!30, densely dotted},
    legend style={draw=black!20, fill=white, font=\footnotesize},
    axis line style={black!20},
    xtick={16,32,64,96,128,160,192,224,256},
    ymin=0,
    ymax=150,
    xticklabel style={font=\footnotesize},
    yticklabel style={font=\footnotesize},
    ytick={20, 40, 60, 80, 100, 120, 140}
]
\addplot[blue, ultra thick] table[col sep=comma, x index=0, y index=1]{gfx/contrib/a_star/data/iter_nar.csv};
\addplot[orange, ultra thick] table[col sep=comma, x index=0, y index=1]{gfx/contrib/a_star/data/iter_random.csv};
\addplot[red, ultra thick] table[col sep=comma, x index=0, y index=1]{gfx/contrib/a_star/data/iter_dijkstra.csv};

\addplot[blue!30, name path=plus_sd_nar, forget plot] table[col sep=comma, x index=0, y expr=\thisrowno{1}+\thisrowno{2}]{gfx/contrib/a_star/data/iter_nar.csv};
\addplot[blue!30, name path=minus_sd_nar, forget plot] table[col sep=comma, x index=0, y expr=\thisrowno{1}-\thisrowno{2}]{gfx/contrib/a_star/data/iter_nar.csv};
\addplot[blue!30] fill between[of=plus_sd_nar and minus_sd_nar];

\addplot[orange!30, name path=plus_sd_rand, forget plot] table[col sep=comma, x index=0, y expr=\thisrowno{1}+\thisrowno{2}]{gfx/contrib/a_star/data/iter_random.csv};
\addplot[orange!30, name path=minus_sd_rand, forget plot] table[col sep=comma, x index=0, y expr=\thisrowno{1}-\thisrowno{2}]{gfx/contrib/a_star/data/iter_random.csv};
\addplot[orange!30] fill between[of=plus_sd_rand and minus_sd_rand];
\end{axis}
\end{tikzpicture}

%% file: content/conar.tex
\chapter{Learning Combinatorial Algorithms}
\label{ch:contribution-conar}
In this chapter, we delve into the second batch of contributions
presented in this PhD thesis, which are mainly about learning
algorithms that target combinatorial optimisation problems (see
\autoref{sec:co}).  Herein, we provide empirical evidence of the
usefulness of different \textit{algorithmic biases} to target,
respectively: algorithmic learning problems; predictive graph learning
tasks; and combinatorial optimisation problems.

In this respect, the section \textit{``Leveraging Strong Duality''}
shows how we can leverage strong duality (see
\autoref{th:strong-duality-theorem}) in neural algorithmic reasoning
in order to: learn algorithms to a better precision, with particular
emphasis on the Ford-Fulkerson algorithm (\autoref{alg:ff}); and
utilise duality as a strong inductive prior for solving edge
classification tasks.

Finally, in \textit{``Learning approximation via algorithmic
  primitives''}, we demonstrate the usefulness of algorithmic priors
for solving combinatorial optimisation problems through quantitative
experiments on NP-hard CO problems. To the best of our knowledge, no
prior work in the field of neural combinatorial optimisation ever
explored algorithmic-based approaches.

\section{Leveraging Strong Duality}
\label{sec:dar}
In \autoref{sec:transfer}, we learnt about multiple ways of teaching
and utilising algorithmic knowledge with neural networks. In
particular, we discussed the MTL paradigm, which can be summarised as
a technique that enables neural networks to solve multiple tasks at
once. As introduced in the aforementioned section,
\cite{xhonneux2021transfer} verify the importance of this methodology
for OOD generalisation in the context of neural algorithmic reasoning,
by tasking neural networks to learn multiple related algorithms
simultaneously. The rationale here is that many classical algorithms
share sub-routines, such as Dijkstra and Bellman-Ford,
(\autoref{alg:dijkstra} and \autoref{alg:bf}), which help the network
learn more effectively and be able to transfer knowledge among the
target algorithms. Interestingly, MTL is possible also on sets of
diverse (possibly unrelated) algorithms, as shown in
\citep{ibarz2022generalist}.

\paragraph{Problem statement.}
However, learning some specific algorithms might require learning of
very specific properties of the input data, for which multi-task
learning may not help. In this work, we particularly target one such
algorithm, which is the Ford-Fulkerson algorithm, presented in
\autoref{sec:graph-algos}. To elaborate on why Ford-Fulkerson requires
stronger reasoning abilities, we have to consider the
\textit{algorithmic problem} that we are solving via the execution of
Ford-Fulkerson. This problem is max-flow (\autoref{prob:maxflow}).  As
discussed in \autoref{sec:problems-on-graphs}, in order to solve
max-flow, a learner is required to \emph{identify} and \emph{reason}
in terms of a set of \emph{bottleneck} edges. These bottleneck edges
are such that a decrease in their capacity causes a proportional
decrease in the maximum flow solution in the flow network, and are
equivalent to the notion of minimum cut of the flow network
(\autoref{prob:mincut}).

From an intuitive standpoint, this is inherently more difficult than
learning, say, the shortest path update rule of equation
\eqref{eq:bf-relaxation}. Furthermore, it is also unclear which other
algorithms possess a similar ``sub-routine'' (i.e., identifying
bottleneck edges) that can positively influence the learning of
Ford-Fulkerson in the MTL scheme.

Furthermore, unlike simpler algorithms, such as shortest path
algorithms, Ford-Fulkerson deals with the optimisation of a
\textit{constrained optimisation problem}, described in
\autoref{eq:mf}, where constraints are non-trivial (see equations
\eqref{eq:capacity-constraint} and \eqref{eq:flow-conservations}).

\paragraph{Duality as inductive bias.}
Driven by these specific requirements, we explore alternative learning
setups to enable better reasoning abilities of our algorithmic
reasoners. We find a promising candidate in the \emph{duality}
information, introduced and discussed in \autoref{sec:duality}. The
fundamental idea here is that many ``algorithmic problems'', such as
flow problems and shortest paths, can be naturally reduced to linear
programs (\autoref{sec:lp}), enabling us to effectively utilise
duality relationships.

In the context of neural algorithmic reasoning, there are several
reasons why considering primal-dual information is sensible:
\begin{enumerate}
\item By incorporating \textit{primal-dual} objectives, we allow a
  neural network to reason on the task from two different (and often
  complementing) perspectives. This can significantly ease the
  learning of algorithms, especially those that require recognising
  and reasoning about properties not directly present in the input
  data, such as Ford-Fulkerson. Specifically, the \emph{max-flow
    min-cut theorem} \citep{ford2015flows} states that \emph{strong
    duality} (see \autoref{th:strong-duality-theorem}) holds for this
  pair of problems. Consequently, accurately identifying the minimum
  cut is crucial for deriving an optimal max-flow solution. Hence,
  duality is naturally present when learning Ford-Fulkerson.

\item Producing a more accurate step-by-step approximation reduces
  error propagation throughout the neural execution of the algorithm.
  In the case of Ford-Fulkerson, choosing a \emph{wrong} edge (i.e.,
  an edge which is not in the min-cut) for updating the flow might
  negatively impact the outcomes of subsequent iterations.

\item Ultimately, we relax the ``standard'' MTL assumption of having
  to learn multiple algorithms simultaneously. However, as we will see,
  we still employ MTL on both primal and dual solutions.
\end{enumerate}

Hence, in this work, we study to what extent neural algorithmic
reasoner can utilise dual information to learn max-flow and
min-cut. Note that, graph networks are inherently able to learn
minimum cut, as the problem itself complies with
\emph{permutation-compatible} functions \citep{fereydounian2022exact}
(see also \autoref{sec:algorithmic-alignment}).  Therefore, resolving
the min-cut can serve as a valuable ``milestone'' for a network in the
process of learning to solve max-flow.

Therefore, we suggest integrating duality information directly into
the learning model, both as an additional signal to supervise on and
as an input feature, by allowing the network to reuse its dual
prediction in subsequent algorithmic steps. We name this approach Dual
Algorithmic Reasoning (DAR) \citep{numeroso2023dual}.

In the next sections, we delve into specifications of the model
architecture utilised for learning Ford-Fulkerson as well as
discussion on experimental settings and analysis.

\subsection{Model architecture}
In spirit of algorithmic alignment, we slightly modify the usual
EPD architecture to more closely match the dynamics of the
Ford-Fulkerson algorithm. A visual representation of the used
architecture is provided in \autoref{fig:dar-architecture}.

Precisely, we still rely on the neural algorithmic reasoning blueprint
\citep{velickovic2021neural} and the EPD architecture, but we utilise
two processors $p_{BF}$ and $p_F$, instead of one, to align with the
two sub-routines composing Ford-Fulkerson (see pseudocode of
\autoref{alg:ff}).  As detailed in \autoref{sec:graph-algos}, one
procedure finds the \emph{augmenting paths}, while the other performs
update of the flow within the graph (i.e., $f_{uv} = f_{uv} + df$ and
$f_{vu} = f_{vu} - df$).
\begin{figure}
  \centering
  \input{gfx/contrib/dar/net}
  \caption{High-level architecture of the Dual Algorithmic Reasoner
    (DAR) used for max-flow and min-cut. Note how $p_{BF}$ is iterated
    to retrieve augmenting paths, whereas $p_F$ is required to simultaneously
    solve max-flow ($g_P$) and min-cut ($g_D$). The edges have all unitary
    capacity. Minimum cut is represented as dotted edges.}
  \label{fig:dar-architecture}
\end{figure}
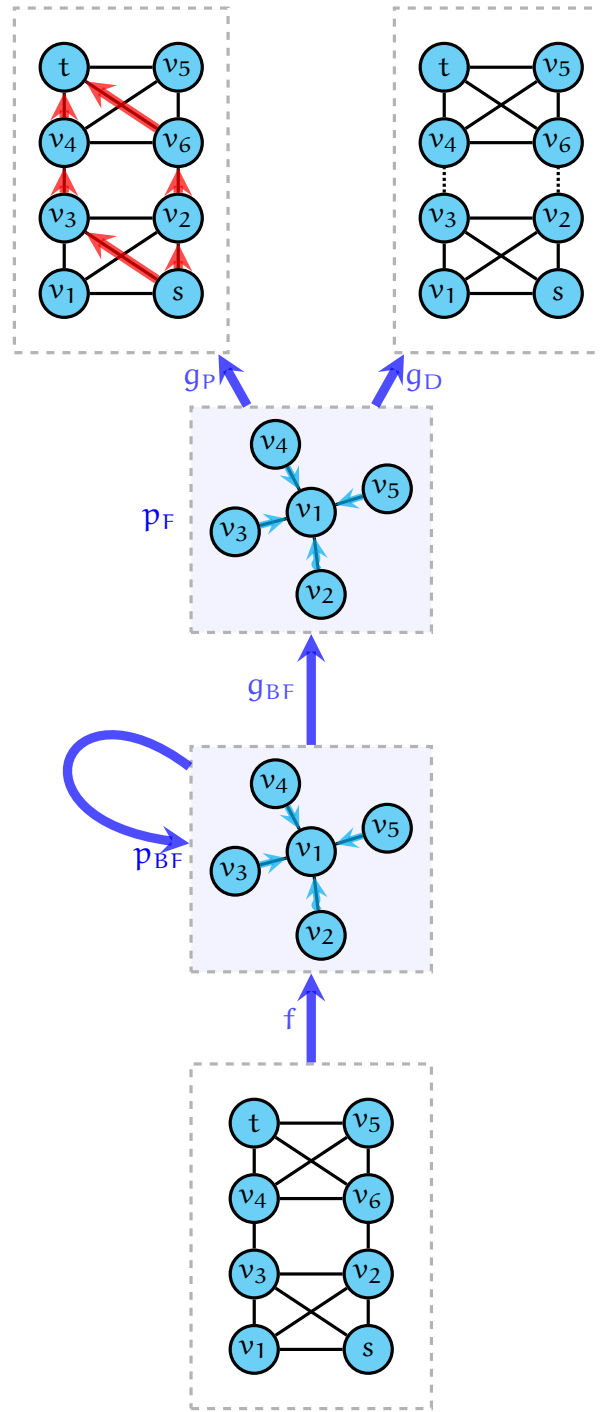

Consequently, $p_{BF}$ learns to retrieve augmenting paths, while
$p_F$ emits simultaneously the minimum cut $\vc$ and the flow
assignments in the form of a matrix
$\mF \in \sR^{|\gV| \times |\gV|}$, where:
\begin{align*}
  &\vc_v = \begin{cases}
            0 & \text{if $v$ is in the same cluster of $s$}\\
            1 & \text{if $v$ is in the same cluster of $t$}
          \end{cases},\\\\
  &\mF_{uv} \approx f_{uv} \quad \forall (u, v) \in \gE.
\end{align*}
Specifically, $p_{BF}$ is trained to find augmenting paths through
execution of the Bellman-Ford algorithm, similar to
\citep{georgiev2020neural}. The sets of \inp{}, \hints{}, and \out{}
used to train the architecture is the same as described in
\autoref{sec:graph-algos}.

The model ``forward pass'', similarly, is mostly unchanged compared to
\autoref{sec:epd}. Indeed, we still implement encoders and decoders as
linear layers, whilst $p_F$ is fed with the output of
$p_{BF}$, as depicted in \autoref{fig:dar-architecture}. Both
processors are instantiated as either MPNNs or PGNs (recall
\autoref{sec:epd}). Precisely, the augmenting paths from $s$ to $t$
$\vp^{(t)}_i$ are first decoded through a decoder $g_{BF}$:
\begin{equation}
  \vp_v^{(t)} = g_{BF}(\vz^{(t)}_{v}, \vh^{(t)}_v) \quad \forall v \in \gV,
\end{equation}
and then passed to $p_F$ as an input feature to get node
representations $\vh^{(t)}_v$. Finally, we decode the target
quantities of the algorithm as:
\begin{align*}
  &\tilde{\mF}^{(t)}_{uv} = g_P(\vh^{(t)}_u, \vh^{(t)}_v),\\
  &\vc^{(t)}_v= g_D(\vh^{(t)}_v).
\end{align*}
Here, $g_P$ and $g_D$ are respectively the primal and dual encoders.
\paragraph{Satisfying Max-Flow constraints.}
In order to comply with max-flow constraints discussed in
\autoref{prob:maxflow}, we ought to pay extra care on the $\mF^{(t)}$
predictions. Given a graph $g=\tuple{\gV, \gE, \gC}$, where
$c_{uv} \in \gC$ is the capacity of the edge $(u, v)$, we have to
ensure that \eqref{eq:capacity-constraint} and
\eqref{eq:flow-conservations} are satisfied. Notably, by looking at
the problem definition in \autoref{prob:maxflow} and
\eqref{eq:flow-conservations}, we derive that the matrix $\mF$ must be
antisymmetric (i.e., $\mF = -\mF^\top$). This can be seen trivially by
noting that if node $u$ sends a certain amount of flow quantity $q$ to
node $v$, then $v$ must receive $q$ flow from $u$ (i.e.,
$\mF_{uv} = -\mF_{vu}$). Hence, we have a ``0-th constraint'' to
satisfy, which is that of antisymmetry of $\mF$. To comply with this,
we apply the following transformation to $\mF$:
\begin{equation*}
  \mF_{\mathsc{a-sym}} = \tilde{\mF} - \tilde{\mF}^\top,
\end{equation*}
where $\tilde{\mF}$ is the raw prediction of the model and
$\mF_{\mathsc{a-sym}}$ is inherently ensured to preserve antisymmetry.
Moving forward, to ensure \eqref{eq:capacity-constraint}, we employ
the \textit{tanh trick}.  Precisely, at each step of the algorithm, we
scale each entry of $\mF^{(t)}$ according to the hyperbolic tangent
and the capacity matrix $\mC = [c_{uv}] \in \sR^{|\gV| \times |\gV|}$,
as follows:
\begin{equation} \label{eq:rescaled-flow}
  \mF^{(t)} = \tanh(\mF^{(t)}_{\mathsc{a-sym}}) \odot \mC,
\end{equation}
where $\tanh$ and $\odot$ are applied element-wise. Here, applying
$\tanh$ rescales the prediction to $[-1, 1]$. Hence, each entry of
$(u, v)$ of the resulting matrix can be thought of as a prediction of
how much capacity $c_{uv}$ we want to use (positively or negatively).

Finally, ensuring conservation of flows (equation
\eqref{eq:flow-conservations}) is more challenging. Indeed, there is
no clear way of how we can satisfy the constraint in an end-to-end
differentiable manner. For this reason, to address this constraint we
resort to a hand-crafted corrective procedure that is executed at the
end of the neural execution of Ford-Fulkerson, whose pseudocode is
reported in \autoref{alg:corrective-procedure}.

\input{content/algs/dar_corrective_procedure.tex}

As detailed in the pseudo-code, the procedures discriminate between
negative nodes $\gV^-$ (i.e., nodes $v$ that \emph{keep} some of the
in-flow):
\begin{equation*}
  \gV^- = \set{v \mid v \in \gV \land \sum_{(i, v) \in \gE} \mF_{iv} +
    \sum_{(v, i) \in \gE} \mF_{vi} < 0},
\end{equation*}
and positive nodes $V^+$ (i.e., nodes $v$ that sends out \emph{more}
flow that they receive):
\begin{equation*}
  \gV^+ = \set{v \mid v \in \gV \land \sum_{(i, v) \in \gE} \mF_{iv} +
    \sum_{(v, i) \in \gE} \mF_{vi} > 0}
\end{equation*}

Thus, flow conservations is enforced through two steps:
\begin{enumerate}
\item Each $v \in \gV^-$ sends back to the source node $q$ flow, where
  $q$ is equal to the magnitude of the constraint
  violation. Similarly, all $v \in \gV^+$ take back from successors an
  amount of flow corresponding to the constraint violation.  Note that
  here, we might have re-introduced some edge capacity constraints
  violation in $\mF$.
\item We enforce back the edge capacity constraint by clamping
  $\mF_{uv}$ to $c_{uv}$, and progressively adjusting the flow sent
  from $s$ until no violations of \eqref{eq:capacity-constraint} and
  \eqref{eq:flow-conservations} occur.
\end{enumerate}

\subsection{Experimental evaluation}
We test DAR with two objectives in mind.

First, we aim to assess whether the augmented reasoning capabilities
of dual algorithmic reasoners really help in the task of learning to
execute Ford-Fulkerson. For this purpose, we train and test the DAR
arhictecture on synthetic graphs. We dive into details of
this setting in \autoref{sec:dar-synthetic-exp}.

Second, we deploy DAR to solve a challenging graph learning task, with
real world data, in order to prove the usefulness of algorithmic
priors discussed in \autoref{sec:nar-motivation} and
\autoref{sec:transfer}. Precisely, we test DAR against
state-of-the-art baselines on a biologically relevant edge
classification task on Brain-Vessel Graphs (BVG)
\citep{paetzold2021whole}. This benchmark comprehends large-scale
graphs and is discussed more thoroughly in \autoref{sec:dar-bvg-exp}.

\subsubsection{Executing Ford-Fulkerson}
\label{sec:dar-synthetic-exp}
In this section, we provide all the details for reproducing our experiments,
as well as a complete discussion on the obtained results.

\paragraph{Ablation studies.}
Ablation studies are conducted to rigorously evaluate the impact of
the dual information, by training the DAR architecture devoid of the
additional \emph{dual decoder} $g_D$ (i.e., outputting min-cut
predictions), thereby preventing joint supervision together with the
primal max-flow problem. We refer to this architecture as the
\emph{primal} architecture, marked as \emph{(P)} in the following
paragraphs. The full model (that of \autoref{fig:dar-architecture}),
instead, is referred to as the \emph{dual} model
\emph{(D)}. Additionally, we also test a neural architecture with an
extra processor designed to learn the minimum cut before employing the
Ford-Fulkerson algorithm. This is intended to check whether training
jointly on both primal and dual problems is really needed, or having
the minimum cut as an input feature would suffice. We refer to such
architectures as \emph{pipeline} models, marked as \emph{(PL)}.

Furthemore, two distinct processor types are used to instantiate both
$p_{BF}$ and $p_F$ with:
\begin{enumerate}
\item A fully-connected MPNN, executing \eqref{eq:gn-mp} on a full
  graph regardless of its connectivity. We note that this approach
  is commonly used in NAR.
\item A PGN, which imposes exchange of node messages on the graph true
  connectivity and its algorithmically defined neighbourhoods, especially
  via \hints{} (e.g., augmenting paths output by $p_{BF}$ are also used to
  restrict the exchanged messages).
\end{enumerate}

Furthermore, given the difficulties of linking the max-flow problem to
the dynamic programming paradigm, we experiment with other types of GN
aggregators other than $\max$ as well. Specifically,
$\oplus = \set{\max, \mean, \sumt}$. All models undergo training for
20,000 epochs using the SGD optimiser, with results averaged over five
runs. \emph{Teacher forcing} is employed with a decay factor of 0.999
on the \hints{}, offering the network ground-truth \hints{} during
initial training stages while allowing network predictions during the
majority of the training.

\paragraph{Model selection.}
Similarly to experiments performed in \autoref{sec:neural-planning}, a
bi-level random search optimisation scheme is employed here as well.
Details on the hyperparameter space are reported in
\autoref{tab:rs-bilevel-dar}. Specifically, we sample $n=50$
configurations for each level of the bi-level random search.  The best
hyperparameters are chosen based on the validation error on $\mF$
predictions, also shown in \autoref{fig:dar-val-synthetic}. In the
figure, \emph{(P)}, \emph{(D)}, and \emph{(PL)} follow the notation
introduced in the previous paragraph.

\begin{table}
  \centering
  \footnotesize
  \begin{tabular}{l | c C c}
    \toprule
    Level & Hyperparameter & [a,b] & Distribution \\
    \midrule
    \multirow{3}{4em}{Level 1} & Learning rate & [1 \cdot 10^{-5}, 1 \cdot 10^{-1}] & log-uniform\\
          & Weight decay & [1 \cdot 10^{-5}, 1 \cdot 10^{-1}] & log-uniform\\
          & Hidden dimension & [16, 512] & uniform\\
    \midrule
    \multirow{3}{4em}{Level 2} & Learning rate & [1 \cdot 10^{-3}, 1 \cdot 10^{-2}] & log-uniform\\
          & Weight decay & [1 \cdot 10^{-3}, 4 \cdot 10^{-3}] & log-uniform\\
          & Hidden dimension & [60, 100] & uniform\\
    \bottomrule
  \end{tabular}
  \caption{DAR first-level and second-level random search configurations.}
  \label{tab:rs-bilevel-dar}
\end{table}

\paragraph{Data generation.}
For this set of experiments, we generate two sets of graphs, following
two different distributions.  First, \emph{2-community} graphs (see
\autoref{sec:graph-types}), in which communities are Erdős--Rényi
graphs with $p=0.75$. Thus, nodes in different communities are
interconnected with $p=0.05$.  Second, \emph{bipartite} graphs (see
also \autoref{sec:graph-types}).  To generate algorithmic
trajectories, we still employ the CLRS benchmark and
training/validation/test setups are in CLRS-style, which is: 1000
training samples in the form of 2-community graphs having 16 nodes;
128 validation samples of 2-community graphs (same size as training
one); 256 test samples of size 64 nodes, of which 128 are 2-community
graphs and 128 are bipartite graphs.  Following this approach, we not
only enable testing OOD (on bigger 2-community graphs), but also
\emph{out-of-family} (OOF) (on larger bipartite graphs).

Lastly, edge capacities for 2-community graphs are sampled in $[0,10]$
and then rescaled via a min-max normalisation, whereas bipartite graphs
have discrete edge capacities in $\set{0, 1}$.

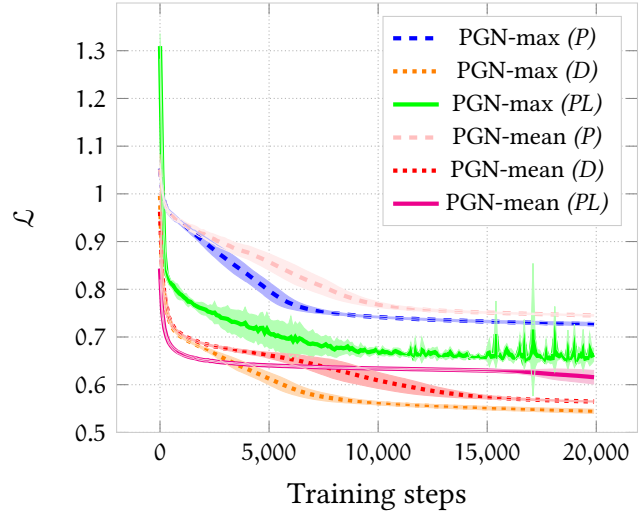
\begin{figure}
  \centering
  \input{gfx/contrib/dar/valid}
  \caption{Ford-Fulkerson validation loss on synthetic data
    (PGNs). Note how the family of primal models underperform also
    in-distribution.}
  \label{fig:dar-val-synthetic}
\end{figure}

\begin{figure}
  \centering
  \input{gfx/contrib/dar/reconstruction}
  \caption{Normalised loss curve of reconstructing Ford-Fulkerson with
    new encoders for $l_{ij}, d_{ij}, \rho_{ij}$. Contrarily to the
    primal model, PGN-max \emph{(D)} seems to be able to generalise
    the algorithm with the new data.}
  \label{fig:loss-reconstruction}
\end{figure}
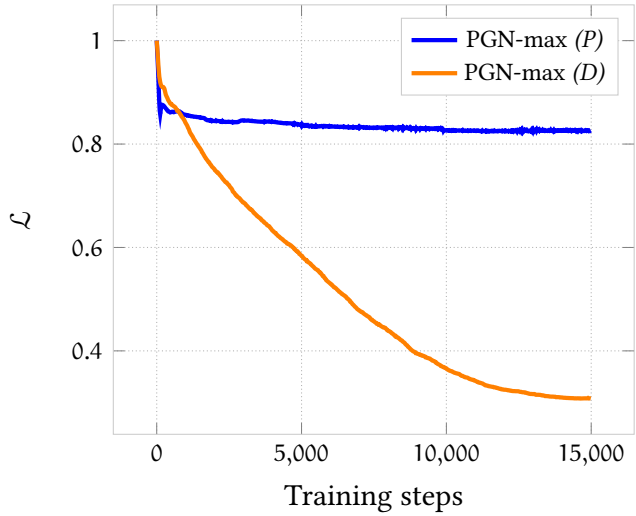

\paragraph{Quantitative analysis.}
Results regarding the neural approximation of Ford-Fulkerson are
documented in \autoref{tab:dar-synthetic-results}, employing Mean
Absolute Error (MAE) as a performance metric for evaluating final flow
assignment predictions $\mF$.  Concurrently, we measure how well the
algorithm is approximated across all steps of the algorithm (i.e.,
$\bar{\mF}^{(t)}$ column in \autoref{tab:dar-synthetic-results}). We
also report minimum cut accuracy with respect to ground truths. This
metric is only relevant for \emph{dual} and \emph{pipeline} models.
We also compare all networks to an algorithm-agnostic graph network
(marked as \emph{(NA)} in the table) which directly produces the
matrix $\mF$ without approximating Ford-Fulkerson. As clearly shown in
the table, such variant demonstrated inferior to its
algorithmically-informed counterparts, testifying the need for
algorithmic reasoning to output reliable max-flow predictions.

In our testing, the \textit{dual} PGN model utilising a max aggregator
emerged as the top-performing model, at least concerning
\textit{2-community} graphs, demonstrating the capability to also
flawlessly predict the minimum cuts of all graphs in the test set.
More in general, models that integrate dual problem prediction (i.e.,
\emph{dual} and \emph{pipeline} models) consistently surpass the
\emph{primal} models on the two test sets (\emph{OOD} and \emph{OOF}),
with respect to both the final prediction $\mF$ as well as throughout
the algorithm executions (i.e., less error on $\bar{\mF}^{(t)}$).
This implies that, indeed, dual min-cut information helps to output
more accurate max-flow predicitons. To further strenghten this claim,
note how performance on max-flow quickly degrades towards that of the
primal baseline whenever min-cut predictions are inaccurate (e.g., see
\{PGN, MPNN\}-max \emph{(PL)} on \emph{2-community} graphs). Notably,
note how \emph{(PL)} predictions are less stable on average. This
confirms prior findings of \citep{xhonneux2021transfer} that deploying
MTL when applicable yields, on average, better OOD results.

Testing OOF, however, proves to be still a challenging task even
for neural algorithmic reasoners. Even though the performance gap
enlarges for bipartite graphs (with both \textit{dual} and
\textit{pipeline} proving highly competitive), we achieve on average
higher mean error and standard deviations. For the sake of completeness,
we have to take into account that capacities here are sampled as either
0 or 1, thus amplying potential for prediction errors.

\begin{table}
  \centering
  \setlength{\tabcolsep}{2.1pt}
  \footnotesize
  \begin{tabular}{l C C C C C C}
    \toprule
    & \multicolumn{3}{c}{Community \textit{(OOD)}} & \multicolumn{3}{c}{Bipartite \textit{(OOF)}}\\
    \midrule
    \textbf{Model} & \mF (\downarrow) & \bar{\mF}^{(t)} (\downarrow) & \vc (\uparrow)& \mF (\downarrow) & \bar{\mF}^{(t)} (\downarrow)& \vc (\uparrow)\\
    PGN-max \emph{(P)} & 0.266_{\pm 0.001} & 0.294_{\pm 0.002} & -
                                      & 0.56_{\pm 0.23} & 0.82_{\pm 0.17} & - \\
    PGN-mean \emph{(P)} & 0.274_{\pm 0.001} & 0.311_{\pm 0.004} & -
                                      & 1.09_{\pm 0.47} & 1.13_{\pm 0.18} & - \\
    MPNN-max \emph{(P)}   & 0.263_{\pm 0.008} & 0.289_{\pm 0.004} & -
                                      & 0.75_{\pm 0.47} & 0.78_{\pm 0.11} & -\\
    MPNN-mean \emph{(P)} & 0.278_{\pm 0.008} & 0.313_{\pm 0.003} & -
                                      & 0.75_{\pm 0.47} & 0.92_{\pm 0.22} & -\\
    \midrule
    PGN-max \emph{(D)} & \mathbf{0.234_{\pm 0.002}} & \mathbf{0.269_{\pm 0.001}} & 100\%_{\pm 0.0}
                                      & 0.49_{\pm 0.22} & \mathbf{0.78_{\pm 0.29}} & 100\%_{\pm 0.0} \\
    PGN-mean \emph{(D)} & 0.240_{\pm 0.004} & 0.285_{\pm 0.004} & 100\%_{\pm 0.0}
                                      & 1.10_{\pm 0.30} & 1.05_{\pm 0.12} & 99\%_{\pm 0.7}  \\
    MPNN-max \emph{(D)} & 0.236_{\pm 0.002} & 0.288_{\pm 0.005} & 100\%_{\pm 0.0}
                                      & 0.71_{\pm 0.32} & 0.98_{\pm 0.22} & 100\%_{\pm 0.0} \\
    MPNN-mean \emph{(D)} & 0.258_{\pm 0.008} & \mathbf{0.268_{\pm 0.002}} & 100\%_{\pm 0.0}
                                      & 0.81_{\pm 0.09} & 1.06_{\pm 0.35} & 100\%_{\pm 0.0}\\
    \midrule
    PGN-max \emph{(PL)} & 0.256_{\pm 0.001} & 0.293_{\pm 0.003} & 61\%_{\pm 0.1}
                                      & \mathbf{0.45_{\pm 0.18}} & \mathbf{0.77_{\pm 0.26}} & 95\%_{\pm 0.1} \\
    PGN-mean \emph{(PL)} & 0.244_{\pm 0.001} & 0.304_{\pm 0.001} & 100\%_{\pm 0.0}
                                      & 0.98_{\pm 0.44} & 1.03_{\pm 0.32} & 99\%_{\pm 0.8} \\
    MPNN-max \emph{(PL)} & 0.261_{\pm 0.002} & 0.312_{\pm 0.005} & 61\%_{\pm 0.3}
                                      & \mathbf{0.47_{\pm 0.23}} & 0.95_{\pm 0.34} & 90\%_{\pm 1.1}\\
    MPNN-mean \emph{(PL)} & 0.255_{\pm 0.002} & 0.292_{\pm 0.002} & 100\%_{\pm 0.0}
                                      & 0.64_{\pm 0.35} & 0.92_{\pm 0.20} & 100\%_{\pm 0.0}\\
    \midrule
    Random & 0.740_{\pm 0.002} & - & 50\%_{\pm 0.0} & 1.00_{\pm 0.00} & - & 50\%_{\pm 0.0} \\
    PGN-max \emph{(NA)} & 0.314_{\pm 0.013} & - & - & 0.78_{\pm 0.02} & - & - \\
    \bottomrule
  \end{tabular}
  \caption{Here, we report different quantitative metrics to assess
    the performance of all different models. Columns marked with
    $(\downarrow)$ are evaluated through the Mean Absolute Error
    (MAE), whereas for columns marked with $(\uparrow)$ accuracy is
    used. $\mF$ refers to the final flow assignment,
    $\bar{\mF}^{(t)}$ refers to the intermediate flow matrices
    (error is averaged), and $\vc$ refers to minimum cut. We report
    all tested models as rows, where \textit{(P)} corresponds to
    training on max-flow only, \textit{(D)} corresponds to training
    with both primal and dual decoders and \textit{(PL)} corresponds to
    the pipeline model. Finally, \textit{(NA)} corresponds to a model
    utilising no algorithms, hence learning max-flow directly.}
  \label{tab:dar-synthetic-results}
\end{table}

Finally, we provide illustrations showcasing how application of
\autoref{alg:corrective-procedure} enforce constraint
\eqref{eq:flow-conservations}. In \autoref{fig:dar-flow-heatmaps}, we
represent negative nodes $v^-$ as blue and positive nodes $v^+$ as
red. Sources $s$ should always be ``positive'' nodes (flow is always
sent out) while sinks should be ``negative'' (flow is always
received). White nodes represent nodes that are ``neutral'' (i.e.,
algebraic sum of their in-flow and out-flow is 0), hence
\eqref{eq:flow-conservations} holds. The figure then shows nodes
violating \eqref{eq:flow-conservations} pre- and post-application
of \autoref{alg:corrective-procedure}.
\begin{figure}
    \centering
    \subcaptionbox{Flow pre-correction}{
      \includegraphics[width=.4\linewidth]{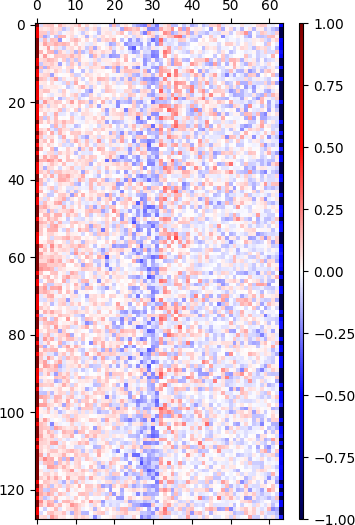}
    }
    \subcaptionbox{Flow post-correction}{
      \includegraphics[width=.4\linewidth]{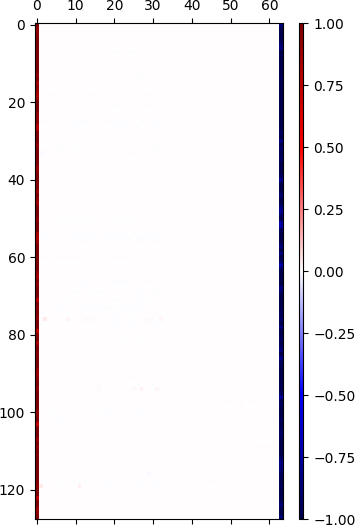}
    }
    \caption{This figure shows the (normalised) flow conservation
      error for all nodes (x axis) of each test graph (y axis). (a)
      shows the predicted $\mF$ conservation error prior
      \autoref{alg:corrective-procedure}. (b) shows the corrected
      $\mF$ conservation error.  Through
      \autoref{alg:corrective-procedure} we are able to correct all
      the inner error (white represents zero error), with $s$ and $t$
      being the only nodes which respectively send and receive flow.
    }
    \label{fig:dar-flow-heatmaps}
\end{figure}

\paragraph{Qualitative analysis.}
To further evaluate the performance of DAR, we conducted two
qualitative studies. The first one aims to assess to what extent DAR
architecture can identify and leverage strong duality
properties. Indeed, as we discussed in \autoref{sec:dar}, we know that
max-flow and min-cut are \emph{strong duals} and, as such, they share
they same optimal value (see \autoref{th:strong-duality-theorem}).
Hence, given a dual model architecture in
\autoref{tab:dar-synthetic-results} that can correctly predict the
minimum cut of a graph, this study aims at answering the following
question: \emph{can such an architecture identify the optimum max-flow
  value more accurately compared to a primal-only model?}

To investigate this property, we have to extract from $\mF$ the flow
quantity that our learnt reasoners ``believe'' to be the optimum.  We
consider such quantity as being either the predicted \textbf{out-flow}
from the source node or \textbf{in-flow} of the target node $t$, by
``marginalising'' over all $s$'s successors and $t$'s predecessors,
respectively. Thus, we consider the maximum (in absolute value) of the
two quantities. We denote this value as $f_s$ and compute it as
follows:
\begin{equation}\label{eq:f-optimum-nn}
  f_s = \max(\sum_{(s, j) \in \gE} \mF_{sj},|\sum_{(j, t) \in \gE} \mF_{jt}|).
\end{equation}
Note that, in doing so, we are ignoring potential violations of
constraint \eqref{eq:flow-conservations}. However, our objective here
is to measure how well the network can guess the optimal value.
Hence, we compare $f_s$ to the ground truth optimal max-flow value
$f^*_s$, which is computed exactly as \eqref{eq:f-optimum-nn} but on
the ground-truth $\mF^*$. Precisely, we measure the discrepancy as MAE
and report it in \autoref{tab:dar-qualitative-results-1}. From the
presented results, one can immediately draw a very interesting
conclusion: DAR models can effectively leverage the fact that the two
learning targets share the same optimum value. This is testified by
noting that DAR models, especially the ones trained jointly (i.e.,
\emph{(D)}), achieve an error of $\approx 0.34$ from the optimum,
which is one order of magnitude lower than primal models (i.e.,
$\approx 7.86$). This analysis strengthens our claim that a DAR model
can effectively transfer knowledge from the dual to the primal
problem, leading to solutions that are both more accurate (see
\autoref{tab:dar-synthetic-results}) and of higher quality. This also
finds additional support in the observation that both primal
architectures and dual architectures, where the minimum cut results
are inferior, deviate significantly from the optimal solution.  This
can be seen by comparing the PGN-max \emph{(PL)} minimum-cut results
in \autoref{tab:dar-synthetic-results} and the corresponding MAE value
in \autoref{tab:dar-qualitative-results-1}.
\begin{table}
  \centering
  \footnotesize
  \setlength{\tabcolsep}{2pt}
  \begin{tabular}{l c c c c c c}
    \toprule
    &\multicolumn{2}{c}{\emph{(P)}}
    &\multicolumn{2}{c}{\emph{(D)}}
    &\multicolumn{2}{c}{\emph{(PL)}}\\
    \textbf{Metric}
    & \textbf{PGN-max} & \textbf{PGN-mean}
    & \textbf{PGN-max} & \textbf{PGN-mean}
                       & \textbf{PGN-max} & \textbf{PGN-mean}
    \\
    \midrule
    $|f_s - f^*_s| (\downarrow)$ & $7.86_{\pm 0.47}$ & $8.68_{\pm 0.21}$
    & $\mathbf{0.34_{\pm 0.04}}$ & $0.41_{\pm 0.01}$
                       & $7.58_{\pm 0.10}$ & $0.38_{\pm 0.08}$ \\
    \bottomrule
  \end{tabular}
  \caption{Qualitative analysis. Here, $f^*_s$ is the ground truth
    maximum flow value. $f_s$, instead, is the quantity that our
    learnt reasoners predict to be the amount of flow \textbf{exiting}
    the source node $s$ (i.e.,
    $f_s = \sum_{(s, j) \in \gE} \mF_{sj}$). For simplicity, we only
    report results of the best-performing models (PGNs).}
  \label{tab:dar-qualitative-results-1}
\end{table}
The second qualitative study, instead, is embeddings-centric. Here,
our objective is to assess the degree to which $f^*_s$ can be linearly
decodable from the learnt node representations. Ideally, this would
provide clear indication of which model representations carry the most
information. We setup this study as follows. First, we neurally
execute Ford-Fulkerson through all PGN-based models in
\autoref{tab:dar-synthetic-results}, and collect learnt node
representations $\vh_v$. Thus, we compute a graph representation
$\vh_g$ as follows:
\begin{equation*}
  \vh_g = \max(\set{\vh_v \mid v \in \gV}).
\end{equation*}
Finally, we learn linear mappings $\vh_g \rightarrow f^*_s$ and
compute the $R^2$ score to evaluate the quality of the
predictions. \autoref{tab:dar-qualitative-results-2} provide clear
indications that \emph{dual} representations are more expressive,
since $f^*_s$ can be decoded linearly almost perfectly (i.e.,
$\approx 0.99$ on average on the two graph distributions).
\begin{table}
  \centering
  \scriptsize
  \setlength{\tabcolsep}{2pt}
  \begin{tabular}{l c c c c c c}
    \toprule
    & \multicolumn{2}{c}{\emph{(P)}}
    & \multicolumn{2}{c}{\emph{(D)}}
    & \multicolumn{2}{c}{\emph{(PL)}}\\
    \textbf{Graphs}
    & \textbf{PGN-max} & \textbf{PGN-mean}
    & \textbf{PGN-max} & \textbf{PGN-mean}
    & \textbf{PGN-max} & \textbf{PGN-mean}\\
    \midrule
    \textit{community} $(\downarrow)$ & $0.69_{\pm 0.12}$ & $-0.59_{\pm 0.33}$
    & $\mathbf{1.00_{\pm 0.00}}$ & $0.93_{\pm 0.01}$
                       & $0.80_{\pm 0.05}$
    & $-0.13_{\pm 0.29}$ \\
    \textit{bipartite} $(\downarrow)$ & $-0.79_{\pm 0.81}$ & $-115.6_{\pm 64}$
    & $\mathbf{0.98_{\pm 0.01}}$ & $-115.2_{\pm 64}$
                       & $-0.69_{\pm 0.33}$
    & $\mathbf{0.98_{\pm 0.01}}$ \\
    \bottomrule
  \end{tabular}
  \caption{$R^2$ score of predicting the maximum flow value from
    learnt graph representations $\vh_g$. $R^2$ metric emits score in
    ($-\infty, 1$], with $1$ being the best possible score.}
  \label{tab:dar-qualitative-results-2}
\end{table}

\subsubsection{Brain Vessel Graph benchmark} \label{sec:dar-bvg-exp} In this
section, we prove for the first time the real-world utility of neural
algorithmic reasoners, in particular of DAR architectures.  Here, we
target tranferring algorithmic knowledge to a task where prior
knowledge of max-flow might be helpful, providing a practical
instatiation of the discussions details in
\autoref{sec:nar-motivation}. In the following, we will be providing a
thorough description of the real-world benchmark we used as well as
experimental details.

\paragraph{Benchmark description.}
We consider the Brain Vessel Graph (BVG) benchmark
\citep{paetzold2021whole}. Here, graphs represent mice's brains, where
nodes represent bifucartion points and edges represent vessels. The
objective is to perform \textbf{edge classification} task. In
particular, edges need to be classified in each of the following three
categories: \emph{capillaries}, \emph{veins} and
\emph{arteries}. Equivalently, we can frame this task as conducting
blood-flow simulation\footnote{This can be seen by considering that
  the more blood can traverse an edge, the more likely the edge is to
  be an artery rather than a capillary}. Hence, a network that is able
to \emph{simulate} blood-flow simulation is likely to be at advantage
compared to a simpler model. This represents the ideal test case for
our DAR architectures.

Furthermore, for each edge we have access to the following input
feature list:
\begin{enumerate}
\item[$l_{uv}$]: which is the length of the vessel.
\item[$d_{uv}$]: which is the shortest distance between the two
  bifurcation points $u$ and $v$.
\item[$\rho_{uv}$]: which is the curvature of the vessel.
\end{enumerate}
Notably, we do not have access to the capacity of the vessels. This
makes direct application of the Ford-Fulkerson algorithm inapplicable
to conduct ``predictions'' on these types of data. This is a scenario
where algorithms do exhibit the limitations discussed in
\autoref{sec:nar-motivation}.

Additionally, the classification task is highly imbalanced (i.e.,
$95\%$ of samples are capillaries, $4\%$ veins and only $1\%$
arteries), posing an additional learning challenge.

We test our models on three large-scale BVG graphs: CD1-E-1 (the
largest, with 5,791,309 edges); CD1-E-2 (2,150,326 edges); and CD1-E-3
(3,130,650 edges). We also exploit a validation synthetic brain vessel
graph provided by the BVG benchmark, comprising 3159 nodes and 3234
edges.

\paragraph{Transfer algorihtmic knowledge to BVG.}
We employ the PF transfer learning strategy discussed in
\autoref{sec:transfer}, with slight modifications, to re-use
algorithmic knowledge lying in our processors. Specifically, since we
have access to a validation synthetic graph from BVG, we can attempt
at re-learning to execute Ford-Fulkerson on such data, even if the
capacity is no longer an input feature. To do that, we perform
the following steps:
\begin{itemize}
\item Extract and freeze the PGN-max processors $p$ from the models
  trained on synthetic data of \autoref{sec:dar-synthetic-exp}.
\item Drop the encoders $f$, but \textbf{keep} the decoders $g$.
\item Introduce a new encoder $\tilde{f}$.
\item Learn to approximate Ford-Fulkerson on the validation BVG graph
  by optimising:
  \begin{equation*}
    \argmin_\phi \gL\left(\mathsc{ff}(\vx),
      (g \circ p \circ \tilde{f})(\vx \mid \phi)\right),
  \end{equation*}
  where \textsc{ff} represent the Ford-Fulkerson algorithm and $\phi$
  are the parameters of the new encoder.
\end{itemize}
Through the above steps, we effectively \textit{reconstruct} the
execution of the Ford-Fulkerson algorithm on the new data. Precisely,
the model is tasked to learn to combine the new inputs in a sensible
way to estimate the edge flows and, consequently, the capacity of the
edges (i.e., our main target in the BVG task). For reconstruction
purposes, source and sink nodes $s, t$ are chosen as two random nodes
whose shortest distance is equal to the diameter of the graph. We
train to reconstruct the algorithm for 15000 epochs, with Adam
optimiser \citep{kingma2015adam} and learning rate 1e-5.
\autoref{fig:loss-reconstruction} compares the loss curves for the
primal and DAR models.

Finally, we perform one Ford-Fulkerson step on the three target
graphs, CD1-E-1, CD1-E-2 and CD1-E-3 and extract the learnt
node representations. We then obtain edge representations of edges
$(u,v)$ as:
\begin{equation*}
  \vh_{uv} = \vh_u + \vh_v.
\end{equation*}
These representations will constitute an additional input feature for
the neural networks trained to solve the BVG task. Note how this
approach enables us to easily dump the embeddings, as no further
training will occur on such models.

\paragraph{Neural architectures.}
We have selected graph networks from the BVG benchmark paper as our
baseline models.

These include Graph Convolutional Networks (GCNs)
\citep{kipf2017semi}, GraphSAGE \citep{hamilton2017inductive}, and
ClusterGCN \citep{chiang2019cluster} with GraphSAGE convolution
(C--SAGE). We augment such networks with extra input features
$\vh_{uv}$ obtained as explained in the previous paragraph.
Additionally, we also utilise a simple linear classifier (LC) to
assess the quality of the $\vh_{uv}$ representations. We also train a
linear classifier on the set of original input features
$l_{ij}, d_{ij}, \rho_{ij}$, to assess whether $\vh_{uv}$ add
information compared to the basic features.

As an additional sanity check, we also train Node2Vec
\citep{grover2016node2vec} embeddings on each of the three datasets
and use them to augment the GN architectures.

\paragraph{Model selection.}
Here, we consider the optimal BVG hyperparameters reported in
\citep{paetzold2021whole}. For the CD1-E-1 graph and the C--SAGE
model, we performed minimal model selection as the optimal
hyperparameters for C--SAGE did not yield good results.
Specifically, we optimised \textit{number of layers} and
\textit{number of hiddens} by grid-searching over \{2, 3, 4\} and
\{64, 128\}, respectively. For consistency, all models are trained
with early stopping applied after 300 epochs, using the
hyperparameters reported in \autoref{tab:dar-bvg-hp}.
\begin{table}
  \centering
  \footnotesize
  \setlength{\tabcolsep}{3.25pt}
  \begin{tabular}{l | l l l l}
    \toprule
    Dataset & GCN & SAGE & C--SAGE & Node2Vec\\
    \midrule
    \multirow{5}{4em}{CD1-E-3} & lr=$3 \cdot 10^{-3}$
                  & lr=$3 \cdot 10^{-3}$ & lr=$3 \cdot 10^{-3}$ & lr=$1 \cdot 10^{-2}$\\
            & no. layers=3 & no. layers=4 & no. layers=4 & walk length=40 \\
            & no. hiddens=256 & no. hiddens=128 & no. hiddens=128 & walks per node=10 \\
            & dropout=0.4 & dropout=0.4 & dropout=0.2 & no. hiddens=128 \\
            & epochs=1500 & epochs=1500 & epochs=5000 & epochs=5 \\
    \midrule
    \multirow{5}{4em}{CD1-E-2} &
                  & & lr=$3 \cdot 10^{-3}$ &\\
            & & & no. layers=4 &  \\
            & {as above} & {as above} & no. hiddens=128 & {as above}\\
            & & & dropout=0.2 & \\
            & & & epochs=1500 & \\
    \midrule
    \multirow{5}{4em}{CD1-E-1} &
                  & & lr=$3 \cdot 10^{-3}$ & \\
            & & & no. layers=3 & \\
            & {as above} & {as above} & no. hiddens=64 & {as above}\\
            & & & dropout=0.2 &\\
            & & & epochs=1500 & \\
    \bottomrule
  \end{tabular}
  \caption{Optimal hyperparameters used for the BVG benchmark. Note
    that C--SAGE needs more epochs for CD1-E-3.}
  \label{tab:dar-bvg-hp}
\end{table}

\paragraph{Quantitative analysis.}
We perform rigorous and thorough testing of all combinations of the
the involved models (i.e., GNs only, GNs with \emph{primal}
embeddings, GNs with \emph{dual} embeddings, and GNs with \emph{N2V}
embeddings).  Given the highly imbalanced nature of the learning
problem, we employ the balanced accuracy score, which represents the
average recall for each class, and the area under the ROC curve as our
metrics to assess performance. Results are showcased in
\autoref{tab:dar-bvg-results-1}.

First, consider performance related to linear classifiers (LC).  It is
evident that the representations $\vh_{uv}$ generated by the
algorithmic reasoners, both the primal and dual versions, carry
valuable information. This is demonstrated by an average increase of
$16.6\%$ in balanced accuracy and $10.5\%$ in ROC across the three
datasets when compared to using simple features
$[l_{uv}, d_{uv}, \rho_{uv}]$ alone. Notably, dual representations
still outperform their primal counterparts, as a possible consequence
of a better algorithm reconstruction (depicted in
\autoref{fig:loss-reconstruction}). Lastly, LC's performance also
serves as a clear indicator of the successfull knowledge transferring
from the synthetic world to \textit{natural} real-world data.

Moving forward, augmented GN architectures significantly improves on
their vanilla (i.e., non-algorithmically informed) versions.  Among
these, C--SAGE with dual embeddings (i.e., C--SAGE \emph{(D)})
consistently achieves the highest performance across all three
datasets, displaying a consistent performance advantage, especially
for CD1-E-3 and CD1-E-2. More interestingly, the computed dual
$\vh_{uv}$ demonstrated to be more informative than N2V embeddings in
this task. Remarkably, N2V representations are directly trained on
CD1-E-\emph{X} data, whereas DAR is trained to reconstruct the
algorithm on validation (and much smaller) BVG data. Note that N2V
computes statistics unsupervisedly through random walks, in order to
capture global information of the graph. Hence, although N2V
embeddings are still valuable (as testified by the improvements w.r.t
vanilla GNs), DAR embeddings possess more principled flow information,
which better suits the BVG task. We also run an additional experiment
concatenating N2V and DAR embeddings, as an attempt to verify whether
the two representations encode information of ``different natures'' --
\emph{structural} (N2V) and \emph{flow-based} information (DAR).
We record higher results for these combined embeddings compared to
using N2V and DAR separately (especially for CD1-E-2), thus empirically confirming our
conjecture. These results are reported in
\autoref{tab:dar-bvg-results-2}.

\begin{table}
  \centering
  \footnotesize
  \setlength{\tabcolsep}{3pt}
  \begin{tabular}{l c c c c c c}
    \toprule
    & \multicolumn{2}{c}{CD1-E-3 $(\uparrow)$} & \multicolumn{2}{c}{CD1-E-2 $(\uparrow)$} & \multicolumn{2}{c}{CD1-E-1 $(\uparrow)$} \\
    \textbf{Model} & \textbf{Bal. Acc.} & \textbf{ROC} & \textbf{Bal. Acc.} & \textbf{ROC} & \textbf{Bal. Acc.} & \textbf{ROC} \\
    \midrule
    LC & $39.3\%_{\pm 0.2}$ & $52.3\%_{\pm 0.6}$ & $36.9\%_{\pm 0.5}$ & $55.9\%_{\pm 0.1}$ & $45.5\%_{\pm 0.1}$ & $61.7\%_{\pm 0.0}$\\
    LC \textit{(N2V)} & $43.9\%_{\pm 0.2}$ & $55.5\%_{\pm 0.1}$  & $71.9\%_{\pm 0.1}$ & $62.6\%_{\pm 0.0}$ & $46.1\%_{\pm 0.1}$ & $60.0\%_{\pm 0.0}$\\
    LC \textit{(P)} & $48.6\%_{\pm 0.4}$ & $59.4\%_{\pm 0.4}$ & $58.7\%_{\pm 0.1}$ & $63.8\%_{\pm 0.2}$ & $45.3\%_{\pm 0.1}$ & $59.9\%_{\pm 0.1}$\\
    LC \textit{(D)} & $53.8\%_{\pm 0.3}$ & $66.2\%_{\pm 0.2}$ & $67.3\%_{\pm 0.1}$ & $71.8\%_{\pm 0.0}$ & $48.1\%_{\pm 0.5}$ & $62.1\%_{\pm 0.3}$\\
    \midrule
    GCN & $58.1\%_{\pm 0.5}$ & $67.9\%_{\pm 0.2}$ & $74.6\%_{\pm 1.7}$ & $78.7\%_{\pm 0.1}$  & $59.0\%_{\pm 0.2}$ & $67.9\%_{\pm 0.2}$ \\
    \midrule
    SAGE & $63.5\%_{\pm 0.2}$ & $70.9\%_{\pm 0.3}$ & $73.9\%_{\pm 0.6}$ & $82.5\%_{\pm 0.2}$  & $64.7\%_{\pm 0.7}$ & $74.2\%_{\pm 0.2}$\\
    SAGE \textit{(N2V)} & $65.0\%_{\pm 0.1}$ & $71.9\%_{\pm 0.1}$ & $84.1\%_{\pm 1.9}$ & $82.5\%_{\pm 0.4}$  & $65.9\%_{\pm 0.1}$ & $74.8\%_{\pm 0.1}$  \\
    SAGE \textit{(P)} & $64.5\%_{\pm 0.2}$ & $72.0\%_{\pm 0.2}$  & $83.8\%_{\pm 0.4}$ & $83.7\%_{\pm 0.4}$ & $66.2\%_{\pm 0.5}$ & $74.8\%_{\pm 0.3}$ \\
    SAGE \textit{(D)} & $66.7\%_{\pm 0.4}$ & $75.0\%_{\pm 0.2}$ & $85.2\%_{\pm 0.1}$ & $85.5\%_{\pm 0.2}$ & $66.4\%_{\pm 0.3}$ & $74.8\%_{\pm 0.1}$ \\
    \midrule
    C--SAGE & $68.6\%_{\pm 0.8}$ & $74.2\%_{\pm 0.5}$ & $81.8\%_{\pm 0.5}$ & $85.6\%_{\pm 0.2}$ & $59.3\%_{\pm 0.9}$ & $68.3\%_{\pm 0.5}$\\
    C--SAGE \textit{(N2V)} & $68.6\%_{\pm 0.2}$ & $74.1\%_{\pm 0.1}$ & $84.8\%_{\pm 0.2}$ & $84.8\%_{\pm 0.5}$ & $67.4\%_{\pm 0.6}$ & $\mathbf{75.9\%_{\pm 0.2}}$\\
    C--SAGE \textit{(P)} & $67.3\%_{\pm 0.2}$ & $73.6\%_{\pm 1.9}$ & $82.5\%_{\pm 1.9}$ & $84.0\%_{\pm 1.5}$ & $67.7\%_{\pm 0.1}$ & $\mathbf{75.8\%_{\pm 0.2}}$\\
    C--SAGE \textit{(D)} & $\mathbf{70.2\%_{\pm 0.2}}$ & $\mathbf{76.3\%_{\pm 0.1}}$ & $\mathbf{85.6\%_{\pm 0.2}}$ & $\mathbf{86.7\%_{\pm 0.3}}$ & $\mathbf{68.1\%_{\pm 0.2}}$ & $\mathbf{75.8\%_{\pm 0.1}}$\\
    \bottomrule
  \end{tabular}
  \caption{BVG results. LC refers to a linear classifier, whereas
    $(\cdot)$ brackets indicate that the models takes in additional
    features from either Node2Vec \emph{(N2V)}, primal PGN-max
    \emph{(P)} and dual PGN-max \emph{(D)}.}
  \label{tab:dar-bvg-results-1}
\end{table}

\begin{table}
  \centering
  \footnotesize
  \setlength{\tabcolsep}{3pt}
  \begin{tabular}{l c c c c c c}
    \toprule
    & \multicolumn{2}{c}{CD1-E-3 $(\uparrow)$} & \multicolumn{2}{c}{CD1-E-2 $(\uparrow)$}  & \multicolumn{2}{c}{CD1-E-1 $(\uparrow)$}  \\
    \textbf{Model} & \textbf{Bal. Acc.} & \textbf{ROC} & \textbf{Bal. Acc.} & \textbf{ROC} & \textbf{Bal. Acc.} & \textbf{ROC} \\
    \midrule
    LC \emph{(D, N2V)} & $53.9\%_{\pm 0.1}$ & $66.3\%_{\pm 0.1}$ & $77.8\%_{\pm 0.1}$ & $74.3\%_{\pm 0.1}$ & $48.2\%_{\pm 0.2}$ & $62.1\%_{\pm 0.1}$\\
    SAGE \emph{(D, N2V)} & $70.1\%_{\pm 0.3}$ & $76.5\%_{\pm 0.1}$ & $\mathbf{89.2\%_{\pm 0.2}}$ & $\mathbf{86.3\%_{\pm 0.2}}$ & $\mathbf{67.5\%_{\pm 0.2}}$ & $\mathbf{76.0\%_{\pm 0.1}}$ \\
    C--SAGE \emph{(D, N2V)} & $\mathbf{70.4\%_{\pm 0.1}}$ & $\mathbf{76.7\%_{\pm 0.1}}$ & $88.3\%_{\pm 0.3}$ & $85.0\%_{\pm 0.3}$ & $67.0\%_{\pm 0.2}$ & $75.8\%_{\pm 0.3}$\\
    \bottomrule
  \end{tabular}
  \caption{BVG results following concatenation of N2V and embeddings.}
  \label{tab:dar-bvg-results-2}
\end{table}

Finally, a last note on DAR/NAR generalisation capabilities. Recall
that these networks were only initially trained on graphs with 16
nodes and are able to produce meaningful representations for graphs
with millions of nodes, being able to provide a clear performance
advantage over previously SOTA baselines.

\section{Learning CO-approximation via algorithmic primitives}
\label{sec:conar}
As explained in \autoref{sec:co}, solving combinatorial optimisation
problems is an extremely important long-standing challenge, studied by
theoretical computer science either by \textit{exact} solving methods
\citep{balas1983branch} or approximated solution
\citep{karlin2022slightly}. As explained in
\autoref{sec:nar-motivation}, the ML community has also interest in
targeting combinatorial problems \citep{vinyals2015pointer,
  bello2017neural, kool2019attention, joshi2022learning}, mainly under
the lens of learning approximation algorithms (i.e., heuristics) for
getting ``good'' solutions in an acceptable amount of time --
the Neural Combinatorial Optimisation (NCO) field.

The ultimate NCO goal is to learn heuristics that can surpass the
performance of manually-designed CO algorithms, typically involved in
theoretical computer science.  Notably, previous research addresses
NCO by primarily focusing on either supervised learning
\citep{vinyals2015pointer, bello2017neural, joshi2022learning} or
reinforcement learning \citep{kool2019attention}. However, despite the
existence of inherent ``algorithmic'' solutions for numerous
combinatorial problems (see also \autoref{sec:graph-algos}), no
studies have yet explored the integration of algorithmic knowledge in
this context.  Precisely, in \autoref{sec:graph-algos} we discussed
how the Christofides' algorithm is heavily based on algorithmic
primitives, such as matching and MST procedures. Such observation
supports NAR's view that algorithms are useful knowledge in the
context of NCO.

In this section, therefore, we target another important ML and NAR
application, which is attempting at leveraging algorithmic primitives
for learning combinatorial optimisation problems (see
\autoref{sec:nar-motivation}). Specifically, we target transfer
learning from algorithms that solve problems in P to NP-hard problems
(refer to \autoref{sec:complexity-classes}), focusing especially on
the TSP (\autoref{prob:tsp}) and the VKC problem
(\autoref{prob:vkc}).

\subsection{Selecting relevant primitives}
By now, it is clear that we intend to pre-train neural solvers on a set of
algorithms first, in order to leverage such algorithmic knowledge to better
perform when targeting CO. However, we have not yet discussed \emph{which}
algorithms we aim to use in the pre-training phase.

Here, we \textbf{prioritise} algorithms that show degrees of
\emph{relevance} for the target combinatorial problem of interest
(i.e., in this section, either TSP or VKC), where relevance is
defined, informally, as follows.
\begin{definition}[Algorithm relevance]
  \label{def:algo-relevance}
  Consider a combinatorial problem $\gP$, an approximation algorithm
  $\tilde{A}$ for $\gP$ and decomposition of $\tilde{A}$ in
  its sub-routines (i.e.,
  $\tilde{A} = \tilde{A}_1 \circ \dots \circ \tilde{A}_n$). Now,
  consider any algorithm $A$. Thus, $A$ is said to be ``relevant'' for
  $\gP$ if $A$ is one of $\tilde{A}$'s sub-routines
  (i.e., $A = \tilde{A}_i$ for any $i$).
\end{definition}
Hence, we have a natural way of discerning which algorithms are
relevant for TSP and VKC, by investigating two well-known heuristics:
the Christofides algorithm for TSP (see \autoref{alg:christofides}),
and the Gon algorithm (see \autoref{alg:gon-algorithm}) for VKC.

However, despite its importancy, relevancy is not the only property
that should be taken into account. Indeed, as we will see in the
experimental section, selecting algorithms for which our ``base''
reasoner struggles to learn is highly detrimental for the subsequent
CO learning. Furthermore, \cite{ibarz2022generalist} clearly show the
benefits of learning also un-related algorithms together.  For such
reasons, while prioritising relevant algorithms we investigate usage
of other algorithms as well. We detail the final choice of algorithms
in the two following paragraphs, dedicated to TSP and VKC,
respectively.

\paragraph{TSP.}
Here, we take into account that the TSP is often framed in a euclidean
geometrical space setting (i.e., often $\sR^2$), from which we derive
a fully-connected undirected graph. Hence, the set of algorithms we
choose for pre-training have to perform well on fully-connected
graphs. As the TSP is in fact a geometric graph, computational
geometry algorithms in the CLRS-30 benchmark
\citep{velickovic2022clrs} are also considered. We also plan to learn
a heuristic for TSP that runs in linear time (i.e., $O(|\gV|)$), hence
we monitor performance of the candidate algorithms w.r.t the number of
iterations they perform (on average) and discard those with
\emph{long} rollouts (e.g., Dijkstra).

In light of the above considerations, we identify a set of 3
algorithms -- Bellman-Ford (\autoref{alg:bf}), Prim's algorithm
(\autoref{alg:prim}) and Graham's scan \citep{cormen2009introduction}
for convex hull. Regrettably, upon initial testing, the inclusion of
Graham's scan resulted in poor performance, leading us to exclude it
from our finalised choice. This might be linked to the lack of
geometrical invariances in NAR architectures
\citep{bronstein2021geometric}.

\paragraph{Vertex k-center problem.} Similar discussion regarding
algorithm performance still holds for selectin VKC-related algorithms.
Notably, upon investigation of the Gon algorithm's sub-routines we
derive that our final model needs knowledge of shortest paths between
nodes (i.e., to effectively select the ``farthest'' node from the
centers), and greedy algorithms. The final choice for VKC is then:
Bellman-Ford, Minimum finding, greedy activity selection
\citep{gavril1972algorithms} and greedy task scheduling algorithms
\citep{lawler2001combinatorial}.

\subsection{Model architecture.}

Similarly to previously presented work, the used architecture here is
still an encode-process-decode (\autoref{sec:epd}), with linear
encoders and decoders and a MPNN-max architecture. However, here we
add a gating mechanism as in \citep{ibarz2022generalist} to let the
processor update only the node representations that it deems important.
Hence, we introduce a gating network $f_{\mathsc{gat}}$  which modifies
representations $\vh_v$ as follows:
\begin{align*}
  \vg^{(t)}_v &= f_{\mathsc{gat}}(\vz^{(t)}_v, \vh^{(t)}_v),\\
  \vh^{(t)}_v &= \sigma(\vg^{(t)}_v) \odot \vh^{(t)}_v + (1 - \sigma(\vg^{(t)}_v))  \odot \vh^{(t)}_v,
\end{align*}
where $\vz^{(t)}_v$ are encoded inputs as in \autoref{sec:epd} and
$\sigma$ is the logistic function. We also tested a variant that
computes and keeps track of the evolution of the edge representations
$\vh^{(t)}_{uv}$.

\paragraph{Decoding CO solutions.}
Finally, since we are targeting combinatorial problems it is important
to ensure the correctness of the solutions.  For instance, a solution
to a TSP problem has to be a permutation of nodes (see
\autoref{def:hamiltonian-path} and \autoref{prob:tsp}), whereas for
VKC we should choose exactly $k$ nodes. However, neural networks might
generate outputs that do not comply with such constraint (similarly
to the max-flow scenario discussed in \autoref{sec:dar}).

Hence, for the TSP we let our reasoner generate a \emph{matrix of
  pointers} $\mP \in \sR^{|\gV| \times |\gV|}$, where each row $\vp_v$
corresponds to the vector of predecessors used to represent outputs of
shortest path algorihms (see \autoref{sec:graph-algos}). In
particular, $\vp_v$ defines a probability distribution over
neighbours, where $p_{vu}$ defines the likelihood of $u$ being the
predecessor of $v$. Therefore, we use a beam search decoding scheme
\citep{freitag2017beam, joshi2022learning} to identify the most likely
set of pointers that constitute a valid TSP solution.

For VKC, instead, we simply sort the neural network's predictions
in order to select the top $k$ nodes that constitute the set $C$.

\subsection{Experimental evaluation}
In this section, we provide all information for replicating our
experiments.

\paragraph{Model selection.}
Here, we select the best hyperparameters for our reasoners, provided
by the CLRS benchmark in \citep{velickovic2022clrs}. Hence, we use
hidden dimension of 128, learning rate of 0.0003 with no weight decay.
We optimise the networks using Adam \citep{kingma2015adam} with a
batch size of 64.

In the pre-training phase, we let algorithmic reasoners be optimised
for 100 epochs. In the transfer learning settings, we experimented
with every modality discussed in \autoref{sec:transfer} (i.e., PF,
PFT, MST, 2PROC). Hence, we fine-tune our models (or
encoders/decoders) for 20 epochs on TSP and 40 epochs for the Vertex K
Center. We report all transfer learning results averaged across 5
different seeds.

\paragraph{Data generation.}
The data generation pipeline for algorithms is similar to the
previously presented contributions. However, as we are dealing with
full graphs and Erdős--Rényi graphs at a later stage (i.e., TSP and
VKC, respectively), we choose to learn algorithmic reasoners on a
training distribution that matches that of the downstream task.  We
sample TSP graphs by randomly selecting points in a 2D space,
restricted to $[0,1]$, and assigning edge weights based on the
euclidean distance between the nodes. On the other hand, VKC-related
graphs are generated following the Erdős--Rényi distribution, setting
$p=0.5$. Following prior works \citep{ibarz2022generalist,
  mahdavi2023towards}, we enhance the traninig data by sampling
varying grah sizes, choosing uniformly in $[8, 16]$.  Validation data
are 16-node graphs (fixed size), whereas test graphs comprehend
64-node graphs (always of fixed size).

For the target TSP task, we consider larger datasets of graphs of
varying size. Notably, learning to generalise in NCO tasks typically
requires large amount of data.  Specifically, we sample $100,000$
training graphs of each of the sizes $[10, 13, 16, 19, 20]$ ($500 000$
total). It is important to point out that our training set size, with
respect to both number of samples and graph sizes, is considerably
smaller of previous works. Indeed, prior work
\citep{khalil2017learning, kool2019attention, joshi2022learning} used
training graphs that may have even more than $100$
nodes. Additionally, \cite{kool2019attention} and
\cite{joshi2022learning} are less data-efficient, using training sets
from several order of magnitude bigger than ours (the former), to
twice as much as our data (the latter).

To perform validation, we create a set of $100$ graphs, each of size
$20$. Additionally, we generate several test sets with the intentions
of assessing our models OOD. Thus, test sets exhibit the following
sizes: $[40, 60, 80, 100, 200, 1000]$, comprising
$[1000, 32, 32, 32, 32, 4]$ graphs for each size, respectively. To
compute the optimal ground-truth tours, we utilize the Concorde
solver, as detailed in \cite{applegate2006concorde}.

For the target VKC task, instead, training, validation and test sets
have the same configurations as those generated for TSP. To generate
ground-truth solutions, however, here we rely on the Gurobi
\citep{gurobi} solver. Specifically, we generate solutions involving
$k=5$ centers.

\paragraph{Ablation studies.}
To thoroughly evaluate the benefits gained by the algorithmic
knowledge, we conducted several ablation studies as well as
comparisons against various baselines (both classical algorithms and
other neural models). In detail, we test the performance of the learnt
reasoner with and without the pre-training phase, using the same set
of hyperparameters.

For both TSP and VKC, we compare against various algorithmic
approaches. TSP baselines comprehend: the Christofides heuristic
(\autoref{alg:christofides}), a beam search approach with beam width
of 1280 and a greedy algorithm that, given a node $v$, selects the
node $u$ which is the closest to $v$ to expand the current solution.
VKC baselines, instead, include: the Gon algorithm and the critical
dominating set (CDS) heuristic \citep{garciadiaz2017when}.

For TSP, we also test the performance against a
reinforcement learning \citep{sutton2018reinforcement} based neural
model \citep{joshi2022learning}.

Finally, we monitor and report the optimality $\textsc{gap}$ metric
(also used in \autoref{sec:neural-planning}) as our performance metric.

\paragraph{Results analysis.}
Our main results are presented in
\autoref{tab:conar-synthetic-results-1}, for TSP, and in
\autoref{tab:conar-synthetic-results-2}, for VKC. In the TSP
\autoref{tab:conar-synthetic-results-1}, subscripts refer to transfer
learning methods presented in \autoref{sec:transfer}, with the
addition of the subscript ``ueh'' which indicates that the model
learns and \textbf{u}pdates \textbf{e}dge \textbf{h}idden
representations. Models with no subscripts are directly trained to
solve the TSP/VKC, without algorithmic pre-training. Finally, we
report results related to the PF strategy only visually in
\autoref{fig:freeze-vs-finetune}, and do not include them in the
tables as employing such strategy failed to produce solutions that
generalise.

Commencing our analysis of TSP results, we observe that, in nearly all
instances, pretraining our model to execute algorithms yields improved
performance, with the transfer PFT (i.e., pre-train and fine-tune)
setting that exhibits the best performance. This is a surprising
result and goes in contrast to common choices of using PF
\citep{velickovic2021neural, deac2021neural, numeroso2023dual,
  cappart2023combinatorial}, 2PROC or MST strategies
\citep{xhonneux2021transfer, ibarz2022generalist}. However, we recall
that our task requires transferring knowledge from algorithms and
problems that are in P to NP-hard. Hence, in our settings the
target task $\gT$ is inherently more difficult than the base task
$\gB$ of learning P algorithms. Hence, our results suggest that in
order to implement such transfer successfully, a reasoner might need
to use $\gB$-representations only as a starting point. Indeed,
pre-trained reasoners learn representations from which we can
``easily'' decode steps of, for instance, Prim and
Bellman-Ford. However, turning this information into a ``good
heuristic'' for the TSP requires extra non-linear learning steps
(i.e., PFT) as linearly transforming such information is clearly not
enough (i.e., PF) -- see also \autoref{fig:freeze-vs-finetune}.

Our learnt reasoners also demonstrate superior performance compared to
the architecture proposed in \citep{joshi2022learning} showing a
substantial OOD performance advantage, with the exception of the
smallest test instances. This is due to well-known reinforcement
learning tendencies of fitting well the training environment
\citep{whiteson2011protecting, zhang2018study, song2019observational}.

An interesting observation is that the performance gap of learnt
reasoners seems to diminish compared to the algorithm-agnostic
baseline, especially for sizes of $\geq 200$ nodes.
\autoref{fig:degradation} seeks to explain this behaviour. In
particular, this trend seems to be influenced by the progressive
deterioration of algorithmic reasoning performance as test graph size
increases. In fact, note how algorithmic reasoners are not able to
generalise Prim for sizes above $200$, as testified by the drop of
performance in the figure (i.e., $\sim 30\%$).  As computing the MST
is an important property for the TSP solution, MPNN\textsubscript{PFT}
can simply not rely on this inductive bias for sizes $\geq
200$. Consequently, the model ``converges'' to an algorithm-agnostic
MPNN. This underscores the significance of developing neural
algorithmic reasoning models capable of robust generalisation even for
significantly larger graphs, to effectively address CO problems using
our approach.

Additionally, \autoref{fig:algo-vs-noalgo} shows that integrate
algorithms leads to faster convergence and delayed overfitting,
compared to standard MPNNs.

In conclusion, it is worth highlighting that we consistently achieve
superior performance compared to the majority of non-parametric
baselines across various test sizes, including both greedy and beam
search. Notably, for test instances of size 40 and 60, we even
outperform or perform on par with Christofides, although our
performance lags behind for larger graphs.

\begin{figure}
  \centering
  \subcaptionbox{Validation \textsc{gap}}{
    \centering
    \input{gfx/contrib/conar/freeze-vs-finetune-val}
  }
  \subcaptionbox{Test \textsc{gap}}{
    \centering
    \input{gfx/contrib/conar/freeze-vs-finetune-test}
  }
  \caption{The standard PF transfer learning approach is not
    applicable for our purposes: the resulting models are unable to
    generalise both in- and out-of-distribution, as shown in
    \textbf{(a)} and \textbf{(b)} respectively.}
  \label{fig:freeze-vs-finetune}
\end{figure}

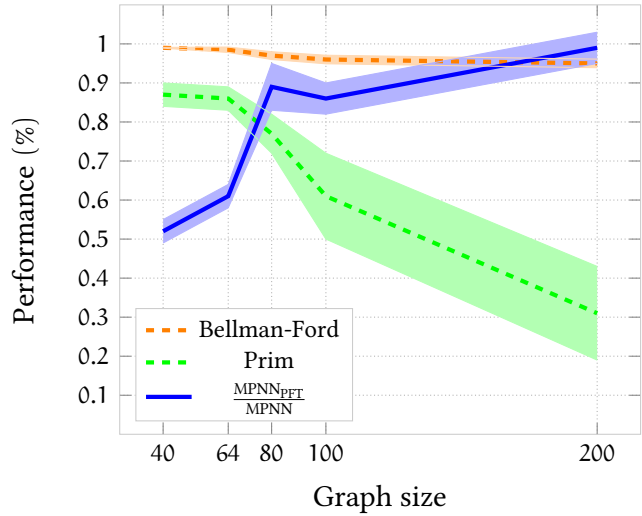
\begin{figure}
  \centering
  \input{gfx/contrib/conar/degradation}
  \caption{Relative performance between MPNN and
    MPNN\textsubscript{PFT}. This figure illustrates that the initial
    superior performance of MPNN\textsubscript{PFT} is most likely due
    to the fact that the network can predict the MST with high
    accuracy (see Prim's curve). As MPNN\textsubscript{PFT} struggles
    to generalise on MST for larger graphs, the performance gain w.r.t
    MPNN also decreases
    ($\frac{\text{MPNN\textsubscript{PFT}}}{\text{MPNN}}$ approaches
    1).  }
  \label{fig:degradation}
\end{figure}

\begin{table}
  \footnotesize
  \centering
  \begin{tabular}{l r r r r r r}
    \toprule
    & \multicolumn{6}{c}{Optimality \textsc{gap} $(\downarrow)$}\\
    \cmidrule(r){2-7}
    \textbf{Model} & \textit{40 nodes}  & \textit{60 nodes} & \textit{80 nodes} & \textit{100nodes} & \textit{200 nodes} & \textit{1000 nodes}\\
    \midrule
    \multicolumn{7}{c}{\textbf{Beam search with width w=128}}\\
    \midrule
    % \multirow{3}{*}{\texttt{w=128}}
    MPNN & $17.7_{\pm 5}\%$ & $23.9_{\pm 3}\%$ & $25.7_{\pm 8}\%$ & $31.9_{\pm 6}\%$ & $38.9_{\pm 7}\%$ & $39.7_{\pm 7}\%$\\
    MPNN\textsubscript{PFT} & $9.1_{\pm 0.1}\%$ & $15.5_{\pm 4}\%$ & $23.1_{\pm 3}\%$ & $28.9_{\pm 2}\%$ & $35.4_{\pm 2}\%$ & $44.5_{\pm 9}\%$\\
    MPNN\textsubscript{PFT+ueh} & $15.4_{\pm 5}\%$ & $23.5_{\pm 8}\%$ & $29.4_{\pm 7}\%$ & $35.7_{\pm 6}\%$ & $37.9_{\pm 3}\%$ & $48.8_{\pm 14}\%$\\
    MPNN\textsubscript{2PROC} & $12.2_{\pm 5}\%$ & $22.1_{\pm 4}\%$ & $28.4_{\pm 6}\%$ & $34.0_{\pm 5}\%$ & $36.6_{\pm 6}\%$ & $38.8_{\pm 3}\%$\\
    MPNN\textsubscript{MTL} & $18.1_{\pm 6}\%$ & $26.2_{\pm 4}\%$ & $31.2_{\pm 5}\%$ & $34.8_{\pm 4}\%$ & $37.1_{\pm 3}\%$ & $38.5_{\pm 3}\%$\\
    \begin{tabular}{@{}l@{}}\citeauthor{joshi2022learning}\\(AR)\end{tabular} & $6.2_{\pm 4}\%$ & $37.1_{\pm 12}\%$ & $74.8_{\pm 13}\%$ & $102_{\pm 10}\%$ & $195_{\pm 20}\%$ & $419_{\pm 20}\%$\\
    \begin{tabular}{@{}l@{}}\citeauthor{joshi2022learning}\\(nAR)\footnotemark\end{tabular} & $20.6_{\pm 10}\%$ & $62.5_{\pm 11}\%$ & $110_{\pm 17}\%$ & $156_{\pm 24}\%$ & $273_{\pm24}\%$ & $497_{\pm 7}\%$\\
    \midrule
    \multicolumn{7}{c}{\textbf{Beam search with width w=1280}}\\
    \midrule
    MPNN & $14.4_{\pm 4}\%$ & $21.8_{\pm 5}\%$ & $22.4_{\pm 4}\%$ & $31.1_{\pm 6}\%$ & $33.6_{\pm 3}\%$ & $41.2_{\pm 8}\%$\\
    MPNN\textsubscript{PFT} & $7.5_{\pm 1}\%$ & $\mathbf{13.3_{\pm 3}}\%$ & $\mathbf{20.1_{\pm 5}}\%$ & $\mathbf{26.9_{\pm 3}}\%$ & $\mathbf{33.8_{\pm 2}}\%$ & $\mathbf{40.5_{\pm 7}}\%$\\
    MPNN\textsubscript{PFT+ueh} & $12.9_{\pm 5}\%$ & $20.6_{\pm 6}\%$ & $26.0_{\pm 3}\%$ & $32.0_{\pm 5}\%$ & $37.2_{\pm 3}\%$ &
                                                                                                                                 $44.3_{\pm 9}\%$\\
    MPNN\textsubscript{2PROC} & $9.5_{\pm 6}\%$ & $18.0_{\pm 5}\%$ & $24.1_{\pm 6}\%$ & $28.8_{\pm 5}\%$ & $36.6_{\pm 6}\%$ & $38.5_{\pm 3}\%$\\
    MPNN\textsubscript{MTL} & $15.5_{\pm 5}\%$ & $26.5_{\pm 6}\%$ & $30.2_{\pm 6}\%$ & $33.4_{\pm 4}\%$ & $36.8_{\pm 5}\%$ & $42.3_{\pm 5}\%$\\
    \begin{tabular}{@{}l@{}}\citeauthor{joshi2022learning}\\(AR)\end{tabular} & $\mathbf{3.5_{\pm 3}}\%$ &  $33.0_{\pm 11}\%$ & $63.8_{\pm 12}\%$ & $97.6_{\pm 12}\%$ & $193.3_{\pm 17}\%$ & $417_{\pm 4}\%$\\
    \begin{tabular}{@{}l@{}}\citeauthor{joshi2022learning}\\(nAR)\footnotemark[2]\end{tabular}& $14.2_{\pm 8}\%$ & $54.7_{\pm 15}\%$ & $100.0_{\pm 15}\%$ & $141.9_{\pm 25}\%$ & $265.5_{\pm 27}\%$ & $500.7_{\pm 23}\%$\\
    \midrule
    \multicolumn{7}{c}{\textbf{Models not using beam search}}\\
    \midrule
    \begin{tabular}{@{}l@{}}\citeauthor{joshi2022learning}\\({\scriptsize AR+greedy})\end{tabular} & $22.1_{\pm 10}\%$ & $57.7_{\pm 13}\%$ & $94.0_{\pm 14}\%$ & $124.0_{\pm 12}\%$ & $219.6_{\pm 22}\%$ & $469_{\pm 17}\%$\\
    \begin{tabular}{@{}l@{}}\citeauthor{joshi2022learning}\\({\scriptsize AR+sampling})\end{tabular} & $9.5_{\pm 5}\%$ & $46.7_{\pm 9}\%$ & $93.1_{\pm 10}\%$ & $137.0_{\pm 14}\%$ & $313.2_{\pm 15}\%$ & $1102_{\pm 1}\%$\\
    \midrule
    \multicolumn{7}{c}{\textbf{Deterministic baselines}}\\
    \midrule
    \textit{Greedy} & $31.9_{\pm 12}\%$ & $32.8_{\pm 10}\%$ & $33.3_{\pm 9}\%$ & $30.0_{\pm 6}\%$ & $32.1_{\pm 6}\%$ & $28.8_{\pm 3}\%$\\
    \textit{\begin{tabular}{@{}l@{}}Beam search\\{\scriptsize (w=1280)}\end{tabular}} & $19.7_{\pm 8}\%$ & $23.1_{\pm 7}\%$ & $29.4_{\pm 7}\%$ & $29.7_{\pm 5}\%$ & $33.2_{\pm 4}\%$ & $38.9_{\pm 2}\%$\\
    \textit{Christofides} & $10.1_{\pm 3}\%$ & $11.0_{\pm 2}\%$ & $11.3_{\pm 2}\%$ & $12.1_{\pm 2}\%$ & $12.2_{\pm 1}\%$ & $12.2_{\pm 0.1}\%$\\
    \bottomrule
  \end{tabular}
  \caption{Optimality \textsc{gap} across different TSP sizes. Subscripts
    refer to the according transfer learning settings or if the model
    \textbf{u}pdates \textbf{e}dge \textbf{h}iddens.}
  \label{tab:conar-synthetic-results-1}
\end{table}

Moving forward, similar arguments to what we made for TSP results
still hold when analysing VKC
(\autoref{tab:conar-synthetic-results-2}). Here, we test only the PF
transfer strategy, as it gave best results on the TSP task.

In these experiments, we also tried integrating two more relevant
algorithms such as Floyd-Warshall and Insertion Sort
\citep{cormen2009introduction}. However, our learnt reasoners did not
achieve satisfactory approximations in the pre-training phase
(Floyd-Warshall, mean accuracy of $\approx23\%$; Insertion Sort,
$\approx43\%$) and this negatively cascade on the VKC learning task.
This testifies that is important to appropriately choosing the test of
algorithms to pre-train on, and exclude those that learnt reasoners
struggle with. Replacing these algorithms by an additional greedy
algorithm (i.e., Activity Selection) does yield increased performance,
outperforming the baseline neural model on sizes $40$ and $60$.

When compared to the deterministic baseline methods, all neural models
exhibit superior performance to the basic Gon heuristic. Regrettably,
even the top-performing models do not reach the level of performance
achieved by the CDS heuristic. We posit that this disparity can be
attributed to the complexity of the CDS heuristic, which comprises
multiple intertwined subroutines such as binary search and graph
pruning. It is possible that a more modular NAR approach in the future
could match such performance.

\begin{table}
  \footnotesize
  \centering
  \begin{tabular}{lrrrrr}
    \toprule
    & \multicolumn{5}{c}{Optimality \textsc{gap} $(\downarrow)$}\\
    \cmidrule(r){2-6}
    \textbf{Model} & \textit{40 nodes} & \textit{60 nodes} & \textit{80 nodes} & \textit{100 nodes} & \textit{200 nodes}\\
    \midrule
    MPNN & $15.88_{\pm 1.11}\%$ & $20.85_{\pm 2.64}\%$ & $24.78_{\pm 7.17}\%$ & $26.48_{\pm 6.65}\%$ & $\mathbf{24.63_{\pm 6.50}}\%$ \\
    MPNN\textsubscript{FMITB} & $14.73_{\pm 1.21}\%$ & $21.97_{\pm 4.87}\%$ & $23.91_{\pm 4.52}\%$ & $28.50_{\pm 5.06}\%$ & $27.04_{\pm 4.22}\%$ \\
    MPNN\textsubscript{MTAB} & $\mathbf{13.58_{\pm 0.60}}\%$ & $\mathbf{16.20_{\pm 3.02}}\%$ & $\mathbf{23.10_{\pm 3.99}}\%$ & $\mathbf{26.40_{\pm 6.62}}\%$ & $26.43_{\pm 3.84}\%$ \\
    \midrule
    \multicolumn{6}{c}{\textbf{Deterministic baselines}}\\
    \midrule
    Farthest First & $41.04_{\pm 14.21}\%$ & $43.89_{\pm 10.70}\%$ & $38.31_{\pm 9.22}\%$ & $36.32_{\pm 7.28}\%$ & $37.91_{\pm 7.91}\%$\\
    CDS & $7.15_{\pm 5.42}\%$ & $7.82_{\pm 4.78}\%$ & $6.49_{\pm 3.65}\%$ & $6.55_{\pm 3.29}\%$ & $5.98_{\pm 2.00}\%$\\
    \midrule
    \bottomrule
  \end{tabular}
  \caption{VKC optimality gap. Each letter in the MPNNs subscripts
    denotes an algorithm pretrained on: \textbf{F}loyd-Warshall,
    \textbf{M}inimum, \textbf{I}nsertion sort, \textbf{T}ask
    scheduling, \textbf{B}ellman-Ford, \textbf{A}ctivity selection }
  \label{tab:conar-synthetic-results-2}
\end{table}

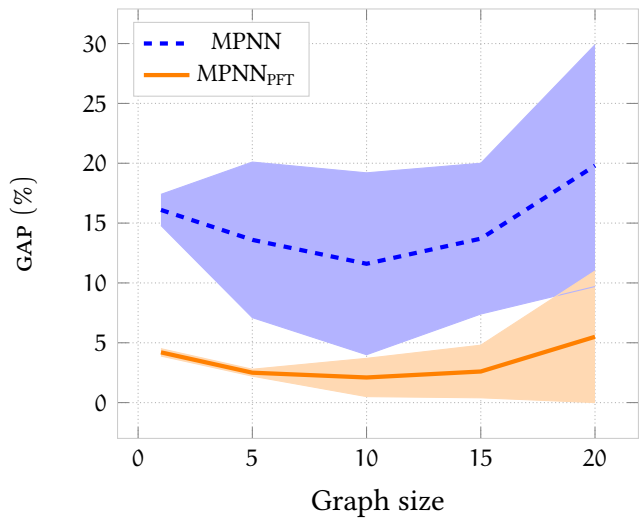
\begin{figure}
  \centering
  \input{gfx/contrib/conar/algo-vs-noalgo}
  \caption{NAR modules converge faster on TSP learning. Both
    models use the same architecture and hyperparameters.}
  \label{fig:algo-vs-noalgo}
\end{figure}

\paragraph{Qualitative study.}
Finally, we perform a qualitative study to evaluate the impact of the
chosen algorithms to pre-train our neural CO solvers on. We pick a
neural network for solving TSP, utilising the PFT strategy, and
pre-train it on three random algorithms that were excluded by our
initial selections: \emph{Breadth-First Search}, \emph{Topological
  Sorting}, and \emph{Longest Common Subsequence length} (a string
algorithm). \autoref{tab:conar-rel-vs-unrel} shows that
choosing sub-optimal algorithms, while still comparable or 
slightly better for sizes 80 and 100 nodes, makes the model diverge 
for larger sizes ($\geq 200$ nodes).

\begin{table}
  \footnotesize
  \begin{tabular}{l r r r r r}
    \toprule
    & \multicolumn{5}{c}{Optimality \textsc{gap} $(\downarrow)$ }\\
    \cmidrule(r){2-6}
    \textbf{Model} & \textit{40 nodes}
                   & \textit{60 nodes}
                   & \textit{80 nodes}
                   & \textit{100 nodes}
                   & \textit{200 nodes}\\
    \midrule
    \multicolumn{6}{c}{\textbf{Beam search with width w=128}}\\
    \midrule
    MPNN & $17.7_{\pm 5}\%$ & $23.9_{\pm 3}\%$ & $25.7_{\pm 8}\%$ & $31.9_{\pm 6}\%$ & $38.9_{\pm 7}\%$ \\
    Unrelated & $12.3_{\pm 2.2}\%$ & $17.5_{\pm 3.1}\%$ & $\mathbf{22.3_{\pm 3.1}}\%$ & $\mathbf{26.4_{\pm 2.6}}\%$ & $199.0_{\pm 373.3}\%$\\
    Related & $\mathbf{9.1_{\pm 0.1}}\%$ & $\mathbf{15.5_{\pm 4}}\%$ & $23.1_{\pm 3}\%$ & $28.9_{\pm 2}\%$ & $\mathbf{35.4_{\pm 2}}\%$\\
    \bottomrule
  \end{tabular}
  \caption{TSP optimality \textsc{gap} comparison between no-transfer,
    transferring unrelated algorithms and transferring related
    algorithms. When transferring, the experimental setup matches
    MPNN\textsubscript{PFT}, only algorithms differ.}
  \label{tab:conar-rel-vs-unrel}
  \centering
\end{table}

%% file: gfx/contrib/dar/net.tex
\begin{tikzpicture}[very thick,scale=1, rotate=90]

  \tikzset{line/.style={draw,line width=1.5pt}}
  \tikzset{arrow/.style={->,>=stealth}}
  \tikzset{snake/.style={arrow,line width=1.3pt,decorate,decoration={snake,amplitude=1,segment length=4,post length=5}}}

  \tikzstyle{box}=[dash pattern=on 5pt off 2pt,inner sep=5pt,rounded corners=3pt]
  \tikzstyle{vertex}=[circle,fill=cyan!50!,draw,minimum size=18pt, inner sep=1pt]

  \tikzstyle{message}=[arrow, decorate,
  decoration={snake,
    amplitude=1,
    segment length=6,
    post length=7
  },
  line width=2.5pt]

  \node[rectangle, draw, dashed, very thick, black!30, minimum height=130pt, minimum width=90pt] (Xbar) at (-5.5, 0) {};

  \node[rectangle, draw, dashed, very thick, black!30, minimum height=130pt, minimum width=80pt] (Ybar) at (8.5, 2.5) {};

  \node[rectangle, draw, dashed, very thick, black!30, minimum height=130pt, minimum width=80pt] (Hbar) at (8.5, -2.5) {};

  \node[rectangle, draw, dashed, very thick, black!30, minimum height=85pt, minimum width=90pt, fill=blue, fill opacity=0.05] (P) at (-0.5, 0) {};

  \node[rectangle, draw, dashed, very thick, black!30, minimum height=85pt, minimum width=90pt, fill=blue, fill opacity=0.05] (P2) at (4, 0) {};

  \draw[line width=3.5pt, -stealth, blue, opacity=0.7] (Xbar) -- node[left] {\large $f$} (P);
  \draw[line width=3.5pt, -stealth, blue, opacity=0.7] (P) -- node[left] {\large $g_{BF}$} (P2);
  \draw[line width=3.5pt, -stealth, blue, opacity=0.7] (P2) -- node[left] {\large $g_P$} (Ybar);
  \draw[line width=3.5pt, -stealth, blue, opacity=0.7] (P2) -- node[right] {\large $g_D$} (Hbar);

  \draw[blue] (P) edge[loop above=1, line width=3.5pt, looseness=12, stealth-, opacity=0.7] (P);

  \node[blue] at (-0.5, 2) {\large $p_{BF}$};
  \node[blue] at (4, 2) {\large $p_{F}$};

  \foreach \pos/\name/\lab in {{(-0.4,0)/m1/$v_1$}, {(-1.51,-0.13)/m2/$v_2$}, {(-0.66,1)/m3/$v_3$}, {(0.5,0.47)/m4/$v_4$}, {(-0.09,-1)/m5/$v_5$}}
  \node[vertex] (\name) at \pos {\lab};

  \foreach \pos/\name/\lab in {{(4.1,0)/m12/$v_1$}, {(3.01,-0.13)/m22/$v_2$}, {(3.84,1)/m32/$v_3$}, {(5.0,0.47)/m42/$v_4$}, {(4.41,-1)/m52/$v_5$}}
  \node[vertex] (\name) at \pos {\lab};

  \foreach \pos/\name/\nb in {{(-7, -0.75)/s/$s$},{(-7, 0.75)/a/$v_1$},{(-6, -0.75)/b/$v_2$},{(-6, 0.75)/c/$v_3$},{(-5, 0.75)/d/$v_4$},{(-4, -0.75)/e/$v_5$},{(-5, -0.75)/f/$v_6$},{(-4, 0.75)/t/$t$}}
  \node[vertex] (\name1) at \pos {\nb};

  \foreach \pos/\name/\nb in {{(7, 1.75)/s/$s$},{(7, 3.25)/a/$v_1$},{(8, 1.75)/b/$v_2$},{(8, 3.25)/c/$v_3$},{(9, 3.25)/d/$v_4$},{(10, 1.75)/e/$v_5$},{(9, 1.75)/f/$v_6$},{(10, 3.25)/t/$t$}}
  \node[vertex] (\name2) at \pos {\nb};

  \foreach \pos/\name/\nb in {{(7, -3.25)/s/$s$},{(7, -1.75)/a/$v_1$},{(8, -3.25)/b/$v_2$},{(8, -1.75)/c/$v_3$},{(9, -1.75)/d/$v_4$},{(10, -3.25)/e/$v_5$},{(9, -3.25)/f/$v_6$},{(10, -1.75)/t/$t$}}
  \node[vertex] (\name3) at \pos {\nb};

  % \foreach \pos/\name/\nb in {{(5 , 1)/a/$s_0$},{(4, 1)/b/$s_1$},{(4, 2.5)/c/$s_2$},{(5, 2.75)/d/$s_3$},{(5, 4)/e/$s^*$}}
  % \node[vertex] (\name2) at \pos {\nb};

  % \foreach \pos/\name/\nb in {{(5 , -4)/a/$\tilde{h}_{0}$},{(4, -4)/b/$\tilde{h}_{1}$},{(4, -2.5)/c/$\tilde{h}_{2}$},{(5, -2.25)/d/$\tilde{h}_{3}$},{(5, -1)/e/$\tilde{h}_{*}$}}
  % \node[vertex] (\name3) at \pos {\small\nb};

  % \foreach \name/\nb in {{a/$\infty$},{b/$0$},{c/$\infty$},{d/$\infty$},{e/$\infty$}}
  % \node[vertex, below = 3.75cm of \name1] (\name3) {$x_\name$};

  \draw (m2) edge (m1);
  \draw (m3) edge (m1);
  \draw (m4) edge (m1);
  \draw (m5) edge (m1);
  \draw (m2) edge[message, cyan, opacity=0.7] (m1);
  \draw (m3) edge[message, cyan, opacity=0.7] (m1);
  \draw (m4) edge[message, cyan, opacity=0.7] (m1);
  \draw (m5) edge[message, cyan, opacity=0.7] (m1);

  \draw (m22) edge (m12);
  \draw (m32) edge (m12);
  \draw (m42) edge (m12);
  \draw (m52) edge (m12);
  \draw (m22) edge[message, cyan, opacity=0.7] (m12);
  \draw (m32) edge[message, cyan, opacity=0.7] (m12);
  \draw (m42) edge[message, cyan, opacity=0.7] (m12);
  \draw (m52) edge[message, cyan, opacity=0.7] (m12);

  \draw (s1) edge (a1);
  \draw (s1) edge (b1);
  \draw (s1) edge (c1);
  \draw (a1) edge (b1);
  \draw (a1) edge (c1);
  \draw (b1) edge (c1);

  \draw (t1) edge (d1);
  \draw (t1) edge (e1);
  \draw (t1) edge (f1);
  \draw (d1) edge (e1);
  \draw (d1) edge (f1);
  \draw (e1) edge (f1);

  \draw (b1) edge (f1);
  \draw (c1) edge (d1);

  \draw (s2) edge (a2);
  \draw (s2) edge (b2);
  \draw (s2) edge (c2);
  \draw (a2) edge (b2);
  \draw (a2) edge (c2);
  \draw (b2) edge (c2);

  \draw (t2) edge (d2);
  \draw (t2) edge (e2);
  \draw (t2) edge (f2);
  \draw (d2) edge (e2);
  \draw (d2) edge (f2);
  \draw (e2) edge (f2);

  \draw (b2) edge (f2);
  \draw (c2) edge (d2);

  \draw (s3) edge (a3);
  \draw (s3) edge (b3);
  \draw (s3) edge (c3);
  \draw (a3) edge (b3);
  \draw (a3) edge (c3);
  \draw (b3) edge (c3);

  \draw (t3) edge (d3);
  \draw (t3) edge (e3);
  \draw (t3) edge (f3);
  \draw (d3) edge (e3);
  \draw (d3) edge (f3);
  \draw (e3) edge (f3);

  \draw (b3) edge[densely dotted] (f3);
  \draw (c3) edge[densely dotted] (d3);

  \draw[red, line width=3.5pt, opacity=0.7, -stealth] (s2) edge (b2);
  \draw[red, line width=3.5pt, opacity=0.7, -stealth] (s2) edge (c2);
  \draw[red, line width=3.5pt, opacity=0.7, -stealth] (b2) edge (f2);
  \draw[red, line width=3.5pt, opacity=0.7, -stealth] (c2) edge (d2);
  \draw[red, line width=3.5pt, opacity=0.7, -stealth] (d2) edge (t2);
  \draw[red, line width=3.5pt, opacity=0.7, -stealth] (f2) edge (t2);

  % \draw[red, line width=3.5pt, opacity=0.7] (b3) edge (f3);
  % \draw[red, line width=3.5pt, opacity=0.7] (c3) edge (d3);

\end{tikzpicture}

%% file: content/algs/dar_corrective_procedure.tex
\begin{algorithm}
\caption{Flow conservation correction algorithm}
\label{alg:corrective-procedure}
\begin{algorithmic}[1]
  \Procedure{EnsureFlowConservation}{$G=\tuple{\gV, \gE, \gC}$, $\mF$, $\mC$, $s$, $t$}
  \State $\gV^-= \{v \mid v \in \gV \land \sum_{(i, v) \in \gE} \mF_{iv} +
  \sum_{(v, i) \in E} \mF_{vi} < 0\}$
  \State $\gV^+ = \{v \mid v \in \gV \land \sum_{(i, v) \in \gE} \mF_{iv} + \sum_{(v, i) \in E} \mF_{vi} > 0\}$
  \For {each $v \in \gV^-$} \State $\varepsilon = |\sum_{(i, v) \in \gE} \mF_{iv} + \sum_{(v, i) \in E} \mF_{vi}|$
  \State{find a path from $v$ to $s$ and send back
    $\varepsilon$ amount of flow to $s$}
  \EndFor
  \For {each $v \in V^+$} \State $\varepsilon = |\sum_{(i, v) \in \gE} \mF_{iv} + \sum_{(v, i) \in \gE} \mF_{vi}|$
  \State{find a path from $v$ to $t$ and reduce the out-flow
    by $\varepsilon$}
  \EndFor
  \State {$\varepsilon_s = |\sum_{(s, i) \in E} \mF_{si} -
    \sum_{(s, i) \in \gE} \mC_{si}|$} \While {$\varepsilon_s > 0$}
  \State{find a path from $s$ to $t$ and reduce the out-flow by enforcing capacity constraints}
  \State{recompute $\varepsilon_s$}
  \EndWhile
  \State \textbf{return} $\mF$
  \EndProcedure
\end{algorithmic}
\end{algorithm}

%% file: gfx/contrib/dar/valid.tex
\begin{tikzpicture}
  \begin{axis}[
    scaled x ticks=false,
    xlabel={Training steps},
    ylabel={$\gL$},
    grid=major,
    legend entries={PGN-max \emph{(P)},
      PGN-max \emph{(D)},
      PGN-max \emph{(PL)},
      PGN-mean \emph{(P)},
      PGN-mean \emph{(D)},
      PGN-mean \emph{(PL)}},
    legend pos=north east,
    grid style={black!30, densely dotted},
    legend style={draw=black!20, fill=white, font=\footnotesize},
    axis line style={black!20},
    xtick={0, 5000, 10000, 15000, 20000},
    ymin=0.5,
    ymax=1.4,
    xticklabel style={font=\footnotesize},
    yticklabel style={font=\footnotesize},
    ytick={0.5, 0.6, 0.7, 0.8, 0.9, 1.0, 1.1, 1.2, 1.3},
]
\addplot[dashed, blue, ultra thick] table[col sep=comma, x index=0, y index=1]{gfx/contrib/dar/data/primal-max.csv};
\addplot[dotted, orange, ultra thick] table[col sep=comma, x index=0, y index=1]{gfx/contrib/dar/data/dual-max.csv};
\addplot[green, ultra thick] table[col sep=comma, x index=0, y index=1]{gfx/contrib/dar/data/pipe-max.csv};
\addplot[dashed, pink, ultra thick] table[col sep=comma, x index=0, y index=1]{gfx/contrib/dar/data/primal-mean.csv};
\addplot[dotted, red, ultra thick] table[col sep=comma, x index=0, y index=1]{gfx/contrib/dar/data/dual-mean.csv};
\addplot[magenta, ultra thick] table[col sep=comma, x index=0, y index=1]{gfx/contrib/dar/data/pipe-mean.csv};

\addplot[blue!30, name path=plus_sd, forget plot] table[col sep=comma, x index=0, y expr=\thisrowno{1}+\thisrowno{2}]{gfx/contrib/dar/data/primal-max.csv};
\addplot[blue!30, name path=minus_sd, forget plot] table[col sep=comma, x index=0, y expr=\thisrowno{1}-\thisrowno{2}]{gfx/contrib/dar/data/primal-max.csv};
\addplot[blue!30] fill between[of=plus_sd and minus_sd];

\addplot[orange!30, name path=plus_sd, forget plot] table[col sep=comma, x index=0, y expr=\thisrowno{1}+\thisrowno{2}]{gfx/contrib/dar/data/dual-max.csv};
\addplot[orange!30, name path=minus_sd, forget plot] table[col sep=comma, x index=0, y expr=\thisrowno{1}-\thisrowno{2}]{gfx/contrib/dar/data/dual-max.csv};
\addplot[orange!30] fill between[of=plus_sd and minus_sd];

\addplot[green!30, name path=plus_sd, forget plot] table[col sep=comma, x index=0, y expr=\thisrowno{1}+\thisrowno{2}]{gfx/contrib/dar/data/pipe-max.csv};
\addplot[green!30, name path=minus_sd, forget plot] table[col sep=comma, x index=0, y expr=\thisrowno{1}-\thisrowno{2}]{gfx/contrib/dar/data/pipe-max.csv};
\addplot[green!30] fill between[of=plus_sd and minus_sd];

\addplot[pink!30, name path=plus_sd, forget plot] table[col sep=comma, x index=0, y expr=\thisrowno{1}+\thisrowno{2}]{gfx/contrib/dar/data/primal-mean.csv};
\addplot[pink!30, name path=minus_sd, forget plot] table[col sep=comma, x index=0, y expr=\thisrowno{1}-\thisrowno{2}]{gfx/contrib/dar/data/primal-mean.csv};
\addplot[pink!30] fill between[of=plus_sd and minus_sd];

\addplot[red!30, name path=plus_sd, forget plot] table[col sep=comma, x index=0, y expr=\thisrowno{1}+\thisrowno{2}]{gfx/contrib/dar/data/dual-mean.csv};
\addplot[red!30, name path=minus_sd, forget plot] table[col sep=comma, x index=0, y expr=\thisrowno{1}-\thisrowno{2}]{gfx/contrib/dar/data/dual-mean.csv};
\addplot[red!30] fill between[of=plus_sd and minus_sd];

\addplot[magenta!30, name path=plus_sd, forget plot] table[col sep=comma, x index=0, y expr=\thisrowno{1}+\thisrowno{2}]{gfx/contrib/dar/data/pipe-mean.csv};
\addplot[magenta!30, name path=minus_sd, forget plot] table[col sep=comma, x index=0, y expr=\thisrowno{1}-\thisrowno{2}]{gfx/contrib/dar/data/pipe-mean.csv};
\addplot[magenta!30] fill between[of=plus_sd and minus_sd];

\end{axis}
\end{tikzpicture}

%% file: gfx/contrib/dar/reconstruction.tex
\begin{tikzpicture}
  \begin{axis}[
    scaled x ticks=false,
    xlabel={Training steps},
    ylabel={$\gL$},
    grid=major,
    legend entries={PGN-max \emph{(P)}, PGN-max \emph{(D)}},
    legend pos=north east,
    grid style={black!30, densely dotted},
    legend style={draw=black!20, fill=white, font=\footnotesize},
    axis line style={black!20},
    xtick={0, 5000, 10000, 15000},
    xticklabel style={font=\footnotesize},
    yticklabel style={font=\footnotesize}
]
\addplot[blue, ultra thick] table[col sep=comma, x index=0, y index=1]{gfx/contrib/dar/data/reconstruction.csv};
\addplot[orange, ultra thick] table[col sep=comma, x index=0, y index=2]{gfx/contrib/dar/data/reconstruction.csv};

\end{axis}
\end{tikzpicture}

%% file: gfx/contrib/conar/freeze-vs-finetune-val.tex
\begin{tikzpicture}[scale=1]
  \begin{axis}[
    scaled x ticks=false,
    xlabel={Epoch},
    ylabel={\textsc{gap} $(\%)$},
    grid=major,
    legend entries={$\text{MPNN}_{\text{PF}}$,
      $\text{MPNN}_{\text{PFT}}$},
    legend pos=north west,
    grid style={black!30, densely dotted},
    legend style={draw=black!20, fill=white, font=\footnotesize},
    axis line style={black!20},
    xtick={0, 5, 10, 15, 20},
    ytick={0,5,10,15,20,25,30},
]
\addplot[dashed, blue, ultra thick] table[col sep=comma, x index=0, y index=1]{gfx/contrib/conar/data/mpnn_pf.csv};
\addplot[solid, orange, ultra thick] table[col sep=comma, x index=0, y index=1]{gfx/contrib/conar/data/mpnn_pft.csv};

\addplot[blue!30, name path=plus_sd, forget plot] table[col sep=comma, x index=0, y expr=\thisrowno{1}+\thisrowno{2}]{gfx/contrib/conar/data/mpnn_pf.csv};
\addplot[blue!30, name path=minus_sd, forget plot] table[col sep=comma, x index=0, y expr=\thisrowno{1}-\thisrowno{2}]{gfx/contrib/conar/data/mpnn_pf.csv};
\addplot[blue!30] fill between[of=plus_sd and minus_sd];

\addplot[orange!30, name path=plus_sd, forget plot] table[col sep=comma, x index=0, y expr=\thisrowno{1}+\thisrowno{2}]{gfx/contrib/conar/data/mpnn_pft.csv};
\addplot[orange!30, name path=minus_sd, forget plot] table[col sep=comma, x index=0, y expr=\thisrowno{1}-\thisrowno{2}]{gfx/contrib/conar/data/mpnn_pft.csv};
\addplot[orange!30] fill between[of=plus_sd and minus_sd];

\end{axis}
\end{tikzpicture}

%% file: gfx/contrib/conar/freeze-vs-finetune-test.tex
\begin{tikzpicture}[scale=1]
  \begin{axis}[
    ylabel={\textsc{gap} $(\%)$},
    ybar=0pt,
    bar width=40pt,
    xmin=0,
    xmax=9,
    xtick={3,6},
    xticklabels={MPNN\textsubscript{PF}, MPNN\textsubscript{PFT}},
    ymin=0.0,
    ymax=55.0,
    ]
    \addplot+[error bars/.cd, y dir=both, y explicit]
    coordinates {
      (3,30.40) +- (21,21)
      (6,9.10) +- (2,2)
    };
\end{axis}
\end{tikzpicture}

%% file: gfx/contrib/conar/degradation.tex
\begin{tikzpicture}
  \begin{axis}[
    scaled x ticks=false,
    xlabel={Graph size},
    ylabel={Performance $(\%)$},
    grid=major,
    legend entries={Bellman-Ford,
      Prim,
      $\frac{\text{MPNN}_{\text{PFT}}}{\text{MPNN}}$},
    legend pos=south west,
    grid style={black!30, densely dotted},
    legend style={draw=black!20, fill=white, font=\footnotesize},
    axis line style={black!20},
    xtick={40, 64, 80, 100, 200},
    ymin=0.0,
    ymax=1.1,
    xticklabel style={font=\footnotesize},
    yticklabel style={font=\footnotesize},
    ytick={0.1,0.2,0.3,0.4,0.5,0.6,0.7,0.8,0.9,1.0},
]
\addplot[dashed, orange, ultra thick] table[col sep=comma, x index=0, y index=1]{gfx/contrib/conar/data/bf.csv};
\addplot[dashed, green, ultra thick] table[col sep=comma, x index=0, y index=1]{gfx/contrib/conar/data/mst.csv};
\addplot[blue, ultra thick] table[col sep=comma, x index=0, y index=1]{gfx/contrib/conar/data/ratio.csv};

\addplot[orange!30, name path=plus_sd, forget plot] table[col sep=comma, x index=0, y expr=\thisrowno{1}+\thisrowno{2}]{gfx/contrib/conar/data/bf.csv};
\addplot[orange!30, name path=minus_sd, forget plot] table[col sep=comma, x index=0, y expr=\thisrowno{1}-\thisrowno{2}]{gfx/contrib/conar/data/bf.csv};
\addplot[orange!30] fill between[of=plus_sd and minus_sd];

\addplot[green!30, name path=plus_sd, forget plot] table[col sep=comma, x index=0, y expr=\thisrowno{1}+\thisrowno{2}]{gfx/contrib/conar/data/mst.csv};
\addplot[green!30, name path=minus_sd, forget plot] table[col sep=comma, x index=0, y expr=\thisrowno{1}-\thisrowno{2}]{gfx/contrib/conar/data/mst.csv};
\addplot[green!30] fill between[of=plus_sd and minus_sd];

\addplot[blue!30, name path=plus_sd, forget plot] table[col sep=comma, x index=0, y expr=\thisrowno{1}+\thisrowno{2}]{gfx/contrib/conar/data/ratio.csv};
\addplot[blue!30, name path=minus_sd, forget plot] table[col sep=comma, x index=0, y expr=\thisrowno{1}-\thisrowno{2}]{gfx/contrib/conar/data/ratio.csv};
\addplot[blue!30] fill between[of=plus_sd and minus_sd];

\end{axis}
\end{tikzpicture}

%% file: gfx/contrib/conar/algo-vs-noalgo.tex
\begin{tikzpicture}[scale=1]
  \begin{axis}[
    scaled x ticks=false,
    xlabel={Graph size},
    ylabel={\textsc{gap} $(\%)$},
    grid=major,
    legend entries={MPNN,
      $\text{MPNN}_{\text{PFT}}$},
    legend pos=north west,
    grid style={black!30, densely dotted},
    legend style={draw=black!20, fill=white, font=\footnotesize},
    axis line style={black!20},
    xtick={0, 5, 10, 15, 20},
    xticklabel style={font=\footnotesize},
    yticklabel style={font=\footnotesize},
    ytick={0,5,10,15,20,25,30},
]
\addplot[dashed, blue, ultra thick] table[col sep=comma, x index=0, y index=1]{gfx/contrib/conar/data/mpnn.csv};
\addplot[solid, orange, ultra thick] table[col sep=comma, x index=0, y index=1]{gfx/contrib/conar/data/mpnn_pft.csv};

\addplot[blue!30, name path=plus_sd, forget plot] table[col sep=comma, x index=0, y expr=\thisrowno{1}+\thisrowno{2}]{gfx/contrib/conar/data/mpnn.csv};
\addplot[blue!30, name path=minus_sd, forget plot] table[col sep=comma, x index=0, y expr=\thisrowno{1}-\thisrowno{2}]{gfx/contrib/conar/data/mpnn.csv};
\addplot[blue!30] fill between[of=plus_sd and minus_sd];

\addplot[orange!30, name path=plus_sd, forget plot] table[col sep=comma, x index=0, y expr=\thisrowno{1}+\thisrowno{2}]{gfx/contrib/conar/data/mpnn_pft.csv};
\addplot[orange!30, name path=minus_sd, forget plot] table[col sep=comma, x index=0, y expr=\thisrowno{1}-\thisrowno{2}]{gfx/contrib/conar/data/mpnn_pft.csv};
\addplot[orange!30] fill between[of=plus_sd and minus_sd];

\end{axis}
\end{tikzpicture}

%% file: content/conclusions.tex
\part{Conclusion}

\chapter{Closing Remarks}\label{ch:conclusions}
This dissertation has inspected the impact of incorporating
algorithmic inductive bias in machine learning models, primarily
looking at the intersection of deep learning for graphs, graph
algorithms and neural combinatorial optimisation.

Throughout its chapters, two interlaced questions were pursued:
\begin{enumerate}
\item[\textsc{q}.1] \label{item:q1}\textit{Can neural networks learn to execute classical algorithms?}
\item[\textsc{q}.2] \label{item:q2}\textit{Is this of any use?}
\end{enumerate}
While it is undoubtedly difficult -- and likely quite a few years away
-- that neural networks will reach the level of \textit{algorithmic
  guarantees} inherent of classical algorithms in \textit{any}
scenarios, this thesis provides theoretical and empirical evidence
that it is indeed possible to learn \emph{reasoners} that approximate
algorithmic properties (e.g., size-invariance). This thesis also
sought to answer concerns regarding the practical value of such
approaches, by discussing limitations of classical algorithms, such as
the \emph{algorithmic bottleneck dilemma}, and showcasing tasks of
practical relevance where neural algorithmic reasoning proves useful.

In details, \autoref{ch:contribution-dp} attempted at giving an answer
to \hyperref[item:q1]{\textsc{q}.1} through the lens of tropical
algebra, succeeding in providing practical examples for reachability
and shortest path problems. There, an investigation of algorithmic
alignment properties is conducted through the use of bijections
between semirings, demonstrating how NAR architectures can be derived
from the algebraic properties of said semirings. Interestingly, our
setting enabled us to \emph{re-discover} some of the most common
empirical findings in NAR, such as the use of $\max$ aggregators in
GNs, and put them under a theoretical framework where their impact on
the reasoning performance can be \emph{quantified}. To elucidate this
point, recall the study performed in \autoref{sec:tropical-effects},
where a sum-aggregated GN implicitly constrained its representations
to learn exponential mappings and consequently requiring a higher
number of neurons. Overall, this contribution seeks to foster more
research in the theoretical aspect of neural algorithmic reasoning,
especially considering the wide range of possible tropical semirings
from which we can draw connections and/or derive aligned
architectures. Indeed, a direct extension to this contribution would
be performing similar explorations on different tropical semirings as
in \citep{gunawardena1998correspondence} in order to build NAR modules
that align with algorithms therein and proving approximation
capabilities similarly to what we have shown in our contribution.

To further expand possible future extentions, establishing connections
between tropical semirings and graph networks unlocked possibilities
to potentially demonstrate that \emph{any} min-aggregated DP algorithm
can be arbitrarily well approximated by an EPD architecture. To pursue
this objective, a further investigation is necessary. Namely, in our
tropical setting, we should consider the use of an external memory to
be used by the processor \citep{jayalath2023recursive}. The rationale
behind this is that GNs alone, with their fixed-size node and edge
representations, can not model recursion typical of certain DP
algorithms, such as Depth First Search
\citep{cormen2009introduction}, up to arbitrary depth. Consequently,
there is an inherent limitation due to GN architectures to prove
approximations for the full spectrum of DP algorithms. However,
our tropical framework did not consider possible external memory
usage, and therefore needs to be revised and extended to account for
storing and recalling of previous processor states.

Moving forward, Dual Algorithmic Reasoning (\autoref{sec:dar})
testified that not only DP algorithms can be targeted by NAR modules,
effectively seeking to answer \hyperref[item:q1]{\textsc{q}.1} for a
broader class of algorithms.  Indeed, we successfully learnt a good
approximation for the Ford-Fulkerson algorithm, which does not fit in
the DP paradigm. Furthermore, such approximation was obtained by
leveraging strong duality as an inductive bias, hence borrowing a
fundamental concept from the combinatorial optimisation
world. Interestingly, strong duality was proven to be effectively
modelled and leveraged in DAR predictions -- a behaviour which was
deeply discussed and analysed in
\autoref{sec:dar-synthetic-exp}). Futhermore, this behaviour emerged
from simply training the architecture on both primal and dual
solutions, without imposing any constraint nor adding any
regularisation term to the objective. As such, DAR formulation is
general and can seamlessly be integrated with prior NAR modules with
slight, non-disruptive architectural adjustments. While we argue that
max-flow and min-cut are \emph{representative} for a broad class of
problem pairs and algorithms (i.e., those for which strong duality
holds), it remains uncertain whether such advantages would extend to
cases involving weak duality. In such scenarios, duality still
provides valuable information in the form of upper and lower bounds on
the primal solutions. However, whether DAR can harness this property
is an unexplored area and a potentially exciting direction for future
research.

On a more general note, DAR represents also one of the most compelling
examples supporting a positive answer to
\hyperref[item:q2]{\textsc{q}.2}. Indeed, DAR demonstrated that
algorithms are useful prior knowledge for solving related learning
problems. For the first time, we recorded a performance improvement in
standard graph learning tasks with real-world utility through
incorporation of algorithms in a learnable classifier. In this
scenario, dual representations also exhibited outstanding
generalisation, being able to being able to output informative
representations for graphs of size 100,000 times larger than its
training data. This achievement also positions DAR as the most
scaled-up NAR architecture, yet.

Further attempts at demonstrating that integrating algorithms in
neural networks is useful are carried out in
\autoref{sec:neural-planning} and \autoref{sec:conar}, where we target
learning of heuristic functions for two different purposes. In the
former, we leveraged DP approximation capabilities of GNs to learn
heuristic functions for planning problems, particularly in the context
of the A* path-finding algorithm. There, NAR was used to improve OOD
generalisation of the computed heuristics, and was proven of effective
use to enhance the performance of A*. In the latter, instead, we dived
into the exploration of transferring algorithmic knowledge from P
$\to$ NP, by training algorithmic reasoners on P algorithms and trying
to leverage this knowledge to learn heuristics (or approximate
algorithms) for NP-hard problems. This investigation exploits the
observation that many approximate algorithms, such as Christofides and
the Gon algorithm, build on simpler algorithmic
``primitives''. Similarly to this manual crafting of algorithms, our
learnt approach seeks to leverage algorithms as starting points (i.e.,
by pre-training NAR modules to execute them) to perform optimisation
steps on when learning CO tasks (i.e., by fine-tuning the
representations). Through quantitative and qualitative analysis we
provided sufficient indications that algorithmically-informed neural
networks perform best on combinatorial tasks. There are still, however,
plenty of research opportunities to deepen the analysis
of whether algorithms, neural networks and their application to NP-hard
problems is effective. One possible, and
highly fascinating, direction is that such approach has potential to
unlock \emph{automatic discovery} of new highly-performing algorithms
for NP-hard problems, possibly looking at interactions with the
Reinforcement Learning (RL) paradigm
\citep{sutton2018reinforcement}. In particular, deep reinforcement
learning has been successfully applied to discover novel algorithms
for matrix multiplication \citep{fawzi2022discovering} and sorting of
short sequences \citep{mankowitz2023faster}. One major difference to
our approach is that deep reinforcement learning is inherently more
\emph{interpretable}. Specifically, it is possible to \emph{write
  down} the learnt algorithm by trivial investigation of the sequence
of actions employed by the learnt agent. In our approach, instead,
hidden representations are not interpretable and it is unclear how we
can devise algorithms from them. However, aforementioned RL
methodologies do not have knowledge of algorithms and are then
required to undergo a heavy trial-and-error phase (typical in RL
settings). Hence, there is evident room for application of algorithmic
reasoning in this scenario. Here, NAR has the potential to navigate
the combinatorial space of algorithmic solutions much more
efficiently, given its knowledge of algorithms.

There remains, however, much to explore within Neural Algorithmic
Reasoning, both from a theoretical standpoint as well as concerning
practical applications.  Beside all the future prospects discussed
above (e.g., theoretical analysis through different tropical
semirings, algorithm discovery, \dots), other two valuable directions
are worth to be mentioned: \emph{fixed point} iterations and
modularity. The former represents one of the most prominent
algorithmic invariances that spans many algorithms and applications
\citep{karamardian2014fixed}. Informally, a fixed point of an
algorithm is an input $x$ for which the application of an algorithm
$A$ results in no change (i.e., $A(x)=x$). To incorporate this
invariance into graph networks, it is necessary to ensure that the
iterative message-passing scheme also ``converges'' to a fixed-point.
It is worth noting that when training algorithmic reasoners on
algorithms that naturally fit the fixed-point iteration framework,
such as Bellman-Ford, this behaviour already emerges empirically, as
investigated in \citep{mirjanic2023latent}. A more rigorous
research question to explore is whether we can derive theoretical
guarantees from this integration, such as termination guarantees.

On the other hand, integrating modularity into algorithmic reasoning
would greatly benefit the development of general architectures capable
of handling algorithms composed of multiple sub-routines. For
instance, in the DAR architecture, modularity in the Ford-Fulkerson
algorithm was achieved by manually combining two neural submodules,
each approximating its own sub-routine.  However, an intriguing
research question arises: \emph{can this property be automatically
  attained by one general architecture?} This is precisely the
question that modular deep learning seeks to answer
\citep{pfeiffer2023modular} and its integration into algorithmic
reasoning is one of certain positive impact.

With these final thoughts, I conclude the dissertation. I hope that
all the contributions presented in this thesis will inspire further
studies within the world of Neural Algorithmic Reasoning, perhaps
starting by investigating one of the future research avenues discussed
in this final chapter.

%% file: frontbackmatter/bibliography.tex
% \defbibheading{bibintoc}[\bibname]{%
%   \phantomsection
%   \manualmark
%   \markboth{\spacedlowsmallcaps{#1}}{\spacedlowsmallcaps{#1}}%
%   \addtocontents{toc}{\protect\vspace{\beforebibskip}}%
%   \addcontentsline{toc}{chapter}{\tocEntry{#1}}%
%   \chapter*{#1}%
% }
% \printbibliography[heading=bibintoc]

\bibliographystyle{ref}
\bibliography{ref}

%% file: frontbackmatter/publications.tex
\pdfbookmark[1]{Publications}{publications}
\chapter{List of publications}

\begin{itemize}
\item Numeroso Danilo, and Davide Bacciu. Explaining Deep Graph
  Networks with Molecular Counterfactuals. In \emph{Workshop on
    Machine Learning for Molecules, Neural Information Processing
    Systems (NeurIPS), 2020.}\\
  URL: \href{https://github.com/danilonumeroso/meg}{github.com/danilonumeroso/meg}

\item Numeroso Danilo, and Davide Bacciu. MEG: Generating Molecular
  Counterfactual Explanations for Deep Graph Networks. In
  \emph{IEEE International Joint Conference on Neural Networks
    (IJCNN), 2021.}\\
  URL: \href{https://github.com/danilonumeroso/meg}{github.com/danilonumeroso/meg}

\item Bacciu Davide, and Danilo Numeroso. Explaining Deep Graph
  Networks via Input Perturbation. In \emph{IEEE Transactions on
    Neural Networks and Learning Systems (TNNLS) Journal, 2022.}\\
  URL:\href{https://github.com/danilonumeroso/legit}{github.com/danilonumeroso/legit}

\item Numeroso, Danilo, Davide Bacciu, and Petar
  Veličković. Learning heuristics for A*. In \emph{Workshop on
    Anchoring Machine Learning in Classical Algorithmic Theory,
    International Conference on Learning Representations (ICLR),
    2022.}
  \href{https://github.com/danilonumeroso/dar}{github.com/danilonumeroso/dar}

\item Kool Wouter, Laurens Bliek, Danilo Numeroso, Yingqian Zhang,
  Tom Catshoek, Kevin Tierney, Thibaut Vidal and Joaquim
  Gromicho. The EURO meets NeurIPS 2022 Vehicle Routing
  Competition. In \emph{Proceedings of the Neural Information
    Processing Systems (NeurIPS) Competitions Track, PMLR,
    2022.}\\
  URL: \href{https://github.com/danilonumeroso/dar}{github.com/danilonumeroso/dar}

\item Numeroso, Danilo, Davide Bacciu, and Petar Veličković. Dual
  Algorithmic Reasoning. In \emph{International Conference on
    Learning Representations (ICLR), 2023.}\\
  URL: \href{https://github.com/ortec/euro-neurips-vrp-2022-quickstart}{github.com/ortec/euro-neurips-vrp-2022-quickstart}

\item Bacciu Davide, Landolfi Francesco, and Danilo Numeroso. A
  Tropical View of Graph Neural Networks. In \emph{European Symposium
    on Artificial Neural Networks (ESANN), 2023.}
\end{itemize}

%% file: frontbackmatter/talks.tex
\pdfbookmark[1]{Talks}{talks}
\chapter{List of talks}

\begin{itemize}
\item Oral presentation of \emph{Explaining Deep Graph Networks with
    Molecular Counterfactuals}, Workshop on Machine Learning for
  Molecules, NeurIPS 2020.

\item Oral presentation of \emph{MEG: Generating Molecular
    Counterfactual Explanations for Deep Graph Networks}, IJCNN 2021.

\item Poster presentation of \emph{Learning heuristics for A*},
  Workshop on Anchoring Machine Learning in Classical Algorithmic
  Theory, ICLR 2022.

\item Invited talk at TU/e, \emph{Neural Algorithmic Reasoning}, Eindhoven, 2023.

\item Spotlight presentation of \emph{Dual Algorithmic Reasoning},
  ICLR 2023.

\item Invited talk at QualcommAI, \emph{Algorithmically-Informed Graph
    Neural Networks}, 2023.

\end{itemize}